\documentclass{article}

\usepackage[nonatbib,final]{neurips_2022}
\usepackage[numbers]{natbib}

\usepackage{microtype}
\usepackage{graphicx}
\usepackage{booktabs} %

\usepackage{hyperref}
\usepackage{xcolor}		%
\definecolor{darkblue}{rgb}{0, 0, 0.5}
\definecolor{beaublue}{rgb}{0.74, 0.83, 0.9}
\definecolor{gainsboro}{rgb}{0.86, 0.86, 0.86}
\definecolor{kleinblue}{rgb}{0,0.18,0.65}
\hypersetup{colorlinks=true,citecolor=kleinblue, linkcolor=kleinblue, urlcolor=kleinblue}
\usepackage[utf8]{inputenc} %
\usepackage[T1]{fontenc}    %
\usepackage{hyperref}       %
\usepackage{url}            %
\usepackage{booktabs}       %
\usepackage{amsfonts}       %
\usepackage{nicefrac}       %
\usepackage{microtype}      %
\usepackage{xcolor}         %

\usepackage{amsmath}
\usepackage{amssymb}
\usepackage{mathtools}
\usepackage{amsthm}
\usepackage{subfigure,wrapfigure}
\usepackage[export]{adjustbox}

\usepackage{enumitem,algorithm,algorithmic}
\usepackage[OT1]{fontenc}
\usepackage{booktabs,arydshln} %
\usepackage{tikz}
\usetikzlibrary{bayesnet} %
\usetikzlibrary{arrows}

\usepackage{amsmath,amsfonts,bm}

\newcommand{\norm}[1]{\left\lVert#1\right\rVert}
\newcommand{\ind}{\perp\!\!\!\!\perp}

\def\eqref#1{equation~\ref{#1}}

\def\1{\bm{1}}
\newcommand{\train}{\mathcal{D_{\mathrm{tr}}}}

\def\mX{{\bm{X}}}

\DeclareMathAlphabet{\mathsfit}{\encodingdefault}{\sfdefault}{m}{sl}
\SetMathAlphabet{\mathsfit}{bold}{\encodingdefault}{\sfdefault}{bx}{n}

\def\gC{{\mathcal{C}}}

\def\gE{{\mathcal{E}}}

\def\gG{{\mathcal{G}}}

\def\gS{{\mathcal{S}}}

\def\gY{{\mathcal{Y}}}
\def\gZ{{\mathcal{Z}}}

\def\sP{{\mathbb{P}}}

\newcommand{\R}{\mathbb{R}}

\DeclareMathOperator*{\argmax}{arg\,max}
\DeclareMathOperator*{\argmin}{arg\,min}

\theoremstyle{plain}
\newtheorem{theorem}{Theorem}[section]
\newtheorem{proposition}[theorem]{Proposition}
\newtheorem{lemma}[theorem]{Lemma}

\theoremstyle{definition}
\newtheorem{definition}[theorem]{Definition}
\newtheorem{assumption}[theorem]{Assumption}
\theoremstyle{remark}

\newcommand{\dataset}{{\cal D}}

\newcommand{\revision}[1]{\textcolor{black}{#1}}

\newcommand{\inv}{{\text{inv}}}
\newcommand{\spu}{{\text{spu}}}

\newcommand{\ego}{\text{ego}}
\newcommand{\envtrain}{{\gE_{\text{tr}}}}
\newcommand{\envtest}{{\gE_{\text{te}}}}
\newcommand{\envall}{{\gE_{\text{all}}}}
\newcommand{\gen}{{{\text{gen}}}}
\newcommand{\invrat}{{\textsc{InvRAT}}} 
\newcommand{\dir}{{\textsc{DIR}}}
\newcommand{\ginv}{{\textsc{CIGA}}} %
\newcommand{\fullginv}{{\text{\textbf{C}ausality Inspired \textbf{I}nvariant \textbf{G}raph {L}e\textbf{A}rning}}}
\newcommand{\doop}{{\text{do}}}
\newcommand{\var}{{\text{Var}}}
\newenvironment{myquotation}{\setlength{\leftmargini}{0em}\quotation}{\endquotation}

\title{Learning Causally Invariant Representations for Out-of-Distribution Generalization on Graphs}

\author{Yongqiang Chen$^{1}$\thanks{Work done during an internship at Tencent AI Lab.}, Yonggang Zhang$^2$, Yatao Bian$^3$, Han Yang$^1$, Kaili Ma$^1$, Binghui Xie$^1$\\
	 $^1$The Chinese University of Hong Kong $^2$Hong Kong Baptist University \\
	 \texttt{\{yqchen\!,hyang\!,klma\!,bhxie21\!,jcheng\}@cse\!.\!cuhk\!.\!edu\!.\!hk}\
	\texttt{yatao\!.\!bian@gmail\!.\!com}\\\vspace{-0.35in}
	 \AND
	 Tongliang Liu$^4$, Bo Han$^2$, James Cheng$^1$ \\
	$^3$Tencent AI Lab, $^4$TML Lab, The University of Sydney\\ \texttt{tongliang\!.liu@sydney\!.\!edu\!.au} \ \texttt{\{csygzhang\!,bhanml\}@comp\!.\!hkbu\!.\!edu\!.\!hk} \\
\setcounter{footnote}{0}
}

\usepackage{etoc}
\etocdepthtag.toc{mtchapter}
\etocsettagdepth{mtchapter}{subsection}
\etocsettagdepth{mtappendix}{none}

\usepackage{framed}
\definecolor{shadecolor}{rgb}{0.94, 0.97, 1.0}

\begin{document}

\maketitle

\begin{abstract}

  Despite recent success in using the invariance principle
  for out-of-distribution (OOD) generalization on Euclidean data (e.g., images), studies on graph data are still limited. Different from images, the complex nature of graphs poses unique challenges to adopting the invariance principle.
  In particular, distribution shifts on graphs can appear in a variety of forms  such as attributes and structures, making it difficult to identify the invariance.
  Moreover, domain or environment partitions, which are often required by OOD methods  on Euclidean data, could be highly expensive to obtain for graphs.
  To bridge this gap, we propose a new framework, called $\fullginv$ ($\ginv$), to capture the invariance of graphs for guaranteed OOD generalization under various distribution shifts.
  Specifically, we characterize potential distribution shifts on graphs with causal models,
  concluding that OOD generalization on graphs is achievable when models focus \emph{only} on subgraphs containing the most information about the causes of labels.
  Accordingly,
  we propose an information-theoretic objective to extract the desired subgraphs that maximally preserve the invariant intra-class information.
  Learning with these subgraphs is immune to distribution shifts.
  Extensive experiments on $16$ synthetic or real-world datasets, including a challenging setting -- DrugOOD,\! from AI-aided drug discovery,
  validate the superior OOD performance of $\ginv$\footnote{Code is available at \url{https://github.com/LFhase/CIGA}.}.
\end{abstract}

\section{Introduction}
Graph representation learning with graph neural networks (GNNs) has gained great success in tasks involving relational information~\citep{gcn,sage,gat,jknet,gin}.
However, it assumes that the training and test graphs are drawn from the same distribution, which
is often violated in reality~\citep{ogb,wilds,TDS,drugood}.
The mismatch between training and test distributions, i.e., \textit{distribution shifts},
introduced by some underlying environmental factors related to data collection or processing,
could seriously degrade the performance of deployed models~\citep{camel_example,covid19_application}.
Such \textit{out-of-distribution} (OOD) generalization failures become the major roadblock for practical applications of
graph representation learning~\citep{drugood}.

Meanwhile, enabling OOD generalization on regular Euclidean data has received surging attention and several solutions were proposed~\citep{irmv1,groupdro,meta-transfer,v-rex,env_inference,ood_max_inv,ib-irm}.
In particular, the invariance principle from causality is at the heart of those works~\citep{inv_principle,causality,causal_transfer}.
The principle leverages the Independent Causal Mechanism (ICM) assumption~\citep{causality,elements_ci} and implies that,
model predictions that only focus on the causes of the label can stay invariant to a large class of distribution shifts~\citep{inv_principle,irmv1,ib-irm}.

Despite the success of the invariance principle on Euclidean data,
the complex nature of graphs raises several new challenges that
prohibit direct adoptions of the principle.
First, distribution shifts on graphs are more complicated.
They can happen at both attribute-level and  structure-level,
and be observed in multiple forms such as graph sizes, subgraph densities and homophily~\citep{size_gen1,size_gen2,reliable_gnn_survey}.
On the other hand, each of the shifts can spuriously correlate with labels in different modes~\citep{irmv1,failure_modes,ib-irm}.
Consequently, the entangled complex distribution shifts make it more difficult to identify and capture the invariance on graphs.
Second, OOD algorithms developed and analyzed on Euclidean data
often require additional environment (or domain) labels for distinguishing the sources of distribution shifts~\citep{irmv1}.
However, the environment labels could be highly expensive to obtain and thus often unavailable for graphs,
as collecting the labels usually requires expert knowledge due to the abstraction of graphs~\citep{ogb}.
These challenges render the problem studied in this paper even more challenging: %
\begin{myquotation}
  \emph{How could one generalize the invariance principle to enable OOD generalization on graphs?}
\end{myquotation}
\begin{figure}[t]
  \subfigure[]{
    \includegraphics[width=0.5\textwidth]{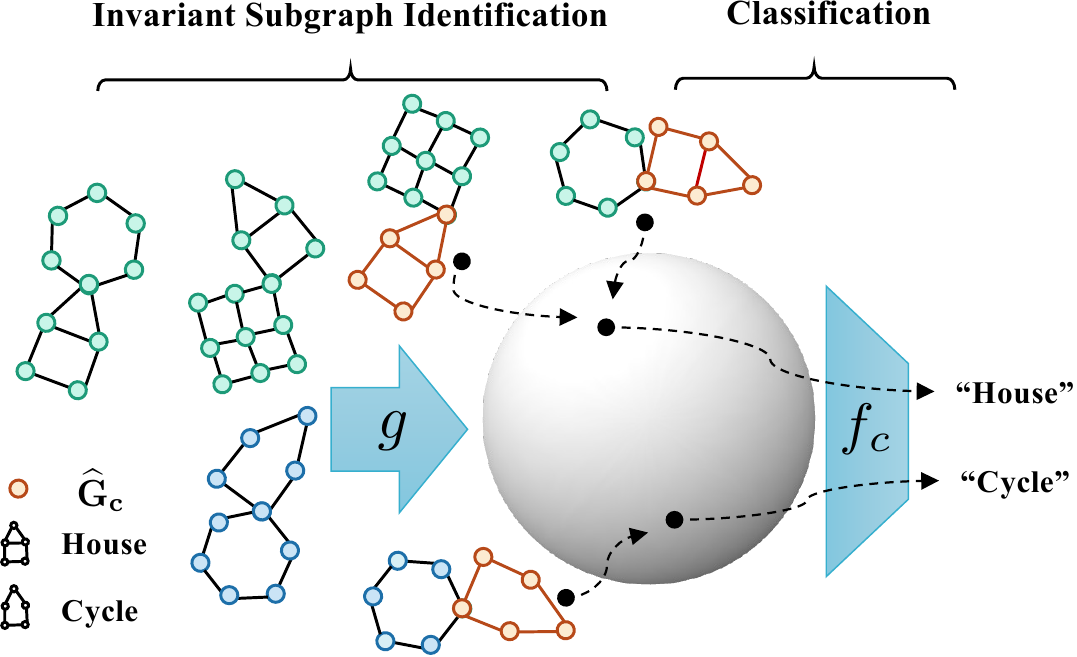}
    \label{fig:motivation}
  }
  \subfigure[]{
    \adjustbox{valign=b}{\resizebox{0.48\textwidth}{!}{
        \bgroup
        \def\arraystretch{1.2}
        \begin{tabular}{ccccc}
          \toprule
          Algorithm                       & OOD Guarantee & Regime         & $E$ Known   & SCM Support         \\
          \midrule
          IRM~\citep{irmv1}               & Yes           & $\R$           & Yes         & PIIF                \\
          IB-IRM~\citep{ib-irm}           & Yes           & $\R$           & Yes         & PIIF\&FIIF          \\
          EIIL~\citep{env_inference}      & Yes           & $\R$           & No          & PIIF                \\
          DANN~\citep{DANN}               & N/A           & $\R$           & Yes         & N/A                 \\
          MatchDG~\citep{causal_matching} & N/A           & $\R$           & Yes         & FIIF                \\
          GroupDro~\citep{groupdro}       & N/A           & $\R$           & Yes         & N/A                 \\
          CNC~\citep{cnc}                 & N/A           & $\R$           & No          & N/A                 \\
          GIB~\citep{gib}                 & Yes           & $\gG$          & No          & FIIF                \\
          DIR~\citep{dir}                 & No            & $\gG$          & No          & FIIF                \\
          \textbf{$\ginv$ (Ours)}         & \textbf{Yes}  & \textbf{$\gG$} & \textbf{No} & \textbf{PIIF\&FIIF} \\
          \bottomrule
                                          &               &                &             &                     \\
                                          &               &                &             &                     \\
        \end{tabular}
        \egroup
      }}
    \label{tab:potential_alg}
  }
  \vspace{-0.15in}
  \caption{(a) Illustration of $\fullginv$ ($\ginv$):
    GNNs need to classify graphs based on the specific motif (``House'' or ``Cycle'').
    The featurizer $g$ will extract an (orange colored) subgraph $\widehat{G}_c$ from each input
    for the classifier $f_c$ to predict the label.
    The training objective of $g$ is implemented in a contrastive strategy where
    the distribution of $\widehat{G}_c$ at the latent sphere
    will be optimized to maximize the intra-class mutual information, hence predictions will be invariant to distribution shifts;
    (b) An overview of potential algorithms for OOD generalization on graphs.}
  \vspace{-0.15in}
\end{figure}

To solve the above problem, we propose  $\fullginv$ ($\ginv$),
a new framework for
capturing the invariance of graphs to enable guaranteed OOD generalization under different distribution shifts.
Specifically, we build three Structural Causal Models (SCMs)~\citep{causality}
to characterize the distribution shifts that could happen on graphs:
one is to model the graph generation process, and the other two are to model two possible
interactions between invariant and spurious features during the graph generation,
i.e., Fully Informative Invariant Feature (FIIF)
and Partially Informative Invariant Feature (PIIF) (Sec.~\ref{sec:data_gen}).
Then, we generalize the invariance principle to graphs for OOD generalization:
GNN models are invariant to distribution shifts
if they focus only on an invariant and critical subgraph $G_c$
that contains the most of the information in $G$ about the underlying causes of the label.
Thus, the problem of achieving OOD generalization on graphs can be rephrased into two processes:
invariant subgraph identification and label prediction.
Accordingly, shown as Fig.~\ref{fig:motivation}, we introduce a prototypical invariant graph learning algorithm
that decomposes a GNN into:
a) a featurizer $g$ for identifying the underlying invariant subgraph $G_c$ from $G$;
b) a classifier $f_c$ for making predictions based on $G_c$.
To extract the desired subgraph $G_c$, we derive an information-theoretic objective for the featurizer
to identify subgraphs that maximally preserves the invariant intra-class information
across a set of different (unknown) environments.
We theoretically show that this approach can provably identify the underlying $G_c$ under mild assumptions (Sec.~\ref{sec:good_framework}).

Experiments on $16$ synthetic and real-world datasets with various distribution shifts,
including a challenging setting from AI-aided drug discovery~\citep{drugood},
show that $\ginv$ can significantly outperform all of existing methods
up to $10\%$, demonstrating its promising OOD generalization ability (Sec.~\ref{sec:exp}).

\paragraph{Related Work.}
We review existing methods that might improve the OOD generalization on graphs,
summarize the main differences between our solution and them in Table~\ref{tab:potential_alg},
and leave thorough discussions to Appendix~\ref{sec:related_work_appdx}.
On Euclidean data,
Invariant Learning~\citep{irmv1,env_inference,ib-irm},
Group Distributionally Robust Optimization~\citep{v-rex,groupdro,cnc},
Domain Adaption and Domain Generalization~\citep{DANN,CORAL,deep_DG,DouCKG19,causal_matching,DG_survey}
are three widely adopted approaches to enable OOD generalization.
However, they all have their own limitations when being applied to graphs.
First, previous invariant learning methods are mostly developed
and analyzed for Euclidean data~\citep{irmv1,ib-irm,env_inference},
or under specific SCM assumptions~\citep{irmv1},
making the theoretical results hardly able to generalize to the complicated graph data~\citep{risk_irm}
that can have multiple types of distribution shifts~\citep{failure_modes}.
Group Distributionally Robust Optimization that minimizes the gap between worst group risk and average risk~\citep{v-rex,groupdro,cnc},
and Domain Adaption/Generalization methods that aim to learn class-conditional domain invariant representations~\citep{DANN,CORAL,deep_DG,DouCKG19,DG_survey},
cannot guarantee a min-max optimal predictor without additional assumptions~\citep{DG_discussion,irmv1,ib-irm}.
Moreover, most existing methods require environment labels that are however expensive to obtain in graphs,
which limits their applications to graphs~\citep{irmv1,v-rex,ib-irm,groupdro,DANN,CORAL,DouCKG19,causal_matching}.
In contrast, we aim to develop OOD algorithms for graphs that are provably generalizable  under different types of distribution shifts.

Another line of relevant works is about GNN explainability
that aims to find a subgraph of the input as the explanation
for a GNN prediction~\citep{gnn_explainer,xgnn_tax}.
Although some may leverage causality to justify the generated explanation~\citep{gen_xgnn},
they mostly focus on understanding the predictions of GNNs instead of for OOD generalization.
The  closest works to ours are two interpretable GNNs that aim to explicitly extract a subgraph
for both predictions and explanations guided by information theory~\citep{gib} and causality~\citep{dir}, respectively.
However, they focus on graphs and shifts generated under a specific SCM.
Although one of them can provide theoretical guarantee for OOD generalization~\citep{gib}
by using the information bottleneck criteria~\citep{ib-irm},
they would inevitably fail to generalize to graphs generated under different SCMs.
More discussions about the failure are deferred to Appendix~\ref{sec:discussion_ood_obj_appdx}.
Besides, \citet{size_gen2} also discuss OOD generalization on graphs
but limited to a specific graph family and graph size shifts.
\citet{handle_node} propose OOD generalization algorithms on graphs
for the task of node classification, also limited to graphs and shifts under a specific SCM.

To the best of our knowledge, there is no existing work that could handle more comprehensive graph distribution shifts than $\ginv$,
while also achieving provable OOD generalization performance. %

\section{OOD Generalization on Graphs through the Lens of Causality}
\label{sec:graph_ood_causal_lens}

\subsection{Problem Setup}
In this work, we focus on OOD generalization in graph classification.
Specifically,
we are given a set of graph datasets $\dataset=\{\dataset^e\}_e$ collected from multiple environments $\envall$.
Samples $(G^e_i,Y^e_i)\in \dataset^e$ from the same
environment are considered as drawn independently from an identical distribution $\sP^e$.
A GNN $\rho \circ h$ generically has an encoder $h:\gG\rightarrow\R^h$
that learns a meaningful representation $h_G$ for each graph $G$
to help predict the label $\hat{Y}_G=\rho(h_G)$ with a downstream classifier $\rho:\R^h\rightarrow\gY$.
The goal of OOD generalization on graphs is to train a GNN $\rho \circ h$
with data from training environments $\train=\{\dataset^e\}_{e\in\envtrain\subseteq\envall}$
that generalizes well to all (unseen) environments,
i.e., to minimize $\max_{e\in\envall}R^e$,
where $R^e$ is the empirical risk of $\rho\circ h$ under environment $e$~\citep{erm,irmv1}.
We leave more details about the background of GNN for graph classification and invariant learning in Appendix~\ref{sec:background_appdx}.

It is known that OOD generalization is impossible without assumptions on the environments $\envall$~\citep{causality,ib-irm}.
Thus, we will first formulate the data generation process with structural causal model
and latent-variable model~\citep{causality,elements_ci,ssl_isolate},
to characterize the distribution shifts that could happen on graphs.
Then, we investigate whether the existing methods are generalizable under these distribution shifts.

\begin{wrapfigure}{l}{0.6\textwidth}
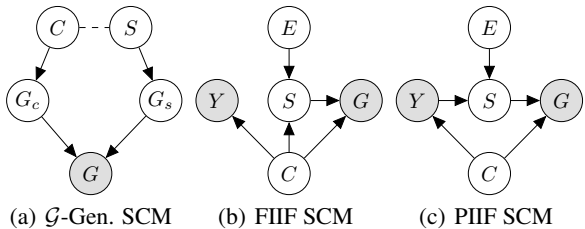

	\vspace{-0.15in}
	\subfigure[\revision{$\gG$-Gen. SCM}]{\label{fig:graph_gen}
		\resizebox{!}{0.18\textwidth}{\tikz{
				\node[latent] (S) {$S$};%
				\node[latent,left=of S,xshift=0.5cm] (C) {$C$};%
				\node[latent,below=of C,xshift=-0.5cm,yshift=0.5cm] (GC) {$G_c$}; %
				\node[latent,below=of S,xshift=0.5cm,yshift=0.5cm] (GS) {$G_s$}; %
				\node[obs,below=of GC,xshift=1.05cm,yshift=0.5cm] (G) {$G$}; %
				\edge[dashed,-] {C} {S}
				\edge {C} {GC}
				\edge {S} {GS}
				\edge {GC,GS} {G}
			}}}
	\subfigure[FIIF SCM]{\label{fig:scm_fiif}
		\resizebox{!}{0.18\textwidth}{\tikz{
				\node[latent] (E) {$E$};%
				\node[latent,below=of E,yshift=0.5cm] (S) {$S$}; %
				\node[obs,below=of E,xshift=-1.2cm,yshift=0.5cm] (Y) {$Y$}; %
				\node[obs,below=of E,xshift=1.2cm,yshift=0.5cm] (G) {$G$}; %
				\node[latent,below=of Y,xshift=1.2cm,yshift=0.5cm] (C) {$C$}; %
				\edge {E} {S}
				\edge {C} {Y,G}
				\edge {S} {G}
				\edge {C} {S}
			}}}
	\subfigure[PIIF SCM]{\label{fig:scm_piif}
		\resizebox{!}{0.18\textwidth}{\tikz{
				\node[latent] (E) {$E$};%
				\node[latent,below=of E,yshift=0.5cm] (S) {$S$}; %
				\node[obs,below=of E,xshift=-1.2cm,yshift=0.5cm] (Y) {$Y$}; %
				\node[obs,below=of E,xshift=1.2cm,yshift=0.5cm] (G) {$G$}; %
				\node[latent,below=of Y,xshift=1.2cm,yshift=0.5cm] (C) {$C$}; %
				\edge {E} {S}
				\edge {C} {Y,G}
				\edge {S} {G}
				\edge {Y} {S}
			}}}
	\caption{
		SCMs on graph distribution shifts.}
	\label{fig:scm}
	\vspace{-0.15in}
\end{wrapfigure}

\subsection{Graph Generation Process}
\label{sec:data_gen}
We take a latent-variable model perspective on the graph generation process and assume
that the graph is generated through a mapping $f_\gen:\gZ\rightarrow \gG$,
where $\gZ\subseteq\R^n$ is the latent space and $\gG=\cup_{N=1}^\infty\{0,1\}^N\times \R^{N\times d}$ is the graph space.
Let $E$ denote environments.
Following previous works~\citep{ssl_isolate,ib-irm},
we partition the latent variable from $\gZ$ into an invariant part $C\in\gC=\R^{n_c}$ %
and a varying part $S\in\gS=\R^{n_s}$, s.t., $n=n_c+n_s$,
according to whether they are affected by $E$ or not.
Similarly in images, $C$ and $S$ can represent content and style
while $E$ can refer to the locations where the images are taken~\citep{camel_example,adv_causal_lens,ssl_isolate}.
Furthermore, $C$ and $S$ control the generation of the observed graphs (Assumption~\ref{assump:graph_gen}) and
can have multiple types of interactions at the latent space (Assumptions~\ref{assump:scm_fiif},~\ref{assump:scm_piif}).

\textbf{Graph generation model.} We elaborate the SCM for the graph generation process in Assumption~\ref{assump:graph_gen}
and Fig.~\ref{fig:graph_gen}, where noises in the structural equations are omitted for simplicity~\citep{elements_ci}.
\begin{assumption}[\revision{Graph Generation Structural Causal Model}]
	\label{assump:graph_gen}
	\[
		\revision{G_c:=f_\gen^{G_c}(C),\qquad G_s:=f_\gen^{G_s}(S),\qquad G:=f_\gen^G(G_c,G_s).}
	\]
\end{assumption}
In Assumption~\ref{assump:graph_gen},
$f_\gen$ is decomposed into $f_\gen^{G_c}$, $f_\gen^{G_s}$ and $f_\gen^G$ to
control the generation of $G_c$, $G_s$, and $G$, respectively.
Among them, $G_c$ inherits the invariant information of $C$ that would not be affected by the interventions (or changes) of $E$~\citep{causality,elements_ci}.
For example, certain properties of a molecule can usually be described by a sub-molecule, or a functional group,
which is invariant across different species or assays~\citep{art_drug,zinc15,drugood}.
\revision{On the contrary, the generation of $G_s$ and $G$ will be affected by
the environment $E$ through $S$.}
Thus, graphs collected from different environments (or domains) can have different distributions of
structure-level properties (e.g., graph sizes~\citep{size_gen2,reliable_gnn_survey})
as well as feature-level properties (e.g., homophily~\citep{homophily1_birds,hao}).
Therefore, the subgraph $G_s$ inherits the spurious feature about $Y$~\citep{adv_causal_lens}.
In fact, Assumption~\ref{assump:graph_gen} is compatible with many graph generation models
by specifying the function classes of $f_\gen^{G_c}$, $f_\gen^{G_s}$ and $f_\gen^G$~\citep{sbm,graphon,graphrnn,graphdf}.
Since our goal is to characterize the potential distribution shifts
in Assumption~\ref{assump:graph_gen},
we focus on building a general SCM that is compatible to many graph families
and leave graph family specifications and their implications to OOD generalization in future works. \revision{More discussions are provided in Appendix~\ref{sec:full_scm_appdx}.}

\textbf{Interactions at latent space.}
Following previous works~\citep{irmv1,ib-irm},
we categorize the latent interactions between $C$ and $S$
into Fully Informative Invariant Features (FIIF, Fig.~\ref{fig:scm_fiif})
and Partially Informative Invariant Features (PIIF, Fig.~\ref{fig:scm_piif})\footnote{Note that FIIF and PIIF can be mixed as Mixed Informative Invariant Features (Appendix~\ref{fig:scm_miif_appdx}) in several ways, while our analysis will focus on the axiom ones for the purpose of generality.},
depending on whether the latent invariant part $C$
is fully informative about label $Y$, i.e., $(S,E)\ind Y|C$.
Formal definitions of the corresponding SCMs are given as follows,
where noises are omitted for simplicity~\citep{causality,elements_ci}.
\begin{assumption}[FIIF Structural Causal Model]
	\label{assump:scm_fiif}$Y:= f_\inv(C),\ S:=f_\spu(C,E),\ G:= f_\gen(C,S).$
\end{assumption}

\begin{assumption}[PIIF Structural Causal Model]
	\label{assump:scm_piif}$Y:= f_\inv(C),\ S:=f_\spu(Y,E),\ G:= f_\gen(C,S).$
\end{assumption}
In the two SCMs above,
$f_\gen$ corresponds to the graph generation process in Assumption~\ref{assump:graph_gen}, and
$f_\spu$ is the mechanism describing how $S$ is affected by $C$ and $E$ at the latent space.
By definition,
$S$ is directly controlled by $C$ in FIIF and
indirectly controlled by $C$ through $Y$  in PIIF,
which can exhibit different behaviors in the observed distribution shifts.
In practice, performances of OOD algorithms can degrade dramatically if one of FIIF or PIIF is excluded~\citep{aubin2021linear,failure_modes}.
This issue can be more serious in graphs, since different distribution shifts can have different interaction modes at the latent space.
Moreover, $f_\inv:\gC\rightarrow\gY$ indicates the labelling process,
which assigns labels $Y$ for the corresponding $G$ merely based on $C$.
Consequently, $\gC$ is better clustered than $\gS$ when given $Y$~\citep{cluster_assump,cluster_assump2,causality4ml,towards_causality},
which also serves as the necessary separation assumption for a classification task~\citep{svm1,svm2,lda}.
\begin{assumption}[Better Clustered Invariant Features]
	\label{assump:latent_sep}$H(C|Y)\leq H(S|Y)$.
\end{assumption}

\subsection{Challenges of OOD Generalization on Graphs}
\label{sec:limitation_prev}

Built upon the graph generation process,
we can formally derive the desired GNN that is able to generalize to OOD graphs under different distribution shifts, which implies the invariant GNN below\footnote{\revision{A discussion on Def. 2.5 and its relation to the SCMs is provided in Appendix~\ref{sec:inv_gnn_discuss_appdx}.}}.
\begin{definition}[Invariant GNN]
	\label{def:inv_gnn}
	Given a set of graph datasets $\{\dataset^e\}_e$ %
	and environments $\envall$ that follow the same graph generation process in Sec.~\ref{sec:data_gen},
	considering a GNN $\rho \circ h$ that has a permutation invariant graph
	encoder $h:\gG\rightarrow\R^h$ and a downstream classifier $\rho:\R^h\rightarrow\gY$,
	$\rho \circ h$ is an invariant GNN if it minimizes the worst case  risk
	among all environments, i.e., $\min \max_{e\in\envall}R^e$.
\end{definition}

Can existing methods produce a desired invariant GNN model?
We find the answers to be negative  unfortunately.
\begin{figure}[t]
	\vspace{-0.15in}
	\subfigure[Failure cases for existing methods.]{
		\includegraphics[width=0.38\textwidth]{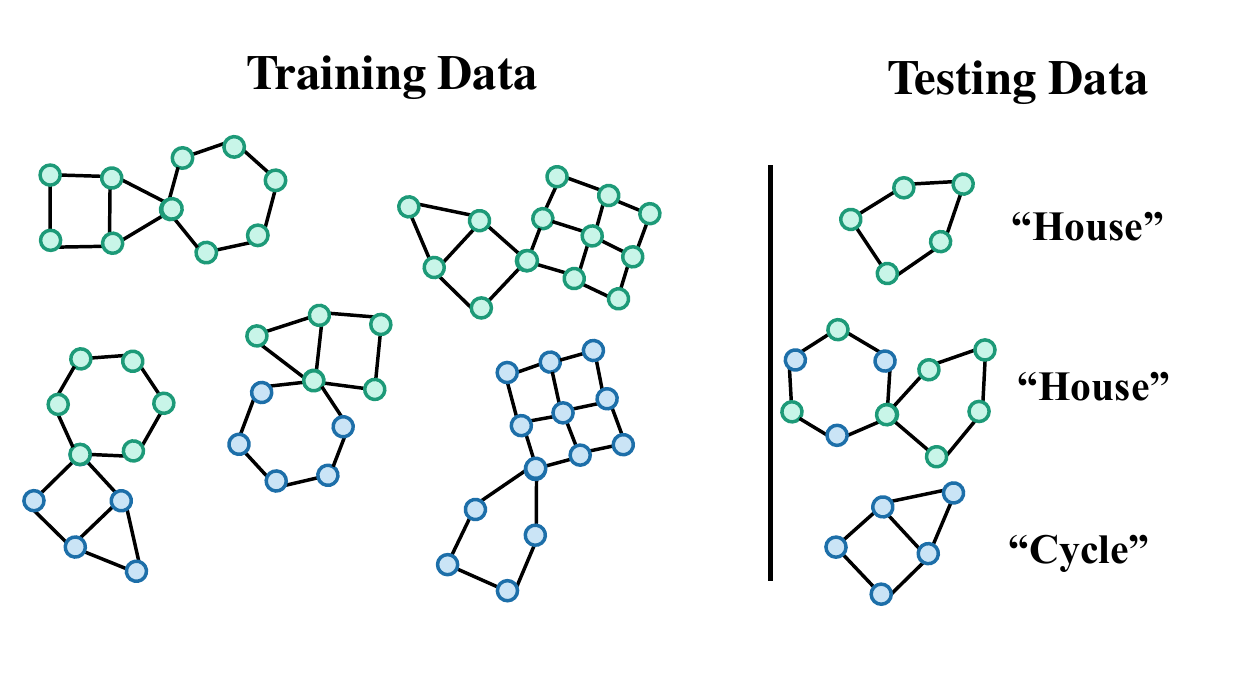}
		\label{fig:good_fail_cases}
	}
	\subfigure[Structure and attribute shifts.]{
		\includegraphics[width=0.28\textwidth]{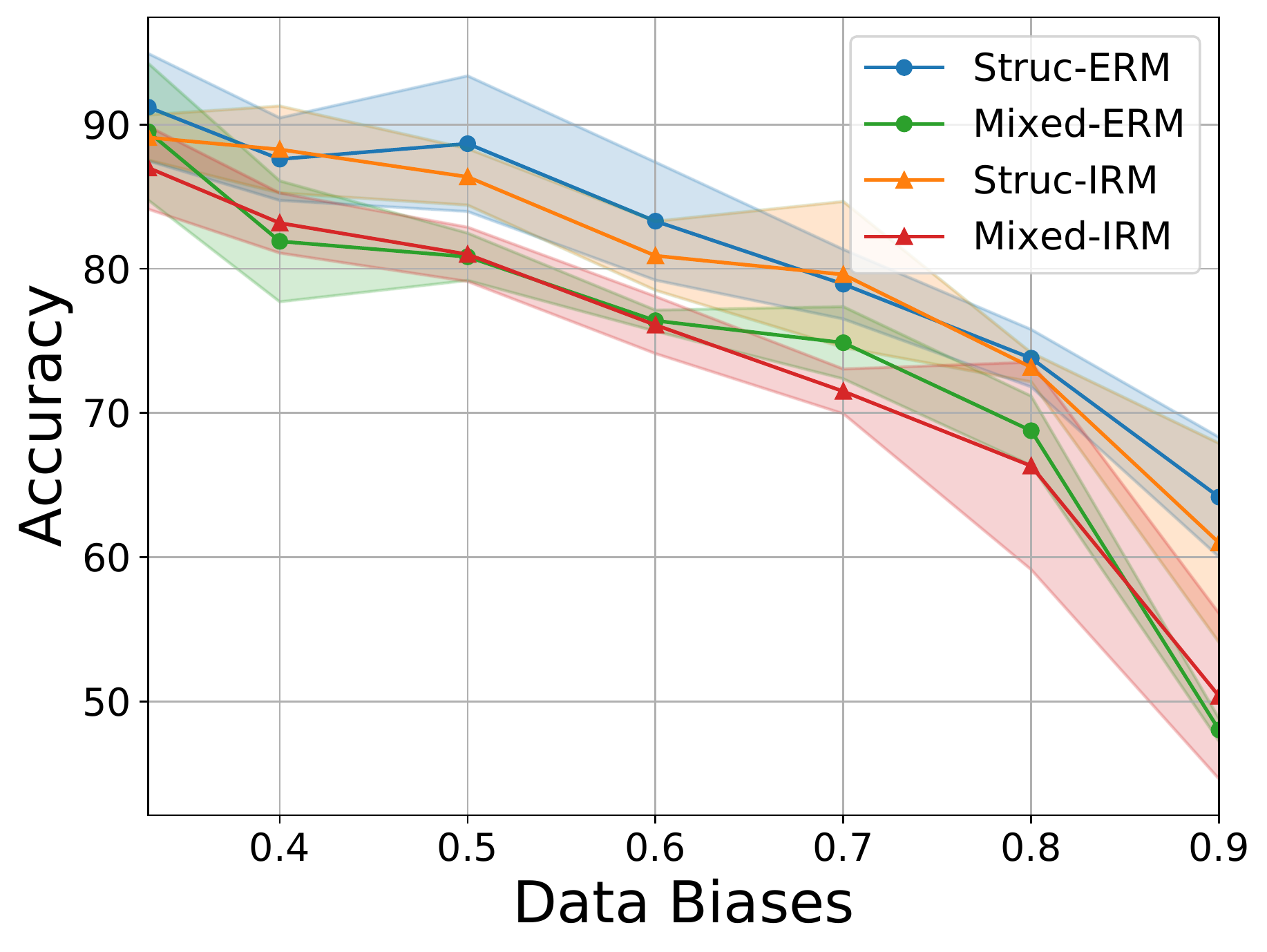}
		\label{fig:ood_failure_wosize}
	}
	\subfigure[Mixed with graph size shifts.]{
		\includegraphics[width=0.28\textwidth]{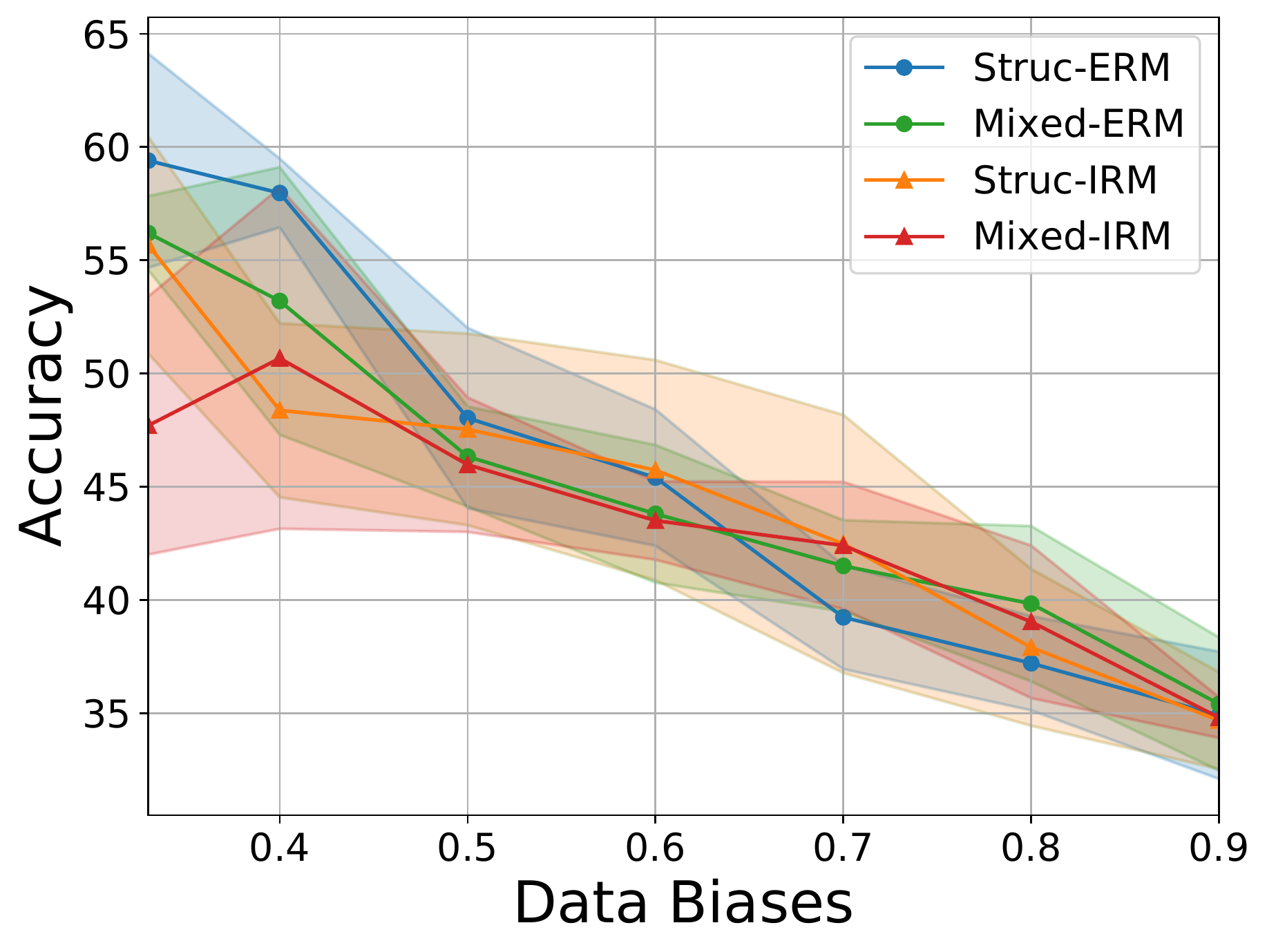}
		\label{fig:ood_failure_size}
	}
	\caption{Failures of OOD generalization on graphs:
		(a) GNNs are required to classify whether the graph contains a ``house'' or ``cycle'' motif,
		where the colors represent node features.
		However, distribution shifts in the training data exist at both structure-level (from left to right: ``house'' mostly co-occur with a hexagon),
		attribute-level (from upper to lower: nodes are mostly  colored green if the graph contains a ``house'', or  colored blue if the graph contains a ``cycle''),
		and graph sizes, making GNNs hard to capture the invariance. Consequently,
		\textit{ERM can fail} for leveraging the shortcuts and predicting graphs that have a hexagon or have nodes mostly colored green as ``house''.
		\textit{IRM can fail} as the test data are not sufficiently supported by the training data.
		(b) GCNs optimized with neither ERM nor IRM can generalize to OOD graphs under
		structure-level shifts (Struc-) or mixed with feature shifts (Mixed-).
		(c) When more complex shifts presented, GNNs can fail more seriously.}
	\label{fig:good_fail}
	\vspace{-0.15in}
\end{figure}
Based on the synthetic BAMotif graph classification task~\citep{pge,dir} shown in Fig.~\ref{fig:good_fail},
we theoretically and empirically analyze whether existing methods
could produce an invariant GNN, through the investigation of the following aspects.
More details and results are given in Appendix~\ref{sec:good_fail_setting_appdx}.

\textbf{Can GNNs trained with ERM generalize to OOD graphs?}
As shown in Fig.~\ref{fig:good_fail}, we find that GNNs trained with the
standard empirical risk minimization (ERM) algorithm~\citep{erm}
are not able to generalize to OOD graphs.
As the data biases grows stronger,
the performances of GNNs drop dramatically.
Furthermore, when graph size shifts are mixed in the data,
GNNs can have larger variance at low data biases, indicating
the instability of learning the desired relationships for the task.
The reason is that ERM tends to overfit to the shortcuts
or spurious correlations presented in specific substructures or attributes
in the graphs~\citep{shortcut_dl}.
This phenomenon has also been shown to exist
in GNNs equipped with more sophisticated architectures such as attention mechanisms~\citep{gat},
under graph size shifts~\citep{understand_att}.

\textbf{Can OOD objectives improve OOD generalization of GNNs?}
Meanwhile, as shown in Fig.~\ref{fig:good_fail}, OOD objectives primarily developed on Euclidean data such as
invariant risk minimization (IRM)~\citep{irmv1}
also cannot alleviate the problem. On the contrary, IRM can
fail catastrophically at non-linear regime if without sufficient support overlap
for the test environments,
i.e., $\cup_{e\in\envtest}\text{supp}(\sP^e)\not\subseteq\cup_{e\in\envtrain}\text{supp}(\sP^e)$~\citep{risk_irm}.
In addition to IRM, the failure would also happen for alternative objectives~\citep{v-rex,gen_inv_conf,ib-irm} as proved by~\citet{risk_irm}.
Besides, different distribution shifts on graphs can be nested with each other where
each one can have distinct spurious correlation type, e.g., FIIF or PIIF.
OOD objectives will also fail seriously if either of the correlation types is not supported~\citep{aubin2021linear,failure_modes}.
Moreover, non-trivial environment partitions or labels
are required for performance guarantee of these OOD objectives~\citep{irmv1,v-rex,groupdro,ib-irm}.
However, collecting meaningful environment partitions of graphs
requires expert knowledge about graph data.
Thus, the environment labels can be expensive to obtain
and are usually not available~\citep{tudataset,benchmark_gnn,ogb}.
Alternative options such as random partitions
tend not to alleviate the issue~\citep{env_inference,zin}, as it can be trivially deemed as mini-batching.

\textbf{Challenges of OOD generalization on graphs.}
The aforementioned failure analysis reveals that existing methods or objectives fail to elicit an invariant GNN primarily due to the following two challenges:
a) Distribution shifts on graphs are more complicated where different types of spurious correlations can be entangled via different graph properties;
b) Environment labels are usually not available due to the abstraction of graphs.
Despite these challenges, we are still highly motivated to address the following research  question:
\emph{Would it be possible to learn an invariant GNN that is generalizable under various distribution shifts by lifting the invariance principle to the graph data?}

\section{Invariance Principle for OOD Generalization on Graphs}
\label{sec:good_framework}
We provide affirmative answers to the previous question
by proposing a new framework, $\ginv$: $\fullginv$. Specifically,
built upon the SCMs in Sec.~\ref{sec:data_gen},
we generalize the invariance principle to graphs
and instantiate the principle with theoretical guarantees.

\subsection{Invariance for OOD Generalization on Graphs}
Towards extending the invariance principle to graphs under SCMs in Sec.~\ref{sec:data_gen},
we need to identify a set of variables that have stable causal relationship with $Y$
under both FIIF and PIIF (Assumption~\ref{assump:scm_fiif},~\ref{assump:scm_piif}).
According to the ICM assumption~\citep{elements_ci}, the labeling process $C\rightarrow Y$ is not informed nor influenced by other processes, implying that the conditional distribution $P(Y|C)$ remains invariant to the interventions on the environment latent variable $E$~\citep{causality}.
Consequently, for a GNN with a permutation invariant encoder $h:\gG\rightarrow \R^h$
and a downstream classifier $\rho:\R^h\rightarrow\gY$,
if $h$ can recover the information of $C$ from $G$ in the learned graph representations,
then the learning of $\rho$ resembles traditional ERM~\citep{erm}
and can achieve the desired min-max optimality required by an invariant GNN (Def.~\ref{def:inv_gnn}).
However, recovering $C$ from $G$ is particularly difficult,
since the generation of $G$ from $C$ involves two causal mechanisms $f_\gen^{G_c}$ and $f_\gen^G$
in Assumption~\ref{assump:graph_gen}.
The unavailability of $E$ further adds up the difficulty of enforcing the
independence between the learned representations and $E$.

\subsection{Invariant Graph Learning Framework}
\label{sec:good_framework_prac}

\textbf{Causal algorithmic alignment.}
To enable a GNN to learn to extract the information about $C$ from $G$, we propose the $\ginv$ framework that \mbox{\textit{explicitly aligns with}} the
two causal mechanisms $f_\gen^{G_c}$ and $f_\gen^G$ in Assumption~\ref{assump:graph_gen}.
The idea of alignment in $\ginv$ is motivated by the algorithmic reasoning results
that a neural network can learn a reasoning process better if its computation structure aligns with the process better~\citep{what_nn_reason,nn_extrapo}.
Specifically, we realize the alignment by decomposing a GNN into two sub-components\footnote{The encoder of the GNN in $\ginv$ can be regarded as the composition of $g$ and the graph encoder in $f_c$.}:
a) a featurizer GNN $g:\gG\rightarrow\gG_c$ aiming to identify the desired $G_c$;
b) a classifier GNN $f_c:\gG_c\rightarrow\gY$ that predicts the label $Y$ based on the estimated $G_c$,
where $\gG_c$ refers to the space of subgraphs of $G$.
Formally, the learning objectives of $f_c$ and $g$ can be formulated as:
\begin{equation}
	\label{eq:good_opt}
	\text{$\max$}_{f_c, \; g} \ I(\widehat{G}_{c};Y), \ \text{s.t.}\ \widehat{G}_{c}\ind E,\ \widehat{G}_{c}=g(G),
\end{equation}
\revision{where maximizing $I(\widehat{G}_c;Y)$ is equivalent
	to minimizing a variational upper bound of $R(f_c(\widehat{G}_{c}))$~\citep{vib,gib}}
that takes $\widehat{G}_{c}$ as inputs to predict label $Y$ for $G$ through $f_c$ and $g$,
and $\widehat{G}_{c}$ is the estimated subgraph containing the information about $C$ and hence needs to be independent of $E$.
Moreover, the extracted $G_c$ can either shares the same graph space with input $G$ or
has its own space with latent node and edge features, depending on the specific graph generation process.
In practice, architectures from the literature of interpretable GNNs are compatible with $\ginv$~\citep{xgnn_tax}, hence can serve as practical choices for the implementation of $\ginv$.
More details are given in Appendix~\ref{sec:good_impl_appdx}.

Although we can technically align with the two causal mechanisms with $g$ and $f_c$,
trivially optimizing this architecture cannot satisfy $\widehat{G}_{c}\ind E$.
Formally, merely maximizing $I(\widehat{G}_{c};Y)$
may include a subgraph from $G_s$ in $\widehat{G}_{c}$ since $G_s$ also shares certain mutual information with $Y$.
Moreover, the unavailability of $E$ prevents the direct usage of $E$ in enforcing the independence
that is often adopted by previous methods~\citep{irmv1,v-rex,groupdro,DANN,CORAL}, making the identification of $G_c$ more challenging.

\textbf{Optimization objective.}
To mitigate this issue,
we need to find and translate other properties of $G_c$
into some differentiable and equivalent objectives to satisfy the independence constraint $\widehat{G}_{c}\ind E$.
\revision{
	{\it\noindent The goal of the desired objective. }
	We begin by considering a simplistic setting where all the invariant subgraphs $G_c$ have the
	same size $s_c$,} i.e., $|G_c|=s_c$\footnote{Throughout the paper, we use generalized set operators for the ease of understanding. They can have multiple implementations in terms of nodes, edges or attributes.}.
When maximizing $I(\widehat{G}_{c};Y)$ in Eq.~\ref{eq:good_opt},
both FIIF and PIIF can introduce part of $G_s$ into $\widehat{G}_{c}$.
In FIIF (Fig.~\ref{fig:scm_fiif}), as $G_c$ already contains the maximal possible information in $G$ about $Y$,
$G_c$ is a solution to $\max I(\widehat{G}_{c};Y)$. However, some subgraph of $G_c$ can be replaced by some subgraph of $G_s$ that is equally informative about $Y$.
In PIIF (Fig.~\ref{fig:scm_piif}), there also exists some subgraph of $G_s$ that
contains additional information about $Y$ than $G_c$, hence $\widehat{G}_{c}$ is more likely
to involve some subgraph of $G_s$.
Thus, the new objective needs to eliminate the auxiliary subgraphs of $\widehat{G}_{c}$ from $G_s$ such that the estimated $\widehat{G}_{c}$ can only contain $G_c$.

\revision{{\it\noindent An important property of $G_c$. }}
Under both FIIF and PIIF SCMs (Fig.~\ref{fig:scm}),
for $G_c^{e_1}$, $G_c^{e_2}$ that relate to the same causal factor $c$
under two environments $e_1$ and $e_2$,
the desired $\widehat{G}_c^{e_1}, \widehat{G}_c^{e_2}$
in $e_1$ and $e_2$ tend to have high mutual information, i.e.,
$(G_c^{e_1},G_c^{e_2})\in \argmax I(\widehat{G}_c^{e_1}; \widehat{G}_c^{e_2})$.
While for $G_c^{e_1}$ and another $G_{c'}^{e_1}$ corresponding to a different $c'\neq c$,
under the same environment $e_1$,
including any subgraph from $G_s^{e_1}$ in $\widehat{G}_c^{e_1},\widehat{G}_{c'}^{e_1}$
will enlarge their mutual information, or in other words,
$(G_c^{e_1},G_{c'}^{e_1})\in \argmin I(\widehat{G}_c^{e_1}; \widehat{G}_{c'}^{e_1})$.
Thus, we can derive an important property of $G_c$, that is, $\forall e_1,e_2\in\envall$,
\begin{equation}
	\label{eq:good_cond_1}
	G_c^{e_1}\in \text{$\argmax$}_{\widehat{G}_c^{e_1}}\  I(\widehat{G}_c^{e_1};\widehat{G}_c^{e_2}|C=c)-I(\widehat{G}_c^{e_1};\widehat{G}_{c'}^{e_2}|C=c',c'\neq c),
\end{equation}
\revision{where $\widehat{G}_c^{e_1}$ and $\widehat{G}_c^{e_2}$ are the estimated
invariant subgraphs corresponding to the same causal factor $c$ under
environment $e_1$ and $e_2$, respectively,
while $\widehat{G}_{c'}^{e_2}$ corresponds to a different causal factor $c'$.}

\revision{\it\noindent Deriving $\ginv$v1 based on the identified property of $G_c$. }
In practice, $C$ is not given.
Nevertheless, since $C$ and $Y$ shares a stable causal relationship in both FIIF and PIIF SCMs,
$Y$ can serve as a proxy of $C$ in Eq.~\ref{eq:good_cond_1}.
Moreover, as Eq.~\ref{eq:good_cond_1} holds for any $\forall e_1,e_2\in\envall$,
the environment superscripts can be eliminated without affecting Eq.~\ref{eq:good_cond_1}.
Furthermore,
when both $I(\widehat{G}^{e_1}_{c};\widehat{G}^{e_2}_c|C=c)$ and $I(\widehat{G}_{c};Y)$ are maximized,
$I(\widehat{G}^{e_1}_{c};\widehat{G}^{e_1}_{c'}|C=c',c'\neq c)$ is automatically minimized,
otherwise all classes will collapse to trivial solutions which is contradictory given
$I(\widehat{G}_{c};Y)$ being maximized.
Therefore, we can derive an alternative objective to Eq.~\ref{eq:good_opt}
by leveraging Eq.~\ref{eq:good_cond_1}
to replace the independence condition:
{\begin{snugshade}\vspace{-0.2in}
	\begin{equation}
		\label{eq:good_opt_contrast_v1}
		(\text{CIGAv1})\qquad\qquad\qquad \ %
		\max_{f_c, g} \ I(\widehat{G}_{c};Y), \ \text{s.t.}\
		\widehat{G}_{c}\in\argmax_{\widehat{G}_{c}=g(G), |\widehat{G}_{c}|\leq s_c} I(\widehat{G}_{c};\widetilde{G}_c|Y),
		\qquad%
	\end{equation}
\end{snugshade}\vspace{-0.1in}}
where $\widetilde{G}_c=g(\widetilde{G})$ and $\widetilde{G}\sim \sP(G|Y)$,
i.e., $\widetilde{G}$ is sampled from training graphs that share the same label $Y$ as $G$.
In Theorem~\ref{thm:good_inv_gnn_new}, we show how Eq.~\ref{eq:good_opt_contrast_v1}
is equivalent to Eq.~\ref{eq:good_opt}.
Nevertheless, Eq.~\ref{eq:good_opt_contrast_v1} requires a strong
assumption on the size of $G_c$.
However, the size of $G_c$ is usually unknown or changes for different $C$s.
In this circumstance, maximizing Eq.~\ref{eq:good_cond_1} without additional constraints will lead to the presence of part of $G_s$ in $\widehat{G}_{c}$.
For instance, $\widehat{G}_{c}=G$ is a trivial solution to Eq.~\ref{eq:good_opt_contrast_v1}
when $s_c = \infty$.

\revision{\it\noindent Deriving $\ginv$v2 by resolving size constraint on $G_c$ in $\ginv$v1. }
To this end, we further resort to the properties of $G_s$.
In both FIIF and PIIF SCMs (Fig.~\ref{fig:scm}), $G_s$ and $G_c$ can share certain
overlapped information about $Y$. When maximizing $I(\widehat{G}_{c};\widetilde{G}_c|Y)$ and $I(\widehat{G}_{c};Y)$,
the appearance of partial $G_s$ in $\widehat{G}_{c}$ will not affect the optimality.
However, it can reduce the mutual information between the
left part $\widehat{G}_s=G-\widehat{G}_{c}$ and $Y$, i.e., $I(\widehat{G}_s;Y)$.
Therefore, by maximizing $I(\widehat{G}_s;Y)$, we can reduce including part of $G_s$ into $\widehat{G}_{c}$.
Meanwhile, to avoid trivial solution that $G_c\subseteq\widehat{G}_s$ during maximizing $I(\widehat{G}_s;Y)$,
we can leverage the better clustering property of $G_c$ implied by Assumption~\ref{assump:latent_sep}
to derive the constraint $I(\widehat{G}_s;Y)\leq I(\widehat{G}_{c};Y)$.
Thus, we can obtain a new objective $\ginv$v2
as follows:
{\begin{snugshade}
	\vspace{-0.2in}
	\begin{equation}
		\label{eq:good_opt_contrast_v3}
		\begin{aligned}
			\text{$\max$}_{f_c, g} \  I(\widehat{G}_{c};Y)+I(\widehat{G}_s;Y), \ \text{s.t.}\
			 & \widehat{G}_{c}\in
			\text{$\argmax$}_{\widehat{G}_{c}=g(G)} I(\widehat{G}_{c};\widetilde{G}_c|Y),\
			\\
			(\text{CIGAv2})\qquad\qquad\qquad\qquad\qquad\qquad\ %
			 & I(\widehat{G}_s;Y)\leq I(\widehat{G}_{c};Y),\ \widehat{G}_s=G-g(G),
			\qquad\ \ %
		\end{aligned}
	\end{equation}
\end{snugshade}\vspace{-0.1in}}
where $\widehat{G}_{c}=g(G),\widetilde{G}_c=g(\widetilde{G})$ and $\widetilde{G}\sim \sP(G|Y)$,
i.e., $\widetilde{G}$ is sampled from training graphs that share the same label $Y$ as $G$.
We also prove the equivalence between Eq.~\ref{eq:good_opt_contrast_v3} and Eq.~\ref{eq:good_opt}
in Theorem~\ref{thm:good_inv_gnn_new}.

\subsection{Theoretical Analysis and Practical Discussions}
\label{sec:good_theory}
\begin{theorem}[$\ginv$ Induces Invariant GNNs]
	\label{thm:good_inv_gnn_new}
	Given a set of graph datasets $\{\dataset^e\}_e$ %
	and environments $\envall$ that follow the same graph generation process in Sec.~\ref{sec:data_gen},
	assuming that \textup{(a)} $f_\gen^G$ and $f_\gen^{G_c}$ in Assumption~\ref{assump:graph_gen} are invertible,
	\textup{(b)} samples from each training environment are equally distributed,
	i.e.,$|\dataset_{\hat{e}}|=|\dataset_{\tilde{e}}|,\ \forall \hat{e},\tilde{e}\in\envtrain$, then:
	\vspace{-0.05in}
	\begin{enumerate}[label=(\roman*).,wide, labelwidth=!]
		\item If $\forall G_c, |G_c|=s_c$,
		      then each solution to Eq.~\ref{eq:good_opt_contrast_v1},
		      elicits an invariant GNN (Def.~\ref{def:inv_gnn}).
		      \vspace{-0.05in}
		\item Each solution to Eq.~\ref{eq:good_opt_contrast_v3},
		      elicits an invariant GNN (Def.~\ref{def:inv_gnn}).
	\end{enumerate}
\end{theorem}
\vspace{-0.05in}
We prove Theorem~\ref{thm:good_inv_gnn_new} (i) and (ii)
in Appendix~\ref{proof:good_inv_gnn_new_appdx},~\ref{proof:good_inv_gnn_v3_appdx}, respectively.

\textbf{Practical implementations of $\ginv$ objectives.}
After showing the power of $\ginv$,
we introduce the practical implementations of $\ginv$v1 and $\ginv$v2 objectives.
Specifically,  an exact estimate of the second term $I(\widehat{G}_{c};\widetilde{G}_c|Y)$ could be highly expensive~\citep{infoNCE,mine}.
However, contrastive learning with supervised sampling provides a practical solution for the approximation~\citep{sup_contrastive,contrast_loss1,contrast_loss2,infoNCE,mine}:
\vspace{-0.05in}
\begin{equation} \label{eq:good_opt_contrast}
	I(\widehat{G}_{c};\widetilde{G}_c|Y) \approx
	\mathbb{E}_{
	\substack{
	\{\widehat{G}_{c},\widetilde{G}_c\} \sim \sP_g(G|\gY=Y)\\\
	\{G^i_c\}_{i=1}^{M} \sim \sP_g(G|\gY \neq Y)
	}
	}
	\log\frac{e^{\phi(h_{\widehat{G}_{c}},h_{\widetilde{G}_c})}}
	{e^{\phi(h_{\widehat{G}_{c}},h_{\widetilde{G}_c})} +
		\sum_{i=1}^M e^{\phi(h_{\widehat{G}_{c}},h_{G^i_c})}},
\end{equation}
where positive samples $(\widehat{G}_{c},\widetilde{G}_c)$ are the extracted subgraphs of graphs that share the same label as $G$,
negative samples are those having different labels, $\sP_g(G|\gY=Y)$ is the push-forward distribution of $\sP(G|\gY=Y)$ by featurizer $g$,
$\sP(G|\gY=Y)$ refers to the distribution of $G$ given the label $Y$,
$\sP(G|\gY\neq Y)$ refers to the distribution of $G$ given the label that is different from $Y$,
$h_{\widehat{G}_{c}},h_{\widetilde{G}_c},h_{G^i_c}$ are the graph presentations of the estimated subgraphs,
and $\phi$ is the similarity metric for graph representations.
As $M\rightarrow \infty$, Eq.~\ref{eq:good_opt_contrast} approximates $I(\widehat{G}_{c};\widetilde{G}_c|Y)$,
which can be regarded as a
non-parameteric resubstitution entropy estimator via the von Mises-Fisher
kernel density~\citep{feat_dist_entropy,vMF_entropy,align_uniform}.
Thus, plugging it into Eq.~\ref{eq:good_opt_contrast_v1} and Eq.~\ref{eq:good_opt_contrast_v3}
can relieve the issue of approximating $I(\widehat{G}_{c};\widetilde{G}_c|Y)$ in practice.

For the implementation of $I(\widehat{G}_s;Y)$ and the constraint $I(\widehat{G}_s;Y)\leq I(\widehat{G}_{c};Y)$ in $\ginv$v2,
a practical choice is to follow the idea of hinge loss, $I(\widehat{G}_s;Y)\approx \frac{1}{N} R_{\widehat{G}_s}\cdot \mathbb{I}(R_{\widehat{G}_c}\leq R_{\widehat{G}_{s}})$,
where $N$ is the number of samples, $\mathbb{I}$ is an indicator function that outputs $1$ when the
inner condition is satisfied otherwise $0$, and
$R_{\widehat{G}_s}$ and $R_{\widehat{G}_{c}}$ are the empirical risk vector of the predictions
for each sample based on the corresponding $\widehat{G}_s$ and $\widehat{G}_{c}$.
More implementation details can be found in Appendix~\ref{sec:good_impl_appdx}.

\textbf{Discussions and implications of $\ginv$.}
Although using contrastive learning to improve OOD generalization is not new
in the literature~\citep{DouCKG19,causal_matching,cnc},
previous methods cannot yield OOD guarantees in graph circumstances due to  the highly non-linearity and the unavailability of domain labels $E$.
In particular, $\ginv$ can \textit{be reduced to directly applying contrastive learning}
when without the decomposition for causal algorithmic alignment.
However, in the experiments we found that merely using the contrastive objective, i.e., CNC~\citep{cnc},
yields unsatisfactory OOD generalization performance,
which further implies the necessity of the decomposition in $\ginv$.

Moreover, the architecture of $\ginv$ can have multiple other implementations
for both the featurizer and classifier, such as identifying $G_c$ at the latent space~\citep{causality4ml,towards_causality}.
Since we cannot enumerate every possible implementation, in this work
we choose interpretable GNN architectures as a prototype validation for $\ginv$ and leave more sophisticated architectures as future works.
In particular,
when optimized with ERM objective,
$\ginv$ can \textit{be reduced to interpretable GNNs}.
However, merely using interpretable GNNs such as ASAP~\citep{asap}, GIB~\citep{gib} or DIR~\citep{dir} cannot yield satisfactory OOD performance.
As shown in Table~\ref{tab:potential_alg}
and discussed in Appendix.~\ref{sec:discussion_ood_obj_appdx},
GIB can only work for FIIF,
while DIR \textit{cannot} yield OOD guarantees for neither FIIF and PIIF SCMs.
These results are also empirically validated in the experiments.
We provide more detailed discussions in Appendix~\ref{sec:discuss_future_appdx}.

\begin{table}[t]
	\vspace{-0.4in}
	\scriptsize
	\caption{OOD generalization performance on structure and mixed shifts for synthetic graphs.}
	\vspace{-0.1in}
	\label{tab:sythetic}
	\begin{sc}
		\begin{center}
			\begin{tabular}{l|rrr|rrr|r}
				\toprule
				                                            &
				\multicolumn{3}{c}{SPMotif-Struc$^\dagger$}
				                                            &
				\multicolumn{3}{c}{SPMotif-Mixed$^\dagger$} &                                                                                        \\

				                                            &
				\multicolumn{1}{c}{bias=$0.33$}
				                                            &
				\multicolumn{1}{c}{bias=$0.60$}
				                                            &
				\multicolumn{1}{c}{bias=$0.90$}
				                                            &
				\multicolumn{1}{c}{bias=$0.33$}
				                                            &
				\multicolumn{1}{c}{bias=$0.60$}
				                                            &
				\multicolumn{1}{c}{bias=$0.90$}             & \multicolumn{1}{c}{Avg}
				\\

				\cmidrule(lr){2-4}\cmidrule(lr){5-7}\cmidrule(lr){8-8}
				ERM                                         & 59.49 (3.50)            & 55.48 (4.84)
				                                            & 49.64 (4.63)
				                                            & 58.18 (4.30)            & 49.29 (8.17)
				                                            & 41.36 (3.29)            & 52.24                                                        \\
				ASAP                                        & 64.87 (13.8)            & 64.85 (10.6)
				                                            & \textbf{57.29 (14.5)}
				                                            & 66.88 (15.0)            & 59.78 (6.78)
				                                            & \textbf{50.45 (4.90)}   & 60.69                                                        \\

				DIR                                         & 58.73 (11.9)            & 48.72 (14.8)
				                                            & 41.90 (9.39)
				                                            & 67.28 (4.06)            & 51.66 (14.1)
				                                            & 38.58 (5.88)            & 51.14                                                        \\

				\hline
				\rule{0pt}{8pt}IRM                          & 57.15 (3.98)            & 61.74 (1.32)
				                                            & 45.68 (4.88)
				                                            & 58.20 (1.97)            & 49.29 (3.67)
				                                            & 40.73 (1.93)            & 52.13                                                        \\
				V-Rex                                       & 54.64 (3.05)            & 53.60 (3.74)
				                                            & 48.86 (9.69)
				                                            & 57.82 (5.93)            & 48.25 (2.79)
				                                            & 43.27 (1.32)            & 51.07                                                        \\
				EIIL                                        & 56.48 (2.56)            & 60.07 (4.47)
				                                            & 55.79 (6.54)
				                                            & 53.91 (3.15)            & 48.41 (5.53)
				                                            & 41.75 (4.97)            & 52.73                                                        \\
				IB-IRM                                      & 58.30 (6.37)            & 54.37 (7.35)
				                                            & 45.14 (4.07)
				                                            & 57.70 (2.11)            & 50.83 (1.51)
				                                            & 40.27 (3.68)            & 51.10                                                        \\

				CNC                                         & 70.44 (2.55)            & \textbf{66.79 (9.42)}
				                                            & 50.25 (10.7)
				                                            & 65.75 (4.35)            & 59.27 (5.29)
				                                            & 41.58 (1.90)            & 59.01                                                        \\

				\hline
				\rule{0pt}{8pt}\textbf{$\ginv$v1}           & \textbf{71.07 (3.60)}   & 63.23 (9.61)
				                                            & 51.78 (7.29)
				                                            & \textbf{74.35 (1.85)}   & \textbf{64.54 (8.19)}
				                                            & 49.01 (9.92)            & \textbf{62.33}                                               \\
				\textbf{$\ginv$v2}                          & \textbf{77.33 (9.13)}   & \textbf{69.29 (3.06)}
				                                            & \textbf{63.41 (7.38)}
				                                            & \textbf{72.42 (4.80)}   & \textbf{70.83 (7.54)}
				                                            & \textbf{54.25 (5.38)}   & \textbf{67.92}                                               \\
				\hline
				\rule{0pt}{8pt}\revision{Oracle (IID)}      &                         & \revision{88.70 (0.17)} &  &  & \revision{88.73 (0.25)} &  & \\
				\bottomrule
				\multicolumn{8}{l}{\rule{0pt}{8pt}$^\dagger$\text{\normalfont \small Higher accuracy and lower variance indicate better OOD generalization ability.}  }
			\end{tabular}
		\end{center}
	\end{sc}
	\vspace{-0.25in}
\end{table}

\section{Empirical Studies}
\label{sec:exp}

We conduct extensive experiments with $16$ datasets
to verify the effectiveness of $\ginv$.

\textbf{Datasets.}
We use the SPMotif datasets from DIR~\citep{dir} where
artificial structural shifts and graph size shifts are nested (SPMotif-Struc).
Besides, we  construct a harder version mixed with attribute shifts (SPMotif-Mixed).
To examine $\ginv$ in real-world scenarios with more complicated relationships and distribution shifts,
we also use DrugOOD~\citep{drugood} from AI-aided Drug Discovery with Assay, Scaffold, and Size splits,
convert the ColoredMNIST from IRM~\citep{irmv1} using the algorithm from~\citet{understand_att} to inject attribute shifts,
and split Graph-SST~\citep{xgnn_tax} to inject degree biases.
To compare with previous specialized OOD methods for graph size shifts~\citep{size_gen1,size_gen2},
we use the datasets in~\citet{size_gen2} that are converted from TU benchmarks~\citep{tudataset}.
More details can be found in Appendix~\ref{sec:exp_data_appdx}.

\textbf{Baselines and our methods.} Besides the ERM, we also compare with SOTA interpretable GNNs,
GIB~\citep{gib}, ASAP Pooling~\citep{asap}, and DIR~\citep{dir}, to validate the effectiveness of the optimization objective in $\ginv$.
We use the same selection ratio (i.e., $s_c$) for all models.
Moreover, to validate the effectiveness of the decomposition in $\ginv$,
we compare $\ginv$ with SOTA OOD objectives including IRM~\citep{irmv1},
v-Rex~\citep{v-rex} and IB-IRM~\citep{ib-irm}, for which we apply random environment partitions following~\citep{env_inference}.
We also compare $\ginv$ with
EIIL~\citep{env_inference} and CNC~\citep{cnc} that do not require environment labels,
where CNC~\citep{cnc} has a more sophisticated contrastive sampling strategy for combating subpopulation shifts.
More implementation and comparison details are deferred to Appendix~\ref{sec:exp_impl_appdx}.

\textbf{Evaluation.} We report the classification accuracy for all datasets,
except for DrugOOD datasets where we use ROC-AUC following~\citep{drugood}, and
for TU datasets where we use Matthews correlation coefficient following~\citep{size_gen2}.
We repeat the evaluation multiple times, select models based on the validation performances,
and report the mean and standard deviation of the corresponding metric. \revision{For each dataset, we also report the ``Oracle'' performances that run ERM on the randomly shuffled data.}

\textbf{OOD generalization performance on structure and mixed shifts.}
In Table~\ref{tab:sythetic}, we report the test accuracy of each method, where we omit GIB due to its poor convergence.
Different biases indicate different strengths of the distribution shifts.
Although the training accuracy of most methods converges to more than $99\%$,
the test accuracy decreases dramatically as the bias increases and as more distribution shifts are mixed,
which concurs with our discussions in Sec.~\ref{sec:limitation_prev} and Appendix~\ref{sec:good_fail_setting_appdx}.
Due to the simplicity of the task as well as the relatively high support overlap between training and test distributions,
interpretable GNNs and OOD objectives can improve certain OOD performance, while they can have \emph{high variance} since they donot have OOD generalization guarantees.
In contrast, $\ginv$v1 and $\ginv$v2 outperform all of the baselines by a significant margin up to $10\%$ with \emph{lower variance},
which demonstrates the effectiveness and excellent OOD generalization ability of $\ginv$.

\begin{table}
	\vspace{-0.35in}
	\scriptsize\center
	\caption{OOD generalization performance on complex distribution shifts for real-world graphs.}
	\vspace{-0.1in}
	\label{table:other_graph}
	\sc
	\resizebox{\columnwidth}{!}{
		\begin{sc}
			\begin{tabular}{lrrrrrrc}
				\toprule
				Datasets                               & \multicolumn{1}{c}{Drug-Assay} & \multicolumn{1}{c}{Drug-Sca} & \multicolumn{1}{c}{Drug-Size} & \multicolumn{1}{c}{CMNIST-sp} & \multicolumn{1}{c}{Graph-SST5} & \multicolumn{1}{c}{Twitter} & \multicolumn{1}{c}{Avg (Rank)$^\dagger$} \\
				\midrule
				ERM                                    & 71.79 (0.27)                   & 68.85 (0.62)                 & 66.70 (1.08)                  & 13.96 (5.48)                  & 43.89 (1.73)                   & 60.81 (2.05)                & 54.33 (6.00)                             \\
				ASAP                                   & 70.51 (1.93)                   & 66.19 (0.94)                 & 64.12 (0.67)                  & 10.23 (0.51)                  & 44.16 (1.36)                   & 60.68 (2.10)                & 52.65 (8.33)                             \\
				GIB                                    & 63.01 (1.16)                   & 62.01 (1.41)                 & 55.50 (1.42)                  & 15.40 (3.91)                  & 38.64 (4.52)                   & 48.08 (2.27)                & 47.11 (10.0)                             \\
				DIR                                    & 68.25 (1.40)                   & 63.91 (1.36)                 & 60.40 (1.42)                  & 15.50 (8.65)                  & 41.12 (1.96)                   & 59.85 (2.98)                & 51.51 (9.33)                             \\\hline
				\rule{0pt}{8pt}IRM                     & 72.12 (0.49)                   & 68.69 (0.65)                 & 66.54 (0.42)                  & 31.58 (9.52)                  & 43.69 (1.26)                   & 63.50 (1.23)                & 57.69 (4.50)                             \\
				V-Rex                                  & 72.05 (1.25)                   & 68.92 (0.98)                 & 66.33 (0.74)                  & 10.29 (0.46)                  & 43.28 (0.52)                   & 63.21 (1.57)                & 54.01 (6.17)                             \\
				EIIL                                   & 72.60 (0.47)                   & 68.45 (0.53)                 & 66.38 (0.66)                  & 30.04 (10.9)                  & 42.98 (1.03)                   & 62.76 (1.72)                & 57.20 (5.33)                             \\
				IB-IRM                                 & 72.50 (0.49)                   & 68.50 (0.40)                 & 66.64 (0.28)                  & \textbf{39.86 (10.5)}         & 40.85 (2.08)                   & 61.26 (1.20)                & 58.27 (5.33)                             \\
				CNC                                    & 72.40 (0.46)                   & 67.24 (0.90)                 & 65.79 (0.80)                  & 12.21 (3.85)                  & 42.78 (1.53)                   & 61.03 (2.49)                & 53.56 (7.50)                             \\
				\hline
				\rule{0pt}{8pt}\textbf{$\ginv$v1}      & \textbf{72.71 (0.52)}          & \textbf{69.04 (0.86)}        & \textbf{67.24 (0.88)}         & 19.77 (17.1)                  & \textbf{44.71 (1.14)}          & \textbf{63.66 (0.84)}       & \textbf{56.19 (2.50)}                    \\
				\textbf{$\ginv$v2}                     & \textbf{73.17 (0.39)}          & \textbf{69.70 (0.27)}        & \textbf{67.78 (0.76)}         & \textbf{44.91 (4.31)}         & \textbf{45.25 (1.27)}          & \textbf{64.45 (1.99)}       & \textbf{60.88 (1.00)}                    \\
				\hline
				\rule{0pt}{8pt}\revision{Oracle (IID)} & \revision{85.56 (1.44)}        & \revision{84.71 (1.60)}
				                                       & \revision{85.83 (1.31)}
				                                       & \revision{62.13 (0.43)}        & \revision{48.18 (1.00)}
				                                       & \revision{64.21 (1.77)}        &                                                                                                                                                                                                        \\
				\bottomrule
				\multicolumn{8}{l}{\rule{0pt}{8pt}$^\dagger$\text{\normalfont \small Averaged rank is also reported in the blankets because of dataset heterogeneity. Lower rank is better.}  }
			\end{tabular}
		\end{sc}
	}
	\vspace{-0.2in}
\end{table}

\begin{wraptable}{r}{0.65\textwidth}
	\vspace{-0.3in}
	\scriptsize
	\caption{OOD generalization performance on graph size shifts for real-world graphs in terms of Matthews correlation coefficient.}
	\vspace{0.05in}
	\label{table:graph_size}
	\resizebox{0.65\columnwidth}{!}{
		\begin{sc}
			\begin{tabular}{lrrrrr}
				\toprule
				Datasets                               & \multicolumn{1}{c}{NCI1} & \multicolumn{1}{c}{NCI109} & \multicolumn{1}{c}{PROTEINS} & \multicolumn{1}{c}{DD} & \multicolumn{1}{c}{Avg} \\
				\midrule
				ERM                                    & 0.15 (0.05)              & 0.16 (0.02)                & 0.22 (0.09)                  & 0.27 (0.09)            & 0.20                    \\
				ASAP                                   & 0.16 (0.10)              & 0.15 (0.07)                & 0.22 (0.16)                  & 0.21 (0.08)            & 0.19                    \\
				GIB                                    & 0.13 (0.10)              & 0.16 (0.02)                & 0.19 (0.08)                  & 0.01 (0.18)            & 0.12                    \\
				DIR                                    & 0.21 (0.06)              & 0.13 (0.05)                & 0.25 (0.14)                  & 0.20 (0.10)            & 0.20                    \\\hline
				\rule{0pt}{8pt}IRM                     & 0.17 (0.02)              & 0.14 (0.01)                & 0.21 (0.09)                  & 0.22 (0.08)            & 0.19                    \\
				V-Rex                                  & 0.15 (0.04)              & 0.15 (0.04)                & 0.22 (0.06)                  & 0.21 (0.07)            & 0.18                    \\
				EIIL                                   & 0.14 (0.03)              & 0.16 (0.02)                & 0.20 (0.05)                  & 0.23 (0.10)            & 0.19                    \\
				IB-IRM                                 & 0.12 (0.04)              & 0.15 (0.06)                & 0.21 (0.06)                  & 0.15 (0.13)            & 0.16                    \\
				CNC                                    & 0.16 (0.04)              & 0.16 (0.04)                & 0.19 (0.08)                  & 0.27 (0.13)            & 0.20                    \\
				\hline
				\rule{0pt}{8pt}WL kernel               & \textbf{0.39 (0.00)}     & 0.21 (0.00)                & 0.00 (0.00)                  & 0.00 (0.00)            & 0.15                    \\
				GC kernel                              & 0.02 (0.00)              & 0.00 (0.00)                & 0.29 (0.00)                  & 0.00 (0.00)            & 0.08                    \\
				$\Gamma_\text{1-hot}$                  & 0.17 (0.08)              & \textbf{0.25 (0.06)}       & 0.12 (0.09)                  & 0.23 (0.08)            & 0.19                    \\
				$\Gamma_\text{GIN}$                    & 0.24 (0.04)              & 0.18 (0.04)                & 0.29 (0.11)                  & \textbf{0.28 (0.06)}   & 0.25                    \\
				$\Gamma_\text{RPGIN}$                  & 0.26 (0.05)              & 0.20 (0.04)                & 0.25 (0.12)                  & 0.20 (0.05)            & 0.23                    \\
				\hline
				\rule{0pt}{8pt}\textbf{$\ginv$v1}      & 0.22 (0.07)              & \textbf{0.23 (0.09)}       & \textbf{0.40 (0.06)}         & \textbf{0.29 (0.08)}   & \textbf{0.29}           \\
				\textbf{$\ginv$v2}                     & \textbf{0.27 (0.07)}     & 0.22 (0.05)                & \textbf{0.31 (0.12)}         & 0.26 (0.08)            & \textbf{0.27}           \\
				\hline
				\rule{0pt}{8pt}\revision{Oracle (IID)} & \revision{0.32 (0.05)}   & \revision{0.37 (0.06)}     & \revision{0.39 (0.09)}       & \revision{0.33 (0.05)} &                         \\
				\bottomrule
			\end{tabular}
		\end{sc}
	}
	\vspace{-0.15in}
\end{wraptable}

\textbf{OOD generalization performance on realistic shifts.}
In Table~\ref{table:other_graph} and Table~\ref{table:graph_size},
we examine the effectiveness of $\ginv$ in real-world data and more complicated distribution shifts. Both averaged accuracy and ranks are reported because of the dataset heterogeneity.
Since the tasks are harder than synthetic ones, interpretable GNNs and OOD objectives perform similar to or even under-perform the ERM baselines, which is also consistent to the observations in non-linear benchmarks~\citep{domainbed,drugood}.
However, both $\ginv$v1 and $\ginv$v2 consistently and significantly outperform previous methods,
including previous specialized methods $\Gamma$ GNNs~\citep{size_gen2} for combating graph size shifts,
demonstrating the generality and superiority of $\ginv$.

\textbf{Comparisons with advanced ablation variants.}
As discussed in Sec.~\ref{sec:good_theory}, $\ginv$ can be reduced to interpretable GNNs and contrastive learning approaches.
However, across all experiments, we can observe that neither the advanced interpretable GNNs (DIR) nor sophisticated contrastive objectives with specialized sampling strategy (CNC) can yield satisfactory OOD performance, which serves as \emph{strong evidence} for the necessities of the decomposition as well as the objective in $\ginv$.
Furthermore, although $\ginv$v1 can outperform $\ginv$v2 when we may have a relatively accurate $s_c$,
the improvements in $\ginv$v1 are not as stable as $\ginv$v2 or even unsatisfactory when the assumption is violated.
This phenomenon also reveals the superiority of $\ginv$v2 in practice.

\revision{
	\textbf{Hyperparameter sensitivity analysis.} 
	To examine how sensitive $\ginv$ is to the hyperparamters $\alpha$ and $\beta$ for contrastive loss and hinge loss, respectively. 
	We conduct experiments based on the hardest datasets from each table (i.e., SPMotif-Mixed with the bias of $0.9$, DrugOOD-Scaffold and the NCI109 datasets from Table~\ref{tab:sythetic}, Table~\ref{table:other_graph}, and Table~\ref{table:graph_size}, respectively.) with different $\alpha$ and $\beta$. 
	When changing the value of $\beta$, we fix the $\alpha$ to a specific value under which the model has a relatively good performance (but not the best, to fully examine the robustness of $\ginv$ in practice). 
}

\revision{The results are shown in Fig.~\ref{fig:hp_sen_alpha} and Fig.~\ref{fig:hp_sen_beta}. 
	It can be found that both $\ginv$v1 and $\ginv$v2 are robust to different values of $\alpha$ and $\beta$, respectively, across different datasets and distribution shifts. 
	Besides, the results also reflect the effects of the additional penalty terms in $\ginv$.
	For example, in Fig.~\ref{fig:hp_sen_alpha_appdx}, when  $\alpha$ is too small, the invariance of the identified invariant subgraphs $\widehat{G}_c$ may not be guaranteed, resulting worse performances. 
	Similarly, as shown in Fig.~\ref{fig:hp_sen_beta_appdx}, when $\beta$ becomes too small, some part of the spurious subgraph may still appear in the estimated invariant subgraphs, which yields worse performances.
	Besides, when $\alpha$ and $\beta$ become too large, the optimization of $\ginv$ can be affected due to their intrinsic conflicts with ERM, hence a better optimization scheme for $\ginv$ can be a promising future direction~\citep{pair}.
	We provide more details and additional analysis on the efficiency of $\ginv$ and single environment OOD generalization performance of $\ginv$ in Appendix~\ref{sec:additional_exp_appdx}, as well as the visualization examples of the identified invariant subgraph in Appendix~\ref{sec:interpret_visualize_appdx}.}

\begin{figure}[t]
	\vspace{-0.35in}
	\centering
	\subfigure[SPMotif-Mixed (bias=$0.9$)]{
		\includegraphics[width=0.31\textwidth]{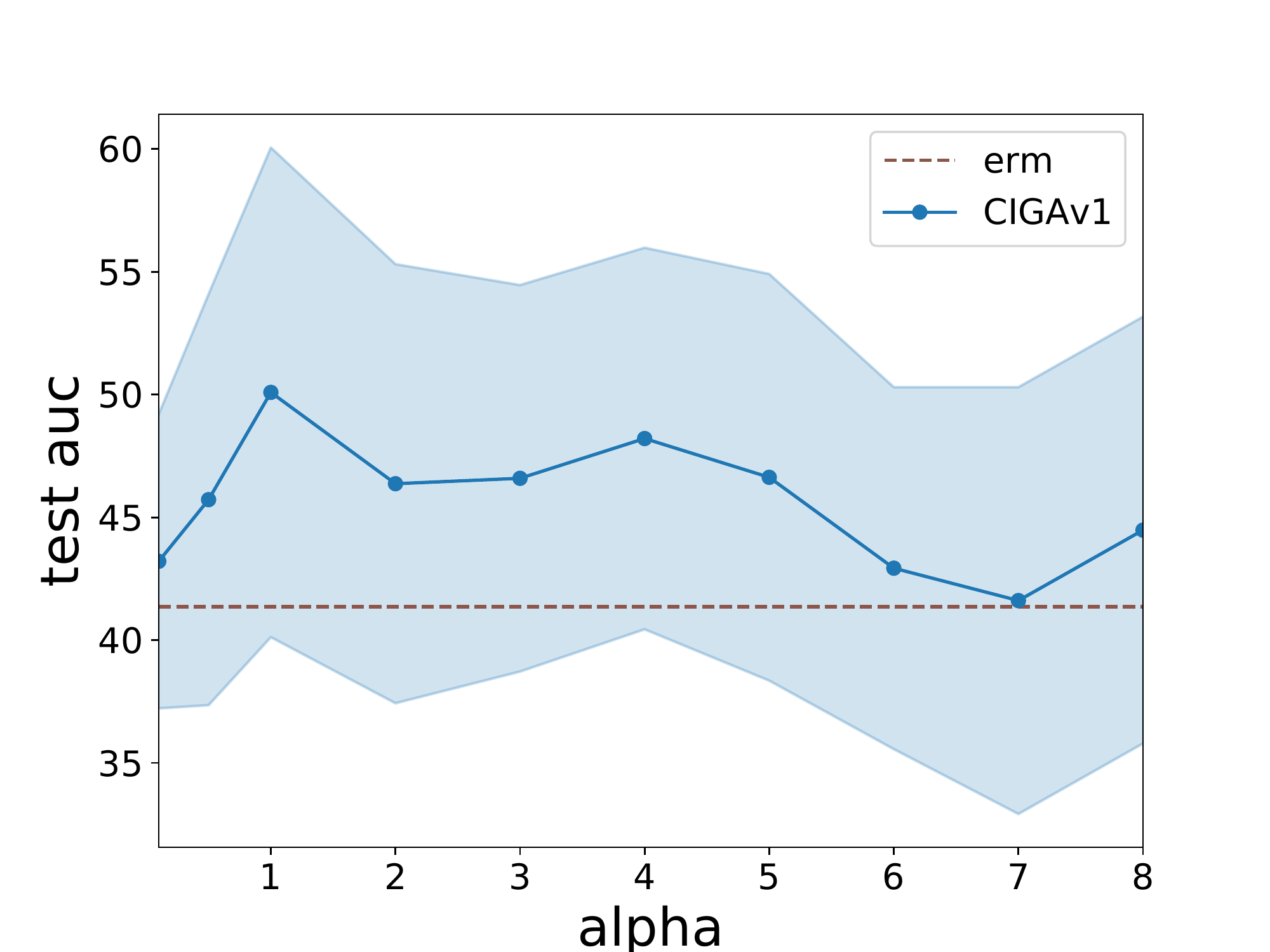}
	}
	\subfigure[DrugOOD-Scaffold]{
		\includegraphics[width=0.31\textwidth]{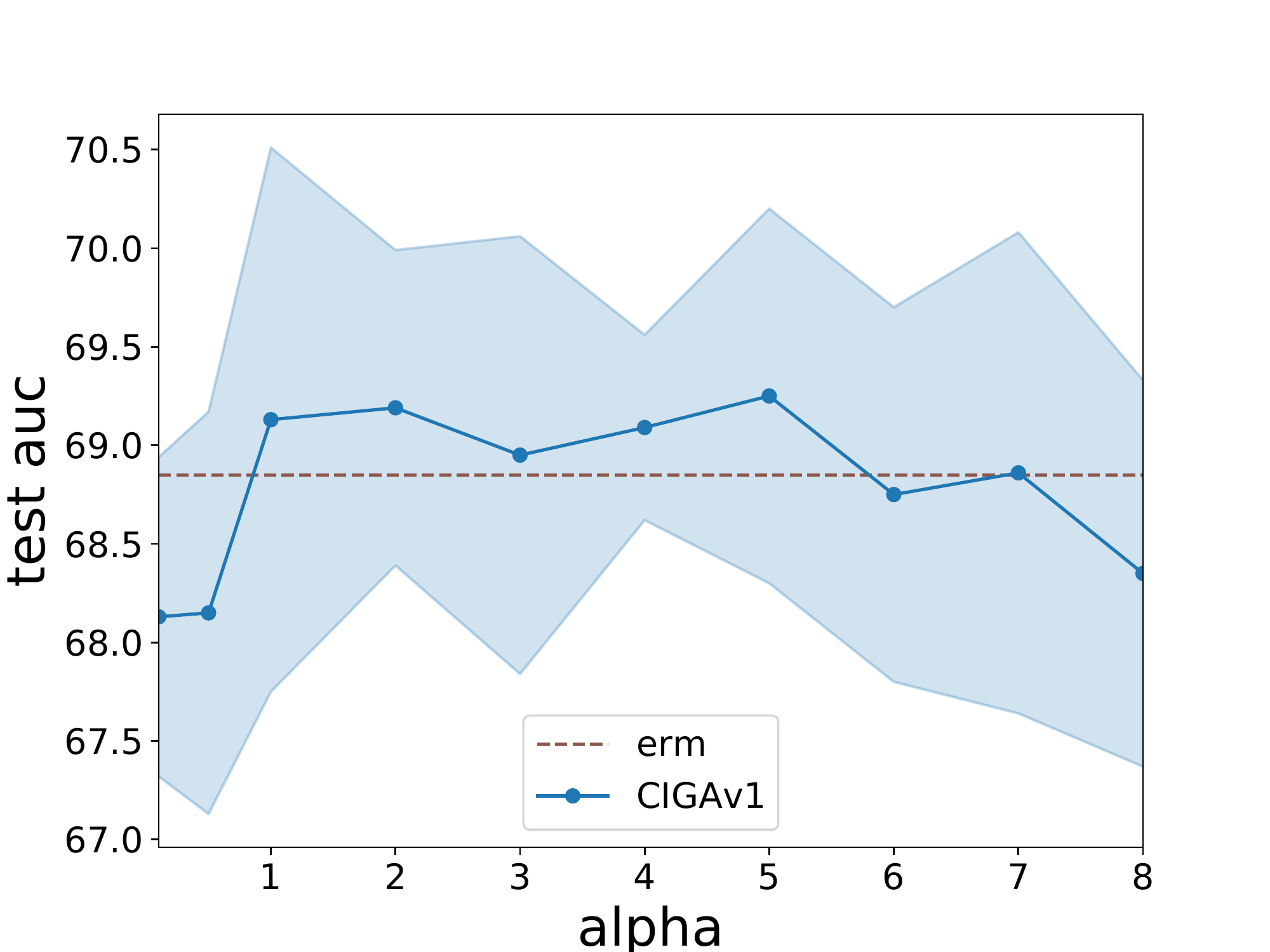}
	}
	\subfigure[NCI109]{
		\includegraphics[width=0.31\textwidth]{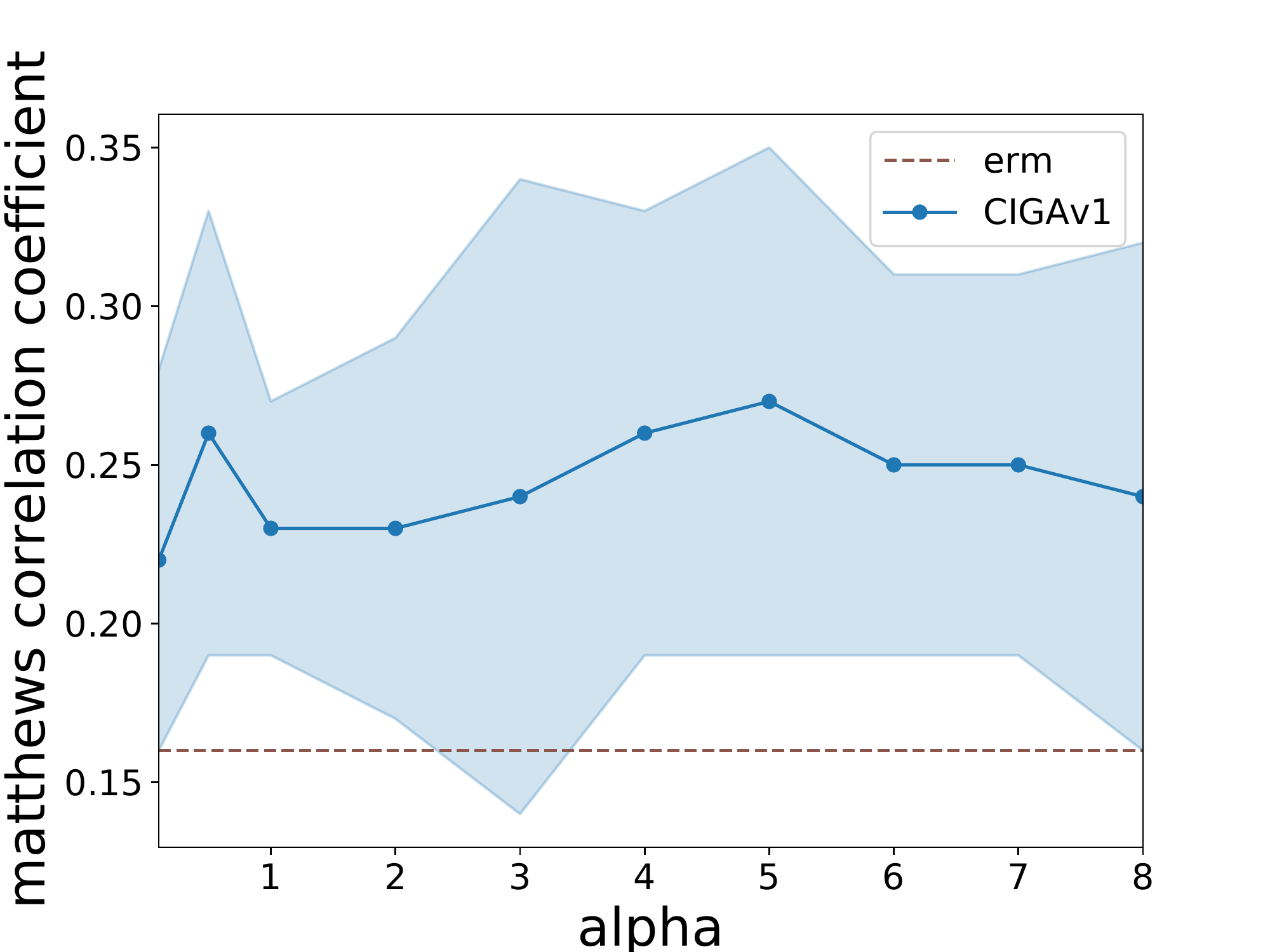}
	}
    \vspace{-0.05in}
	\caption{
		Hyperparameter sensitivity analysis on the coefficient of contrastive loss ($\alpha$).}
	\label{fig:hp_sen_alpha}
	\vskip -0.15in
\end{figure}

\begin{figure}[t]
	\centering
	\subfigure[SPMotif-Mixed (bias=$0.9$, $\alpha$=$4$)]{
		\includegraphics[width=0.31\textwidth]{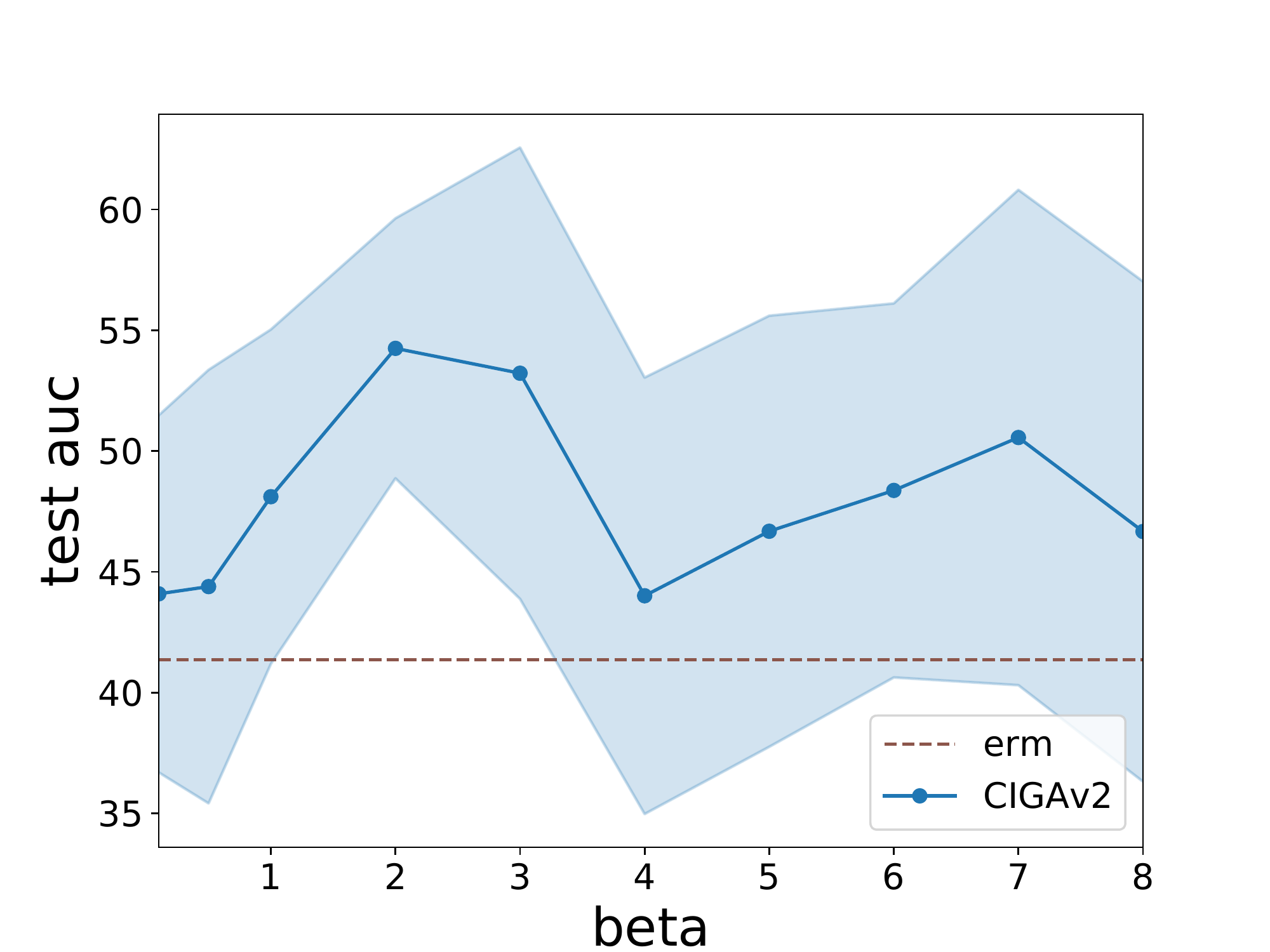}
	}
	\subfigure[DrugOOD-Scaffold ($\alpha$=$1$)]{
		\includegraphics[width=0.31\textwidth]{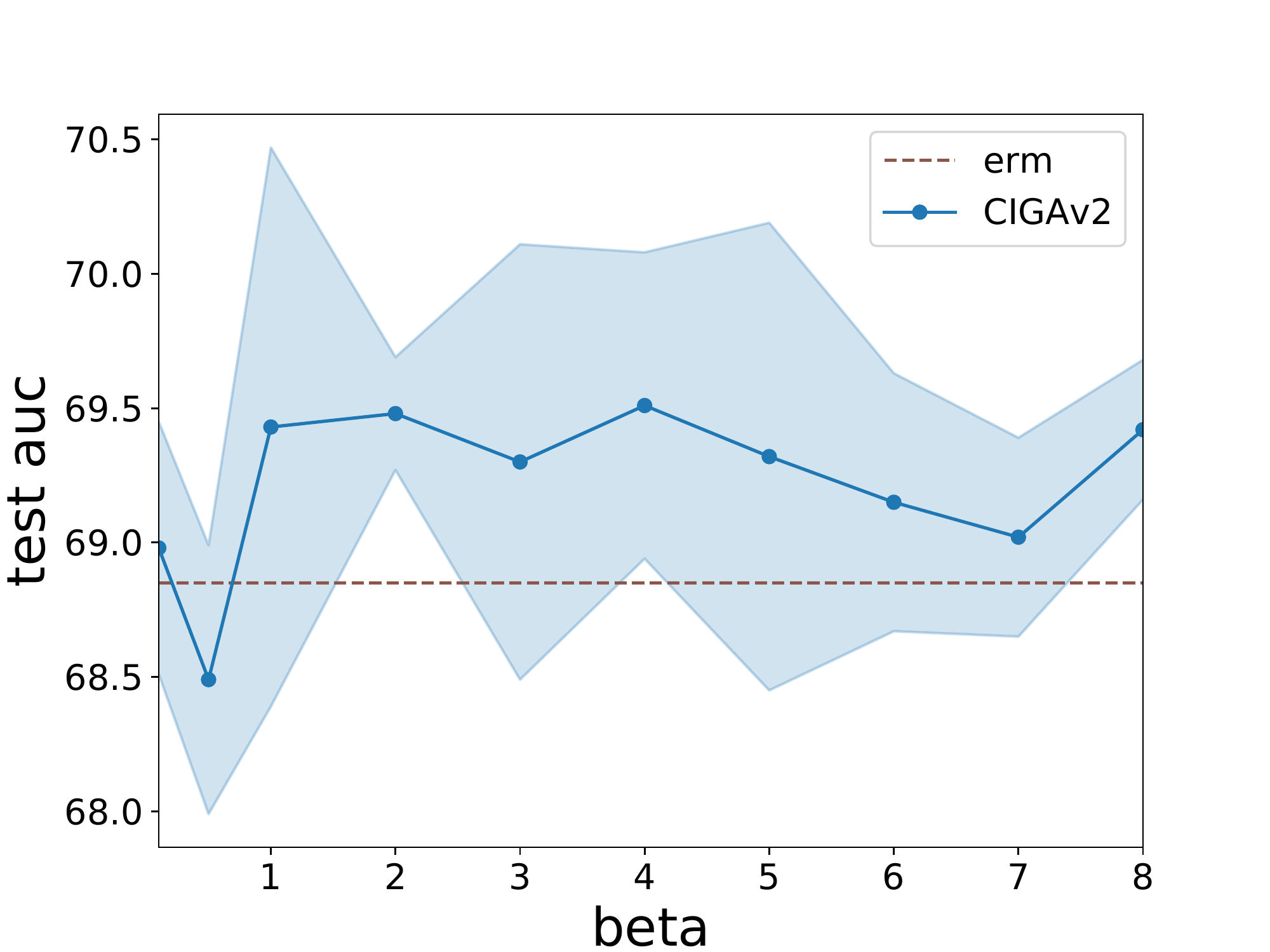}
	}
	\subfigure[NCI109 ($\alpha$=$1$)]{
		\includegraphics[width=0.31\textwidth]{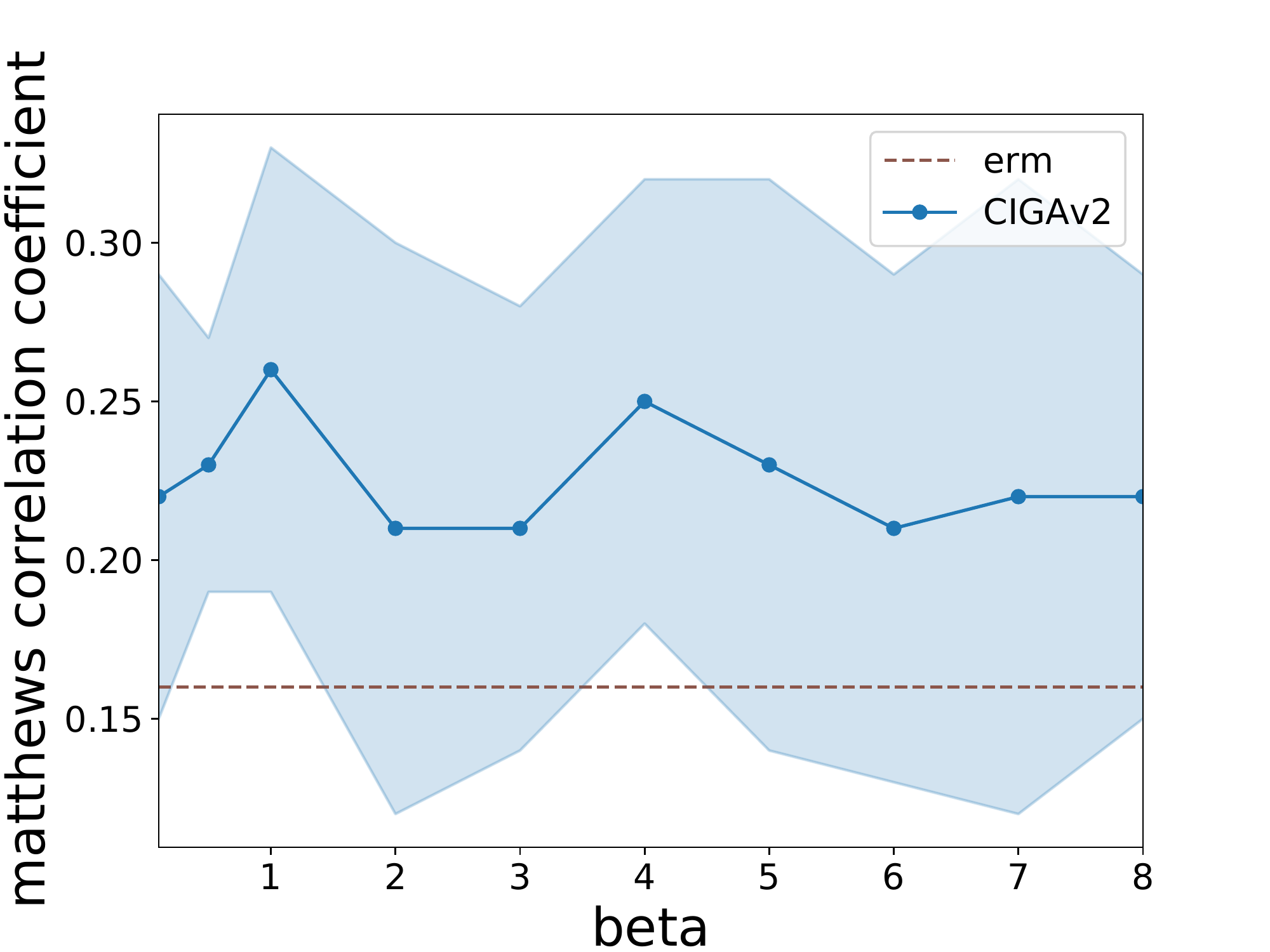}
	}
    \vspace{-0.05in}
	\caption{
		Hyperparameter sensitivity analysis on the coefficient of hinge loss ($\beta$).}
	\label{fig:hp_sen_beta}
	\vskip -0.15in
\end{figure}

\section{Conclusions}
We studied the OOD generalization on graphs via graph classification,
and propose a new solution $\ginv$ through the lens of causality.
By modeling potential distribution shifts on graphs with SCMs,
we generalized and instantiated the invariance principle to graphs,
which was shown to have promising theoretical and empirical OOD generalization ability
under a variety of distribution shifts.

\begin{ack}
  We thank the reviewers for their valuable comments. This work was supported by GRF 14208318 from the RGC of HKSAR and CUHK direct grant 4055146.
  TL was partially supported by Australian Research Council Projects DP180103424, DE-190101473, IC-190100031, DP-220102121, and FT-220100318.
  YZ and BH were supported by the RGC Early Career Scheme No. 22200720, NSFC Young Scientists Fund No. 62006202, Guangdong Basic and Applied Basic Research Foundation No. 2022A1515011652, and Tencent AI Lab Rhino-Bird Gift Fund.
\end{ack}

\bibliography{reference}
\bibliographystyle{abbrvnat}
\clearpage
\section*{Checklist}
\begin{enumerate}
  \item For all authors...
        \begin{enumerate}
          \item Do the main claims made in the abstract and introduction accurately reflect the paper's contributions and scope?
                \answerYes{}
          \item Did you describe the limitations of your work?
                \answerYes{See Sec.~\ref{sec:good_limit_future_appdx} in the appendix.}
          \item Did you discuss any potential negative societal impacts of your work?
                \answerYes{See Sec.~\ref{sec:broader_impacts_appdx} in the appendix.}
          \item Have you read the ethics review guidelines and ensured that your paper conforms to them?
                \answerYes{}
        \end{enumerate}

  \item If you are including theoretical results...
        \begin{enumerate}
          \item Did you state the full set of assumptions of all theoretical results?
                \answerYes{See Sec.~\ref{sec:full_scm_appdx} and Sec.~\ref{sec:theory_appdx} in the appendix.}
          \item Did you include complete proofs of all theoretical results?
                \answerYes{See Sec.~\ref{sec:theory_appdx} in the appendix.}
        \end{enumerate}

  \item If you ran experiments...
        \begin{enumerate}
          \item Did you include the code, data, and instructions needed to reproduce the main experimental results (either in the supplemental material or as a URL)?
                \answerYes{The code and the other required materials are provided in \url{https://github.com/LFhase/CIGA}.}
          \item Did you specify all the training details (e.g., data splits, hyperparameters, how they were chosen)?
                \answerYes{See Sec.~\ref{sec:exp_appdx}.}
          \item Did you report error bars (e.g., with respect to the random seed after running experiments multiple times)?
                \answerYes{See Sec.~\ref{sec:exp_appdx}.}
          \item Did you include the total amount of compute and the type of resources used (e.g., type of GPUs, internal cluster, or cloud provider)?
                \answerYes{See Sec.~\ref{sec:exp_software_appdx}.}
        \end{enumerate}

  \item If you are using existing assets (e.g., code, data, models) or curating/releasing new assets...
        \begin{enumerate}
          \item If your work uses existing assets, did you cite the creators?
                \answerYes{}
          \item Did you mention the license of the assets?
                \answerYes{}
          \item Did you include any new assets either in the supplemental material or as a URL?
                \answerNA{}
          \item Did you discuss whether and how consent was obtained from people whose data you're using/curating?
                \answerNA{The data used are all publicly available datasets.}
          \item Did you discuss whether the data you are using/curating contains personally identifiable information or offensive content?
                \answerNA{}
        \end{enumerate}

  \item If you used crowdsourcing or conducted research with human subjects...
        \begin{enumerate}
          \item Did you include the full text of instructions given to participants and screenshots, if applicable?
                \answerNA{We didn't conduct research with human subjects.}
          \item Did you describe any potential participant risks, with links to Institutional Review Board (IRB) approvals, if applicable?
                \answerNA{We didn't conduct research with human subjects.}
          \item Did you include the estimated hourly wage paid to participants and the total amount spent on participant compensation?
                \answerNA{We didn't conduct research with human subjects.}
        \end{enumerate}
\end{enumerate}

\newpage

\newpage
\appendix

\begin{center}
	\LARGE \bf {Appendix of $\ginv$}
\end{center}

\etocdepthtag.toc{mtappendix}
\etocsettagdepth{mtchapter}{none}
\etocsettagdepth{mtappendix}{subsection}
\tableofcontents
\newpage
\section{Broader Impacts}
\label{sec:broader_impacts_appdx}
Considering the wide applications and high sensitivity of GNNs to distribution shifts and spurious correlations,
it is important to develop GNNs that are able to generalize to OOD data,
especially for realistic scenarios such as AI-aided Drug Discovery where OOD data are ubiquitous.
By formulating OOD generalization problem on graphs using causality,
our work can serve as an initiate step towards tackling OOD generalization problem on graphs,
with the hope to empower GNNs for broader applications and social benefits.
Besides, this paper does not raise any ethical concerns.
This study does not involve
any human subjects, practices to data set releases, potentially harmful insights, methodologies and
applications, potential conflicts of interest and sponsorship, discrimination/bias/fairness concerns,
privacy and security issues, legal compliance, and research integrity issues.

\section{More Discussions on Related Work and Future Directions}
\label{sec:discuss_future_appdx}

\subsection{More backgrounds}
\label{sec:background_appdx}
We give more background introduction about GNNs and Invariant Learning in this section.

\textbf{Graph Neural Networks.} Let $G=(A,X)$ denote a graph with $n$ nodes and $m$ edges,
where $A \in \{0,1\}^{n\times n}$ is the adjacency matrix, and $X\in \R^{n \times d}$ is the node feature matrix
with a node feature dimension of $d$.
In graph classification, we are given a set of $N$ graphs $\{G_i\}_{i=1}^N\subseteq \gG$
and their labels $\{Y_i\}_{i=1}^N\subseteq\gY=\R^c$ from $c$ classes.
Then, we train a GNN $\rho \circ h$ with an encoder $h:\gG\rightarrow\R^h$ that learns a meaningful representation $h_G$ for each graph $G$ to help predict their labels $y_G=\rho(h_G)$ with a downstream classifier $\rho:\R^h\rightarrow\gY$.
The representation $h_G$ is typically obtained by performing pooling with a $\text{READOUT}$ function on the learned node representations:
\begin{equation}
	\label{eq:gnn_pooling}
	h_G = \text{READOUT}(\{h^{(K)}_u|u\in V\}),
\end{equation}
where the $\text{READOUT}$ is a permutation invariant function (e.g., $\text{SUM}$, $\text{MEAN}$)~\citep{gin,diff_pooling,relation_pooling,gin,can_gnn_count,wl_goml},
and $h^{(K)}_u$ stands for the node representation of $u\in V$ at $K$-th layer that is obtained by neighbor aggregation:
\begin{equation}
	\label{eq:gnn}
	h^{(K)}_u = \sigma(W_K\cdot a(\{h^{(K-1)}_v\}| v\in\mathcal{N}(u)\cup\{u\})),
\end{equation}
where $\mathcal{N}(u)$ is the set of neighbors of node $u$,
$\sigma(\cdot)$ is an activation function, e.g., $\text{ReLU}$, and $a(\cdot)$ is an aggregation function over neighbors, e.g., $\text{MEAN}$.

\textbf{Invariant Learning.}
Invariant learning typically considers
a supervised learning setting based on the data $\dataset=\{\dataset^e\}_e$ %
collected from multiple environments $\envall$,
where $\dataset^e=\{G^e_i,y^e_i\}$ is the dataset from environment $e\in\envall$.
$(G^e_i,y^e_i)$ from a single
environment $e$ are considered as drawn independently from an identical distribution $\sP^e$.
The goal of OOD generalization is to train a GNN $\rho \circ h:\gG\rightarrow\gY$
with data from training environments $\train=\{\dataset^e\}_{e\in\envtrain\subseteq\envall}$, and
generalize well to all (unseen) environments, i.e., to minimize:
\begin{equation}
	\label{eq:ood}
	\min_{\rho,h}\max_{e\in\envall}R^e(\rho \circ h),
\end{equation}
where $R^e$ is the empirical risk under environment $e$~\citep{erm,inv_principle,irmv1}.
More details can be referred in~\citep{ib-irm}.

\subsection{Detailed related work}
\label{sec:related_work_appdx}

\textbf{GNN Explainability.}
Works in GNN explainability  aim to find a subgraph of the input graph as the explanation
for the prediction of a GNN model~\citep{gnn_explainer,xgnn_tax}.
Although some may leverage causality in explanation generation~\citep{gen_xgnn},
they mostly focus on understanding the predictions of GNNs in a post-hoc manner instead of OOD generalization.
Recently there are two works aiming to provide robust explanations under distribution shifts, i.e., GIB~\citep{gib} and DIR~\citep{dir},
and both of them focus on tackling FIIF spurious correlations (Assumption~\ref{assump:scm_fiif_appdx}).
The theoretical guarantees of GIB follows the theory of information bottleneck~\citep{ib},
while GIB can not solve PIIF spurious correlations (Assumption~\ref{assump:scm_piif_appdx}).
As both FIIF and PIIF widely exist in realistic scenarios, failing to solve either of them could result in severe performance degradation in practice~\citep{irmv1,ib-irm,aubin2021linear,failure_modes}.
While for DIR, though as a generalization of~\citet{inv_rat} to graphs, can not provide any theoretical guarantees under FIIF spurious correlations as shown in Appendix~\ref{sec:discussion_ood_obj_appdx},
nor under PIIF spurious correlations.

\textbf{GNN Extrapolation.}
Recently there is a surge of attention in improving the extrapolation ability of GNNs and apply them to various applications,
such as mathematical reasoning~\citep{math_reasoning1,math_reasoning2},
physics~\citep{physics_reasoning1,physics_reasoning2},
and graph algorithms~\citep{size_inv_extra,neural_exe,what_nn_reason,transfer_alg}.
\citet{nn_extrapo} study the neural network extrapolation ability from a geometrical perspective.
\citet{reliable_drug} improve OOD drug discovery by mitigating
the overconfident misprediction issue.
\citet{understand_att,size_gen1} focus on the extrapolation of GNNs in terms of
graph sizes, while making additional assumptions on the knowledge about
ground truth attentions and access to test inputs.
\citet{size_gen2} study the graph size extrapolation problem of GNNs through a causal lens,
while the induced  invariance principle is built upon assumptions on the specific family of graphs.
Different from these works, we consider the GNN extrapolation as a causal
problem, establish generic SCMs that are compatible with several graph generation models,
as well as, more importantly, different types of distribution shifts.
Hence, the induced the invariance principle and provable algorithms built upon the SCMs in our work can generalize to multiple graph families and distribution shifts.

Additionally, \citet{handle_node} propose causal models as well as specialized objectives to extrapolate
nodes with different neighbors. However, their formulation is limited to node classification task and specific spurious correlation type.
In contrast, the induced invariance principle in~\citet{handle_node},
can be seen as a extension of $\ginv$ for node classification,
where we cab identify an invariant subgraph
from the $K$-hop neighbor graph of each node, and
making predictions based on it, i.e., $Y\ind E|G_c^\ego\subseteq G_u^\ego$ for node $u$.
We leave specific formulation and implementation to future works.

\textbf{Causality and OOD Generalization.}
Causality comes to the stage for demystifying and improving the huge success
of machine learning algorithms to further advances~\citep{seven_tools_causality,causality4ml,towards_causality}.
One of the most widely applied concept from causality is the Independent Causal Mechanism (ICM)
that assumes conditional distribution of each variable given its causes (i.e., its mechanism)
does not inform or influence the other conditional distributions~\citep{causality,elements_ci}.
The invariance principle is also induced from the ICM assumption.
Once proper assumptions about the underlying data generation process
via Structural Causal Models (SCM) are established,
it is promising to apply the invariance principle to machine learning models for finding
an invariant representation about the causal relationship between the underlying causes and the label~\citep{inv_principle,irmv1}.
Consequently, models built upon the invariant representation can generalize to unseen environments or
domains with guaranteed performance~\citep{inv_principle,causal_transfer,irmv1,groupdro,meta-transfer,ood_max_inv,domainbed,v-rex,env_inference,ib-irm}.
The arguably first formulation of invariance principle was introduced by~\citet{inv_principle}.
\citet{irmv1} propose a novel formulation of learning causal invariance in representation learning, i.e., IRM,
show how it connects with existing areas such as  distributional robust optimization~\citep{dro} and generalization~\citep{memorization},
and prove its effectiveness in addressing PIIF spurious correlations (Assumption~\ref{assump:scm_piif_appdx}).
However, in practice, both PIIF and FIIF (Assumption~\ref{assump:scm_fiif_appdx}) can appear in data, while IRM can fail in these cases~\citep{aubin2021linear,failure_modes}.
\citet{ib-irm} then propose to add information bottleneck criteria into the IRM formulation to address the issue.
However, their results are restricted to linear regime and also require environment partitions to distinguish the sources of distribution shifts.
Recently, \citet{env_inference} and \citet{zin} propose new OOD objectives to relieve the needs for environment partitions, but limited to PIIF spurious types and linear regime.
Besides, \citet{BayesianIRM} identify the overfitting problem as a key challenge when applying IRM on large neural networks. \citet{SparseIRM} propose to alleviate this problem by imposing sparsity constrain.

In parallel invariant learning approaches,
\citet{groupdro} propose to regularize the worst group in group distributionally robust optimization (GroupDro).
\citet{cnc} propose a contrastive approach to tackle GroupDro when the group partitions are not available.
However, minimizing the gap between worst group risk and averaged risk can not yield a OOD generalizable predictors in our circumstances.
Besides, traditional approaches to tackle OOD generalization also include Domain Adaption, Transfer Learning and Domain Generalization\citep{causal_transfer,domain_gen_inv,DANN,CORAL,deep_DG,DouCKG19,causal_matching,DG_survey},
which aim to learn the class conditional invariant representation shared across source domain and target domain.
However, they all require a stronger assumption on the availability of target domain data or the ground truth predictors~\citep{domainbed,ib-irm},
hence are not able to yield predictors with OOD generalization guarantees.
We refer interested readers to~\citet{seven_tools_causality,causality4ml,towards_causality} for an in-depth understanding,
and \citet{domainbed,ib-irm} for a thorough overview.

\subsection{More discussions on connections of $\ginv$ with existing work}
\label{sec:good_connection_appdx}

Although primarily serving for graph OOD generalization
problem, our theory complements the identifiability study on graphs
through contrastive learning, and aligns with the discoveries
in the image domain that contrastive learning learns to isolate
the content ($C$) and style ($S$)~\citep{contrast_inverts,ssl_isolate}.
Moreover, our results also partially explain the success of
graph contrastive learning~\citep{graphcl,contrast_reg,graphcl_auto},
where GNNs may implicitly learn to identify the underlying invariant subgraphs for prediction.

\textbf{On expressivity of graph encoder in $\ginv$.}
The expressivity of $\ginv$ is essentially constrained by the
encoders embedded for learning graph representations.
During isolating $G_c$ from $G$,
if the encoder can not differentiate two isomorphic graphs
$G_c$ and $G_c\cup G_s^p$ where $G_s^p\subseteq G_s$,
then the featurizer will fail to identify the underlying invariant subgraph.
Moreover, the classifier will also fail if the encoder
can not differentiate two non-isomorphic $G_c$s from different classes.
Thus, adopting more powerful graph representation encoders into $\ginv$
can improve the OOD generalization.

\textbf{On $\ginv$ and graph information bottleneck.}
Under the FIIF assumption on latent interaction,
the independence condition derived from causal model can
also be rewritten as $Y\ind S|C$ (similar to that in DIR~\citep{dir} as they also focus on FIIF),
which further implies $Y\ind S|\widehat{G}_c$.
Hence it is natural to use Information Bottleneck (IB) objective~\citep{ib}
to solve for $G_c$:
\begin{equation}
	\label{eq:good_ib}
	\begin{aligned}
		\min_{f_c,g} & \ R_{G_c}(f_c(\widehat{G}_c)),                                                        \\
		\text{s.t.}  & \ G_c=\argmax_{\widehat{G}_c=g(G)\subseteq G}I(\widehat{G}_c,Y)-I(\widehat{G}_c,\gG), \\
	\end{aligned}
\end{equation}
which explains the success of many existing
works in finding predictive subgraph through IB~\citep{gib}.
However, the estimation of $I(\widehat{G}_c,G)$
is notoriously difficult due to the complexity of graph,
which can lead to unstable convergence as observed in our experiments.
In contrast, optimization with contrastive objective in $\ginv$ as Eq.~\ref{eq:good_opt_contrast}
induces more stable convergence.

\textbf{On $\ginv$ for node classifications.}
As the task of node classification can be viewed as
graph classification based on the ego-graphs of a node,
our analysis and discoveries can generalize to
node classification.
More specifically, the invariance principle for node classification
can be implemented by identifying an invariant subgraph
from the $K$-hop neighbor graph of each node, and
making predictions based on it, i.e., $Y\ind E|G_c^\ego\subseteq G_u^\ego$ for node $u$~\citep{handle_node}.

\subsection{Discussions on limitations of $\ginv$ and future directions}
\label{sec:good_limit_future_appdx}
\textbf{Better graph generation modeling.} Compared to~\citet{size_gen2}, we do not specify a specific graph family
in the SCM for graph generation process. Since our focus is to describe the potential distribution shifts with SCMs,
in Assumption~\ref{assump:graph_gen}, we aim to build a SCM that is compatible to many graph generation processes~\citep{sbm,graphon,graphrnn,graphdf}.
However, it is often the case that practitioners have certain inductive knowledge about the graph generation process,
which may imply useful leads and invariance in modeling the generation process~\citep{graph_gen1,graph_gen2,graph_gen3}.
\revision{In Appendix~\ref{sec:case_scm_appdx}, we provide an example about incorporating the graphon~\citep{graphon} knowledge into the SCMs, which derives similar solutions as in the literature~\citep{size_gen1,size_gen2}.
	Therefore, we believe it is promising to leverage more additional knowledge for more precise graph generation modeling and better OOD generalization on graphs.}

\textbf{Better contrastive sampling.} Typical contrastive or graph contrastive learning approaches
leverage augmentation techniques as well as sophisticated sampling strategies during the positive or negative pairs selection~\citep{contrast_loss1,contrast_loss2,infoNCE,graphcl,graphcl_auto}.
A better augmentation or sampling strategy can benefit the OOD generalization in general as shown by \citet{ssl_isolate} and \citet{cnc}.
Since our implementation of $\ginv$ in this work aims to verify the theoretical findings,
we do not apply sophisticated augmentation or sampling during the sampling while simply using the supervised contrastive approach~\citep{sup_contrastive}.
Nevertheless, it is promising to leverage better augmentation and contrastive strategy to improve the generalization ability in $\ginv$~\citep{dps}.

\textbf{More sophisticated architectures/parameter tunning.} The $\ginv$ framework introduced in Sec.~\ref{sec:good_framework} can have multiple implementations.
We choose interpretable architectures in our experiments for the purpose of concept verification.
\revision{Essentially, different architectures can have different advantages and limitations. For the interpretable GNNs used in our experiments, it can provide interpretability for the results (as shown in Appendix~\ref{sec:interpret_visualize_appdx}), but still requires more training time (as shown in Appendix~\ref{sec:additional_exp_appdx}). Therefore, it may not be applicable to some resource-limited scenarios such as Edge-AI. Besides, the approximation may also be limited to the chosen architectures.}
More sophisticated architectures can be incorporated, such as identifying and disentangling $G_c$ at the latent space~\citep{causality4ml,towards_causality}.
\revision{Moreover, as shown in Appendix~\ref{sec:additional_exp_appdx}, $\ginv$ still requires certain additional tunning efforts for the objectives. Hence we believe it is also a promising future direction to reduce the parameter tunning by leveraging better optimization techiniques~\citep{mgda,pair}}

\section{Full Structural Causal Models on Graph Generation}
\label{sec:full_scm_appdx}
Due to the space constraints in the main paper, we make some simplifications when giving the SCMs on the graph generation process.
Hence in this section,
supplementary to the graph generation process in Sec.~\ref{sec:data_gen}, we provide full
SCMs on the graph generation process in this section as shown in Fig.~\ref{fig:scm_appdx}.
Formal descriptions are given as Assumptions~\ref{assump:graph_gen_appdx},~\ref{assump:scm_fiif_appdx},~\ref{assump:scm_piif_appdx},~\ref{assump:scm_miif_appdx}.

To begin with, we take a latent-variable model perspective on the graph generation process and assume
that the graph is generated through a mapping $f_\gen:\gZ\rightarrow \gG$,
where $\gZ\subseteq\R^n$ is the latent space and $\gG=\cup_{N=1}^\infty\{0,1\}^N\times \R^{N\times d}$ is the graph space.
Let $E$ denote environments.
Following previous works~\citep{ssl_isolate,ib-irm},
we partition the latent variable from $\gZ$ into an invariant part $C\in\gC=\R^{n_c}$
and a varying part $S\in\gS=\R^{n_s}$, s.t., $n=n_c+n_s$,
according to whether they are affected by $E$.
Similarly in images, $C$ and $S$ can represent content and style
while $E$ can refer to the locations where the images are taken~\citep{camel_example,adv_causal_lens,ssl_isolate}.
While in graphs, $C$ can be the latent variable that controls the generation of functional groups in
a molecule, which can not be affected by the changes of environments, such as species (or scaffolds), experimental environment for
examining the chemical property (or assays)~\citep{drugood}. On the contrary, the other latent variable $S$
inherits environment-specific information thus can further affect the finally generated graphs.
Besides, $C$ and $S$ can have multiple types of interactions at the latent space with environments $E$ and labels $Y$,
which will generate different types of spurious correlations~\citep{ib-irm}.

\begin{assumption}[\revision{Graph generation SCM}]
	\label{assump:graph_gen_appdx}
	\[\revision{
			\begin{aligned}
				 & (Z^c_A,Z^c_X):=f_\gen^{(A,X)^c}(C),\ G_c:=f_\gen^{G_c}(Z^c_A,Z^c_X), \\
				 & (Z^s_A,Z^s_X):=f_\gen^{(A,X)^s}(S),\ G_s:=f_\gen^{G_s}(Z^s_A,Z^s_X), \\
				 & G:=f_\gen^G(G_c,G_s).
			\end{aligned}}
	\]
\end{assumption}

Specifically, the graph generation process is shown as Fig.~\ref{fig:graph_gen_appdx}.
The generation mapping $f_\gen$ is decomposed into $f_\gen^{(A,X)^c}$,$f_\gen^{G_c}$, $f_\gen^{(A,X)^s}$,$f_\gen^{G_s}$ and $f_\gen^G$ to
control the generation of $(Z^c_A,Z^c_X)$, $G_c$, $(Z^s_A,Z^s_X)$, $G_s$, and $G$, respectively.
Given the variable partitions $C$ and $S$ at the latent space $\gZ$,
they control the generation of the adjacency matrix and features for the invariant subgraph $G_c$ and spurious subgraph $G_s$
through two pairs of latent variables $(Z^c_A,Z^c_X)$ and $(Z^s_A,Z^s_X)$, respectively.
$Z^c_A$ and $Z^s_A$ will control the structure-level properties in the generated graphs, such as degrees, sizes, and subgraph densities.
While $Z^c_X$ and $Z^s_X$ mainly control the attribute-level properties in the generated graphs, such as homophily.
Then, $G_c$ and $G_s$ are entangled into the observed graph $G$ through $f_\gen^{G}$.
It can be a simply $\text{JOIN}$ of
a $G_c$ with one or multiple $G_s$,
or more complex generation processes controlled by the latent variables~\citep{sbm,graphon,graphrnn,graphdf,size_gen2}.
Note that since our focus is to describe the potential distribution shifts with SCMs,
in Assumption~\ref{assump:graph_gen}, we aim to build a SCM that is compatible to many graph generation processes~\citep{sbm,graphon,graphrnn,graphdf}.
\revision{
	In fact, in Appendix~\ref{sec:case_scm_appdx}, we showcase how our SCMs can generalize to specific graph families studied in the literature~\citep{size_gen2,dir,handle_node}, when given more additional knowledge about the graph generation process.
	Nevertheless, we believe integrating specific graph generation processes and their implications to improving OOD generalization on graphs would be a promising future direction, as discussed in Appendix~\ref{sec:good_limit_future_appdx}.}

\revision{
	Due to the correlation between $E$ and $G$, graphs collected from different environments can have different
	structure-level properties} such as degrees, graph sizes, and subgraph densities,
as well as feature-level properties such as homophily~\citep{understand_att,size_gen1,size_gen2,hao}.
Meanwhile, all of them can spuriously correlated with the labels depending on how the underlying latent
variables are interacted with each others. The interaction types can be further
divided into two axiom types FIIF and PIIF, as well as the mixed one MIIF.
Previous OOD methods such as GIB~\citep{gib} and DIR~\citep{dir} mainly focus on FIIF case,
while others such as IRM~\citep{irmv1} mainly focuses on the PIIF case.
Evidences show that failing to model either of them when developing the OOD objectives
can have serious performance degenerations in practice~\citep{aubin2021linear,failure_modes}.
That is why we aim to model both of them in our solution.

\begin{figure*}[t]
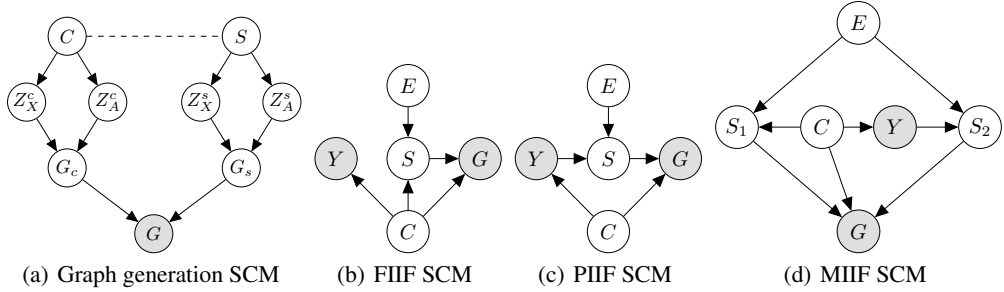

	\centering\hfill
	\subfigure[\revision{Graph generation SCM}]{\label{fig:graph_gen_appdx}
		\resizebox{!}{0.225\textwidth}{\tikz{
				\node[latent] (S) {$S$};%
				\node[latent,left=of S,xshift=-1.5cm] (C) {$C$};%
				\node[latent,below=of C,xshift=-0.75cm,yshift=0.5cm] (ZCA) {$Z_X^c$}; %
				\node[latent,below=of C,xshift=0.75cm,yshift=0.5cm] (ZCX) {$Z_A^c$}; %
				\node[latent,below=of S,xshift=-0.75cm,yshift=0.5cm] (ZSA) {$Z_X^s$}; %
				\node[latent,below=of S,xshift=0.75cm,yshift=0.5cm] (ZSX) {$Z_A^s$}; %
				\node[latent,below=of ZCX,xshift=-0.75cm,yshift=0.5cm] (GC) {$G_c$}; %
				\node[latent,below=of ZSX,xshift=-0.75cm,yshift=0.5cm] (GS) {$G_s$}; %
				\node[obs,below=of GC,xshift=1.6cm,yshift=0.5cm] (G) {$G$}; %
				\edge[dashed,-] {C} {S}
				\edge {C} {ZCX,ZCA}
				\edge {S} {ZSX,ZSA}
				\edge {ZCX,ZCA} {GC}
				\edge {ZSX,ZSA} {GS}
				\edge {GC,GS} {G}
			}}}
	\subfigure[FIIF SCM]{\label{fig:scm_fiif_appdx}
		\resizebox{!}{0.18\textwidth}{\tikz{
				\node[latent] (E) {$E$};%
				\node[latent,below=of E,yshift=0.5cm] (S) {$S$}; %
				\node[obs,below=of E,xshift=-1.2cm,yshift=0.5cm] (Y) {$Y$}; %
				\node[obs,below=of E,xshift=1.2cm,yshift=0.5cm] (G) {$G$}; %
				\node[latent,below=of Y,xshift=1.2cm,yshift=0.5cm] (C) {$C$}; %
				\edge {E} {S}
				\edge {C} {Y,G}
				\edge {S} {G}
				\edge {C} {S}
			}}}
	\subfigure[PIIF SCM]{\label{fig:scm_piif_appdx}
		\resizebox{!}{0.18\textwidth}{\tikz{
				\node[latent] (E) {$E$};%
				\node[latent,below=of E,yshift=0.5cm] (S) {$S$}; %
				\node[obs,below=of E,xshift=-1.2cm,yshift=0.5cm] (Y) {$Y$}; %
				\node[obs,below=of E,xshift=1.2cm,yshift=0.5cm] (G) {$G$}; %
				\node[latent,below=of Y,xshift=1.2cm,yshift=0.5cm] (C) {$C$}; %
				\edge {E} {S}
				\edge {C} {Y,G}
				\edge {S} {G}
				\edge {Y} {S}
			}}}
	\subfigure[MIIF SCM]{\label{fig:scm_miif_appdx}
		\resizebox{!}{0.24\textwidth}{\tikz{
				\node[latent] (E) {$E$};%
				\node[latent,below=of E,xshift=-2cm] (S1) {$S_1$}; %
				\node[latent,below=of E,xshift=-0.6cm] (C) {$C$}; %
				\node[latent,below=of E,xshift=2cm] (S2) {$S_2$}; %
				\node[obs,below=of E,xshift=0.6cm] (Y) {$Y$}; %
				\node[obs,below=of C,xshift=0.6cm] (G) {$G$}; %
				\edge {E} {S1,S2}
				\edge {C} {S1,Y,G}
				\edge {Y} {S2}
				\edge {S1,S2} {G}
			}}}
	\caption{Full SCMs on Graph Distribution Shifts.}
	\label{fig:scm_appdx}
\end{figure*}

\begin{assumption}[FIIF SCM]
	\label{assump:scm_fiif_appdx}
	\[\begin{aligned}
			Y:= f_\inv(C),\ S:=f_\spu(C,E),\ G:= f_\gen(C,S).
		\end{aligned}\]
\end{assumption}

\begin{assumption}[PIIF SCM]
	\label{assump:scm_piif_appdx}
	\[\begin{aligned}
			Y:= f_\inv(C),\ S:=f_\spu(Y,E),\ G:= f_\gen(C,S).
		\end{aligned}\]
\end{assumption}

\begin{assumption}[MIIF SCM]
	\label{assump:scm_miif_appdx}
	\[\begin{aligned}
			Y:= f_\inv(C),\ S_1:=f_\spu(C,E),\ S_2:=f_\spu(Y,E),\ G:= f_\gen(C,S_1,S_2).
		\end{aligned}\]
\end{assumption}

As for the interactions between $C$ and $S$ at the latent space,
we categorize the interaction modes into Fully Informative Invariant Features (FIIF, Fig.~\ref{fig:scm_fiif_appdx}), and Partially Informative Invariant Features (PIIF, Fig.~\ref{fig:scm_piif_appdx}),
depending on whether the latent invariant part $C$
is fully informative about label $Y$, i.e., $(S,E)\ind Y|C$.
It is also possible that FIIF and PIIF are entangled into a Mixed Informative Invariant Features (MIIF,Fig.~\ref{fig:scm_miif_appdx}).
We follow ~\citet{irmv1,ib-irm} to formulate the SCMs for FIIF and PIIF, where we omit noises for simplicity~\citep{causality,elements_ci}.
Since MIIF is built upon FIIF and PIIF, we will focus on the axiom interaction modes (FIIF and PIIF) in this paper,
while most of our discussions can be extended to MIIF or more complex interactions built upon FIIF and PIIF.

Among all of the interaction modes,
$f_\gen$ corresponds to the graph generation process in Assumption~\ref{assump:graph_gen_appdx}.
$f_\spu$ is the mechanism describing how $S$ is affected by $C$ and $E$ at the latent space.
In FIIF, $S$ is directly controlled by $C$
while in PIIF,
indirectly controlled by $C$ through $Y$,
which can exhibit different behaviors in practice~\citep{ib-irm,failure_modes}.
Additionally, in MIIF, $S$ is further partitioned into $S_1$ and $S_2$ depending on whether it is directly or indirectly controlled by $C$, respectively.
Moreover, $f_\inv:\gC\rightarrow\gY$ indicates the labeling process,
which assigns labels $Y$ for the corresponding $G$ merely based on $C$.
Consequently, $\gC$ is better clustered than $\gS$ when given $Y$~\citep{cluster_assump,cluster_assump2,causality4ml,towards_causality},
which also serves as the necessary separation assumption for a classification task~\citep{svm1,svm2,lda}.
\begin{assumption}[Latent Separability]
	\label{assump:latent_sep_appdx}$H(C|Y)\leq H(S|Y)$.
\end{assumption}

\revision{
	\subsection{Discussions on specific cases of the SCMs}
	\label{sec:case_scm_appdx}
	Although our primary focus in this work is to characterize general graph distribution shifts that could happen in practice without any additional knowledge about the underlying graph family, and derive the corresponding solutions, our SCMs (Fig.~\ref{fig:scm_appdx}) can generalize to specific cases studied in previous works, when incorporating more inductive biases about the underlying graph family~\citep{size_gen2,dir,handle_node}.
}
\revision{
	Specifically, we illustrate the specialized SCMs in Fig.~\ref{fig:case_scm_appdx} for the SCM studied in~\citep{size_gen2} which assumes the graphs are generated following a graphon model~\citep{graphon}.
}
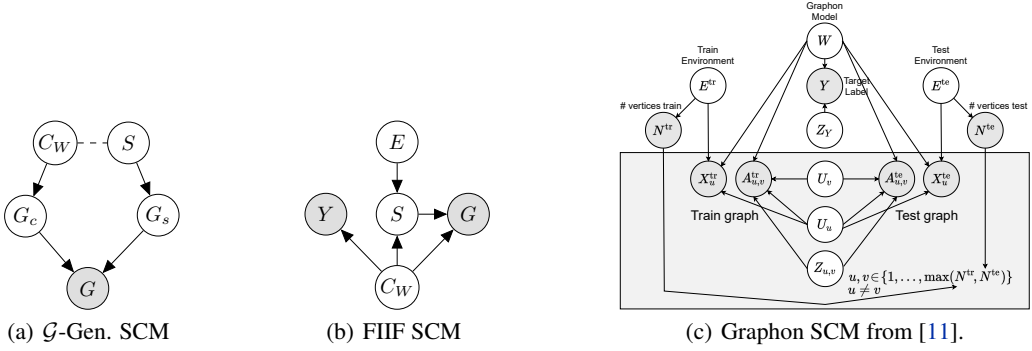
\begin{figure}[ht]
	\centering
	\subfigure[\revision{$\gG$-Gen. SCM}]{\label{fig:graph_gen_case_appdx}
		\resizebox{!}{0.18\textwidth}{\tikz{
				\node[latent] (S) {$S$};%
				\node[latent,left=of S,xshift=0.5cm] (C) {$C_W$};%
				\node[latent,below=of C,xshift=-0.5cm,yshift=0.5cm] (GC) {$G_c$}; %
				\node[latent,below=of S,xshift=0.5cm,yshift=0.5cm] (GS) {$G_s$}; %
				\node[obs,below=of GC,xshift=1.05cm,yshift=0.5cm] (G) {$G$}; %
				\edge[dashed,-] {C} {S}
				\edge {C} {GC}
				\edge {S} {GS}
				\edge {GC,GS} {G}
			}}}
	\hfill
	\subfigure[\revision{FIIF SCM}]{\label{fig:scm_fiif_case_appdx}
		\resizebox{!}{0.18\textwidth}{\tikz{
				\node[latent] (E) {$E$};%
				\node[latent,below=of E,yshift=0.5cm] (S) {$S$}; %
				\node[obs,below=of E,xshift=-1.2cm,yshift=0.5cm] (Y) {$Y$}; %
				\node[obs,below=of E,xshift=1.2cm,yshift=0.5cm] (G) {$G$}; %
				\node[latent,below=of Y,xshift=1.2cm,yshift=0.5cm] (C) {$C_W$}; %
				\edge {E} {S}
				\edge {C} {Y,G}
				\edge {S} {G}
				\edge {C} {S}
			}}}
	\hfill
	\subfigure[\revision{Graphon SCM from~\citep{size_gen2}}.]{\label{fig:scm_graphon_case_appdx}
		\includegraphics[width=0.4\textwidth]{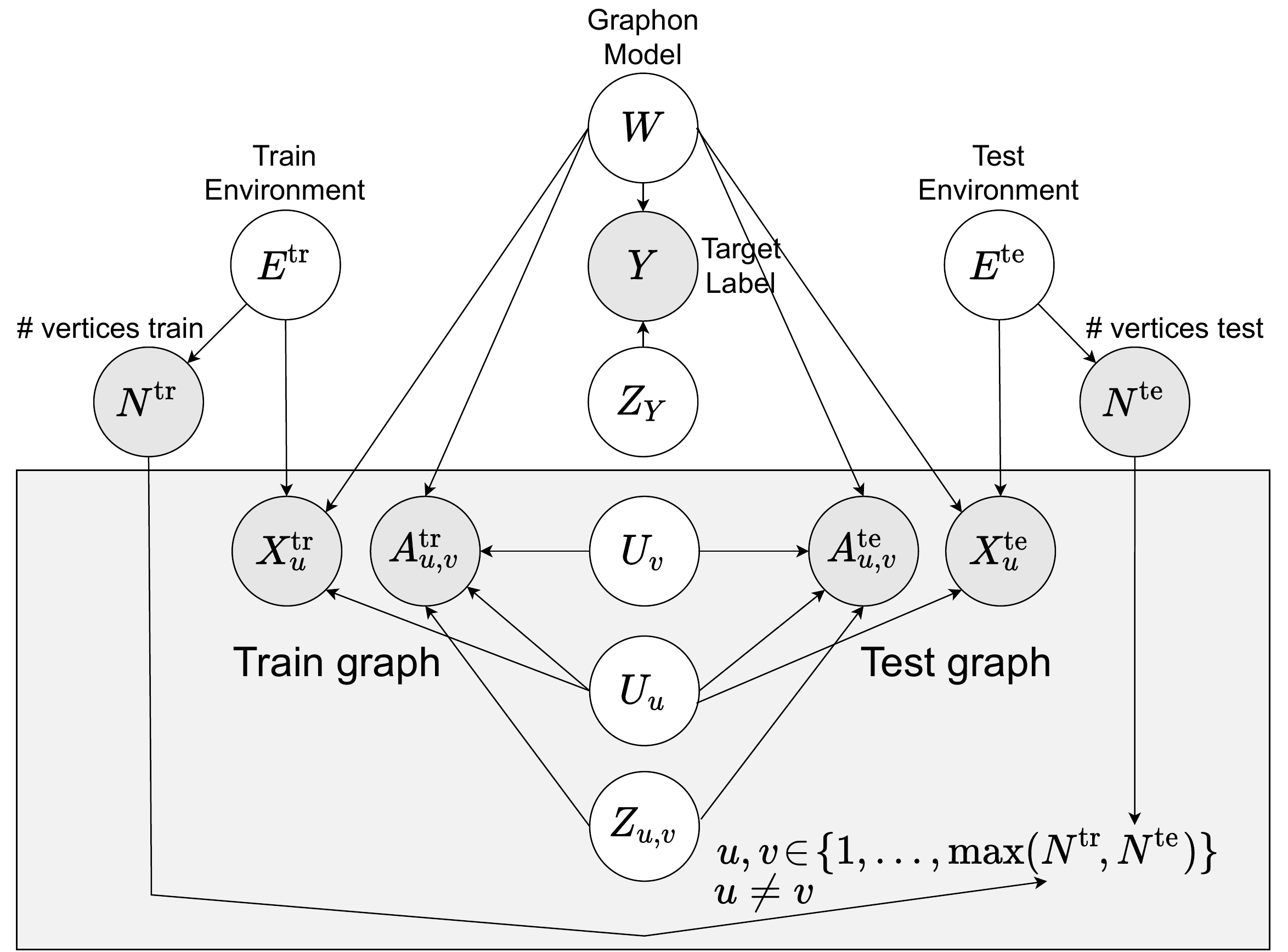}
	}
	\caption{
		\revision{Specialized graph generation SCMs when incorporating additional knowledge.}}
	\label{fig:case_scm_appdx}
\end{figure}

\revision{
	When with the additional knowledge about the underlying graph generative model, the graph generation SCM (Fig.~\ref{fig:graph_gen_appdx}) and the FIIF SCM (Fig.~\ref{fig:scm_fiif_appdx}) together generalizes to the graphon SCM studied in~\citep{size_gen2}. We now give a brief description in the below.
}

\revision{
Specifically, shown as in Fig.~\ref{fig:graph_gen_case_appdx}, $C$ now is instantiated as a graphon model $C_W\sim \mathbb{P}(C_W)$, where $C_W:[0,1]^2\rightarrow[0,1]$ is a random symmetric measurable function sampled from the set of all symmetric measurable functions~\citep{graphon}. Besides, the label $Y$ is determined according to $C_W$. Then, $C_W$ will further control the generation of the adjcency matrix $G_c=A^c$ through graphon generative process:
\[A^c_{u,v}:=\mathbb{I}(Z_{u,v}>C_W(U_u,U_v)),\ \forall u,v\in V,\]
where $Z_{u,v}$ is an independent uniform noises on $[0,1]$ for each possible edge $(u,v)$ in the graph. Bascially, $Z$ and $U$ are inherited from the graphon SCM as Fig.~\ref{fig:scm_graphon_case_appdx}.
}

\revision{On the other hand, as $S$ does not imply any information about $Y$ in this case, it resembles the FIIF SCM (Fig.~\ref{fig:scm_fiif_appdx}). In other words, $(S,E)\ind Y|C$ still holds. Moreover, the node attributes $G_s=X^s$ are generated jointly influenced by the environment $E$ and the graphon $C_W$ through $S$:
	\[X_v:=f_\gen^s(S),\ S:=f_\spu(E,C_W),\]
	which resembles the attribute generation in Fig.~\ref{fig:scm_graphon_case_appdx}.
}

\revision{Then, both $G_c$ and $G_s$ are concatenated together. In a simplistic case intuitively, we can regard $G_c$ only contains the edges in $G$ and $G_s$ only contains the node attributes. Since the graphon model mainly controls the edge connection, the edge connection patterns, e.g., motif appearance frequency or subgraph densities, acts as a informative indicator for the label $Y$. In contrast, the node attributes and its numbers would be affected by the environments. A GNN model is prone to the changes of the environments if it overfits to some spurious patterns about the graph sizes or the attributes. While if the GNN model can leverage the connection patterns to make predictions, it remain invariant to the changes of environments, or the spurious patterns such as graph sizes and node attributes, which resembles the solutions derived in~\citep{size_gen1,size_gen2}. Besides, it also partially explains why $\ginv$ can generalize to OOD graphs studied in these works~\citep{size_gen1,size_gen2}.}

\revision{In addition to the graphon SCM, essentially, the SCM studied in~\citep{dir} resembles the FIIF SCM, and that of~\citep{handle_node} resembles PIIF SCM, which also serves as partial evidence for the superiority OOD generalization performances of $\ginv$.}

\section{More Details about Failure Case Studies in Sec.~\ref{sec:limitation_prev}}
\label{sec:good_fail_setting_appdx}
In this section, we provide details on failure case studies in Sec.~\ref{sec:limitation_prev}.
We first elaborate the empirical evaluation setting where we construct a synthetic graph datasets to probe the
behaviors of existing methods in OOD generalization on graphs.

\subsection{More empirical details about failure case study in Sec.~\ref{sec:limitation_prev}}
\label{sec:more_emp_fail_case_appdx}

To begin with,
we construct 3-class synthetic datasets based on BAMotif~\citep{pge} and follow~\citet{dir} to inject spurious correlations
between motif graph and base graph during the generation.
In this graph classification task, the model needs to tell which motif the graph contains, e.g., ``House'' or ``Cycle'' motif, as shown in Fig.~\ref{fig:good_fail_cases_appdx}.
We inject the distribution shifts in the training data while keeping the test data and validation data without the biases.
For structure-level shifts, we introduce the artificial bias based on FIIF, where the motif and the base graph are spuriously correlated with a probability of various bias.
For mixed shifts, we additionally introduced attribute-level shifts based on FIIF, where all of the node features are spuriously correlated with a probability of various bias.
The number of training graphs is $600$ for each class and the number of graphs in validation and test set is $200$ for each class.
More construction details are given in Appendix~\ref{sec:exp_appdx}.

For the GNN encoders, by default,
we use $3$-layer GCN~\citep{gcn} with mean readout, a hidden dimension of $64$, and JK jump connections~\citep{jknet} at the last layer.
During training, we use a batch size of $32$, learning rate of $1e-3$ with Adam optimizer~\citep{adam}, and
batch normalization between hidden layers~\citep{batch_norm}.
Meanwhile, to stabilize the training,
we also use dropout~\citep{dropout} of $0.1$ and early stop the training when the validation accuracy does not increase till $5$ epoch after first $20$ epochs.
All of the experiments are repeated $5$ times, and the mean accuracy as well as variance are reported and plotted.
When using IRM objective~\citep{irmv1}, as the environment partitions are not available, we generate $2$ environments with random partitions.

\begin{wrapfigure}{r}{0.5\textwidth}
	\vspace{-0.2in}
	\includegraphics[width=0.5\textwidth]{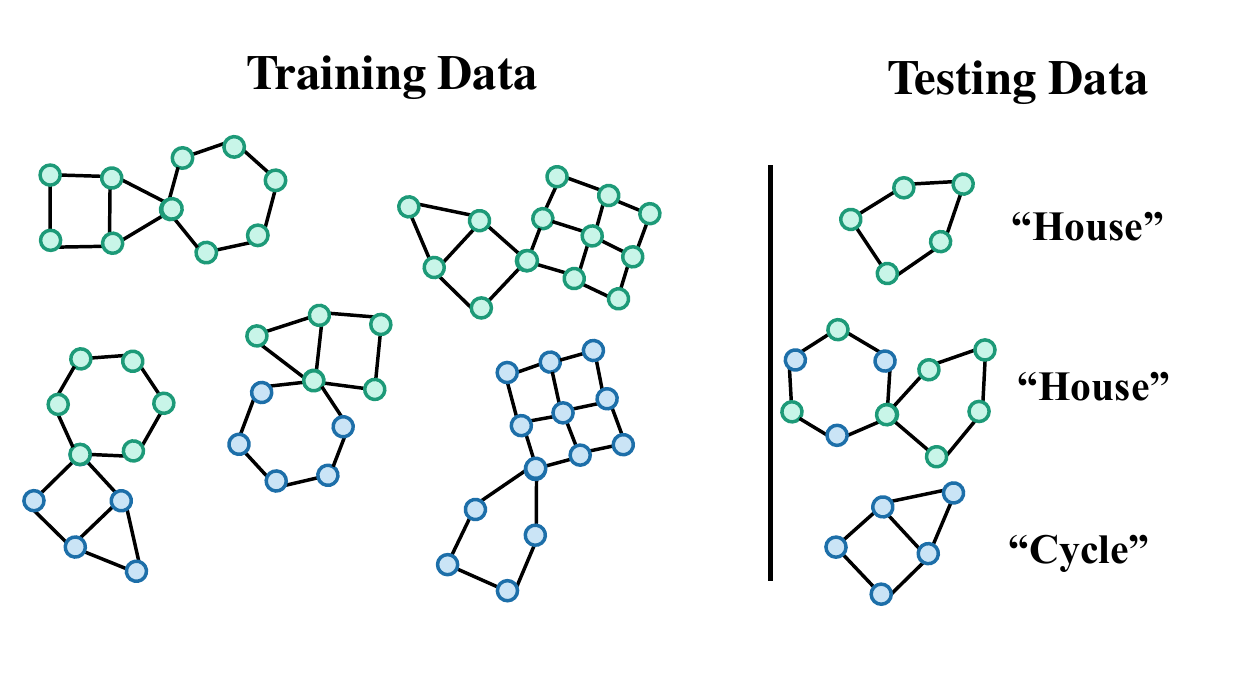}
	\vspace{-0.2in}
	\caption{Failure cases of existing methods. GNNs are required to classify whether the graph contains a ``house'' or ``cycle'',
		where the colors represent node features.
		However, distribution shifts in the training exists at both structure level (From left to right: ``house'' mostly co-occur with a hexagon),
		attribute level (From upper to lower: graphs nodes are mostly green colored if they contain ``house'', or blued colored if they contain ``cycle''),
		and graph sizes, making GNNs hard to capture the invariance.
		\textit{ERM can fail} for leveraging the shortcuts and predict graphs that have a hexagon or have mostly green nodes as ``house''.
		\textit{IRM can fail} when test data is not sufficiently supported by the training data.}
	\label{fig:good_fail_cases_appdx}
	\vspace{-0.2in}
\end{wrapfigure}

\subsection{More discussions about failure case study in Sec.~\ref{sec:limitation_prev}}
In Fig.~\ref{fig:good_fail_wosize_appdx},~\ref{fig:good_fail_size_appdx},~\ref{fig:good_fail_piif_wosize_appdx},~\ref{fig:good_fail_piif_appdx},
we investigate whether existing training objectives (ERM and IRM), adding more message passing, as well as using expressive GNNs, can improve the OOD generalization ability
on graphs. Here we also provide a additional discussion
in complementary to the discussions on OOD generalization performance of ERM and IRM objectives in Sec.~\ref{sec:limitation_prev}.
\begin{myquotation}
	\emph{Can better architectures improve OOD generalization of GNNs?}
\end{myquotation}
\textbf{Adding more message passing turns.} It is a common practice in GNNs to denoise the signals by aggregating more neighbors with higher layers,
or enhance the expressive power with more powerful readout functions~\citep{jknet,gin,p-reg}.
Aggregating neighbor information with more layers to denoise the input signal,
or enhancing the expressivity with more powerful readout functions,
are two common choices in GNNs to improve the generalization ability~\citep{jknet,oversmoothing,gin,p-reg}.
However, in the experiments next, we empirically found that GCNs with more layers and more powerful  readout operations are still sensitive to distribution shifts.
In particular, stacking more layers  helps denoising certain shifts,
while the OOD performance would drop more sharply when the bias increases.
Intuitively, if the spurious features from nodes cannot be eliminated by the denoising property of a deeper GNN,
they would spread among the whole graph more widely, which in turn leads to stronger spurious correlations.
Besides, the spurious correlations would be more difficult to be disentangled
if there are distribution shifts at both structure-level and attribute-level.
Since the node representations from hidden layers can also encode graph topology features~\citep{gin},
distribution shifts introduced through $Z_A^s$ and $Z_X^s$ will doubly mix at the learned features.
In the worst case, the information about $Z_A^c$ and $Z_X^c$ could be partially covered by or even replaced by $Z_A^s$ and $Z_X^s$.
This will make OOD generalization of message passing GNNs trained through ERM much more difficult or even impossible.
Besides, as the node representations of $1\leq i\leq k$-th layer can also encode graph topology features~\citep{gin}, which, if spuriously correlated with labels through $Z^s_A$ and entangled with part of invariant node features, i.e., $Z^c_X$, in the worst case, can greatly improve the difficulty or even make the OOD generalization impossible for neighbor aggregation GNNs trained with ERM.

\textbf{Using more expressive GNNs.} Previous results on the expressivity of GNNs show that GNNs are limited to
distinguish isomorphic graphs at most as 1-WL/2-WL test can distinguish~\cite{gin}. After that, many follow-up variants are proposed
to improve the expressivity of GNNs~\citep{wl_goml}. However, if the labels are spuriously correlated with certain subgraphs,
even the GNN has high expressivity can still be prone to distribution shifts.
In a idealistic case, when classifying a graph with a highly expressive GNN, it reduces to the linear or discrete feature case on the Euclidean regime.
In this case, there exists many evidences showing that neural networks can fail to generalize to OOD data without a proper objective~\citep{camel_example,covid19_application,irmv1,groupdro,meta-transfer,v-rex,env_inference,ood_max_inv,ib-irm}.
Empirically, we use $k$-GNNs~\citep{kgnn} to verify the intuition and observe similar failures for this provably more expressive GNN as basic GNN variants.

\subsection{More empirical results about failure case study in Sec.~\ref{sec:limitation_prev}}

\begin{figure}[ht]
	\centering
	\subfigure[Failures of training objectives.]{
		\includegraphics[width=0.3\textwidth]{figures/ood_failure_wosize.pdf}
	}
	\subfigure[Failures of deeper GNNs.]{
		\includegraphics[width=0.3\textwidth]{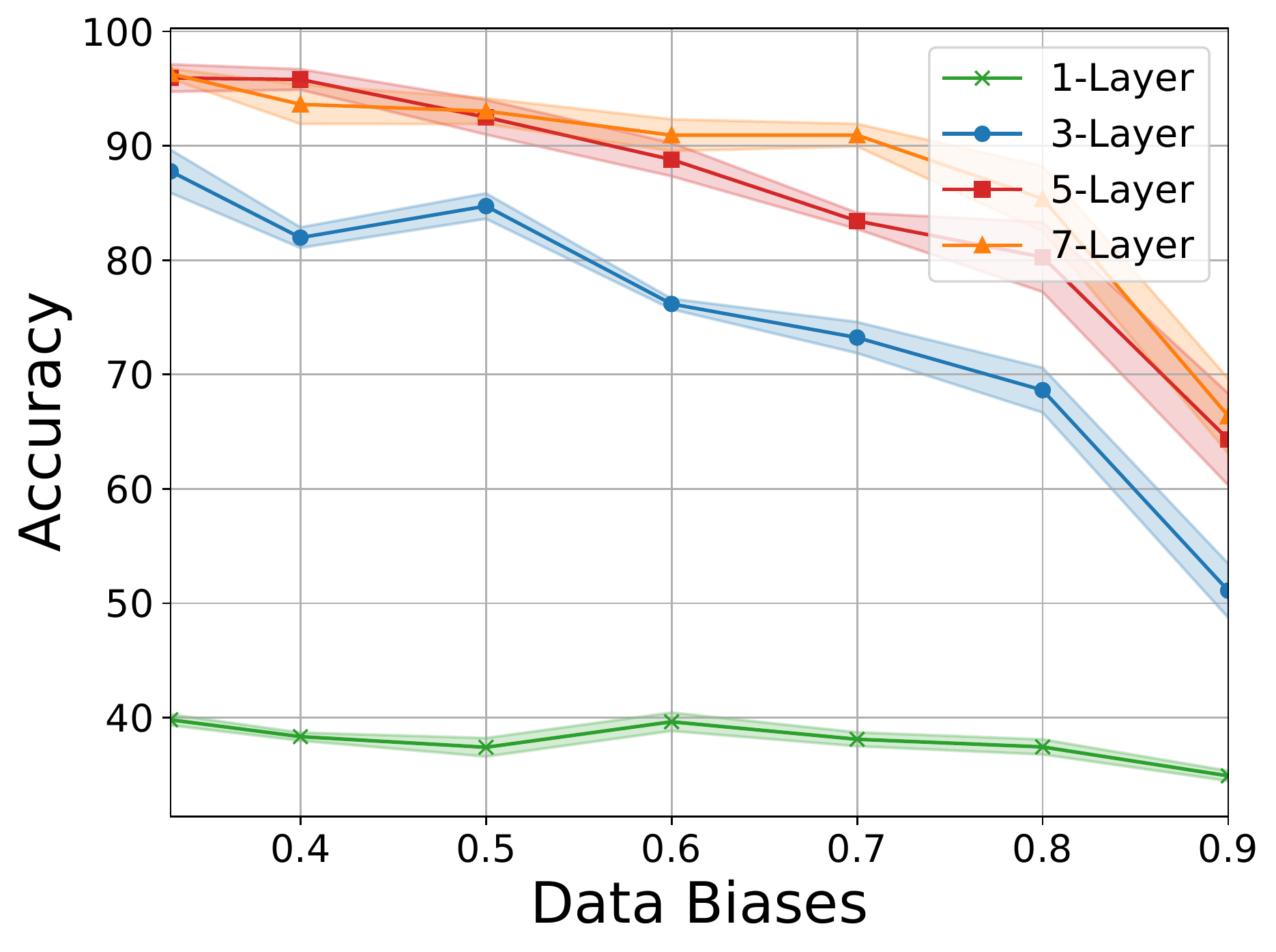}
	}
	\subfigure[Failures of expressive GNNs.]{
		\includegraphics[width=0.3\textwidth]{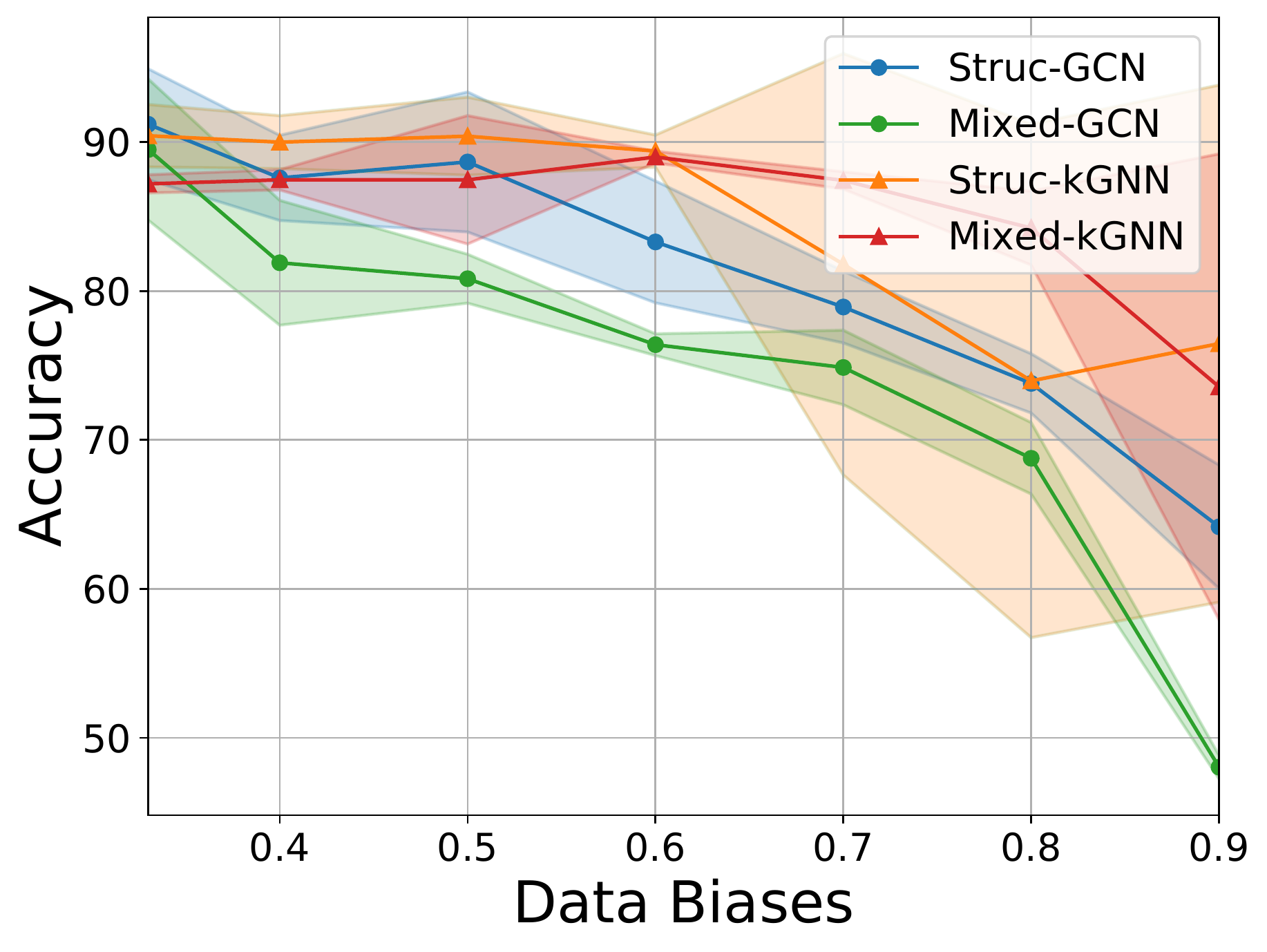}
	}
	\caption{
		Failure of existing methods on SPMotif with FIIF attribute shifts.}
	\label{fig:good_fail_wosize_appdx}
\end{figure}

\begin{figure}[ht]
	\centering
	\subfigure[Failures of training objectives.]{
		\includegraphics[width=0.3\textwidth]{figures/ood_failure_size.pdf}
	}
	\subfigure[Failures of deeper GNNs.]{
		\includegraphics[width=0.3\textwidth]{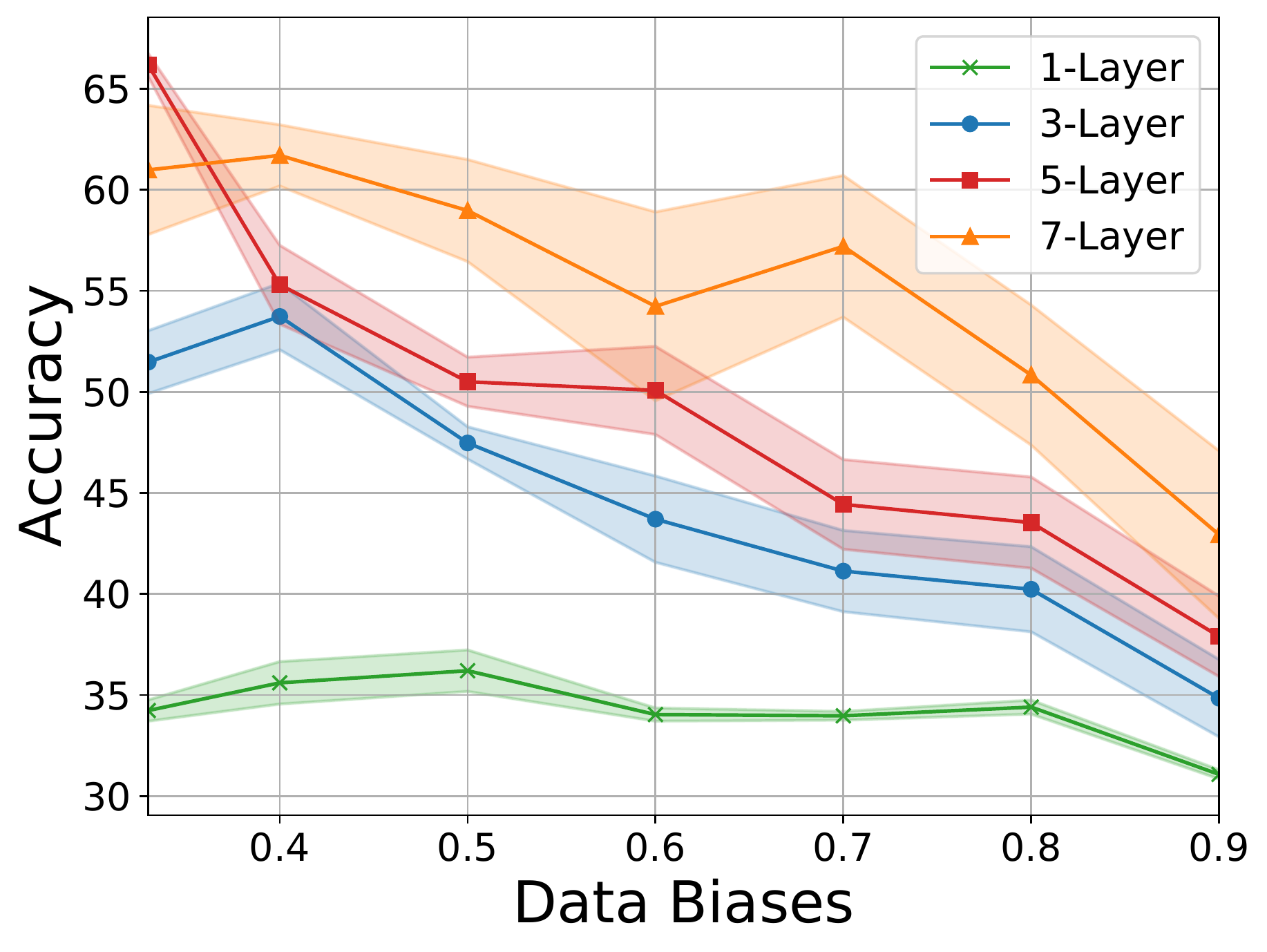}
	}
	\subfigure[Failures of expressive GNNs.]{
		\includegraphics[width=0.3\textwidth]{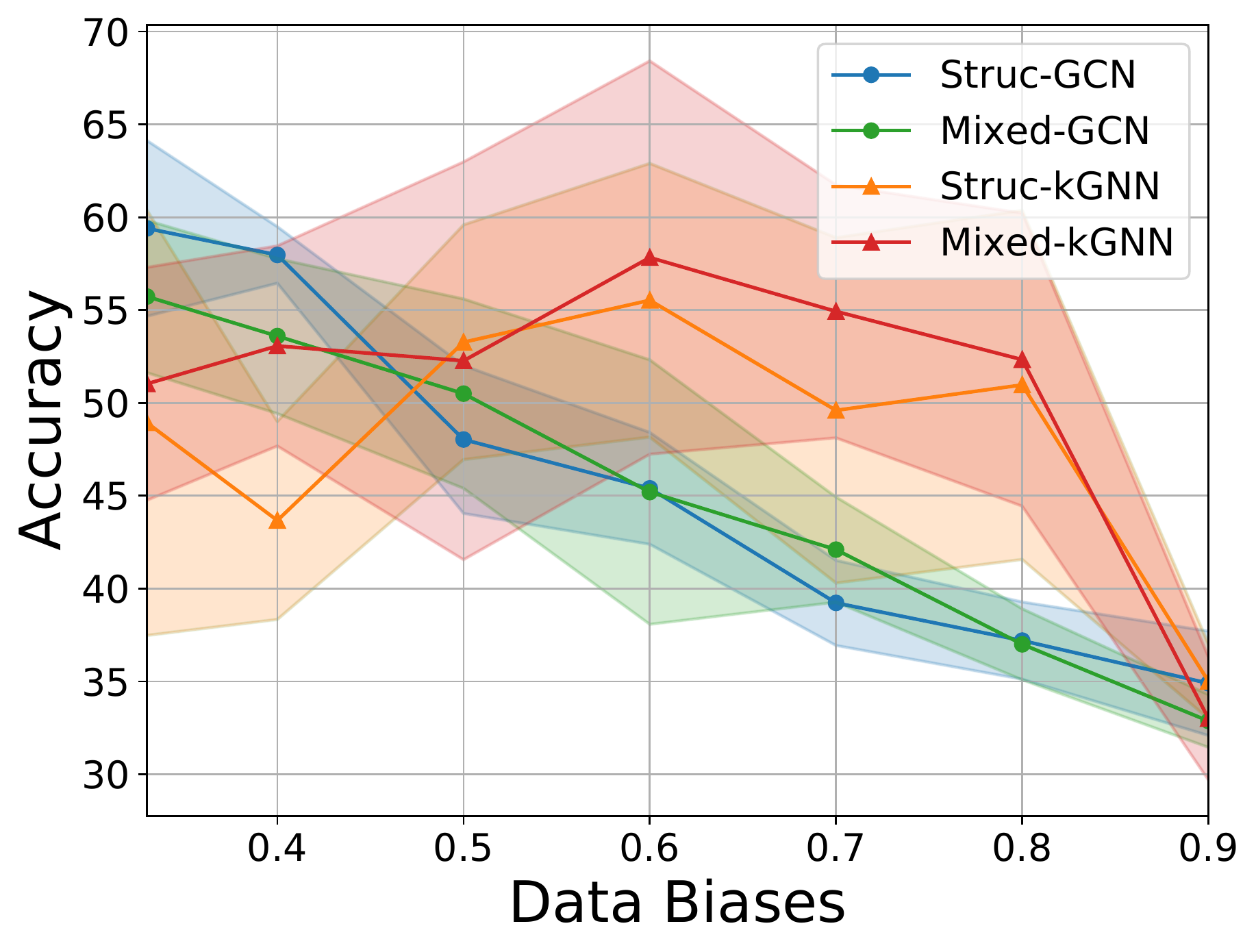}
	}
	\caption{
		Failure of existing methods on SPMotif with FIIF attribute shifts and graph size shifts.}
	\label{fig:good_fail_size_appdx}
\end{figure}

\begin{figure}[H]
	\centering
	\subfigure[Failures of training objectives.]{
		\includegraphics[width=0.3\textwidth]{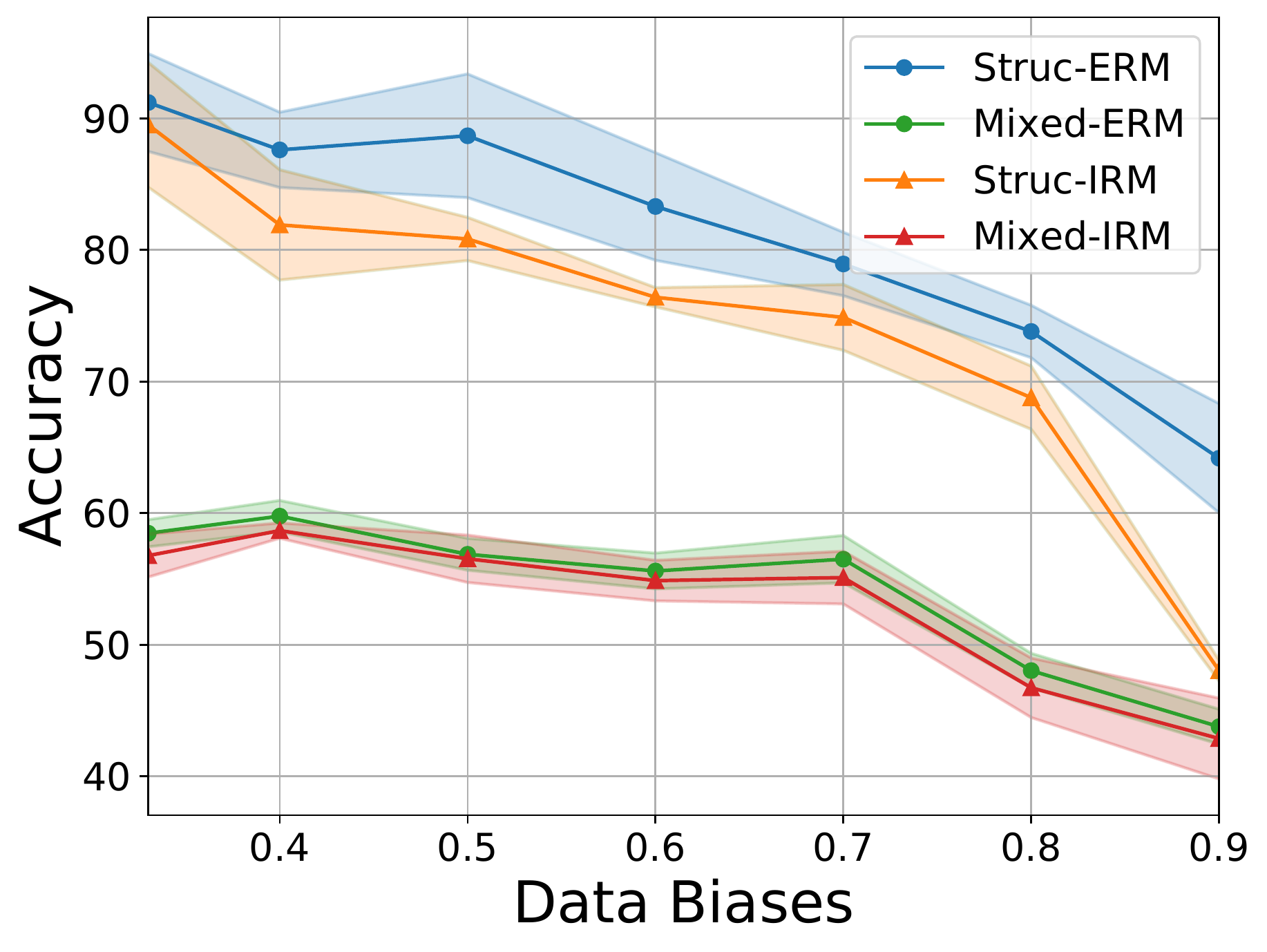}
	}
	\subfigure[Failures of deeper GNNs.]{
		\includegraphics[width=0.3\textwidth]{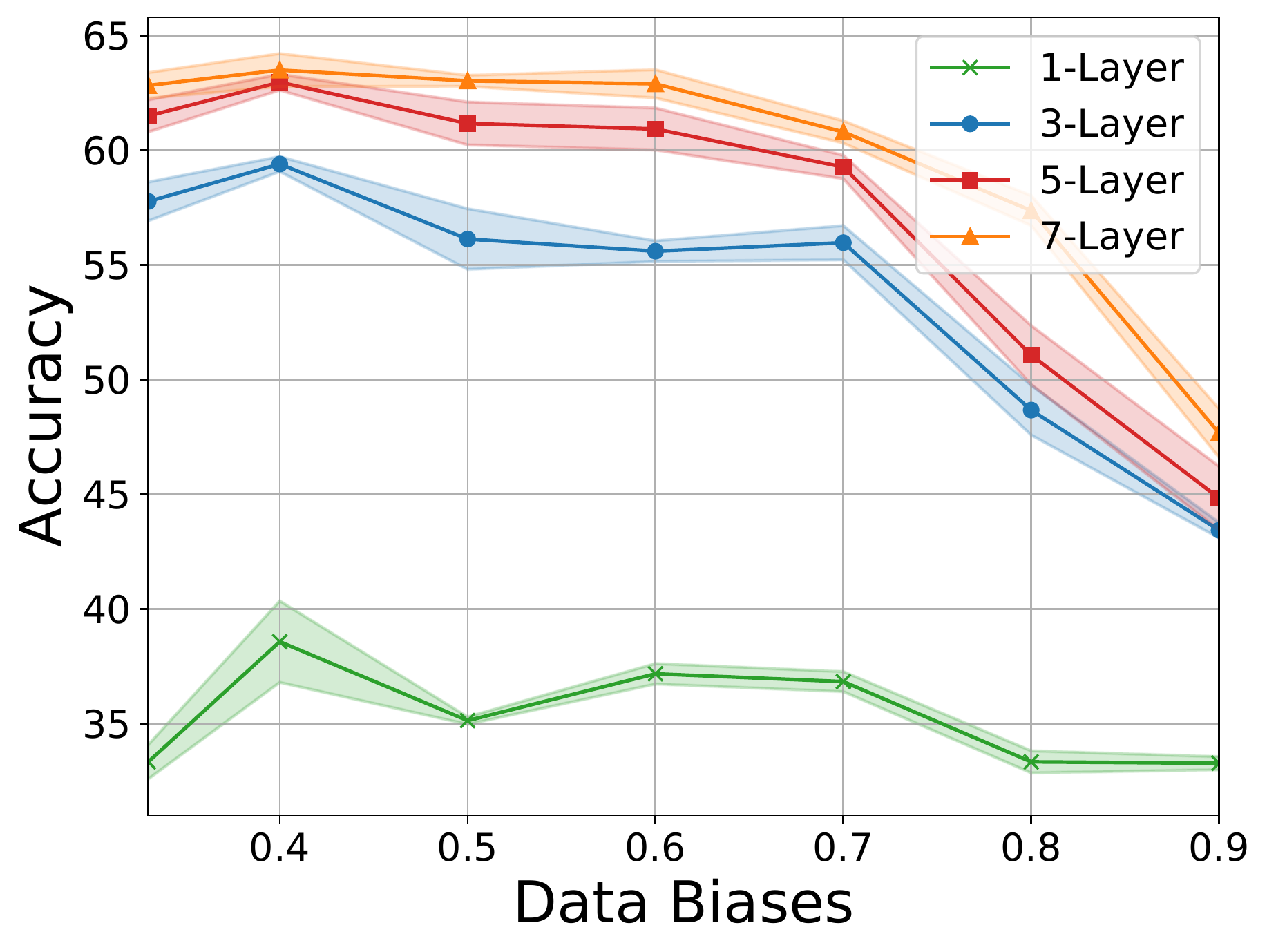}
	}
	\subfigure[Failures of expressive GNNs.]{
		\includegraphics[width=0.3\textwidth]{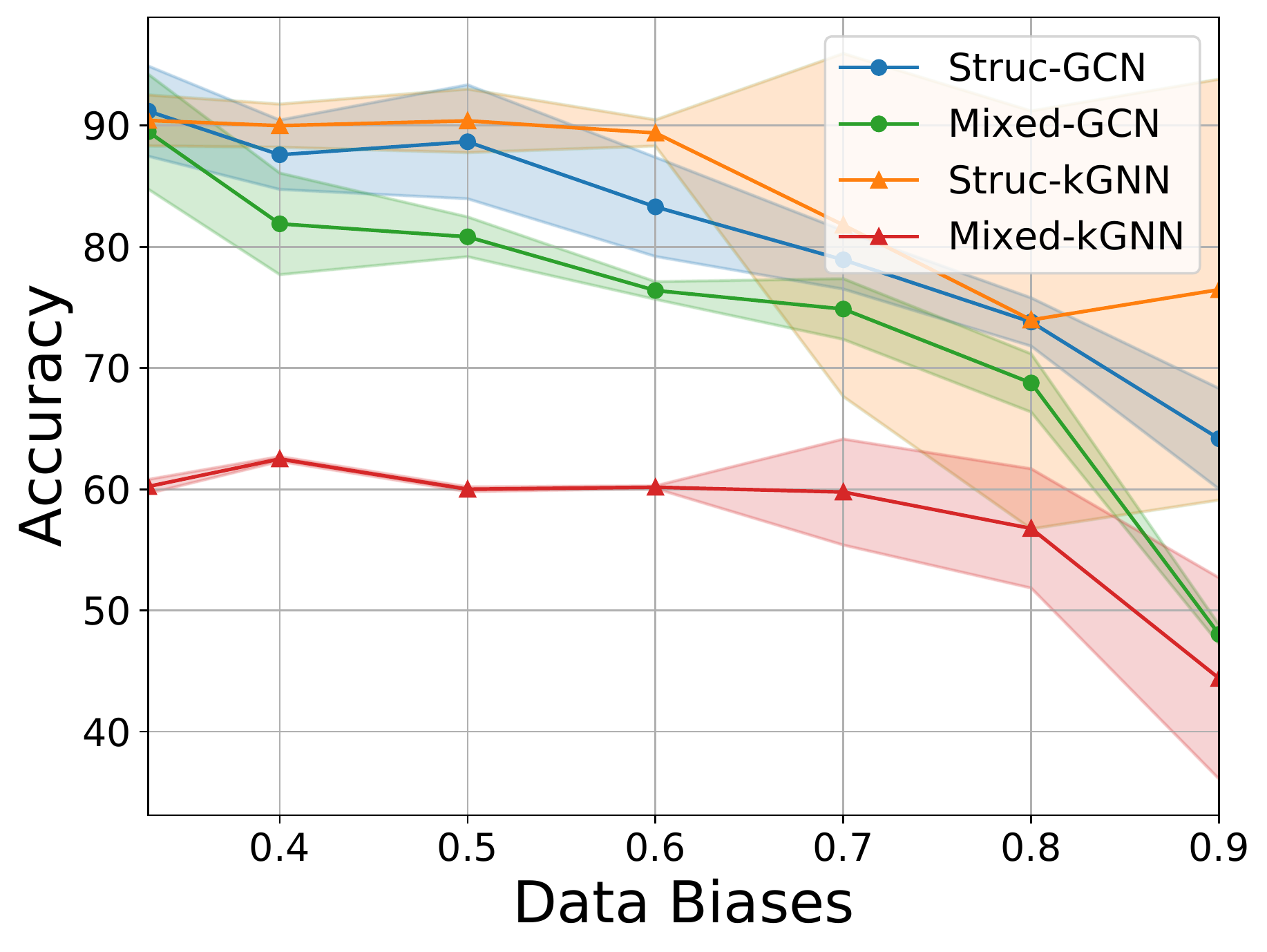}
	}
	\caption{
		Failure of existing methods on SPMotif with PIIF attribute shifts.}
	\label{fig:good_fail_piif_wosize_appdx}
\end{figure}

\begin{figure}[H]
	\centering
	\subfigure[Failures of training objectives.]{
		\includegraphics[width=0.3\textwidth]{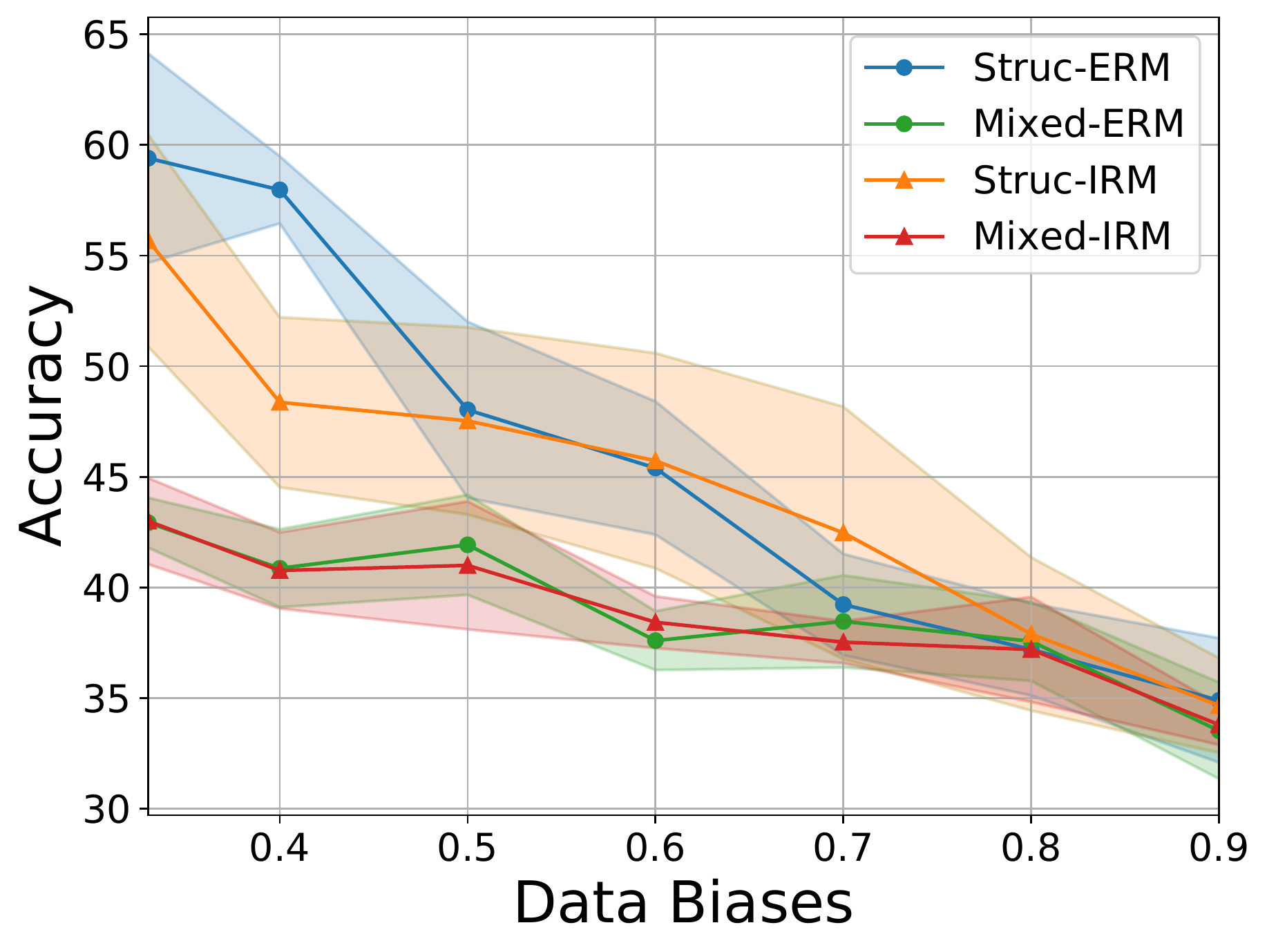}
	}
	\subfigure[Failures of deeper GNNs.]{
		\includegraphics[width=0.3\textwidth]{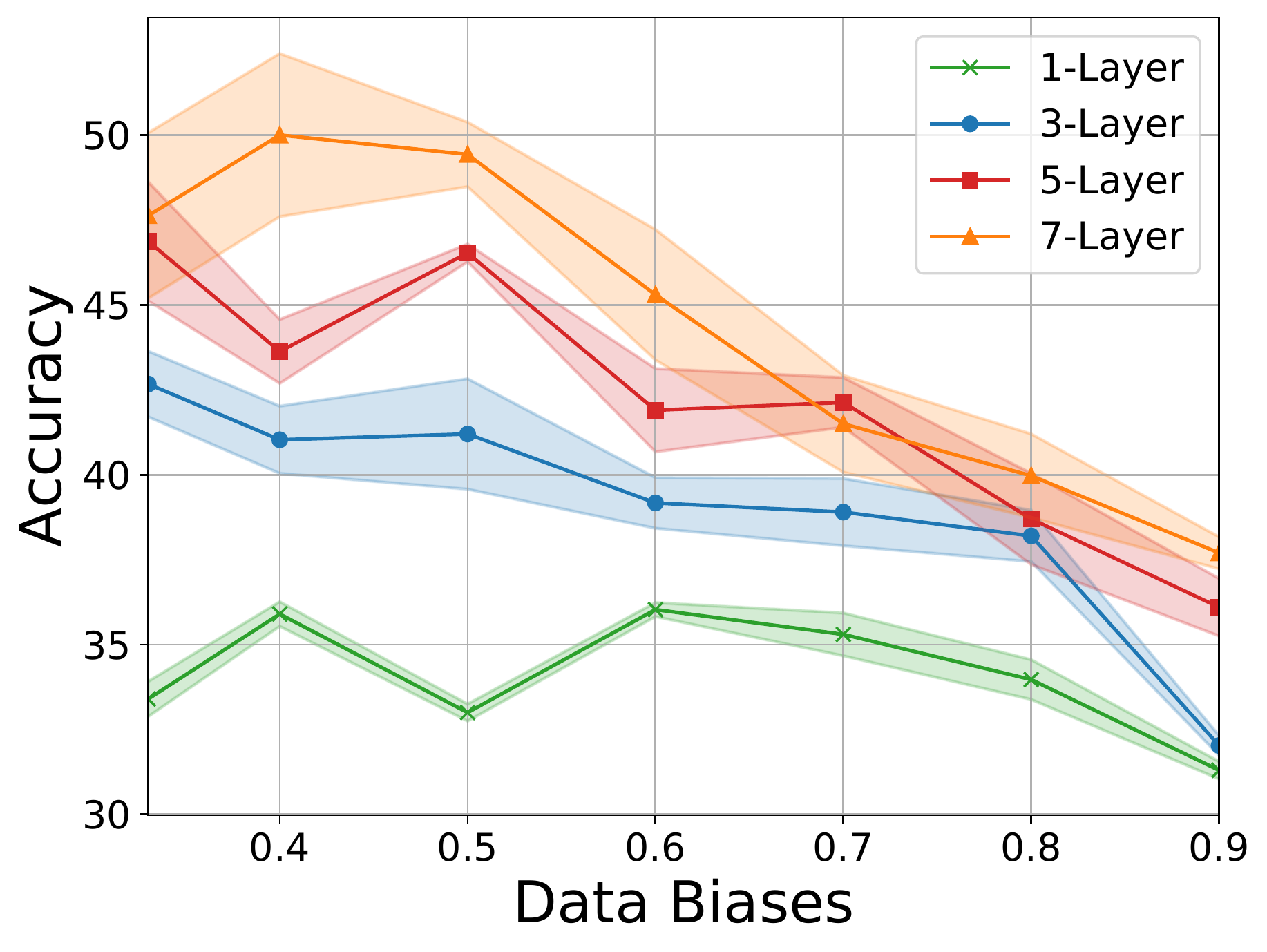}
	}
	\subfigure[Failures of expressive GNNs.]{
		\includegraphics[width=0.3\textwidth]{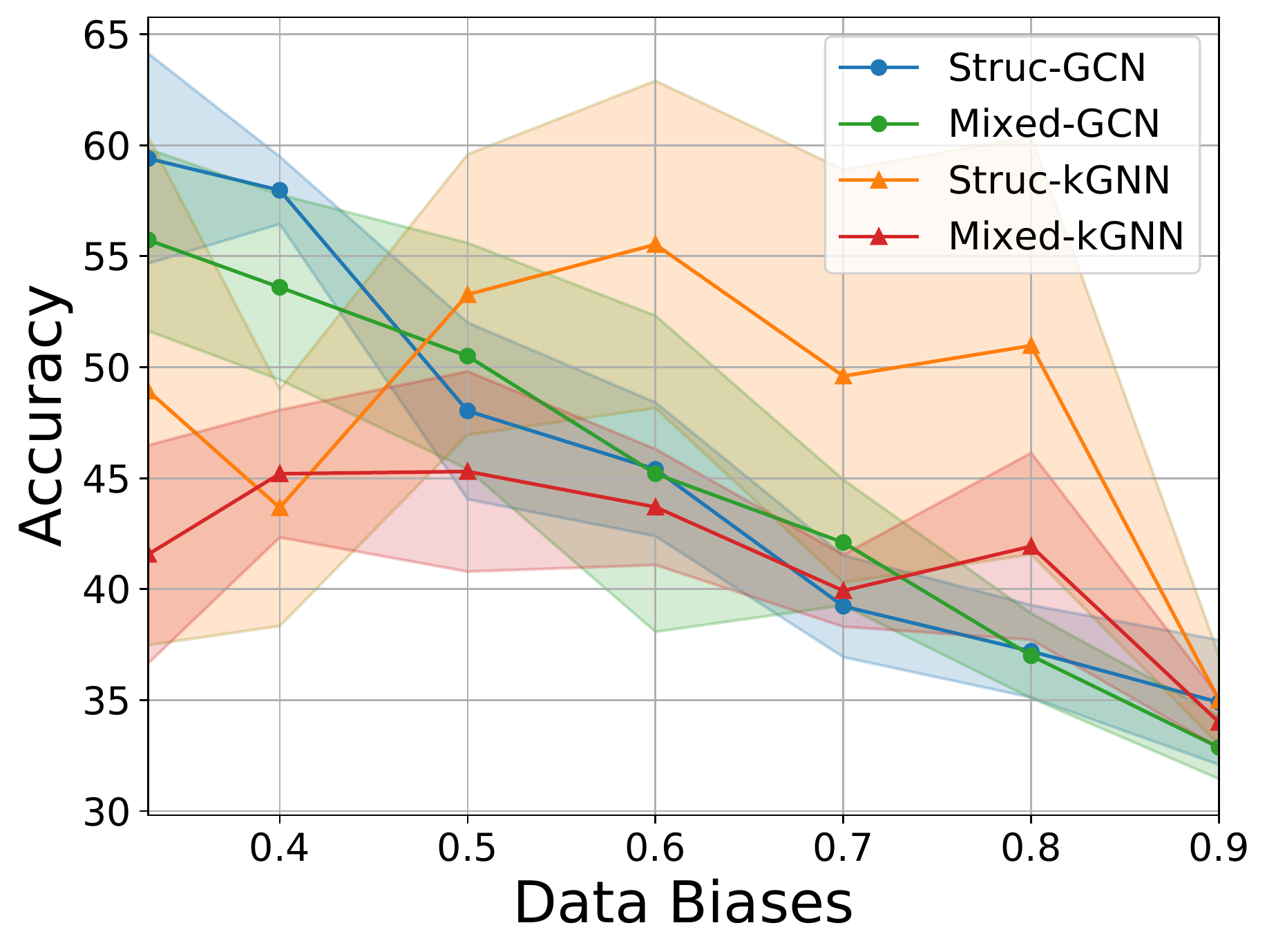}
	}
	\caption{
		Failure of existing methods on SPMotif PIIF attribute shifts with graph size shifts.}
	\label{fig:good_fail_piif_appdx}
\end{figure}

To explore the behaviors of aforementioned methods against complicated distribution shifts on graphs,
we first modify construction method in~\citet{dir} to construct dataset for Fig.~\ref{fig:good_fail_wosize_appdx}, where only
FIIF structure-level spurious correlations are injected.
Then we also inject FIIF attribute-level shifts, by setting the node attributes to constant vectors which is spuriously correlated with the labels.
Furthermore, in Fig.~\ref{fig:good_fail_size_appdx}, graph size shifts are added, which is exactly the SPMotif datasets used in DIR~\citep{dir}.
Besides, in Fig.~\ref{fig:good_fail_piif_wosize_appdx}, we can also change the FIIF attribute-level shifts to PIIF attribute-level shifts, where we flip the labels by a probability of $5\%$
and let the flipped label to be spuriously correlated with the node features, following the PIIF SCM in Fig.~\ref{fig:scm_appdx}.
Graph size shifts can also be injected in this case, shown as Fig.~\ref{fig:good_fail_piif_appdx}. Next, we summarize our findings from the experiments.

\textbf{Observation I: All existing methods are sensitive to distribution shifts.}
From the Fig.~\ref{fig:good_fail_wosize_appdx},~\ref{fig:good_fail_size_appdx},~\ref{fig:good_fail_piif_wosize_appdx},~\ref{fig:good_fail_piif_appdx},
we can observe that \emph{all} GNNs are sensitive to distribution shifts.
As the intensity of spurious correlation grows, GNNs are more likely to overfit to shortcuts
presented either in the structure-level or attribute-level, which is similar to general deep learning models~\citep{shortcut_dl}.

\textbf{Observation II: Higher variance also indicates unstable OOD performance.}
Although GNNs show certain robustness against single distribution shifts, e.g., performances do not decrease sharply at the beginning in Fig.~\ref{fig:good_fail_wosize_appdx},
when the spurious correlation grows stronger, the OOD performance become more \emph{unstable}, e.g., higher variance.
The reason is that, GNNs sometimes can directly learn about the desired information at some random initializations, since the task is relatively simple
compared to reality. Hence the performance will be highly sensitive to the quality of initialized points at the beginning.
Consequently, the performances from multiple runs would exhibit high variance.
However, when the task becomes more difficult, GNNs will consistently be prone to distribution shifts, and the variance will be smaller,
as shown in experiments (Sec.~\ref{sec:exp}).

\textbf{Observation III: Entangling more distribution shifts can degenerate more GNN performance.}
As implied by the graph generation SCMs in Fig.~\ref{fig:scm_appdx}, distribution shifts can happen at both structure-level and attribute-level,
and each of them can have different type of spurious correlation with the label.
In Fig.~\ref{fig:good_fail_wosize_appdx}, we can find that, when the attribute-level distribution shifts are mixed,
the performance will be worse and more unstable.
When the graph size shifts are mixed, this phenomenon will be more obvious, as shown in Fig.~\ref{fig:good_fail_size_appdx}.
This phenomenon also verifies the observations in~\citet{understand_att} that attention mechanism in GNN
is also sensitive to graph size shifts and can hardly learn the desired attention distributions without further guidance.
Moreover, when the structure-level and attribute-level shifts have different spurious correlation types,
i.e., when FIIF structure-level shifts and PIIF attribute-level shifts are both presented,
the performance drop will be more serious, by comparing Fig.~\ref{fig:good_fail_wosize_appdx} to Fig.~\ref{fig:good_fail_piif_wosize_appdx},
as well as Fig.~\ref{fig:good_fail_size_appdx} to Fig.~\ref{fig:good_fail_piif_appdx}.

\textbf{Observation IV: Using more powerful architectures can not improve the OOD performance.}
From the sub-figures (b) and (c) in
Fig.~\ref{fig:good_fail_wosize_appdx},~\ref{fig:good_fail_size_appdx},~\ref{fig:good_fail_piif_wosize_appdx},~\ref{fig:good_fail_piif_appdx},
we can also observe that neither adding more message passing turns nor using more expressive GNN architectures can
be immune to distribution shifts. On the contrary, they also exhibit similar behaviors like basic GNN architectures.
Specifically, adding more message passing runs show certain robustness against distribution shifts since they are more likely to
learn the desired information during the optimization~\citep{gnn_opt}.
However, when the intensity of spurious correlation grows stronger,
deeper GNNs are more likely to overfit to shortcuts hence their performances will drop more sharply.
On the other hand, using provably more expressive GNN architectures can not improve the OOD performance, either.
In Fig.~\ref{fig:good_fail_wosize_appdx},~\ref{fig:good_fail_size_appdx},~\ref{fig:good_fail_piif_wosize_appdx},~\ref{fig:good_fail_piif_appdx}
we use $1$-$2$-$3$-GNN following the algorithm of $k$-GNNs which is provably more expressive than $2$-WL test~\citep{kgnn}.
When there are no graph size shifts, $k$-GNNs will have higher performance at the beginning.
When there are graph size shifts, $k$-GNNs will have a lower initial performance at the beginning.
Then, as the spurious strength grows, $k$-GNNs can suddenly become seriously unstable, though
$k$-GNNs can have higher averaged performance, which reflects unsatisfactory OOD performance as Observation II implies.
When the intensity of spurious correlations grows even stronger, similar to deeper GNNs,
OOD performances of $k$-GNNs will be more unstable and go down to similar level as that of normal GNN architectures.
Hence, it calls for better optimization objectives as well as a suitable architectures to help improve the OOD generalization performance.

Beyond the empirical studies in previous section, we aim to accompany more formal discussions for explaining the failures of existing optimization
objectives and architectures in the next sections.

\subsection{Theoretical discussions for failure case study in Sec.~\ref{sec:limitation_prev}}
\label{sec:discussion_ood_obj_appdx}

\textbf{A motivating example.}
To begin with, we follow~\citet{ib-irm} to introduce a formal example on the failures of GNNs optimized with ERM or IRM~\citep{erm,irmv1} via a linear binary classification problem:
\begin{definition}[Linear classification structural equation model (FIIF)]
	\label{def:linear_fiif_appdx}
	\[
		\begin{aligned}
			 & Y:= (w_\inv^* \cdot C)\oplus N,\ N\sim \text{Ber}(q),\ N\ind (C,S), \\
			 & X\leftarrow S(C,S),
		\end{aligned}
	\]
	where $w_\inv^*\in\R^{n_c}$ with $\norm{w_\inv^*}=1$ is the labeling hyperplane,
	$C\in \R^{n_c},\ S\in\R^{n_s}$ are the corresponding invariant and varying latent variables, $N$ is Bernoulli binary noise with a parameter of $q$ and identical across all environments, $\oplus$ is the $\text{XOR}$ operator, $S$ is invertible.
\end{definition}

Given data generation process as Assumption~\ref{assump:graph_gen_appdx}, and latent space interaction as Assumption~\ref{assump:scm_fiif_appdx} or ~\ref{assump:scm_piif_appdx}, and strictly separable invariant features~\ref{assump:latent_sep},
consider a $k$-layer linearized GNN $\rho \circ h$ using $\text{mean}$ as $\text{READOUT}$ for binary graph classification,
if $\cup_{e\in\envtest}\text{supp}(\sP^e)\not\subseteq\cup_{e\in\envtrain}\text{supp}(\sP^e)$:
\begin{enumerate}[label=(\roman*),nosep]
	\item  For graphs features  generated as Definition~\ref{def:linear_fiif_appdx},
	      $\rho \circ h$ optimized with ERM or IRM will fail to generalize OOD (Eq.~\ref{eq:ood}) almost surely;
	\item For graphs with more than two nodes, globally same node features generated as Definition~\ref{def:linear_fiif_appdx}, and graph labels that are the same as global node labels, $\rho \circ h$ optimized with ERM or IRM will fail to generalize OOD (Eq.~\ref{eq:ood}) almost surely;
\end{enumerate}

For graph classification, if the number of nodes is fixed to one, it covers the linear classification as above. When  $\cup_{e\in\envtest}\text{supp}(\sP^e)\not\subseteq\cup_{e\in\envtrain}\text{supp}(\sP^e)$,
it implies the $S$ from training environments $\envtrain$ does not cover $S$ from testing environments, while $C$ can be covered. Moreover, the condition of strictly separable training data now can be formulated as $\min_{C\in\cup_{e\in\envtrain}(C\subseteq G^e)}\text{sgn}(w^*_\inv\cdot C)(w^*_\inv \cdot C)>0$. Recall that ERM trains the model by minimizing the empirical risk (e.g., 0-1 loss) over all training data, and IRM formulates OOD generalization as:
\begin{equation}
	\label{eq:irmv1_formula_appdx}
	\begin{aligned}
		\min_{\theta, f_c} & \frac{1}{|\envtrain|}\sum_{e\in\envtrain}R^e(\rho \circ h)                   \\
		\text{s.t.}\       & \rho\in\argmin_{\hat{\rho}}R^e(\hat{\rho}\circ h),\ \forall e \in \envtrain.
	\end{aligned}
\end{equation}
However, both ERM and IRM can not enable OOD generalization, i.e., finding the ground truth $w^*_\inv$, following the Theorem 3 from~\citet{ib-irm}:
\begin{theorem}[Insufficiency of ERM and IRM]
	Suppose each $e\in\envall$ follows Definition.~\ref{def:linear_fiif_appdx}, $C$ are strictly separable, bounded and satisfy the support overlap between $\envtrain$ and $\envtest$, and $S$ are bounded, if $S$ does not support the overlap, then both ERM and IRM fail at solving the OOD generalization problem.
\end{theorem}

The reason is that, when $C$ from all environments are strictly separable, there can be infinite many Bayes optimal solutions given training data $\{G^e,y^e\}_{e\in\envtrain}$, while there is only one optimal solution that does not rely on $S$. Hence, the probability of generalization to OOD (finding the optimal solution) tends to be $0$ in probability.

As for case (ii), when the GNN uses mean readout to classify more than one node graphs, assuming the graph label is determined by the node label and all of the nodes have the same label that are determined as Definition~\ref{def:linear_fiif_appdx}, then GNN optimized with ERM and IRM will also fail because of the same reasons as case (i).

\textbf{Discussions on the failures of previous OOD related solutions.}
First of all, for IRM or similar objectives~\citep{groupdro,v-rex,ib-irm,gen_inv_conf} that require environment information or non-trivial data partitions,
they can hardly be applied to graphs due to the lack of such information.
The reason is that obtaining such information can be expensive due to the abstraction of graphs.
Moreover, as proved in Theorem 5.1 of~\citet{risk_irm}, when there is not sufficient support overlap between training environments and testing environments, the IRM or similar objectives can fail catastrophically when being applied to non-linear regime.
The only OOD objective EIIL~\citep{env_inference} that does not require environment labels, also rely on similar assumptions on the support overlap.
We also empirically verify their failing behaviors in our experiments.

Moreover, since part of explainability works also try to find a subset of the inputs for interpretable prediction robustly against distribution shifts.
Here we also provide a discussion for these works.
The first work following this line is $\invrat$~\citep{inv_rat},
which develops an information-theoretic objective (we re-formulate it to suit with OOD generalization problem on graphs):
\begin{equation}
	\label{eq:inv_rat_appdx}
	\min_{g,f_c} \max_{f_s}R(f_c \circ g,Y)+\lambda h(R(f_c \circ g,Y)-R_e(f_s\circ g,Y,E)).
\end{equation}
However, it also requires extra environment labels for optimization that are often unavailable in graphs.
Besides, the corresponding assumption on the data generation for guaranteed performance is essentially PIIF if applied to our case,
while it can not provide any theoretical guarantee on FIIF.

We also notice a recent work, $\dir$~\citep{dir}, as a generalization of  $\invrat$ to graphs while studying FIIF spurious correlations,
that proposes an alternative objective which does not require environment label:
\begin{equation}
	\label{eq:dir_appdx}
	\min \mathbb{E}_s[R(h,Y|\doop(S=s))]+\\\lambda \var_s(\{R(h,Y|\doop(S=s))\}).
\end{equation}
However, the theoretical justification established for $\dir$ (Theorem 1 to Corollary 1 in~\citet{dir}) essentially depends on the
quality of the generator $g$ which can be prone to spurious correlations.
Thus, $\dir$ can hardly provide any theoretical guarantees when applied to our case, neither for FIIF nor PIIF.
In experiments, we empirically find the unstable and relatively high sensitivity of DIR to spurious correlations,
which verifies our finding.
More details about empirical behaviors of DIR can be found in Appendix~\ref{sec:exp_appdx}.

In contrast to $\dir$, GIB~\citep{gib} that focuses on discovering a informative subgraph for explanation,
essentially can provide theoretical guarantees for FIIF spurious correlations. Theoretically,
(we copy the discussion in Appendix~\ref{sec:good_impl_appdx} here to provide an overview of relationships
between GIB and DIR.)
Under the FIIF assumption on latent interaction,
the independence condition derived from causal model can
also be rewritten as $Y\ind S|C$ (similar to that in DIR~\citep{dir} as they also focus on FIIF),
which further implies $Y\ind S|\widehat{G}_c$.
Hence it is natural to use Information Bottleneck (IB) objective~\citep{ib}
to solve for $G_c$:
\begin{equation}
	\label{eq:good_ib}
	\begin{aligned}
		\min_{f_c,g} & \ R_{G_c}(f_c(\widehat{G}_c)),                                                        \\
		\text{s.t.}  & \ G_c=\argmax_{\widehat{G}_c=g(G)\subseteq G}I(\widehat{G}_c,Y)-I(\widehat{G}_c,\gG), \\
	\end{aligned}
\end{equation}
which explains the success of many existing
works in finding predictive subgraph through IB~\citep{gib}.
However, the estimation of $I(\widehat{G}_c,G)$
is notoriously difficult due to the complexity of graph,
which can lead to unstable convergence as observed in our experiments.
In contrast, optimization with contrastive objective in $\ginv$ as Eq.~\ref{eq:good_opt_contrast}
induces more stable convergence.
\subsection{Challenges of OOD generalization on graphs.}
From the aforementioned analysis, we can summarize some key challenges revealed by the failures of both existing optimization
objectives and GNN architectures.
In particular, we are facing two main challenges
a) Distribution shifts on graphs are more complicated where different types of spurious correlations can be entangled via different graph properties;
b) Environment labels are usually not available due to the abstract graph data structure.

\section{Theory and Discussions}
\label{sec:theory_appdx}
In this section, we provide proofs for propositions and theorems mentioned in the main paper.
\revision{
	\subsection{More discussions on Definition~\ref{def:inv_gnn} for Invariant GNNs}
	\label{sec:inv_gnn_discuss_appdx}
	Definition~\ref{def:inv_gnn} is motivated by applying the invariance principle to the established SCMs in Sec.~\ref{sec:data_gen}, following the literature of invariant learning~\citep{inv_principle}. In this section, we will present Proposition~\ref{thm:causal_minmax_appdx} and Proposition~\ref{thm:minmax_causal_appdx}
	to illustrate how satisfying the minmax objective in Definition~\ref{def:inv_gnn_appdx} is equivalent to identifying the underlying invariant subgraph $G_c$ that contains all of the information about causal factor $C$ in $G$, under both FIIF and PIIF SCMs (Fig.~\ref{fig:scm_fiif} and Fig.~\ref{fig:scm_piif}).
}\revision{
	\begin{definition}[Invariant GNN]
		\label{def:inv_gnn_appdx}
		Given a set of graph datasets $\{\dataset^e\}_e$ %
		and environments $\envall$ that follow the same graph generation process in Sec.~\ref{sec:data_gen},
		considering a GNN $\rho \circ h$ that has a permutation invariant graph
		encoder $h:\gG\rightarrow\R^h$ and a downstream classifier $\rho:\R^h\rightarrow\gY$,
		$\rho \circ h$ is an invariant GNN if it minimizes the worst case  risk
		among all environments, i.e., $\min \max_{e\in\envall}R^e$.
	\end{definition}
}

\revision{First, we show that using the invariant subgraphs $G_c$ to predict $Y$ can satisfy the minmax objective $\min \max_{e\in\envall}R^e$ in Proposition~\ref{thm:causal_minmax_appdx}. }

\revision{
	\begin{proposition}
		\label{thm:causal_minmax_appdx}
		Let $\gG_c$ denote the subgraph space for $G_c$,
		given a set of graphs with their labels $\dataset=\{G^{(i)},y^{(i)}\}_{i=1}^N$ and $\envall$ that
		follow the graph generation process in Sec.~\ref{sec:data_gen} (or Sec.~\ref{sec:full_scm_appdx}),
		a GNN $\rho\circ h:\gG_c\rightarrow\gY$ that takes $G_c$ of $G$ as the input to predict $Y$, and solves the following objective can generalize to OOD graphs, i.e., solving the minmax objective in Def.~\ref{def:inv_gnn_appdx}:
		\[\min_{\theta}R_{\gG_c}(\rho\circ h),\]
		where $R_{\gG_c}$ is the empirical risk over $\{G^{(i)}_c,y^{(i)}\}_{i=1}^N$
		and $G^{(i)}_c$ is the underlying invariant subgraph $G_c$ for $G^{(i)}$.
	\end{proposition}
}\revision{
	\begin{proof}
		\label{proof:causal_minmax_appdx}
		We establish the proof with independent causal mechanism (ICM) assumption in SCM~\citep{causality,elements_ci}.
		In particular, given the data generation assumption, i.e.,
		for both FIIF (Assumption~\ref{assump:scm_fiif}) and PIIF (Assumption~\ref{assump:scm_piif}),
		we have: $\forall e,$
		\begin{equation}
			\label{eq:icm_appdx}
			\begin{aligned}
				P(Y|C)                     & =P(Y|C,E=e)                     \\
				P(Y|G_c)\sum_{G_c}P(G_c|C) & =P(Y|G_c)\sum_{G_c}P(G_c|C,E=e) \\
				P(Y|G_c)\sum_{G_c}P(G_c|C) & =P(Y|G_c,E=e)\sum_{G_c}P(G_c|C) \\
				P(Y|G_c)                   & =P(Y|G_c,E=e),                  \\
			\end{aligned}
		\end{equation}
		where we use ICM for the first three equalities.
		From Eq.~\ref{eq:icm_appdx}, it suffices to know $P(Y|G_c)$ is invariant across different environments.
		Hence, a GNN predictor $\rho\circ h:\gG_c\rightarrow\gY$ optimized with
		empirical risk given $G_c$, essentially minimizes the empirical risk
		across all environments, i.e., $\min R_{\gG_c} = \min \max R^e$.
		Thus, if $\rho\circ h$ solves $\min R_{\gG_c}$, it also solves $\min \max R^e$,
		hence it elicits a invariant GNN predictor according to Definition.~\ref{def:inv_gnn_appdx}.
	\end{proof}
}
\revision{Besides, we show in Proposition~\ref{thm:minmax_causal_appdx} that only using the underlying invariant subgraphs $G_c$ to make predictions can satisfy the minmax objectives. Or equivalently, a GNN predictor solving the minmax objective can only rely on the underlying invariant subgraph $G_c$ to predict $Y$. }

\revision{
	\begin{proposition}
		\label{thm:minmax_causal_appdx}
		Given a set of graph datasets $\{\dataset^e\}_e$ %
		and environments $\envall$ that follow the same graph generation process in Sec.~\ref{sec:data_gen},
		considering a GNN $\rho \circ h$ that has a permutation invariant graph
		encoder $h:\gG\rightarrow\R^h$ and a downstream classifier $\rho:\R^h\rightarrow\gY$,
		$\rho \circ h$ that minimizes the worst case  risk
		among all environments, i.e., $\min \max_{e\in\envall}R^e$,
		can not rely on any part of $G_s$, i.e., $\rho \circ h (G) \ind G_s$.
	\end{proposition}
}\revision{
	\begin{proof}
		\label{proof:minmax_causal_appdx}
		The proof for Proposition~\ref{thm:minmax_causal_appdx} is straightforward. Assuming that $\rho \circ h (G) \not\ind G_s$, as $E$ is influenced by the changes of $E$ through $S$ in both FIIF and PIIF SCMs (Fig.~\ref{fig:scm_fiif} and Fig.~\ref{fig:scm_piif}), then $\rho \circ h (G) \not\ind E$ as well.
		Consequently, there exists some graph $G$ corresponding to $G_c,G_s^e$ and $\rho \circ h (G)=Y$ under an environment $e$, such that we can always find a proper $e'$ to make $\rho \circ h (G) \neq Y$.
		In contrast, the prediction of a GNN that satisfies $\rho \circ h (G) \ind G_s$ remains invariant against arbitrary changes of environments.
		Thus, it leads to a contradiction to the condition that $\min \max_{e'\in\envall}R^{e'}$. Therefore, a GNN that solves $\min \max_{e\in\envall}R^e$ must satisfy $\rho \circ h (G) \ind G_s$.
	\end{proof}
}

Combining Proposition~\ref{thm:causal_minmax_appdx} and Proposition~\ref{thm:minmax_causal_appdx}, we are highly motivated to find the underlying invariant subgraphs to make predictions about the original graphs, which converges to Eq.~\ref{eq:good_opt}. Tackling Eq.~\ref{eq:good_opt} under the unavailability of $E$ brings us two variants of CIGA solutions, as illustrated in Section~\ref{sec:good_framework}.

\subsection{Proof for theorem~\ref{thm:good_inv_gnn_new} (i)}

\begin{theorem}[$\ginv$v1 Induces Invariant GNNs]
	\label{thm:good_inv_gnn_new_appdx}
	Given a set of graph datasets $\{\dataset^e\}_e$ %
	and environments $\envall$ that follow the same graph generation process in Sec.~\ref{sec:data_gen},
	assuming that \textup{(a)} $f_\gen^G$ and $f_\gen^{G_c}$ in Assumption~\ref{assump:graph_gen} are invertible,
	\textup{(b)} samples from each training environment are equally distributed,
	i.e.,$|\dataset_{\hat{e}}|=|\dataset_{\tilde{e}}|,\ \forall \hat{e},\tilde{e}\in\envtrain$,
	if $\forall G_c, |G_c|=s_c$,
	then a GNN $f_c\circ g$ solves Eq.~\ref{eq:good_opt_contrast_v3},
	is an invariant GNN (Def.~\ref{def:inv_gnn}).
\end{theorem}

\textit{Proof.}
We re-write the objective as follows:
\label{proof:good_inv_gnn_new_appdx}
\begin{equation}
	\label{eq:good_opt_contrast_new_appdx}
	\max_{f_c,g} \ I(\widehat{G}_c;Y), \ \text{s.t.}\
	\widehat{G}_c\in\argmax_{\widehat{G}_c=g(G),|\widehat{G}_c|\leq s_c} I(\widehat{G}_c;\widetilde{G}_c|Y),
\end{equation}
where $\widehat{G}_c=g(G),\widetilde{G}_c=g(\widetilde{G})$ and $\widetilde{G}\sim \sP(G|Y)$,
i.e., $\widetilde{G}$ and $G$ have the same label.

The proof of Theorem~\ref{thm:good_inv_gnn_new_appdx} is essentially
to show the estimated $\widehat{G}_c$ through Eq.~\ref{eq:good_opt_contrast_new_appdx}
is the underlying $G_c$, then
the maximizer of $I(\widehat{G}_c;Y)$ in Eq.~\ref{eq:good_opt_contrast_new_appdx}
can produce most informative and stable predictions about $Y$ based on $G$,
hence is an invariant GNN (Definition.~\ref{def:inv_gnn_appdx}).

In the next, we are going to take an information-theoretic view of
the first term $I(\widehat{G}_c;Y)$ and
the second term $I(\widehat{G}_c;\widetilde{G}_c|Y)$ to conclude the proof.
We begin by introducing the following lemma:
\begin{lemma}
	\label{thm:env_eq_appdx}
	Given the same conditions as Thm.~\ref{thm:good_inv_gnn_new_appdx},
	$I(\widehat{G}_c;Y)$ is maximized if and only if
	$I(\widehat{G}_c;Y|E=e)$ is maximized, $\forall e\in\envtrain$.
\end{lemma}
The proof for Lemma~\ref{thm:env_eq_appdx} is straightforward, given the
condition that samples from each training environment are equally distributed,
i.e.,$|\dataset_{\hat{e}}|=|\dataset_{\tilde{e}}|,\ \forall \hat{e},\tilde{e}\in\envtrain$.
Obviously, $\widehat{G}_c=G_c$ is a maximizer of $I(\widehat{G}_c;Y)=I(C;Y)=H(Y)$, since
$f_\gen^c:\gC\rightarrow\gG_c$ is invertible and $C$ causes $Y$.
However, there might be some subset $G_s^p\subseteq G_s$ from the underlying $G_s$
that entail the same information about label, i.e., $I(G_c^p\cup G_s^p;Y)=I(G_c;Y)$
where $\widehat{G}_c=G_c^p\cup G_s^p$ and $G_c^p= G_c\cap\widehat{G}_c$.
For FIIF (Assumption~\ref{fig:scm_fiif_appdx}), it can not happen, otherwise,
let $G_c^l= G_c-G_c^p$, then we have:
\begin{equation}
	\begin{aligned}
		I(\widehat{G}_c;Y)=I(G_c^p\cup G_s^p;Y) & =I(G_c^p\cup G_c^l;Y)=I(G_c;Y) \\
		I(G_c^p;Y)+I(G_s^p;Y|G_c^p)             & =I(G_c^p;Y)+I(G_c^l;Y|G_c^p)   \\
		I(G_s^p;Y|G_c^p)                        & =I(G_c^l;Y|G_c^p)              \\
		H(Y|G_c^p)-H(Y|G_c^p,G_s^p)             & =H(Y|G_c^p)-H(Y|G_c^p,G_c^l)   \\
		H(Y|G_c^p)-H(Y|G_c^p,G_s^p)             & =H(Y|G_c^p),                   \\
		H(Y|G_c^l,G_s^p)                        & =0,                            \\
	\end{aligned}
\end{equation}
where the second last equality is due to $C\rightarrow Y$ and the invertibility of $f_\gen^c:\gC\rightarrow\gG_c$ in FIIF,
i.e., $H(Y|C)=H(Y|G_c)=H(Y|G_c^p,G_c^l)=0$.
However, in PIIF,
it can hold since conditioning on $G_c^p,G_s^p$ can not determine $Y$, as $S\not\ind Y|C$. In other words, $G_s\not\ind Y|G_c$, which
means $G_s$ can imply some information about $Y$
that is equivalent to $I(G_c^l;Y|G_c^p)$.

To avoid the presence of spuriously correlated $G_s$ in $\widehat{G}_c$, we will
use the second term to eliminate it:
\begin{equation}
	\label{eq:good_opt_contrast_cond_appdx_yg2}
	\begin{aligned}
		\max_{f_c,g} & \ I(\widehat{G}_c;\widetilde{G}_c|Y),                    \\
		             & = H(\widehat{G}_c|Y)-H(\widehat{G}_c|\widetilde{G}_c,Y), \\
	\end{aligned}
\end{equation}
where $\widehat{G}_c=g(G)$, $\widetilde{G}_c=g(\widetilde{G})$ are two positive samples drawn from the same class (i.e., condition on the same $Y$).
Since the all of the training environments are equally distributed,
maximizing $I(\widehat{G}_c;\widetilde{G}_c|Y)$ is
essentially maximizing $I(\widehat{G}_c,E=\hat{e};\widetilde{G}_c,E=\tilde{e}|Y)$, $\forall \hat{e},\tilde{e}\in\envtrain$.
Hence, we have:
\begin{equation}
	\label{eq:good_opt_contrast_env_appdx_i}
	\begin{aligned}
		\max_{f_c,g} & \ I(\widehat{G}_c;\widetilde{G}_c|Y),                                                    \\
		             & =I(\widehat{G}_c,E=\hat{e};\widetilde{G}_c,E=\tilde{e}|Y)                                \\
		             & = H(\widehat{G}_c,E=\hat{e}|Y)-H(\widehat{G}_c,E=\hat{e}|\widetilde{G}_c,E=\tilde{e},Y). \\
	\end{aligned}
\end{equation}
We claim Eq.~\ref{eq:good_opt_contrast_env_appdx_i}
can eliminate any potential
subsets from $G_s$ in the estimated $\widehat{G}_c$.
\begin{wrapfigure}{r}{0.5\textwidth}
	\includegraphics[width=0.5\textwidth]{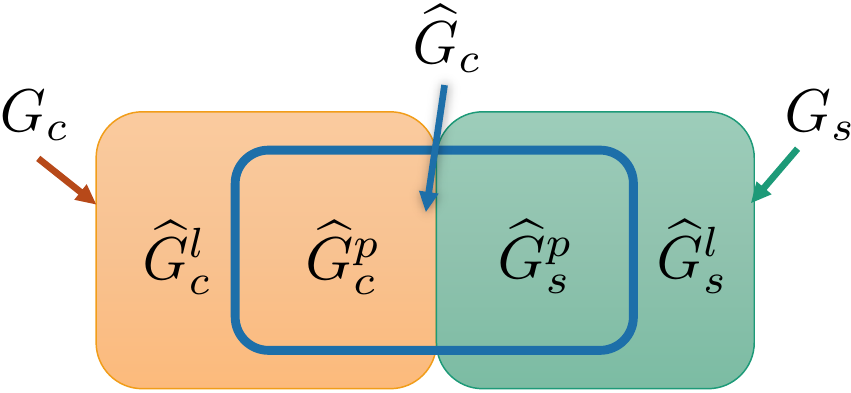}
	\label{fig:good_set_venn_1_appdx}
	\caption{Illustration of the notation. $G_c$ and $G_s$ are two
		disjoint sets. $\widehat{G}_c$ may contain certain subsets from $G_c$ and $G_s$.
		The subsets from $G_c$ and $G_s$ contained in $\widehat{G}_c$ are denoted as $\widehat{G}_c^p$ and $\widehat{G}_s^p$, respectively.
		While the left subsets in $G_c$ and $G_s$ are denoted as $\widehat{G}_c^l$ and $\widehat{G}_s^l$, respectively.}
	\vspace{-0.45in}
\end{wrapfigure}

Otherwise, suppose there are some subsets $\widehat{G}_s^p\subseteq \widehat{G}_s$
and $\widetilde{G}_s^p\subseteq \widetilde{G}_s$
contained in the estimated $\widehat{G}_c$, $\widetilde{G}_c$, where
$\widehat{G}_s,\widetilde{G}_s$ be the corresponding underlying $G_s$s for $\widehat{G}_c,\widetilde{G}_c$.
Let $\widehat{G}_c^*$ and $\widetilde{G}_c^*$ be the ground truth invariant subgraph $G_c$s
of $\widehat{G}$ and $\widetilde{G}$,
$\widehat{G}_c^l=\widehat{G}_c^*-\widehat{G}_c$
and $\widetilde{G}_c^l=\widetilde{G}_c^*-\widetilde{G}_c$ be the \textbf{l}eft (un-estimated) subsets
from corresponding ground truth $G_c$s,
and $\widehat{G}_c^p=\widehat{G}_c^*-\widehat{G}_c^l$ and
$\widetilde{G}_c^p=\widetilde{G}_c^*-\widetilde{G}_c^l$ be the complement,
or equivalently,
the \textbf{p}artial $\widehat{G}_c^*,\widetilde{G}_c^*$ that are estimated in $\widehat{G}_c,\widetilde{G}_c$, respectively.
We can also define similar counterparts for $G_s$: $\widehat{G}_s^p,\widetilde{G}_s^p$
are the partial $\widehat{G}_s,\widetilde{G}_s$s contained in the estimated $\widehat{G}_c,\widetilde{G}_c$
while $\widehat{G}_s^l,\widetilde{G}_s^l$ are the left subsets $\widehat{G}_s,\widetilde{G}_s$, respectively.

Recall the constraint that $|G_c|=s_c$, hence if $\widehat{G}_c^p\subseteq \widehat{G}_c$,
then a corresponding $\widehat{G}_c^l=\widehat{G}_c^*-\widehat{G}_c^p$ will be
replaced by $\widehat{G}_s^p$ in $\widehat{G}_c$.
In this case, we have:
\begin{equation}
	\label{eq:good_opt_contrast_mi_t1_appdx_i}
	\begin{aligned}
		H(\widehat{G}_c,E=\hat{e}|Y) & =H(E=\hat{e}|\widehat{G}_c,Y)+H(\widehat{G}_c|E=\hat{e},Y)                      \\
		                             & =H(\widehat{G}_c^p\cup \widehat{G}_s^p|E=\hat{e},Y)                             \\
		                             & =H(\widehat{G}_c^p|E=\hat{e},Y) +H(\widehat{G}_s^p|\widehat{G}_c^p,E=\hat{e},Y) \\
	\end{aligned}
\end{equation}
where the second equality is due to $E=\hat{e}$ is determined so that $H(E=\hat{e}|\widehat{G}_c,Y)=0$.
Compared Eq.~\ref{eq:good_opt_contrast_mi_t1_appdx_i} to that
when $\widehat{G}_c=\widehat{G}_c^*$, we have the entropy change as:
\begin{equation}
	\label{eq:good_opt_contrast_env_t1_delta_appdx}
	\begin{aligned}
		\Delta H(\widehat{G}_c,E=\hat{e}|Y) & =H(\widehat{G}_c,E=\hat{e}|Y)-H(\widehat{G}_c^*,E=\hat{e}|Y),                                   \\
		                                    & =H(\widehat{G}_s^p|\widehat{G}_c^p,E=\hat{e},Y)-H(\widehat{G}_c^l|\widehat{G}_c^p,E=\hat{e},Y).
	\end{aligned}
\end{equation}
Let $\epsilon=H(\widehat{G}_s^p|\widehat{G}_c^p,E=\hat{e},Y)$.
In a idealistic setting, when the noise of the generation process $S:=f_\spu(Y,E)$ in PIIF
tends to be $0$, i.e., $\epsilon\rightarrow 0$,
$S$ is determined conditioned on $E,Y$,
hence $G_s$ and any subsets of $G_s$ are all determined.
Then, it suffices to know that in Eq.~\ref{eq:good_opt_contrast_env_t1_delta_appdx},
$H(\widehat{G}_s^p|\widehat{G}_c^p,E=\hat{e},Y)=0$ while $H(\widehat{G}_c^l|\widehat{G}_c^p,E=\hat{e},Y)>0$
since $\widehat{G}_c^l$ can not be determined when given $\widehat{G}_c^p,E=\hat{e},Y$.
Thus, when some subset from $G_s$ is included in $\widehat{G}_c$, it will minimize $H(\widehat{G}_c,E=\hat{e}|Y)$.

However in practice, it is usual that $\epsilon>0$.
Therefore, in the next, we will show how $\epsilon=H(\widehat{G}_s^p|\widehat{G}_c^p,E=\hat{e},Y)$
can be cancelled thus leading to a smaller $H(\widehat{G}_c,E=\hat{e}|Y)$, by considering
the second term $H(\widehat{G}_c,E=\hat{e}|\widetilde{G}_c,E=\tilde{e},Y)$.

As for $H(\widehat{G}_c,E=\hat{e}|\widetilde{G}_c,E=\tilde{e},Y)$, without loss of generality,
we can divide all of the possible cases into two:
\begin{enumerate}[label=(\roman*)]
	\item One of $\widehat{G}_c$ and $\widetilde{G}_c$ contains some subset of $G_s$, i.e.,
	      $\widehat{G}_c$ contains some $\widehat{G}_s^p\subseteq \widehat{G}_s$;
	\item Both $\widehat{G}_c$ and $\widetilde{G}_c$ contain some $\widehat{G}_s^p\subseteq \widehat{G}_s$ and $\widetilde{G}_s^p\subseteq \widetilde{G}_s$, respectively.
\end{enumerate}
For (i), we have:
\begin{equation}
	\begin{aligned}
		H(\widehat{G}_c,E=\hat{e}|\widetilde{G}_c,E=\tilde{e},Y) & =H(\widehat{G}_c^p, \widehat{G}_s^p,E=\hat{e}|\widetilde{G}_c,E=\tilde{e},Y)                                                              \\
		                                                         & =H( \widehat{G}_s^p|\widetilde{G}_c,E=\tilde{e},Y,\widehat{G}_c^p,E=\hat{e})+H( \widehat{G}_c^p,E=\hat{e}|\widetilde{G}_c,E=\tilde{e},Y), \\
	\end{aligned}
\end{equation}
Thus, we can write the change of $H(\widehat{G}_c,E=\hat{e}|\widetilde{G}_c,E=\tilde{e},Y)$ between $\widehat{G}_c=\widehat{G}_c^p\cup\widehat{G}_s^p$ and $\widehat{G}_c=\widehat{G}_c^*$ as:
\begin{equation}
	\begin{aligned}
		\Delta H(\widehat{G}_c,E=\hat{e}|\widetilde{G}_c,E=\tilde{e},Y) & =H(\widehat{G}_c,E=\hat{e}|\widetilde{G}_c,E=\tilde{e},Y)-H(\widehat{G}_c^*,E=\hat{e}|\widetilde{G}_c,E=\tilde{e},Y), \\
		                                                                & =H(\widehat{G}_s^p|\widetilde{G}_c,E=\tilde{e},Y,\widehat{G}_c^p,E=\hat{e})                                           \\
		                                                                & \qquad -H( \widehat{G}_c^l|\widetilde{G}_c,E=\tilde{e},Y,\widehat{G}_c^p,E=\hat{e}).
	\end{aligned}
\end{equation}
Combing $\Delta H(\widehat{G}_c,E=\hat{e}|Y)$, we have:
\begin{equation}
	\begin{aligned}
		\Delta I(\widehat{G}_c,E=\hat{e};\widetilde{G}_c,E=\tilde{e}|Y) & =\Delta H(\widehat{G}_c,E=\hat{e}|Y)-\Delta H(\widehat{G}_c,E=\hat{e}|\widetilde{G}_c,E=\tilde{e},Y)                                                \\
		                                                                & =\left\{H(\widehat{G}_s^p|\widehat{G}_c^p,E=\hat{e},Y)-H(\widehat{G}_s^p|\widetilde{G}_c,E=\tilde{e},Y,\widehat{G}_c^p,E=\hat{e})\right\}           \\
		                                                                & \qquad +\left\{-H(\widehat{G}_c^l|\widehat{G}_c^p,E=\hat{e},Y)+H( \widehat{G}_c^l|\widetilde{G}_c,E=\tilde{e},Y,\widehat{G}_c^p,E=\hat{e})\right\}, \\
		                                                                & =-H(\widehat{G}_c^l|\widehat{G}_c^p,E=\hat{e},Y)+H(\widehat{G}_c^l|\widetilde{G}_c,E=\tilde{e},Y,\widehat{G}_c^p,E=\hat{e}),
	\end{aligned}
\end{equation}
where the last equality is because of the independence of $\widehat{G}_s^p$ between
$\widetilde{G}_c,E=\tilde{e}$ conditioned on $Y,E=\hat{e}$.
Since conditioning will lower the entropy for both discrete and continuous variables~\citep{elements_info,network_coding},
we have:
\begin{equation}
	\Delta I(\widehat{G}_c,E=\hat{e};\widetilde{G}_c,E=\tilde{e}|Y) <0,
\end{equation}
which implies the existence of $\widehat{G}_s^p$ in $\widehat{G}_c$ will lower down the second term in Eq.~\ref{eq:good_opt_contrast_new_appdx}
for the case (i).

For (ii), we have:
\begin{equation}
	\begin{aligned}
		H(\widehat{G}_c,E=\hat{e}|\widetilde{G}_c,E=\tilde{e},Y) & =H(\widehat{G}_c^p, \widehat{G}_s^p,E=\hat{e}|\widetilde{G}_c^p,\widetilde{G}_s^p,E=\tilde{e},Y) \\
		                                                         & =H( \widehat{G}_s^p|\widetilde{G}_c^p,\widetilde{G}_s^p,E=\tilde{e},Y,\widehat{G}_c^p,E=\hat{e}) \\
		                                                         & \qquad +H(\widehat{G}_c^p,E=\hat{e}|\widetilde{G}_c^p,\widetilde{G}_s^p,E=\tilde{e},Y),          \\
	\end{aligned}
\end{equation}
Similar to (i),
$H( \widehat{G}_s^p|\widetilde{G}_c^p,\widetilde{G}_s^p,E=\tilde{e},Y,\widehat{G}_c^p,E=\hat{e})$
can be cancelled out with $H(\widehat{G}_s^p|\widehat{G}_c^p,E=\hat{e},Y)$.
Then, we have:
\begin{equation}
	\begin{aligned}
		\Delta I(\widehat{G}_c,E=\hat{e};\widetilde{G}_c,E=\tilde{e}|Y) & =\Delta H(\widehat{G}_c,E=\hat{e}|Y)-\Delta H(\widehat{G}_c,E=\hat{e}|\widetilde{G}_c,E=\tilde{e},Y)                                             \\
		                                                                & =-H(\widehat{G}_c^l|\widehat{G}_c^p,E=\hat{e},Y)+H(\widehat{G}_c^l|\widetilde{G}_c^p,\widetilde{G}_s^p,E=\tilde{e},\widehat{G}_c^p,Y,E=\hat{e}).
	\end{aligned}
\end{equation}
Since additionally conditioning on $\widehat{G}_s^p$ in $H(\widehat{G}_c^l,E=\hat{e}|\widetilde{G}_c^p,\widetilde{G}_s^p,E=\tilde{e},Y)$
can not lead to new information about $\widehat{G}_c^l$, we have:
\begin{equation}
	\begin{aligned}
		H(\widehat{G}_c^l|\widetilde{G}_c^p,\widetilde{G}_s^p,E=\tilde{e},\widehat{G}_c^p,Y,E=\hat{e}) & =H(\widehat{G}_c^l|\widetilde{G}_c^p,E=\tilde{e},\widehat{G}_c^p,Y,E=\hat{e}) \\
		                                                                                               & < H(\widehat{G}_c^l|\widehat{G}_c^p,Y,E=\hat{e}),                             \\
	\end{aligned}
\end{equation}
which follows that $\Delta I(\widehat{G}_c,E=\hat{e};\widetilde{G}_c,E=\tilde{e}|Y)<0$.

To summarize, the ground truth $G_c$ is the only maximizer of
the objective (Eq.~\ref{eq:good_opt_contrast_new_appdx}),
hence solving for the objective (Eq.~\ref{eq:good_opt_contrast_new_appdx}) can elicit an invariant GNN.

\subsection{Proof for theorem~\ref{thm:good_inv_gnn_new} (ii)}
\begin{theorem}[$\ginv$v2 Induces Invariant GNNs]
	\label{thm:good_inv_gnn_v3_appdx}
	Given a set of graph datasets $\{\dataset^e\}_e$ %
	and environments $\envall$ that follow the same graph generation process in Sec.~\ref{sec:data_gen},
	assuming that \textup{(a)} $f_\gen^G$ and $f_\gen^{G_c}$ in Assumption~\ref{assump:graph_gen} are invertible,
	\textup{(b)} samples from each training environment are equally distributed,
	i.e.,$|\dataset_{\hat{e}}|=|\dataset_{\tilde{e}}|,\ \forall \hat{e},\tilde{e}\in\envtrain$,
	a GNN $f_c\circ g$ solves Eq.~\ref{eq:good_opt_contrast_v3},
	is an invariant GNN (Def.~\ref{def:inv_gnn}).
\end{theorem}

\textit{Proof.}
We re-write the objective as follows:
\label{proof:good_inv_gnn_v3_appdx}
\begin{equation}
	\label{eq:good_opt_contrast_v3_appdx}
	\begin{aligned}
		\begin{aligned}
			\max_{f_c,g} \  I(\widehat{G}_c;Y)+I(\widehat{G}_s;Y), \ \text{s.t.}\
			 & \widehat{G}_c\in
			\argmax_{\widehat{G}_c=g(G), \widetilde{G}_c=g(\widetilde{G})} I(\widehat{G}_c;\widetilde{G}_c|Y),\
			\\
			 & I(\widehat{G}_s;Y)\leq I(\widehat{G}_c;Y),\ \widehat{G}_s=G-g(G).
		\end{aligned}
	\end{aligned}
\end{equation}
where $\widehat{G}_c=g(G),\widetilde{G}_c=g(\widetilde{G})$ and $\widetilde{G}\sim \sP(G|Y)$,
i.e., $\widetilde{G}$ and $G$ have the same label.

Similar to the proof for Theorem~\ref{thm:good_inv_gnn_new_appdx},
to prove Theorem~\ref{thm:good_inv_gnn_v3_appdx} is essentially
to show the estimated $\widehat{G}_c$ through Eq.~\ref{eq:good_opt_contrast_v3_appdx}
is the underlying $G_c$, hence the minimizer of Eq.~\ref{eq:good_opt_contrast_v3_appdx}
elicits an invariant GNN predictor (Definition.~\ref{def:inv_gnn_appdx}).

In the next, we also begin with a lemma:
\begin{lemma}
	\label{thm:mi_gc_larger_than_gs_appdx}
	Given data generation process as Theorem~\ref{thm:good_inv_gnn_v3_appdx}, for both FIIF and PIIF, we have:
	\[I(C;Y)\geq I(S;Y),\]
	hence $I(G_c;Y)\geq I(G_s;Y)$.
\end{lemma}
\begin{proof}[Proof for Lemma~\ref{thm:mi_gc_larger_than_gs_appdx}]
	\label{proof:mi_gc_larger_than_gs_appdx}
	For both FIIF and PIIF, Assumption~\ref{assump:latent_sep}
	implies that $H(C|Y)\leq H(S|Y)$.
	It follows that $I(C;Y)=H(Y)-H(C|Y)\geq H(Y)-H(S|Y)=I(S;Y)$.
	Then, since $f^{G_c}_\gen:\gC\rightarrow\gG_c$ is invertible,
	we have $I(G_c;Y)=I(C;Y)\geq I(S;Y)\geq I(G_s;Y)$.
\end{proof}
Given Lemma~\ref{thm:mi_gc_larger_than_gs_appdx}, we know $\widehat{G}_c$ at least contains some subset of the underlying $G_c$,
otherwise the constraint $I(\widehat{G}_s;Y)\leq I(\widehat{G}_c;Y)$ will be violated since $G_c\subseteq \widehat{G}_s$ in this case.

Assuming there are some subset of $G_s$ contained in $\widehat{G}_c$,
without loss of generality,
we can divide all of the possible cases about $\widehat{G}_c$ into two:
\begin{enumerate}[label=(\roman*)]
	\item $\widehat{G}_c$ only contains a subset of the underlying $G_c$;
	\item $\widehat{G}_c$ contains a subset of the underlying $G_c$ as well as part of the underlying $G_s$;
\end{enumerate}

\begin{wrapfigure}{r}{0.5\textwidth}
	\vskip -0.15in
	\includegraphics[width=0.5\textwidth]{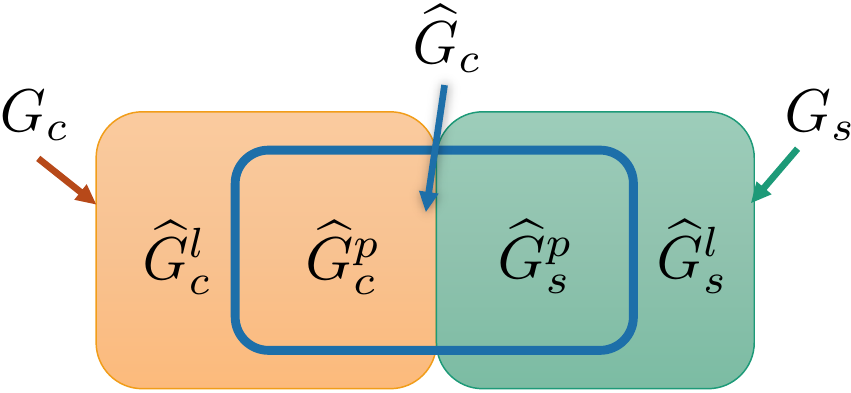}
	\label{fig:good_set_venn_2_appdx}
	\caption{Illustration of the notation for estimated $\widehat{G}_c$ from $G$.
		$G_c$ and $G_s$ are two
		disjoint sets. $\widehat{G}_c$ may contain certain subsets from $G_c$ and $G_s$.
		The subsets from $G_c$ and $G_s$ contained in $\widehat{G}_c$ are denoted as $\widehat{G}_c^p$ and $\widehat{G}_s^p$, respectively.
		While the left subsets in $G_c$ and $G_s$ are denoted as $\widehat{G}_c^l$ and $\widehat{G}_s^l$, respectively.
		Similar notations are also applicable for the estimated $\widetilde{G}_c$ from $\widetilde{G}$.
	}
	\vspace{-0.5in}
\end{wrapfigure}

Before the discussion, let us inherit the notations of subsets of $G_c,G_s$
from the proof for Theorem~\ref{thm:good_inv_gnn_new_appdx}:
Let $\widehat{G}_c^*$ and $\widetilde{G}_c^*$ be the ground truth invariant subgraph $G_c$s
of $\widehat{G}$ and $\widetilde{G}$,
$\widehat{G}_c^l=\widehat{G}_c^*-\widehat{G}_c$
and $\widetilde{G}_c^l=\widetilde{G}_c^*-\widetilde{G}_c$ be the \textbf{l}eft (un-estimated) subsets
from corresponding ground truth $G_c$s,
and $\widehat{G}_c^p=\widehat{G}_c^*-\widehat{G}_c^l$ and
$\widetilde{G}_c^p=\widetilde{G}_c^*-\widetilde{G}_c^l$ be the complement,
or equivalently,
the \textbf{p}artial $\widehat{G}_c^*,\widetilde{G}_c^*$ that are estimated in $\widehat{G}_c,\widetilde{G}_c$, respectively.
Similarly,  $\widehat{G}_s^p,\widetilde{G}_s^p$
are the partial $\widehat{G}_s,\widetilde{G}_s$s contained in the estimated $\widehat{G}_c,\widetilde{G}_c$
while $\widehat{G}_s^l,\widetilde{G}_s^l$ are the left subsets $\widehat{G}_s,\widetilde{G}_s$, respectively.

First of all, case (i) cannot hold because,
when maximizing $I(\widehat{G}_c;\widetilde{G}_c|Y)$,
if $\exists \widehat{G}_c^l=\widehat{G}_c^*-\widehat{G}_c$,
as shown in the proof for Theorem~\ref{thm:good_inv_gnn_new_appdx},
including $\widehat{G}_c^l$ into $\widehat{G}_c$ can always enlarge $I(\widehat{G}_c;\widetilde{G}_c|Y)$,
while not affecting the optimality of $I(\widehat{G}_s;Y)+I(\widehat{G}_c;Y)$
by re-distributing $\widehat{G}_c^l$ from $\widehat{G}_s$ to $\widehat{G}_c$.
Consequently, $\widehat{G}_c^*$ must be included in $\widehat{G}_c$, i.e., $\widehat{G}_c^*\subseteq \widehat{G}_c$.

As for case (ii),
recall that, by the condition of equally distributed training samples from each training environment,
maximizing $I(\widehat{G}_c;\widetilde{G}_c|Y)$ is
essentially maximizing $I(\widehat{G}_c,E=\hat{e};\widetilde{G}_c,E=\tilde{e}|Y)$, $\forall \hat{e},\tilde{e}\in\envtrain$,
hence, we have:
\begin{equation}
	\label{eq:good_opt_contrast_env_appdx_ii}
	\begin{aligned}
		\max_{g,f_c} & \ I(\widehat{G}_c;\widetilde{G}_c|Y),                                                    \\
		             & =I(\widehat{G}_c,E=\hat{e};\widetilde{G}_c,E=\tilde{e}|Y)                                \\
		             & = H(\widehat{G}_c,E=\hat{e}|Y)-H(\widehat{G}_c,E=\hat{e}|\widetilde{G}_c,E=\tilde{e},Y). \\
	\end{aligned}
\end{equation}
We claim Eq.~\ref{eq:good_opt_contrast_env_appdx_ii}
can eliminate any potential
subsets in the estimated $\widehat{G}_c$.
Similarly, we have:
\begin{equation}
	\label{eq:good_opt_contrast_mi_t1_appdx_ii}
	\begin{aligned}
		H(\widehat{G}_c,E=\hat{e}|Y) & =H(E=\hat{e}|\widehat{G}_c,Y)+H(\widehat{G}_c|E=\hat{e},Y)                      \\
		                             & =H(\widehat{G}_c^*\cup \widehat{G}_s^p|E=\hat{e},Y)                             \\
		                             & =H(\widehat{G}_c^*|E=\hat{e},Y) +H(\widehat{G}_s^p|\widehat{G}_c^*,E=\hat{e},Y) \\
		                             & =H(\widehat{G}_c^*|Y) +H(\widehat{G}_s^p|\widehat{G}_c^*,E=\hat{e},Y)           \\
	\end{aligned}
\end{equation}
where the second equality is due to $E=\hat{e}$ is determined.
Compared to the case that $\widehat{G}_c=\widehat{G}_c^*$, we have:
\begin{equation}
	\begin{aligned}
		\Delta H(\widehat{G}_c,E=\hat{e}|Y) & =H(\widehat{G}_c,E=\hat{e}|Y)-H(\widehat{G}_c^*,E=\hat{e}|Y), \\
		                                    & =H(\widehat{G}_s^p|\widehat{G}_c^*,E=\hat{e},Y).
	\end{aligned}
\end{equation}

Then, as for $H(\widehat{G}_c,E=\hat{e}|\widetilde{G}_c,E=\tilde{e},Y)$,
without loss of generality, we can divide all of the possible cases into two:
\begin{enumerate}[label=(\alph*)]
	\item $\widehat{G}_c$ contains some $\widehat{G}_s^p\subseteq \widehat{G}_s$;
	\item Both $\widehat{G}_c$ and $\widetilde{G}_c$ contain some $\widehat{G}_s^p\subseteq \widehat{G}_s$ and $\widetilde{G}_s^p\subseteq \widetilde{G}_s$, respectively.
\end{enumerate}
For (a), we have:
\begin{equation}
	\begin{aligned}
		H(\widehat{G}_c,E=\hat{e}|\widetilde{G}_c,E=\tilde{e},Y) & =H(\widehat{G}_c^*, \widehat{G}_s^p,E=\hat{e}|\widetilde{G}_c,E=\tilde{e},Y)                                                              \\
		                                                         & =H( \widehat{G}_s^p|\widetilde{G}_c,E=\tilde{e},Y,\widehat{G}_c^*,E=\hat{e})+H( \widehat{G}_c^*,E=\hat{e}|\widetilde{G}_c,E=\tilde{e},Y), \\
	\end{aligned}
\end{equation}
Similarly to the proof for Theorem~\ref{thm:good_inv_gnn_new_appdx},
when considering $\Delta I(\widehat{G}_c;\widetilde{G}_c|Y)$, the effects of
$H( \widehat{G}_s^p|\widetilde{G}_c,E=\tilde{e},Y,\widehat{G}_c^*,E=\hat{e})$ is cancelled out by
$H(\widehat{G}_s^p|\widehat{G}_c^*,E=\hat{e},Y)$.
Hence, we have:
\[\Delta I(\widehat{G}_c;\widetilde{G}_c|Y)=0.\]

For (b), we have:
\begin{equation}
	\begin{aligned}
		H(\widehat{G}_c,E=\hat{e}|\widetilde{G}_c,E=\tilde{e},Y) & =H(\widetilde{G}_c^*, \widetilde{G}_s^p,E=\hat{e}|\widetilde{G}_c^*,\widetilde{G}_s^p,E=\tilde{e},Y) \\
		                                                         & =H( \widehat{G}_s^p|\widetilde{G}_c^*,\widetilde{G}_s^p,E=\tilde{e},Y,\widehat{G}_c^*,E=\hat{e})     \\
		                                                         & \qquad+H(\widehat{G}_c^*|\widetilde{G}_c^*,\widetilde{G}_s^p,E=\tilde{e},Y,E=\hat{e}),               \\
	\end{aligned}
\end{equation}
Similarly, $H( \widehat{G}_s^p|\widetilde{G}_c^*,\widetilde{G}_s^p,E=\tilde{e},Y,\widehat{G}_c^*,E=\hat{e})=0$
can also be cancelled out by $H(\widehat{G}_s^p|\widehat{G}_c^*,E=\hat{e},Y)$.
Moreover, for $H(\widehat{G}_c^*|\widetilde{G}_c^*,\widetilde{G}_s^p,E=\tilde{e},Y,E=\hat{e})$,
$\widetilde{G}_s^p$ can not bring no additional information about $\widehat{G}_c^*$,
when conditioning on $\widetilde{G}_c^*,Y,E=\tilde{e}$. Hence, we also have:
\[\Delta I(\widehat{G}_c;\widetilde{G}_c|Y)=0.\]

To summarize, when maximizing $I(\widehat{G}_c;\widetilde{G}_c|Y)$,
including any $\widehat{G}_s^p\subseteq \widehat{G}_s^*$ can not
bring additional benefit while affecting the optimality of $I(\widehat{G}_s;Y)+I(\widehat{G}_c;Y)$.
More specifically, when considering the changes to $I(\widehat{G}_s;Y)+I(\widehat{G}_c;Y)$,
$\forall G_s^p\subseteq G_s$, we have
\[I(G-\widehat{G}_c^*-G_s^p;Y)\leq I(G-\widehat{G}_c^*;Y),\ \forall G_s^p\subseteq G_s,\]
while $I(Y;\widehat{G}_c^*,G_s^p)=I(Y;\widehat{G}_c^*)+I(Y;\widehat{G}_s^p|\widehat{G}_c^*),\ \forall e\in\envtrain$.
Consequently,
\begin{equation}
	\begin{aligned}
		\Delta I(\widehat{G}_s;Y)+I(\widehat{G}_c;Y) & = -I(\widehat{G}_s^p;Y|\widehat{G}_s^l)+I(\widehat{G}_s^p;Y|\widehat{G}_c^*) \\
		                                             & =-I(\widehat{G}_s^p;Y)+I(\widehat{G}_s^p;Y|\widehat{G}_c^*)\leq 0.           \\
	\end{aligned}
\end{equation}

Hence, only the underlying $G_c$ is the solution to Eq.~\ref{eq:good_opt_contrast_v3_appdx},
which implies that solving for the objective (Eq.~\ref{eq:good_opt_contrast_v3_appdx})
can elicit an invariant GNN.

\section{Details of Prototypical $\ginv$ Implementation}
\label{sec:good_impl_appdx}
In fact, the $\ginv$ framework introduced in Sec.~\ref{sec:good_framework} can have multiple implementations.
We choose interpretable architectures in our experiments for the purpose of concept verification.
More sophisticated architectures can be incorporated.
Experimental results in Sec.~\ref{sec:exp} also demonstrates that, even equipped with basic GNN architectures,
$\ginv$ already has the excellent OOD generalization ability, hence
it is promising to incorporate more advanced architectures from the prosperous GNN literature.

We now introduce the details of the architectures used in our experiments. Recall that $\ginv$ decomposes a GNN model for graph classification into two modules,
i.e., a featurizer: $g:\gG\rightarrow\gG_c$ and a classifier $f_c:\gG_c\rightarrow\gY$.
Specifically, for the implementation of Featurizer, we choose one of the common practices GAE~\citep{gae} for calculating the sampled weights for each edge. More formally, the soft mask is predicted through the following equation:
\begin{equation}\label{eq:gae_appdx}\nonumber
	Z=\text{GNN}(G)\in\R^{n\times h},\ M=\sigma(ZZ^T)\in\R^{n\times n}.
\end{equation}

\begin{wrapfigure}{r}{0.5\textwidth}
	\includegraphics[width=0.5\textwidth]{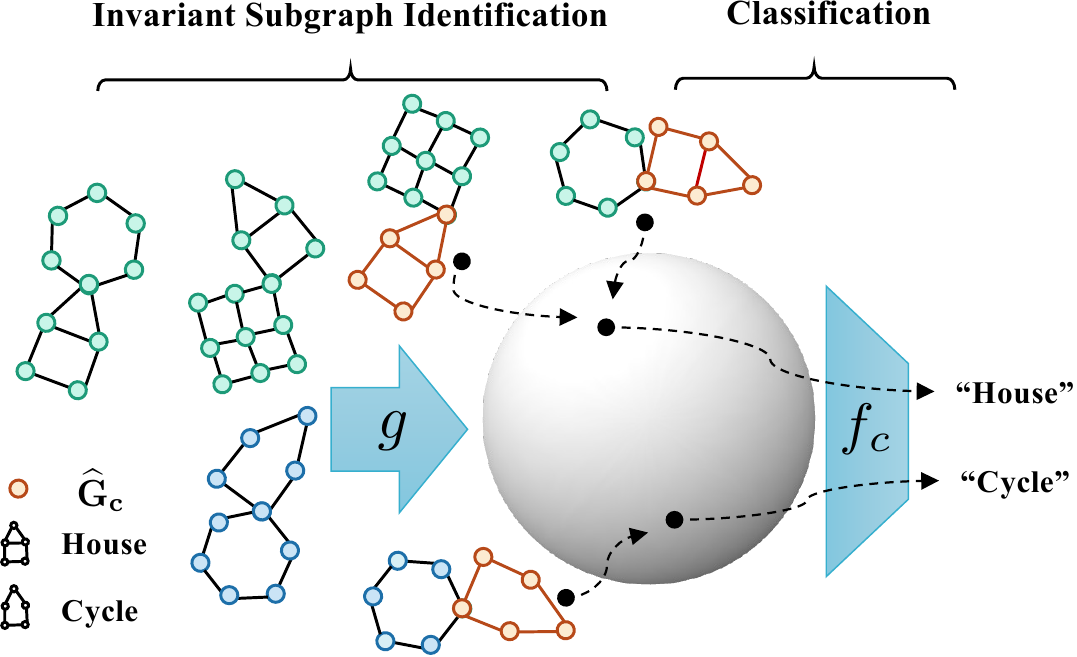}
	\label{fig:motivation_appdx}
	\caption{Illustration of $\fullginv$ ($\ginv$):
		GNNs need to classify graphs based on the specific motif (``House'' or ``Cycle'').
		The featurizer $g$ will extract an (orange colored) subgraph $\widehat{G}_c$ from each input
		for the classifier $f_c$ to predict the label.
		The training objective of $g$ is implemented in a contrastive strategy where
		the distribution of $\widehat{G}_c$ at the latent sphere
		will be optimized to maximize the intra-class mutual information.
		With the identified invariant subgraph $G_c$,
		the predictions made by classifier $f_c$ based on $G_c$ are invariant to distribution shifts;
	}
\end{wrapfigure}

If a sampling ratio $s_c$ is predetermined, we sample $s_c$ of total edges with the largest predicted weights as a soft estimation of $\widehat{G}_c$.
Then, the estimated $\widehat{G}_c$ will be forwarded to the classifier $f_c$ for predicting the labels of the original graph.
Although Theorem~\ref{thm:good_inv_gnn_new_appdx} assumes $s_c$ is known, in real applications we do not know the specific $s_c$.
Hence, in experiments, we select $s_c$ according to the validation performance.
To thoroughly study the effects of $I(\widehat{G}_s;Y)$ comparing to $\ginv$v1, we stick to using the same $s_c$ and sampling process for $\ginv$v2,
while  $\ginv$v2 essentially requires less specific knowledge about ground truth $r_c$ hence achieving better empirical performance.
Moreover, once the sampled edges are determined, the classifier GNN can take either the original feature of the input graph
or the learned feature from the featurizer as the new node attributes for $\widehat{G}_c$.
We select the architecture according to the validation performance from some random runs.

For the implementation of the information theoretic objectives,
we will use $\ginv$v2 for elaboration while the implementation of $\ginv$v1 can be obtained via removing the third term from $\ginv$v2.
Recall that $\ginv$v2 has the following formulation:
\begin{equation}
	\label{eq:good_opt_contrast_v3_appdx2}
	\begin{aligned}
		\max_{f_c,g} \  I(\widehat{G}_c;Y)+I(\widehat{G}_s;Y), \ \text{s.t.}\
		 & \widehat{G}_c\in
		\argmax_{\widehat{G}_c=g(G), \widetilde{G}_c=g(\widetilde{G})} I(\widehat{G}_c;\widetilde{G}_c|Y),\
		\\
		 & I(\widehat{G}_s;Y)\leq I(\widehat{G}_c;Y),\ \widehat{G}_s=G-g(G).
	\end{aligned}
\end{equation}
where $\widehat{G}_c=g(G),\widetilde{G}_c=g(\widetilde{G})$ and $\widetilde{G}\sim P(G|Y)$, i.e., $\widetilde{G}$ and $G$ have the same label.
In Sec.~\ref{sec:good_theory}, we introduce a contrastive approximation for $I(\widehat{G}_c;\widetilde{G}_c|Y)$:
\begin{equation} \label{eq:good_opt_contrast_appdx2}
	I(\widehat{G}_c;\widetilde{G}_c|Y) \approx
	\mathbb{E}_{
	\substack{
	\{\widehat{G}_c,\widetilde{G}_c\} \sim \sP_g(G|\gY=Y)\\\
	\{G^i_c\}_{i=1}^{M} \sim \sP_g(G|\gY \neq Y)
	}
	}
	\log\frac{e^{\phi(h_{\widehat{G}_c},h_{\widetilde{G}_c})}}
	{e^{\phi(h_{\widehat{G}_c},h_{\widetilde{G}_c})} +
		\sum_{i}^M e^{\phi(h_{\widehat{G}_c}h_{G^i_c})}},
\end{equation}
where positive samples $(\widehat{G}_c,\widetilde{G}_c)$ are the extracted subgraphs of graphs that have the same label of $G$,
negative samples are those with different labels, $\sP_g(G|\gY=Y)$ is the pushforward distribution of $\sP(G|\gY=Y)$ by featurizer $g$,
$\sP(G|\gY=Y)$ refers to the distribution of $G$ given the label $Y$,
$h_{\widehat{G}_c},h_{\widetilde{G}_c},h_{G^i_c}$ are the graph presentations of the estimated subgraphs,
and $\phi$ is the similarity metric for the graph presentations.
As $M\rightarrow \infty$, Eq.~\ref{eq:good_opt_contrast_appdx2} approximates $I(\widehat{G}_c;\widetilde{G}_c|Y)$ which can be regarded as a
non-parameteric resubstitution entropy estimator via the von Mises-Fisher
kernel density~\citep{feat_dist_entropy,vMF_entropy,align_uniform}.

While for the third term $I(\widehat{G}_s;Y)$ and the constraint $I(\widehat{G}_s;Y)\leq I(\widehat{G}_c;Y)$,
a straightforward implementation is to imitate the hinge loss:
\begin{equation} \label{eq:good_opt_hinge_appdx2}
	I(\widehat{G}_s;Y)\approx \frac{1}{N} R_{\widehat{G}_s}\cdot \mathbb{I}(R_{\widehat{G}_s}\leq R_{\widehat{G}_c}),
\end{equation}
where $N$ is the number of samples, $\mathbb{I}$ is a indicator function that outputs $1$ when the
interior condition is satisfied otherwise $0$, and
$R_{\widehat{G}_s}$ and $R_{\widehat{G}_c}$ are the empirical risk vector of the predictions
for each sample based on $\widehat{G}_s$ and $\widehat{G}_c$ respectively.
One can also formulate Eq.~\ref{eq:good_opt_contrast_v3_appdx2} from game-theoretic perspective~\citep{inv_rat}.

Finally, we can derive the specific loss for the optimization of $\ginv$v2 combining Eq.~\ref{eq:good_opt_contrast_appdx2} and Eq.~\ref{eq:good_opt_hinge_appdx2}:
\begin{equation} \label{eq:good_opt_loss_impl_appdx}
	\begin{aligned}
		 & R_{\widehat{G}_c} +\alpha
		\mathbb{E}_{
		\substack{
		\{\widehat{G}_c,\widetilde{G}_c\} \sim \sP_g(G|\gY=Y)                                             \\\
		\{G^i_c\}_{i=1}^{M} \sim \sP_g(G|\gY \neq Y)
		}
		}
		\log\frac{e^{\phi(h_{\widehat{G}_c},h_{\widetilde{G}_c})}}
		{e^{\phi(h_{\widehat{G}_c},h_{\widetilde{G}_c})} +
		\sum_{i}^M e^{\phi(h_{\widehat{G}_c}h_{G^i_c})}}                                                  \\
		 & +\beta \frac{1}{N} R_{\widehat{G}_s}\cdot \mathbb{I}(R_{\widehat{G}_c}\leq R_{\widehat{G}_s}), \\
	\end{aligned}
\end{equation}
where $R_{\widehat{G}_c},R_{\widehat{G}_s}$ are the empirical risk when using $\widehat{G}_c,\widehat{G}_s$ to predict $Y$ through the classifier.
Typically, we use a additional MLP downstream classifier $\rho_s$ for $\widehat{G}_s$ in the classifier GNN.
$h_{\widehat{G}_c}$ is the graph representation of $\widehat{G}_c$ which can be induced from the GNN encoder either in the featurizer or in the classifier.
$\alpha,\beta$ are the weights for $I(\widehat{G}_c;\widetilde{G}_c|Y)$ and $I(\widehat{G}_s;Y)$, and $\phi$ is implemented as cosine similarity.
The optimization loss for $\ginv$v1 merely contains the first two terms in Eq.~\ref{eq:good_opt_loss_impl_appdx}.

The detailed algorithm for $\ginv$ is given in the Algorithm~\ref{alg:good_appdx}, assuming the $h_{\widehat{G}_c}$ is obtained via the graph encoder in $f_c$. Fig.~\ref{fig:motivation_appdx} also shows a illustration of the working procedure of $\ginv$.

\begin{algorithm}[tb]
	\caption{Pseudo code for $\ginv$ framework.}
	\label{alg:good_appdx}
	\begin{algorithmic}
		\STATE {\bfseries Input:} Training graphs and labels $\train=\{G_i,Y_i\}_{i=1}^N$; learning rate $l$; loss weights $\alpha,\beta$ required by Eq.~\ref{eq:good_opt_loss_impl_appdx}; number of training epochs $e$; batch size $b$;
		\STATE Randomly initialize parameters of $g,f_c$;
		\FOR{$i=1$ {\bfseries to} $e$}
		\STATE Sample a batch of graphs $\{G^j,Y^j\}_{j=1}^b$;
		\STATE Estimate the invariant subgraph for the batch: $\{\widehat{G}_c^j\}_{j=1}^b=g(\{G^j,Y^j\}_{j=1}^b)$;
		\STATE Make predictions based the estimated invariant subgraph: $\{\widehat{Y}^j\}_{j=1}^b=f_c(\{\widehat{G}_c^j\}_{j=1}^b)$;
		\STATE Calculate the empirical loss $R_{\widehat{G}_c}$ with $\{\widehat{Y}^j\}_{j=1}^b$;
		\STATE Fetch the graph representations of invariant subgraphs from $f_c$ as $\{h_{\widehat{G}_c^j}\}_{j=1}^b$;
		\STATE Calculate the contrastive loss $R_c$ with Eq.~\ref{eq:good_opt_contrast_appdx2}, where positive samples and negative samples are constructed from the batch;
		\STATE Obtain $\widehat{G}_s$
		for the batch: $\{\widehat{G}_c^j\}_{j=1}^b=\{G^j-\widehat{G}_c^j\}_{j=1}^b$;
		\STATE Make predictions based on the $\widehat{G}_s$: $\{\widehat{Y}_s^j\}_{j=1}^b=f_c(\{\widehat{G}_c^j\}_{j=1}^b)$;
		\STATE Calculate the empirical loss $R_{\widehat{G}_s}$ with $\{\widehat{Y}_s^j\}_{j=1}^b$, and weighted as Eq.~\ref{eq:good_opt_hinge_appdx2};
		\STATE Update parameters of $g,f_c$ with respect to $R_{\widehat{G}_c}+\alpha R_c+\beta R_{\widehat{G}_s}$ as Eq.~\ref{eq:good_opt_loss_impl_appdx};
		\ENDFOR
	\end{algorithmic}
\end{algorithm}

\section{Detailed Experimental Settings}
\label{sec:exp_appdx}
In this section, we provide more details about our experimental settings in Sec.~\ref{sec:exp}, including the dataset preparation, dataset statistics, implementations of baselines, selection of models and hyperparameters as well as evaluation protocols.

\bgroup
\def\arraystretch{1.2}
\begin{table}[ht]
	\centering
	\caption{Information about the datasets used in experiments. The number of nodes and edges are taking average among all graphs. MCC indicates the Matthews correlation coefficient.}\small\sc
	\label{tab:datasets_stats_appdx}
	\resizebox{\textwidth}{!}{
		\begin{small}
			\begin{tabular}{lccccccc}
				\toprule
				\textbf{Datasets} & \textbf{\# Training} & \textbf{\# Validation} & \textbf{\# Testing} & \textbf{\# Classes} & \textbf{ \# Nodes} & \textbf{ \# Edges}
				                  & \textbf{  Metrics}                                                                                                                            \\\midrule
				SPMotif           & $9,000$              & $3,000$                & $3,000$             & $3$                 & $44.96$            & $65.67$            & ACC     \\
				PROTEINS          & $511$                & $56$                   & $112$               & $2$                 & $39.06$            & $145.63$           & MCC     \\
				DD                & $533$                & $59$                   & $118$               & 2                   & $284.32$           & $1,431.32$         & MCC     \\
				NCI1              & $1,942$              & $215$                  & $412$               & $2$                 & $29.87$            & $64.6$             & MCC     \\
				NCI109            & $1,872$              & $207$                  & $421$               & $2$                 & $29.68$            & $64.26$            & MCC     \\
				SST5              & $6,090$              & $1,186$                & $2,240$             & $5$                 & $19.85$            & $37.70$            & ACC     \\
				Twitter           & $3,238$              & $694$                  & $1,509$             & $3$                 & $21.10$            & $40.20$            & ACC     \\
				CMNIST-sp         & $40,000$             & $5,000$                & $15,000$            & $2$                 & $56.90$            & $373.85$           & ACC     \\
				DrugOOD-Assay     & $34,179$             & $19,028$               & $19,032$            & $2$                 & $32.27$            & $70.25$            & ROC-AUC \\
				DrugOOD-Scaffold  & $21,519$             & $19,041$               & $19,048$            & $2$                 & $29.95$            & $64.86$            & ROC-AUC \\
				DrugOOD-Size      & $36,597$             & $17,660$               & $16,415$            & $2$                 & $30.73$            & $66.90$            & ROC-AUC \\
				\bottomrule
			\end{tabular}	\end{small}}
\end{table}
\egroup

\begin{table*}[ht]
	\caption{Detailed statistics of selected TU datasets. Table from~\citet{size_gen1,size_gen2}.}
	\label{tab:datasets_stats_tu_appdx}
	\begin{small}
		\begin{sc}
			\begin{center}
				\resizebox{\textwidth}{!}{
					\centering
					\begin{tabular}{|l|r|r|r|r|r|r|}
						\cline{2-7}
						\multicolumn{1}{c|}{}   & \multicolumn{3}{c|}{\textbf{NCI1}} & \multicolumn{3}{c|}{\textbf{NCI109}}                                                                                                                          \\
						\cline{2-7}
						\multicolumn{1}{c|}{}   & \textbf{all}                       & \textbf{Smallest} $\mathbf{50\%}$    & \textbf{Largest $\mathbf{10\%}$} & \textbf{all} & \textbf{Smallest} $\mathbf{50\%}$ & \textbf{Largest $\mathbf{10\%}$} \\
						\hline
						\textbf{Class A}        & $49.95\%$                          & $62.30\%$                            & $19.17\%$                        & $49.62\%$    & $62.04\%$                         & $21.37\%$                        \\
						\hline
						\textbf{Class B}        & $50.04\%$                          & $37.69\%$                            & $80.82\%$                        & $50.37\%$    & $37.95\%$                         & $78.62\%$                        \\
						\hline
						\textbf{Num of graphs}  & 4110                               & 2157                                 & 412                              & 4127         & 2079                              & 421                              \\
						\hline
						\textbf{Avg graph size} & 29                                 & 20                                   & 61                               & 29           & 20                                & 61                               \\
						\hline
					\end{tabular}
				}

				\bigskip

				\resizebox{\textwidth}{!}{
					\centering
					\begin{tabular}{|l|r|r|r|r|r|r|}
						\cline{2-7}
						\multicolumn{1}{c|}{}   & \multicolumn{3}{c|}{\textbf{PROTEINS}} & \multicolumn{3}{c|}{\textbf{DD}}                                                                                                                           \\
						\cline{2-7}
						\multicolumn{1}{c|}{}   & \textbf{all}                           & \textbf{Smallest} $\mathbf{50\%}$ & \textbf{Largest $\mathbf{10\%}$} & \textbf{all} & \textbf{Smallest} $\mathbf{50\%}$ & \textbf{Largest $\mathbf{10\%}$} \\
						\hline
						\textbf{Class A}        & $59.56\%$                              & $41.97\%$                         & $90.17\%$                        & $58.65\%$    & $35.47\%$                         & $79.66\%$                        \\
						\hline
						\textbf{Class B}        & $40.43\%$                              & $58.02\%$                         & $9.82\%$                         & $41.34\%$    & $64.52\%$                         & $20.33\%$                        \\
						\hline
						\textbf{Num of graphs}  & 1113                                   & 567                               & 112                              & 1178         & 592                               & 118                              \\
						\hline
						\textbf{Avg graph size} & 39                                     & 15                                & 138                              & 284          & 144                               & 746                              \\
						\hline
					\end{tabular}
				}
			\end{center}
		\end{sc}
	\end{small}
\end{table*}

\subsection{Details about the datasets}
\label{sec:exp_data_appdx}
We provide more details about the motivation and construction method of the datasets that are used in our experiments. Statistics of the datasets are presented in Table~\ref{tab:datasets_stats_appdx}.

\textbf{SPMotif datasets.} We construct 3-class synthetic datasets based on BAMotif~\citep{gnn_explainer,pge} following~\cite{dir},
where the model needs to tell which one of three motifs (House, Cycle, Crane) that the graph contains.
For each dataset, we generate $3000$ graphs for each class at the training set, $1000$ graphs for each class at the validation set and testing set, respectively.
During the construction, we merely inject the distribution shifts in the training data while keep the testing data and validation data without the biases.
For structure-level shifts (\textbf{SPMotif-Struc}), we introduce the bias based on FIIF, where the motif and one of the three base graphs (Tree, Ladder, Wheel) are artificially (spuriously) correlated with a probability of various biases, and equally correlated with the other two. Specifically, given a predefined bias $b$, the probability of a specific motif (e.g., House) and a specific base graph (Tree) will co-occur is $b$ while for the others is $(1-b)/2$ (e.g., House-Ladder, House-Wheel). We use random node features for SPMotif-Struc, in order to study the influences of structure level shifts.
Moreover, to simulate more realistic scenarios where both structure level and topology level have distribution shifts, we also construct \textbf{SPMotif-Mixed} for mixed distribution shifts.
We additionally introduced FIIF attribute-level shifts based on SPMotif-Struc, where all of the node features are spuriously correlated with a probability of various biases by setting to the same number of corresponding labels. Specifically, given a predefined bias $b$, the probability that all of the node features of a graph has label $y$ (e.g., $y=0$) being set to $y$ (e.g., $\mX=\mathbf{0}$) is $b$ while for the others is $(1-b)/2$ (e.g., $P(\mX=\mathbf{1})=P(\mX=\mathbf{2})=(1-b)/2$). More complex distribution shift mixes can be studied following our construction approach, which we will leave for future works.

\textbf{TU datasets.} To study the effects of graph sizes shifts, we follow~\citet{size_gen1,size_gen2} to study the OOD generalization abilities of various methods on four of TU datasets~\citep{tudataset}, i.e., \textbf{PROTEINS, DD, NCI1, NCI109}. Specifically, we use the data splits generated by~\citet{size_gen1} and use the Matthews correlation coefficient  as evaluation metric following~\cite{size_gen2} due to the class imbalance in the splits. The splits are generated as follows: Graphs with sizes smaller than the $50$-th percentile are assigned to training, while graphs with sizes larger than the $90$-th percentile are assigned to test. A validation set for hyperparameters tuning consists of $10\%$ held out examples from training. We also provide a detailed statistics about these datasets in table~\ref{tab:datasets_stats_tu_appdx}.

\textbf{Graph-SST datasets.} Inspired by the data splits generation for studying distribution shifts on graph sizes, we split the data curated from sentiment graph data~\citep{xgnn_tax}, that converts sentiment sentence classification datasets \textbf{SST5} and \textbf{SST-Twitter}~\citep{sst25,sst_twitter} into graphs, where node features are generated using BERT~\citep{bert} and the edges are parsed by a Biaffine parser~\citep{biaffine}. Our splits are created according to the averaged degrees of each graph. Specifically, we assign the graphs as follows: Those that have smaller or equal than $50$-th percentile averaged degree are assigned into training, those that have averaged degree large than $50$-th percentile while smaller than $80$-th percentile are assigned to validation set, and the left are assigned to test set. For SST5 we follow the above process while for Twitter we conduct the above split in an inversed order to study the OOD generalization ability of GNNs trained on large degree graphs to small degree graphs.

\textbf{CMNIST-sp.} To study the effects of PIIF shifts, we select the ColoredMnist dataset created in IRM~\citep{irmv1}. We convert the ColoredMnist into graphs using super pixel algorithm introduced by~\citet{understand_att}. Specifically, the original Mnist dataset are assigned to binary labels where images with digits $0-4$ are assigned to $y=0$ and those with digits $5-9$ are assigned to $y=1$. Then, $y$ will be flipped with a probability of $0.25$. Thirdly, green and red colors will be respectively assigned to images with labels $0$ and $1$ an averaged probability of $0.15$ (since we do not have environment splits) for the training data. While for the validation and testing data the probability is flipped to $0.9$.

\textbf{DrugOOD datasets.} To evaluate the OOD performance in realistic scenarios with realistic distribution shifts, we also include three datasets from DrugOOD benchmark.
DrugOOD is a systematic OOD benchmark for AI-aided drug discovery, focusing on the task of drug target binding affinity prediction for both macromolecule (protein target) and small-molecule (drug compound).
The molecule data and the notations are curated from realistic ChEMBL database~\citep{chembl}.
Complicated distribution shifts can happen on different assays, scaffolds and molecule sizes.
In particular, we select \texttt{DrugOOD-lbap-core-ic50-assay}, \texttt{DrugOOD-lbap-core-ic50-scaffold}, and \texttt{DrugOOD-lbap-core-ic50-size},
from the task of Ligand Based Afﬁnity Prediction which uses \texttt{ic50} measurement type and contains \texttt{core} level annotation noises.
For more details, we refer interested readers to~\citet{drugood}.

\subsection{Training and Optimization in Experiments}
\label{sec:exp_impl_appdx}
During the experiments, we do not tune the hyperparameters exhaustively while following the common recipes for optimizing GNNs.
Details are as follows.

\textbf{GNN encoder.} For fair comparison, we use the same GNN architecture as graph encoders for all methods.
By default, we use $3$-layer GNN with Batch Normalization~\citep{batch_norm} between layers and JK residual connections at last layer~\citep{jknet}.
For the architectures we use the GCN with mean readout~\citep{gcn} for all datasets except Proteins where we empirically observe better validation performance
with a GIN and max readout~\citep{gin}, and for DrugOOD  datasets where we follow the backbone used in the paper~\citep{drugood}, i.e., $4$-layer GIN with sum readout.
The hidden dimensions are fixed as $32$ for SPMotif, TU datasets, CMNIST-sp, and $128$ for SST5, Twitter and DrugOOD datasets.

\textbf{Optimization and model selection.}
By default, we use Adam optimizer~\citep{adam} with a learning rate of $1e-3$ and a batch size of $32$ for all models at all datasets.
Except for DrugOOD datasets, we use a batch size of $128$ following the original paper~\citep{drugood}.
To avoid underfitting, we pretrain models for $20$ epochs for all datasets, except
for CMNIST and Twitter where we pretrain $5$ epochs and for SST5 we pretrain $10$ epochs, because of the dataset size and the difficulty of the task.
To avoid overfitting, we also employ an early stopping of $5$ epochs according to the validation performance.
Meanwhile, dropout~\citep{dropout} is also adopted for some datasets.
Specifically, we use a dropout rate of $0.5$ for CMNIST, SST5, Twitter, DrugOOD-Assay and DurgOOD-Scaffold,
$0.1$ for DrugOOD-Size according to the validation performance,
and $0.3$ for TU datasets following the practice of~\citet{size_gen2}.

\textbf{Implementations of baselines.}
For implementations of the interpretable GNNs, we use the author released codes~\citep{gib,asap}, where we use the codes provided by the authors\footnote{\url{https://anonymous.4open.science/r/DIR/}} for DIR~c\citep{dir}
which is the same as the author released codes.
During the implementation, we use the same $s_c$ for all interpretable GNN baselines, chosen from $\{0.1,0.2,0.25,0.3,0.4,0.5,0.6,0.7,0.8,0.9\}$ according to the validation performances,
and set to $0.25$ for SPMotif following~\citet{dir}, $0.3$ for Proteins and DD,
$0.6$ for NCI1, $0.7$ for NCI109, $0.8$ for CMNIST-sp, $0.5$ for SST5 and Twitter, and $0.8$ for DrugOOD datasets, respectively.
Empirically, we observe that the optimization process in GIB can be unstable during its nested optimization for approximating  the mutual information of the predicted subgraph and the input graph.
We use a larger batch size of $128$ or reduce the nested optimization steps to be lower than $20$ for stabilizing the performance.
If the optimization failed due to the instability during training, we will select the results with best validation accuracy as the final outcomes.
Although SPMotif-Struc is also evaluated in DIR, we find the results are inconsistent to the results reported by the author, because
DIR adopts \texttt{Last Epoch Model Selection} which is \emph{different} from the claim that they select models according to \texttt{the validation performance},
i.e., \texttt{line $264$} to \texttt{line $278$} in \texttt{train/spmotif\_dir.py} from the commit \texttt{4b975f9b3962e7820d8449eb4abbb4cc30c1025d} of \url{https://github.com/Wuyxin/DIR-GNN}.
We select the hyperparamter for the proposed DIR regularization from $\{0.01,0.1,1,10\}$ according to the validation performances at the datasets, while we stick to the authors claimed hyperparameters for the datasets they also experimented with.

For invariant learning,
we refer to the implementations in DomainBed~\citep{domainbed} for IRM~\citep{irmv1}, V-Rex~\citep{v-rex} and IB-IRM~\citep{ib-irm}.
Since the environment information is not available, we perform random partitions on the training data to obtain two equally large environments for these objectives.
Moreover, we select the weights for the corresponding regularization from $\{0.01,0.1,1,10,100\}$ for these objectives according to the validation performances of IRM and stick to it for others,
since we empirically observe that they perform similarly with respect to the regularization weight choice.
For EIIL~\citep{env_inference}, we use the author released implementations about assigning different samples the weights for being put in each environment and calculating the IRM loss.

Besides, for CNC~\citep{cnc}, we follow the algorithm description to modify the sampling strategy in supervised contrastive loss~\citep{sup_contrastive} based on a pretrained GNN optimized with ERM, and choose the weight for contrastive loss using the same grid search as for $\ginv$.

\textbf{Implementations of $\ginv$.}
For fair comparison, $\ginv$ uses the same GNN architecture for GNN encoders as the baseline methods.
We did not do exhaustive hyperparameters tuning for the loss Eq.~\ref{eq:good_opt_loss_impl_appdx}.
By default, we fix the temperature to be $1$ in the contrastive loss,
and merely search $\alpha$ from $\{0.5,1,2,4,8,16,32\}$ and $\beta$ from $\{0.5,1,2,4\}$ according to the validation performances.
For CMNIST-sp, we find larger $\beta$ are required to get rid of intense spurious node features hence we expand the search range for $\beta$ to $\{0.5,1,2,4,16,32\}$,
For Graph-SST datasets, we search $\alpha$ from $\{0.5,1,2,4\}$ as we empirically find that increasing $\alpha$ does not help increase the performance with few random runs.
Besides, we also have various implementation options for obtaining the features in $\widehat{G}_c$, for obtaining $h_{\widehat{G}_c}$, as well as for obtaining predictions based on $\widehat{G}_s$.
By default, we feed the graph representations of featurizer GNN to the classifier GNN, as well as to the contrastive loss.
For classifying $G$ based on $\widehat{G}_s$, we use a separate MLP downstream classifier in the classifier GNN $f_c$.
The only exception is for the CMNIST-sp dataset where the spurious correlation is stronger than the invariant signal.
Directly feeding the graph representations from the featurizer GNN can easily overfit to the shortcuts hence we instead feed the original features to the downstream classifier GNN.
There  can be more other options, such as using separate graph convolutions on $\widehat{G}_s$ or $\widehat{G}_c$, which we  leave for future work.

\textbf{Evaluation protocol.} We run each experiment $10$ on TU datasets and $5$ times for others where the random seeds start from $1$ to the number of total repeated times.
During each run, we select the model according to the validation performance and report the mean and standard deviation of the corresponding metrics.

\subsection{Software and Hardware}
\label{sec:exp_software_appdx}
We implement our methods with PyTorch~\citep{pytorch} and PyTorch Geometric~\citep{pytorch_geometric}.
We ran our experiments on Linux Servers with 40 cores Intel(R) Xeon(R) Silver 4114 CPU @ 2.20GHz, 256 GB Memory, and Ubuntu 18.04 LTS installed.
GPU environments are varied from 4 NVIDIA RTX 2080Ti graphics cards with CUDA 10.2, 2 NVIDIA RTX 2080Ti and 2 NVIDIA RTX 3090Ti graphics cards with CUDA 11.3, and NVIDIA TITAN series with CUDA 11.3.
\revision{
	\subsection{Additional Analysis}
	\label{sec:additional_exp_appdx}
}

\revision{
	\textbf{Hyperparameter sensitivity analysis.} To examine how sensitive $\ginv$ is to the hyperparamters $\alpha$ and $\beta$ for contrastive loss and hinge loss, respectively, under different distribution shifts. We conduct experiments based on the hardest datasets from each table (i.e., SPMotif-Mixed with the bias of $0.9$, DrugOOD-Scaffold and the NCI109 datasets from Table~\ref{tab:sythetic}, Table~\ref{table:other_graph}, and Table~\ref{table:graph_size}, respectively.) To increase the difficulty, we search for more fine-grained spaces for both parameters, i.e., $\{0.1,0.5,1,2,3,4,5,6,7,8\}$. During changing the value of $\beta$, we will fix the $\alpha$ to a specific value under which the model has a relatively good performance (but not the best, to fully examine the robustness of $\ginv$ in practice). During the sensitivity tests, we follow the evaluation protocol as that used for the main experiments. The results are shown in Fig.~\ref{fig:hp_sen_alpha_appdx} and Fig.~\ref{fig:hp_sen_beta_appdx}.
}

\begin{figure}[H]
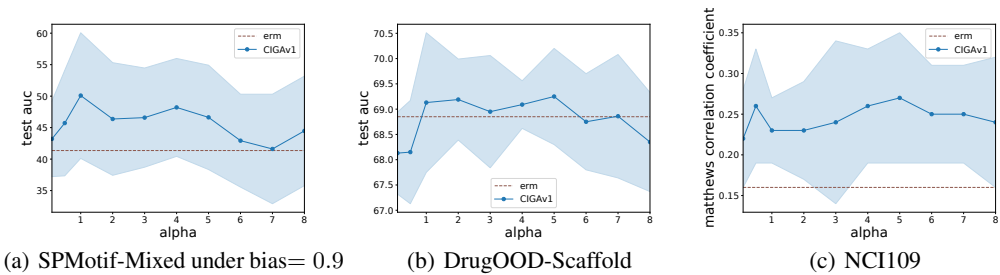

	\centering
	\subfigure[SPMotif-Mixed under bias$=0.9$]{
		\includegraphics[width=0.31\textwidth]{figures/spmotif_ab_contrast.pdf}
	}
	\subfigure[DrugOOD-Scaffold]{
		\includegraphics[width=0.31\textwidth]{figures/drugood_scaffold_ab_contrast.pdf}
	}
	\subfigure[NCI109]{
		\includegraphics[width=0.31\textwidth]{figures/nci_ab_contrast.pdf}
	}
	\caption{
		Hyperparameter sensitivity analysis on the coefficient of contrastive loss ($\alpha$).}
	\label{fig:hp_sen_alpha_appdx}
\end{figure}

\begin{figure}[H]
	\centering
	\subfigure[SPMotif-Mixed under bias$=0.9$ with $\alpha=4$]{
		\includegraphics[width=0.31\textwidth]{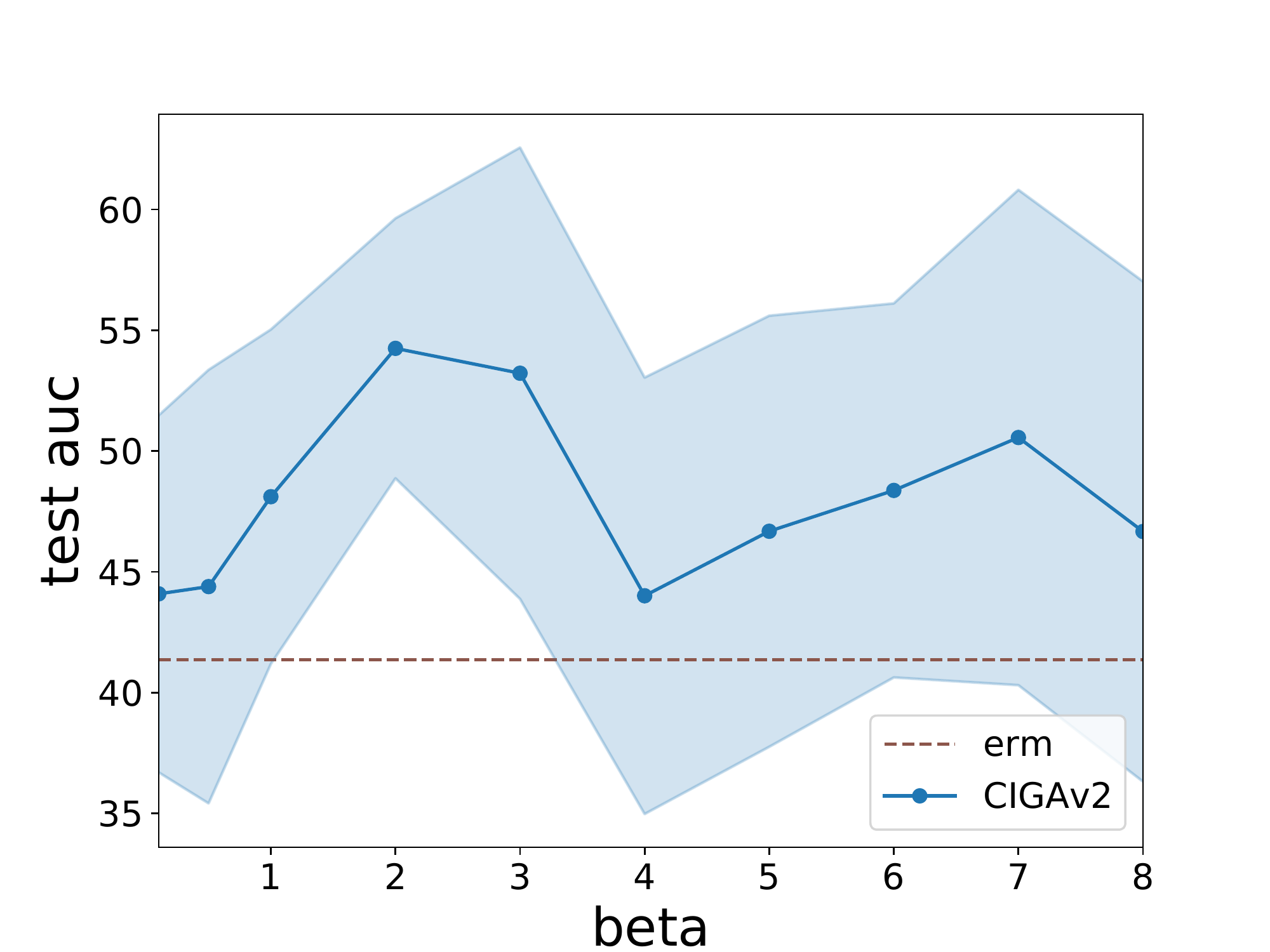}
	}
	\subfigure[DrugOOD-Scaffold with $\alpha=1$]{
		\includegraphics[width=0.31\textwidth]{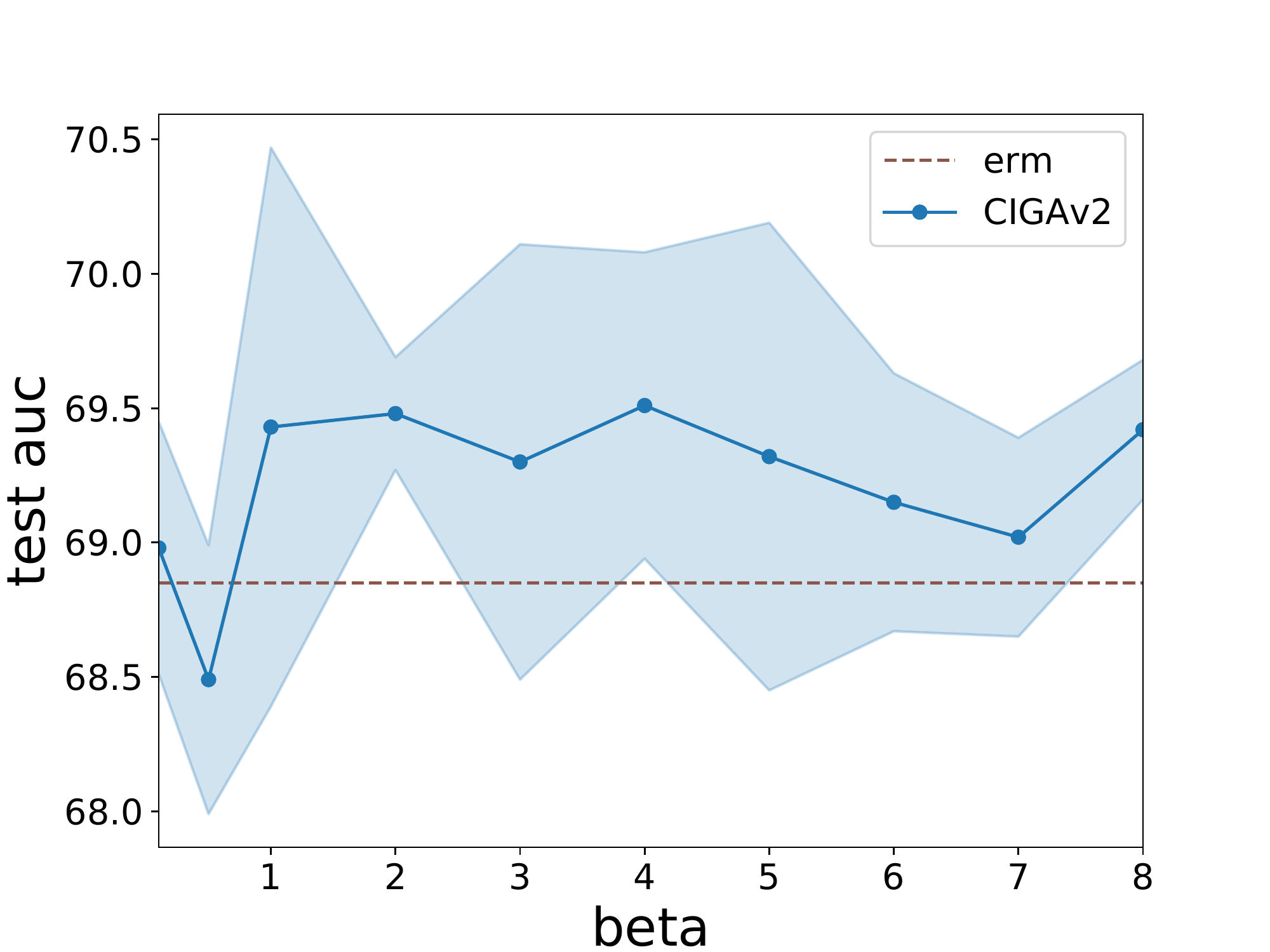}
	}
	\subfigure[NCI109 with $\alpha=1$]{
		\includegraphics[width=0.31\textwidth]{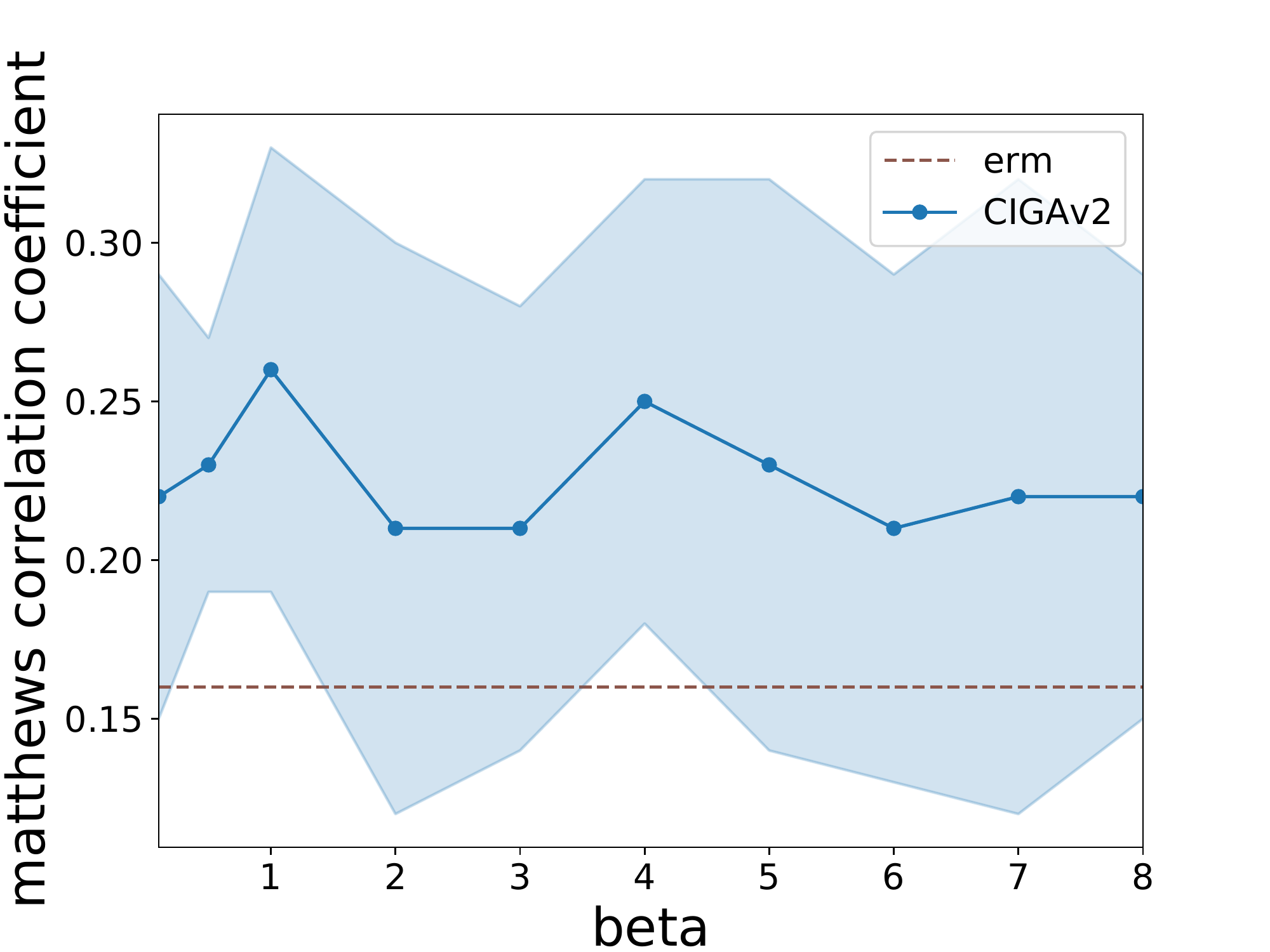}
	}
	\caption{
		Hyperparameter sensitivity analysis on the coefficient of hinge loss ($\beta$).}
	\label{fig:hp_sen_beta_appdx}
\end{figure}

\revision{From the results above, we can see that both $\ginv$v1 and $\ginv$v2 are robust to different values of $\alpha$ and $\beta$, respectively, across different datasets and distribution shifts. Notably, in Fig.~\ref{fig:hp_sen_alpha_appdx}, when the coefficient $\alpha$ for the contrastive loss become too small, the invariance of the identified invariant subgraphs $\widehat{G}_c$ may not be guaranteed, resulting worse performances. Moreover, when $\alpha$ becomes too large, it may affect the optimization and yield worse performances. In SPMotif datasets, the worse performances can be observed via the large variances as well. Similarly for $\beta$, as shown in Fig.~\ref{fig:hp_sen_beta_appdx}, when $\beta$ becomes too small, some part from the spurious subgraph may still be contained in the estimated invariant subgraphs. While if $\beta$ becomes too large, there might be part of $\widehat{G}_c$ being eliminated. Although both $\ginv$v1 and $\ginv$v2 are robust to the changes of $\alpha$ and $\beta$, the intrinsic difficult optimization in OOD generalization algorithms including the proposed $\ginv$ in our work, still require a more proper and smooth optimization process~\citep{pair}.}

\begin{table}[ht]
	\centering
	\caption{\revision{Averaged training time (sec.) per epoch of various methods on DrugOOD-Scaffold.}}\small\sc
	\label{tab:running_time_appdx}
	\resizebox{\textwidth}{!}{
		\begin{small}
			\begin{tabular}{lccccccccc}
				\toprule
				Methods         & ERM   & ASAP   & GIB     & DIR     & IRM   & EIIL   & CNC   & $\ginv$v1 & $\ginv$v2 \\\midrule
				Running time    & 8.055 & 15.578 & 300.304 & 106.919 & 8.73  & 69.664 & 9.795 & 40.065    & 46.181    \\
				OOD Performance & 68.85 & 66.19  & 62.01   & 63.91   & 68.69 & 68.45  & 67.24 & 69.04     & 69.7      \\
				Avg. Rank       & 2     & 5.5    & 9       & 8       & 3     & 6      & 4.5   & 3.5       & 3.5       \\
				\bottomrule
			\end{tabular}	\end{small}}
\end{table}

\revision{
	\textbf{Running time analysis.} To examine how much computational overhead is induced by the architecture and the additional objectives in $\ginv$, we analyze and compare the averaged training time of different methods on DrugOOD-Scaffold. Factors that could affect the running time such as GNN backbone, batch size, and the running devices (NVIDIA RTX 2080Ti, Linux Servers with 40 cores Intel(R) Xeon(R) Silver 4114 CPU @ 2.20GHz, 256 GB Memory, and Ubuntu 18.04 LTS), are fixed the same during the testing. The results are shown as in Table.~\ref{tab:running_time_appdx}. It can be found that $\ginv$ is the only OOD method that outperforms ERM by a non-trivial margin with a relatively low additional computational overhead.
}

\begin{table}[ht]
	\centering
	\caption{\revision{Performances of different methods on Drug-Assay under single environment OOD generalization (i).}}\small\sc
	\label{tab:single_dg_appdx}
	\resizebox{\textwidth}{!}{
		\begin{small}
			\begin{tabular}{lccccccc}
				\toprule
				Methods         & ERM         & ASAP        & GIB         & DIR         & $\ginv$v1             & $\ginv$v2             & Oracle (IID) \\\midrule
				OOD Performance & 63.29(2.67) & 63.41(0.70) & 62.72(0.59) & 62.56(0.79) & \textbf{63.86 (0.57)} & \textbf{64.31 (0.92)} & 84.71 (1.60) \\
				Rank            & 5           & 4           & 8           & 9           & 2                     & 1                     &              \\
				\bottomrule
			\end{tabular}	\end{small}}
\end{table}

\begin{table}[ht]
	\centering
	\caption{\revision{Performances of different methods on Drug-Assay under single environment OOD generalization (ii).}}\small\sc
	\label{tab:single_dg_2_appdx}
	\resizebox{\textwidth}{!}{
		\begin{small}
			\begin{tabular}{lccccccccc}
				\toprule
				Methods         & ERM         & IRM         & V-Rex       & EIIL        & IB-IRM      & CNC         & $\ginv$v1             & $\ginv$v2             & Oracle (IID) \\\midrule
				OOD Performance & 63.29(2.67) & 63.25(1.45) & 62.18(1.71) & 62.95(1.37) & 61.95(1.72) & 63.61(0.96) & \textbf{63.86 (0.57)} & \textbf{64.31 (0.92)} & 84.71 (1.60) \\
				Rank            & 5           & 6           & 10          & 7           & 11          & 3           & 2                     & 1                     &              \\
				\bottomrule
			\end{tabular}	\end{small}}
\end{table}

\revision{
	\textbf{Single environment OOD generalization.} The theory of invariant learning fundamentally assume the presence of multiple environments~\citep{inv_principle,irmv1}. However in practice, it does not always hold, which would inevitably fail all of the invariant learning solutions~\citep{irmv1,v-rex,env_inference,ib-irm}, including $\ginv$.
}

\revision{	Nevertheless, to examine how $\ginv$ performs under various realistic scenarios, we conduct an additional experiment based on DrugOOD-Assay. We select samples that are from the largest assay group (i.e., the biochemical functionalities of these molecules are tested and reported under the same experimental setup in the lab)~\citep{drugood}. The results are separated and shown in Table~\ref{tab:single_dg_appdx} and Table~\ref{tab:single_dg_2_appdx}. Besides the baselines, we also show the ``Oracle'' performances from the main table, to demonstrate the performance gaps.}

\revision{From the Table~\ref{tab:single_dg_appdx} and Table~\ref{tab:single_dg_2_appdx}, we can see that, both $\ginv$v1 and $\ginv$v2 maintain their state-of-the-art performances even in the single training environment setting. We hypothesize that enforcing the mutual information between the estimated $\widehat{G}_c$ also helps to retain the invariance even under the single training environment setting. That may partially explain why CNC can bring some improvements. We believe it is an interesting and promising future direction to develop in-depth understanding and better solutions under this circumstance.
}

\revision{
	\subsection{Interpretation Visualization}
	\label{sec:interpret_visualize_appdx}
	Since we use the interpretable GNN architecture to implement $\ginv$\footnote{We use the code provided by~\citep{gsat}.}, it brings an additional benefit that provides certain interpretation for the predictions automatically, which may facilitate human understanding in practice.
}

\revision{First, we provide some interpretation visualizations in SPMotif-Struc and SPMotif-Mixed datasets, under the biases of $0.6$ and $0.9$. Shown in Fig.~\ref{fig:spmotif_b6_appdx} to Fig.~\ref{fig:spmotifm_b9_appdx}, we use pink to color the ground truth nodes in $G_c$, and denote the relative attention strength with edge color intensities.}

\revision{
	Besides, we also provide some interpretation visualization examples in DrugOOD datasets.
	Shown in Fig.~\ref{fig:assay_viz_act_appdx} to Fig.~\ref{fig:size_viz_inact_appdx}, we use the edge color intensities to denote the attentions of models that pay to the corresponding edge.
	Some interesting patterns can be found in the molecules shared with the same label, which could provide insights to the domain experts when developing new drugs.
	We believe that, because of its superior OOD generalization performance on graphs, $\ginv$ can have high potential to push forward the developments of AI-Assisted Drug Discovery, and enrich the AI tools for facilitating the fundamental practice of science in the future.
}

\begin{figure}[H]
	\centering
	\subfigure[]{
		\includegraphics[width=0.31\textwidth]{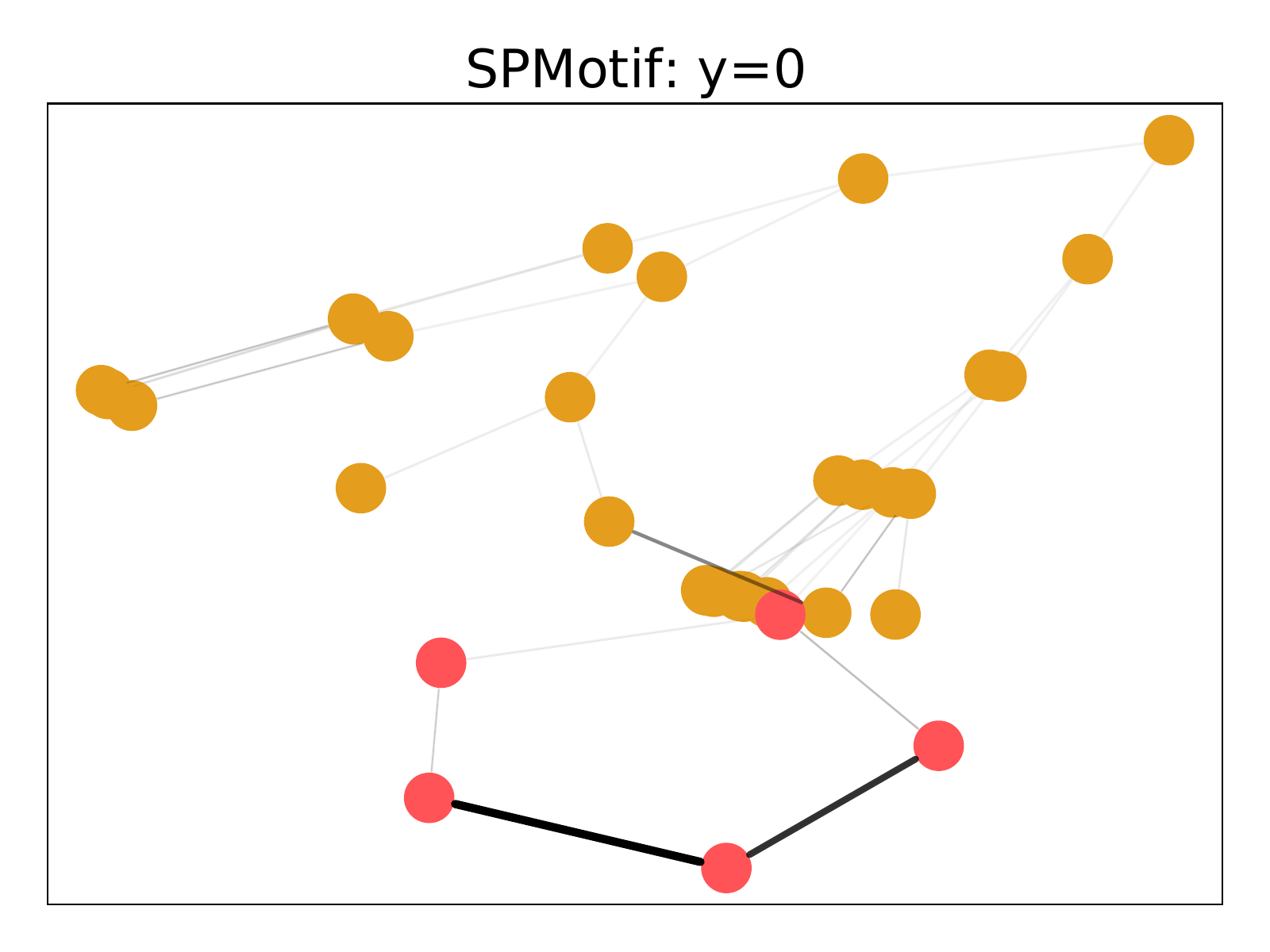}
	}
	\subfigure[]{
		\includegraphics[width=0.31\textwidth]{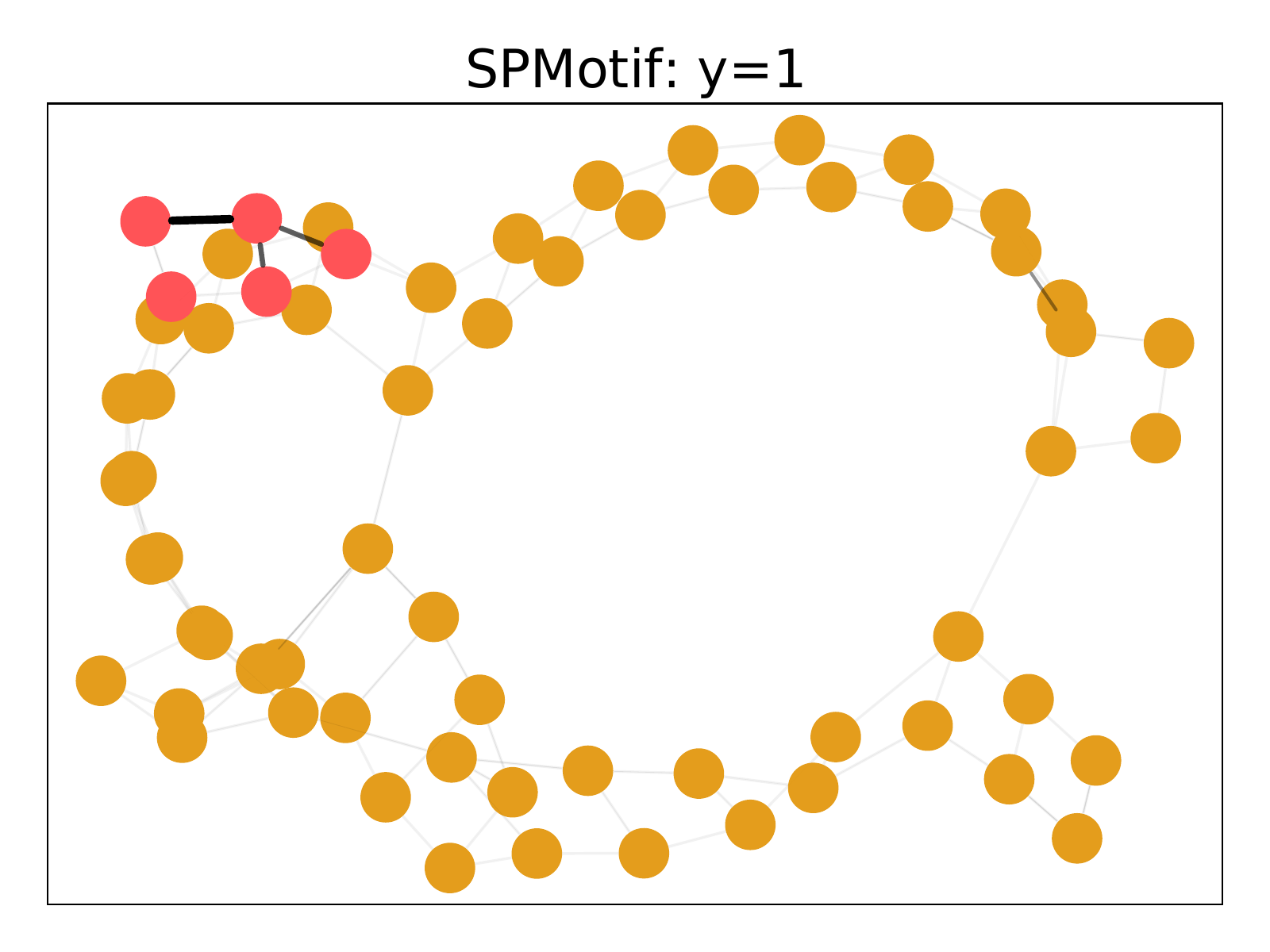}
	}
	\subfigure[]{
		\includegraphics[width=0.31\textwidth]{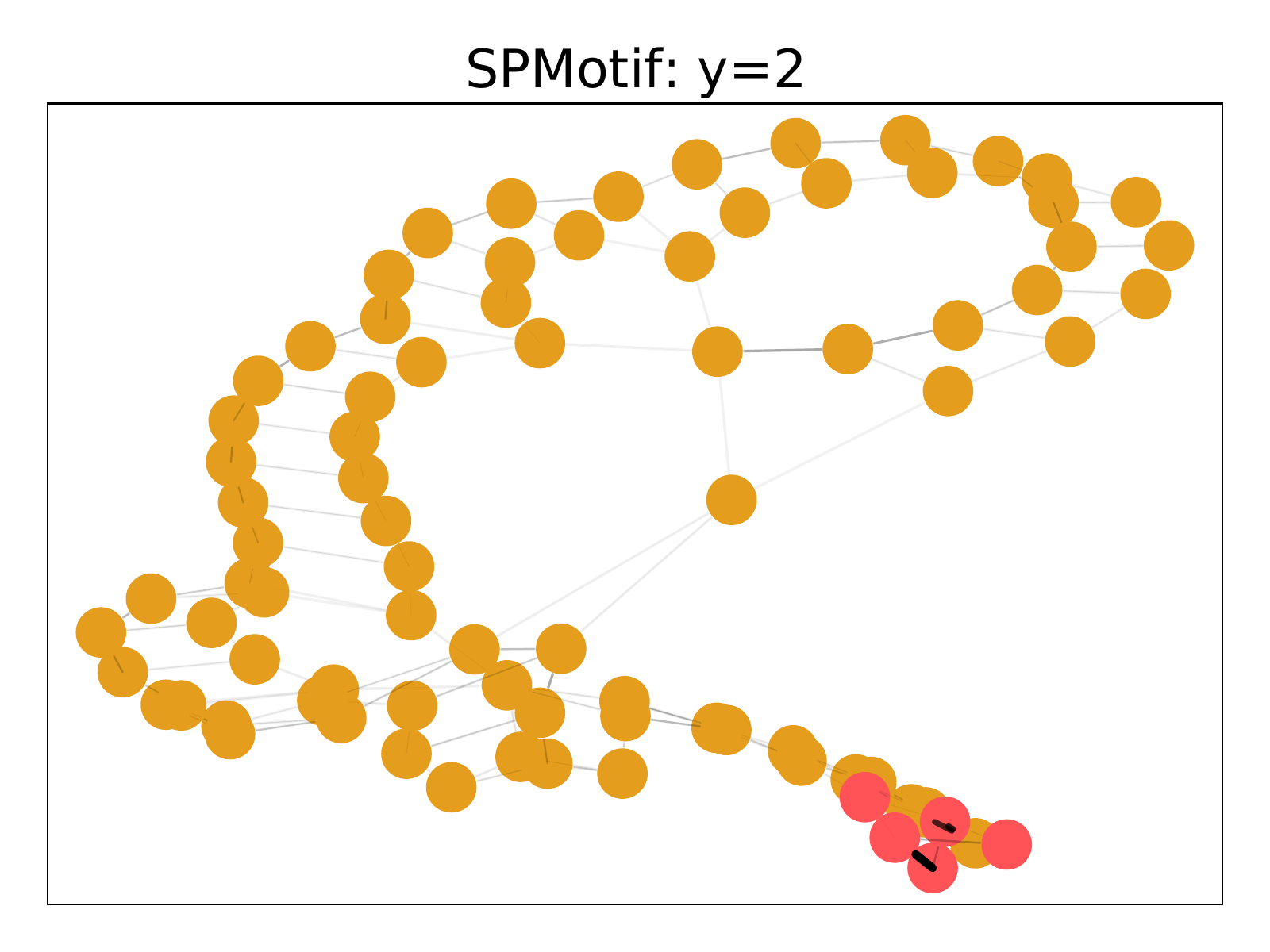}
	}
	\caption{
		Interpretation visualization of examples from SPMotif-Struc under bias$=0.6$.}
	\label{fig:spmotif_b6_appdx}
\end{figure}

\begin{figure}[H]
	\centering
	\subfigure[]{
		\includegraphics[width=0.31\textwidth]{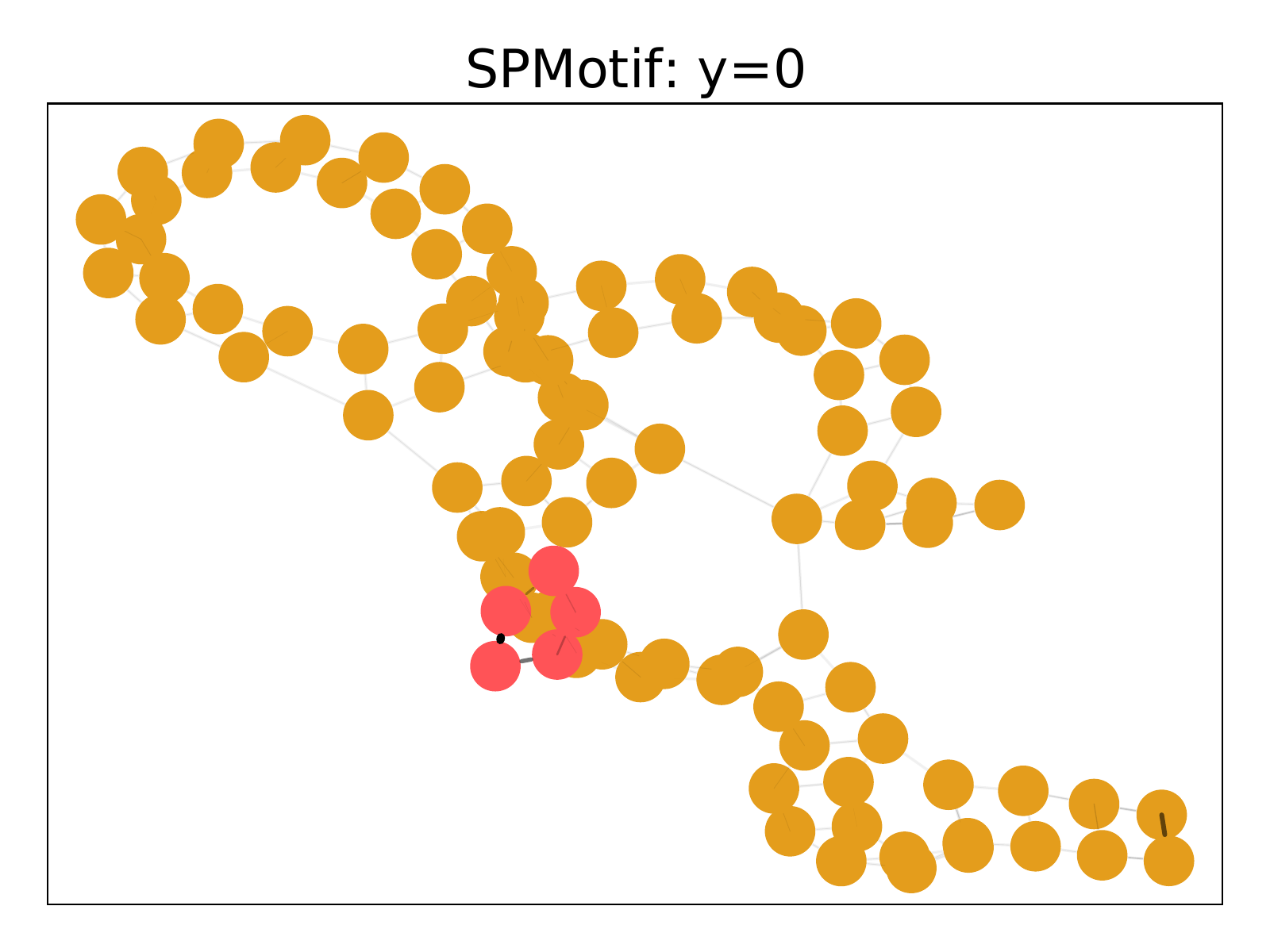}
	}
	\subfigure[]{
		\includegraphics[width=0.31\textwidth]{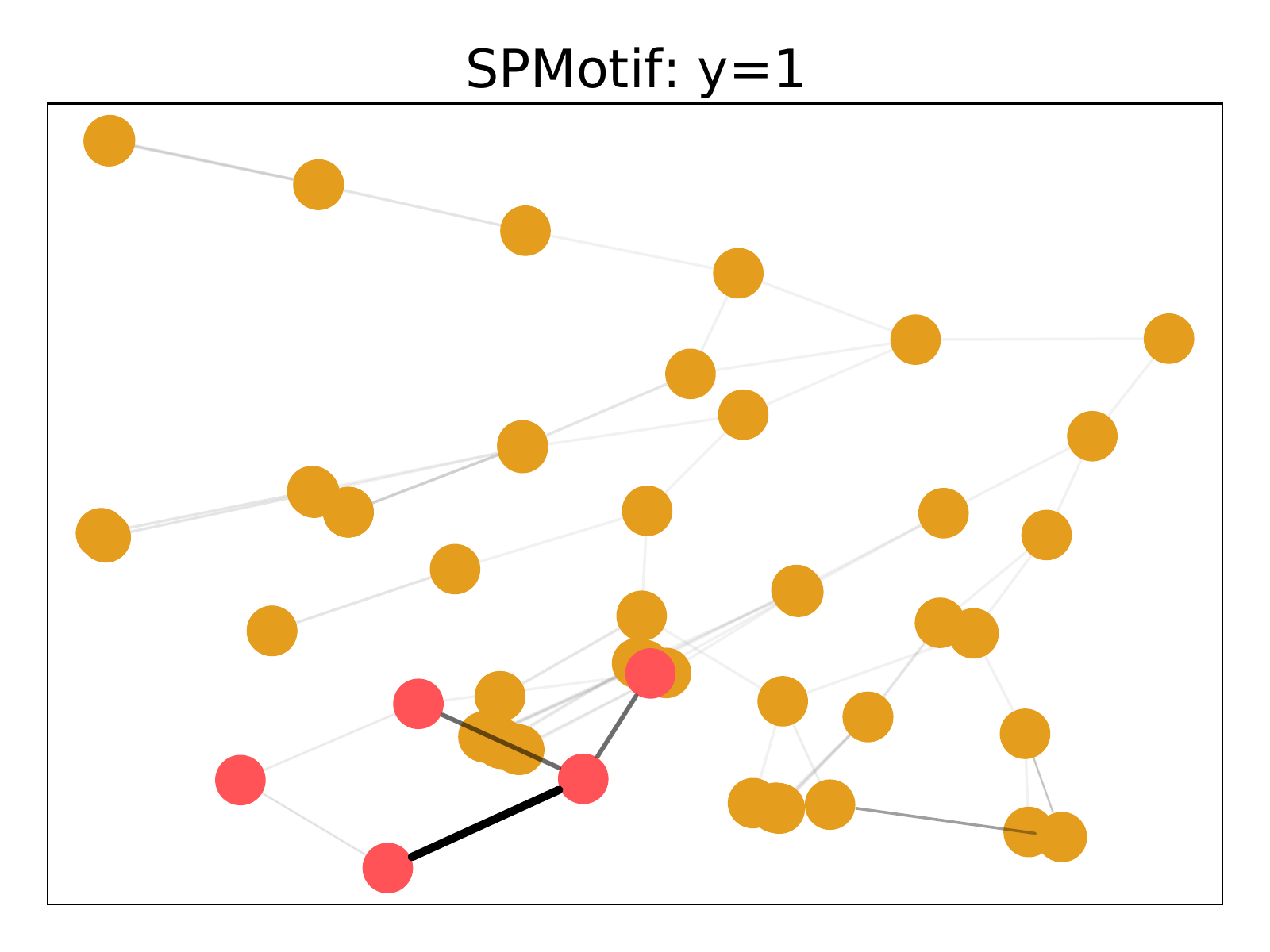}
	}
	\subfigure[]{
		\includegraphics[width=0.31\textwidth]{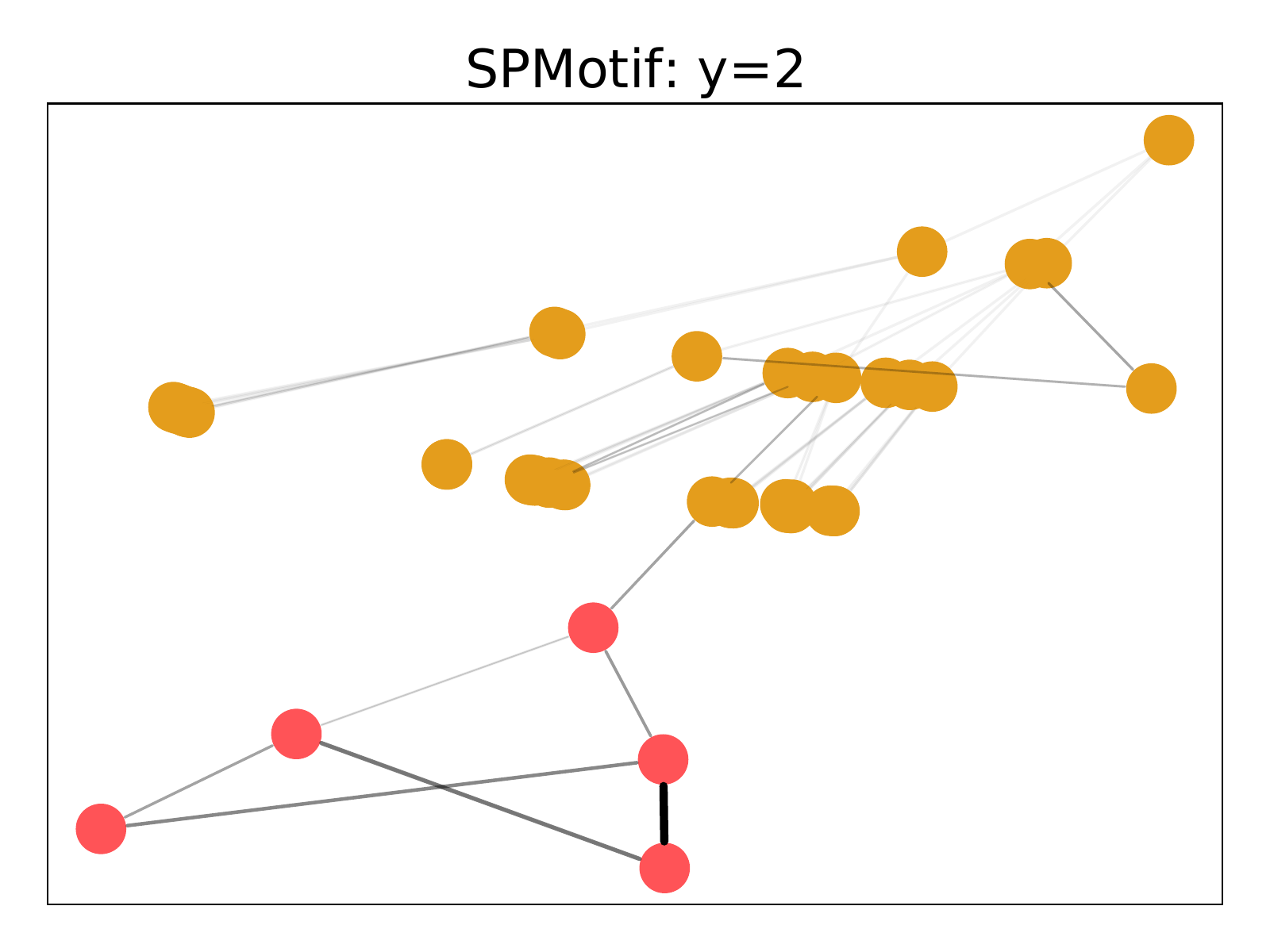}
	}
	\caption{
		Interpretation visualization of examples from SPMotif-Struc under bias$=0.9$.}
	\label{fig:spmotif_b9_appdx}
\end{figure}

\begin{figure}[H]
	\centering
	\subfigure[]{
		\includegraphics[width=0.31\textwidth]{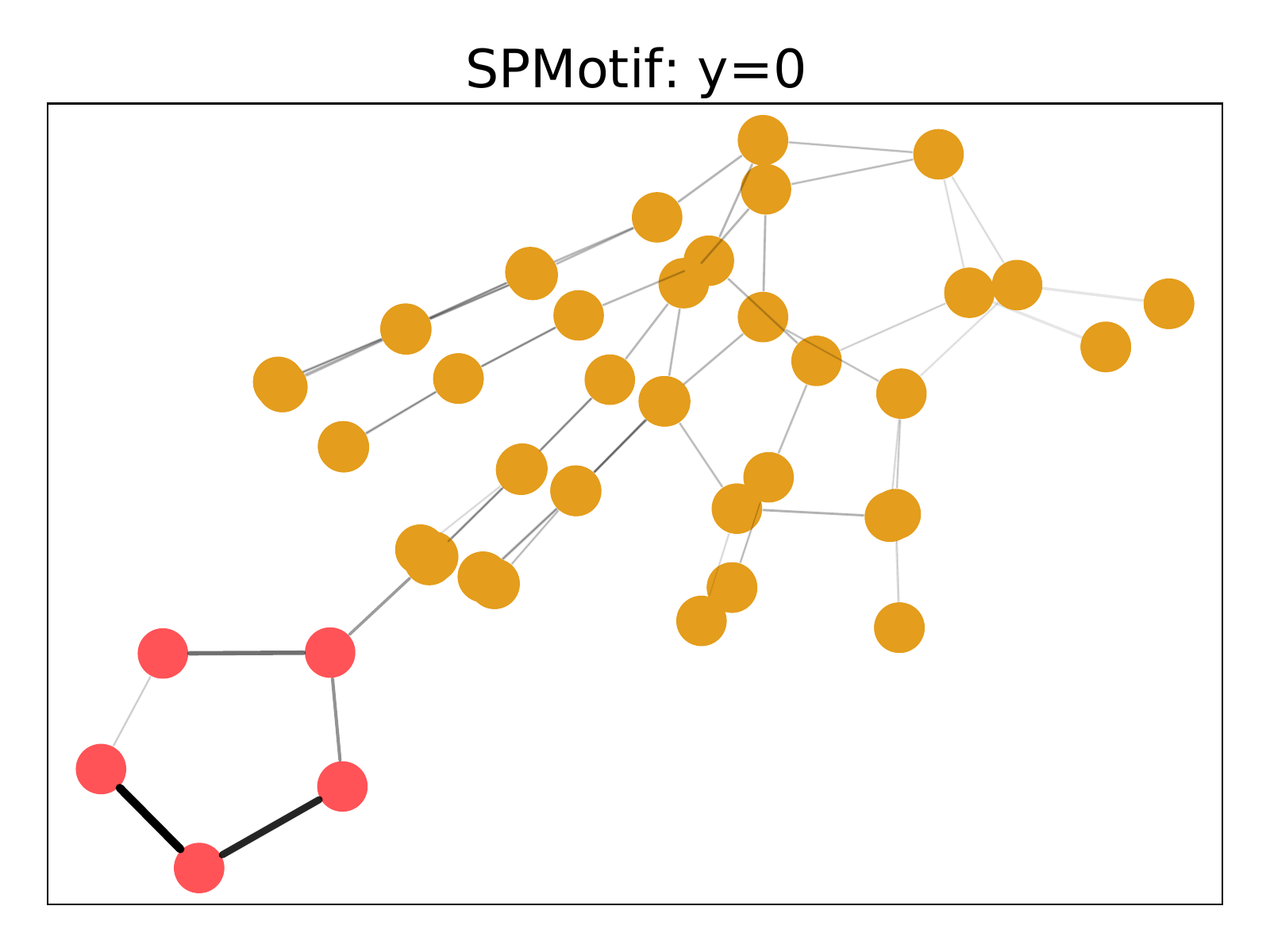}
	}
	\subfigure[]{
		\includegraphics[width=0.31\textwidth]{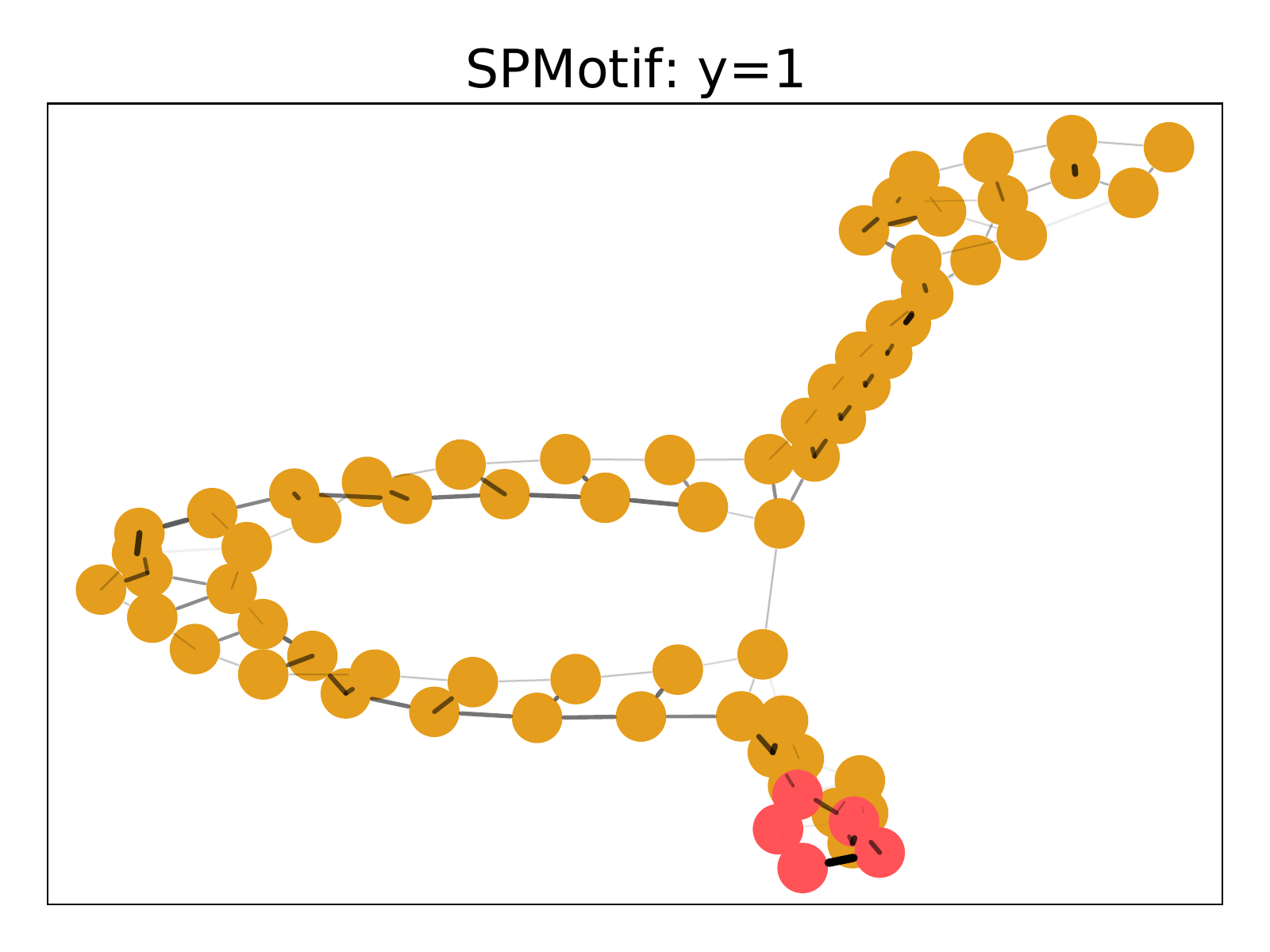}
	}
	\subfigure[]{
		\includegraphics[width=0.31\textwidth]{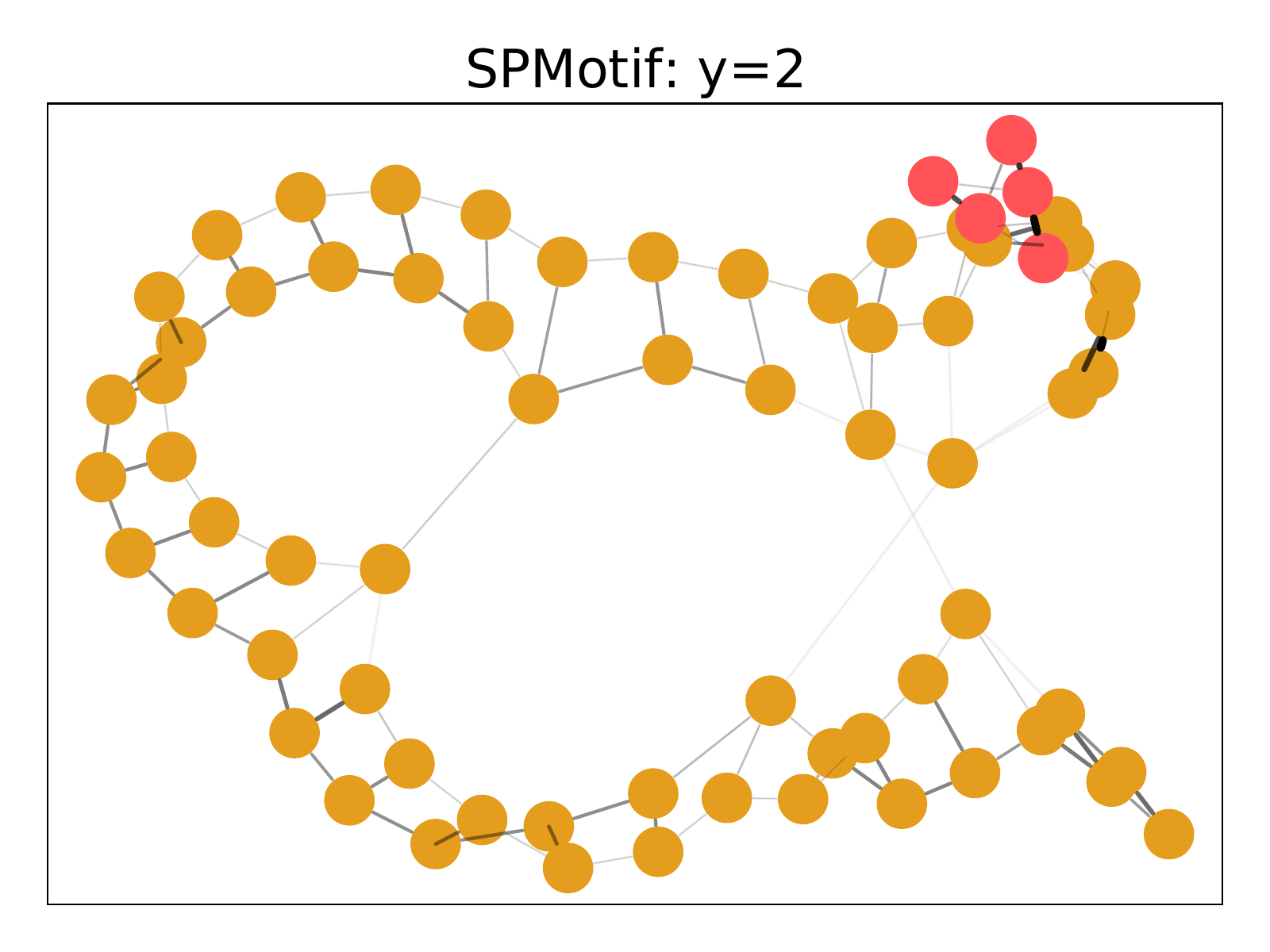}
	}
	\caption{
		Interpretation visualization of examples from SPMotif-Mixed under bias$=0.6$.}
	\label{fig:spmotifm_b6_appdx}
\end{figure}

\begin{figure}[H]
	\centering
	\subfigure[]{
		\includegraphics[width=0.31\textwidth]{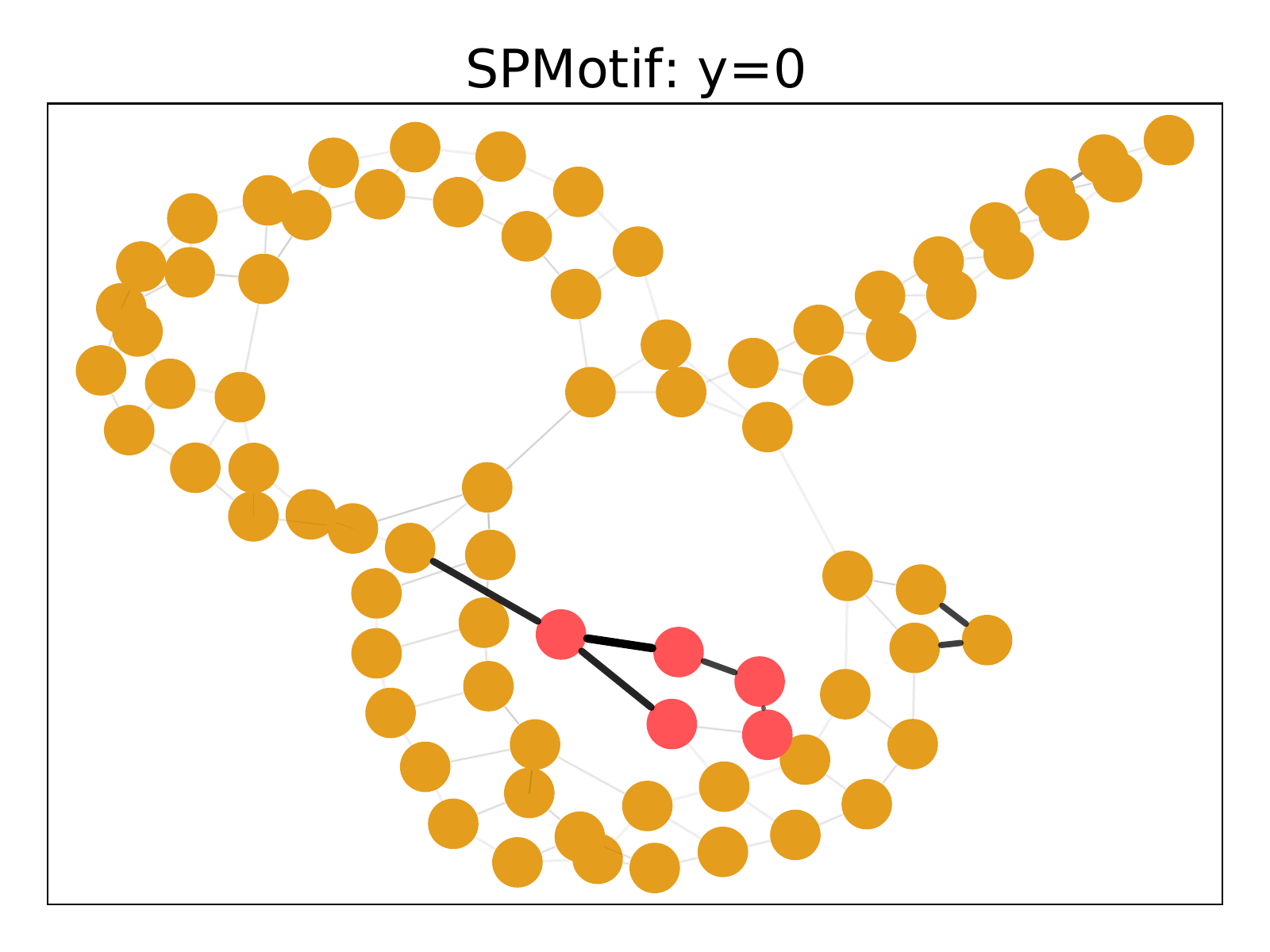}
	}
	\subfigure[]{
		\includegraphics[width=0.31\textwidth]{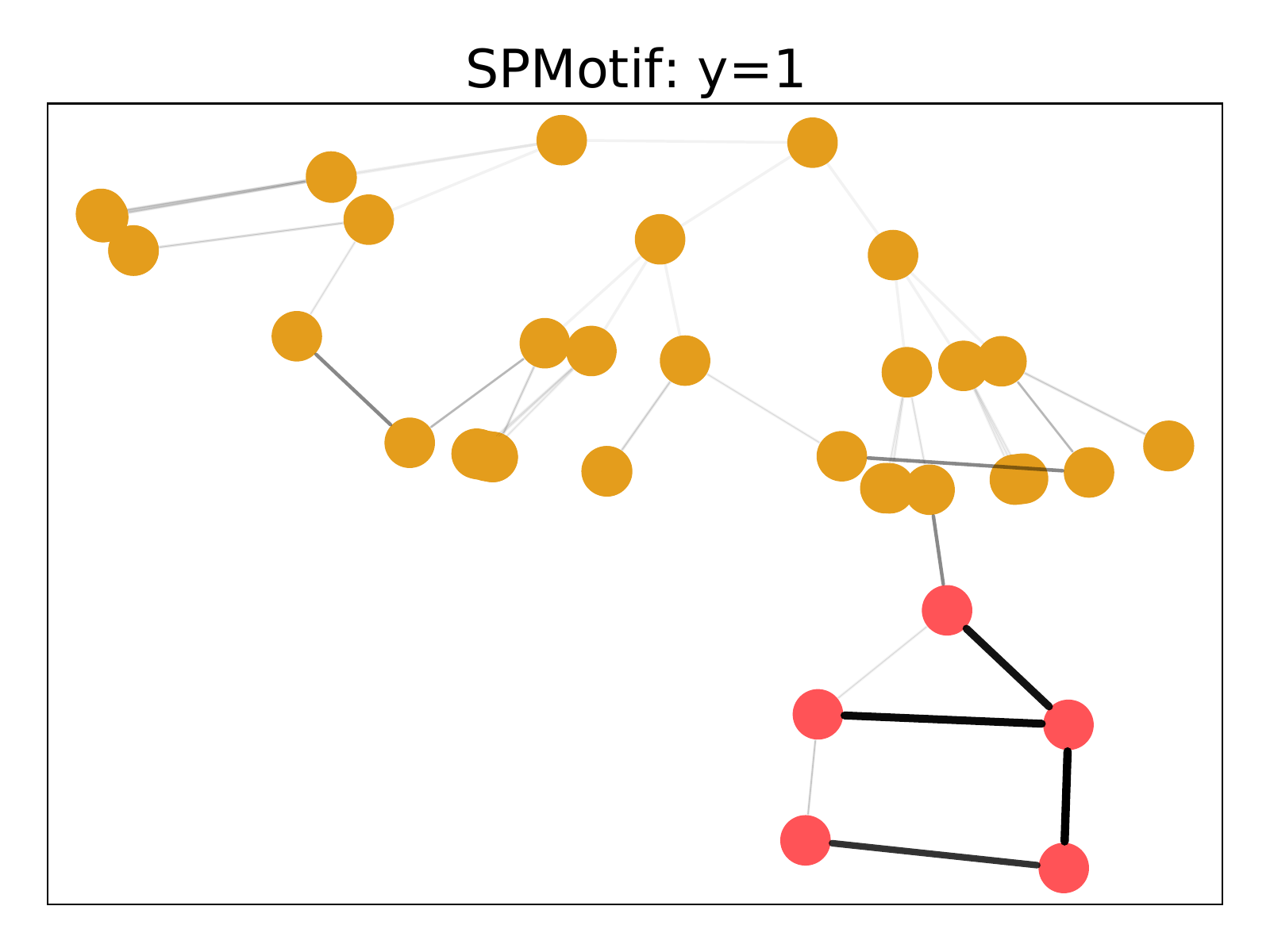}
	}
	\subfigure[]{
		\includegraphics[width=0.31\textwidth]{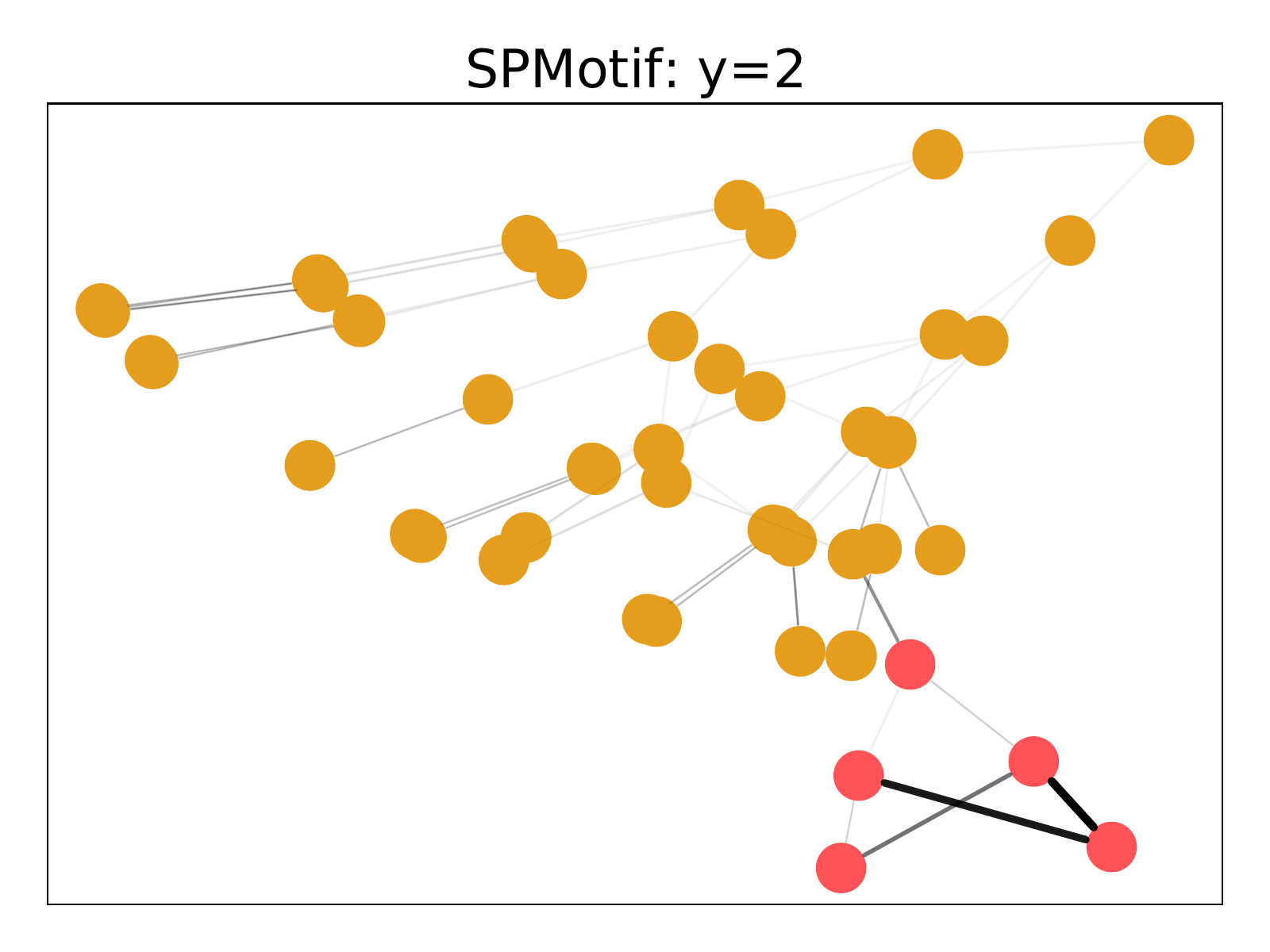}
	}
	\caption{
		Interpretation visualization of examples from SPMotif-Mixed under bias$=0.9$.}
	\label{fig:spmotifm_b9_appdx}
\end{figure}

\begin{figure}[H]
	\centering
	\subfigure[]{
		\includegraphics[width=0.31\textwidth]{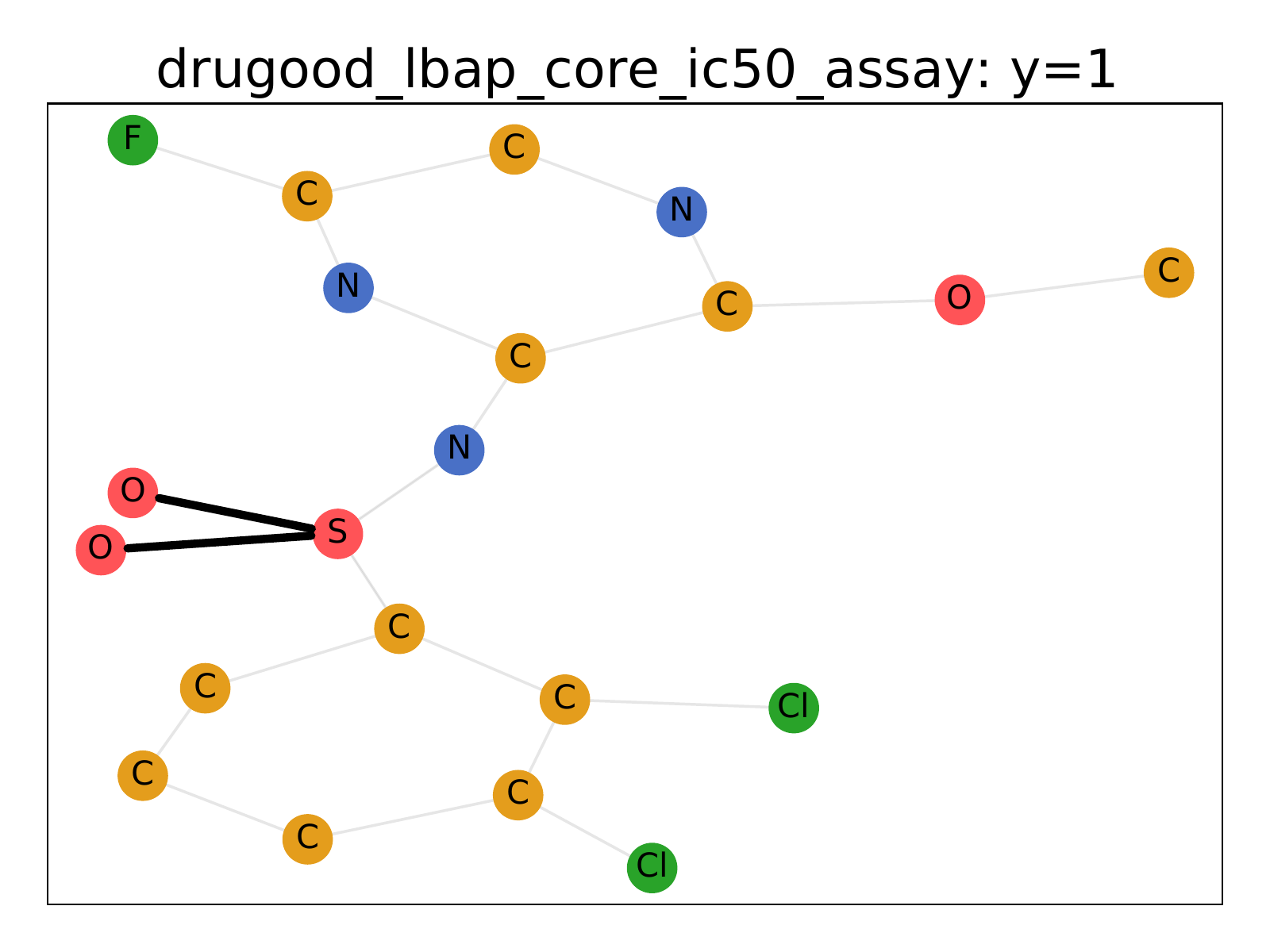}
	}
	\subfigure[]{
		\includegraphics[width=0.31\textwidth]{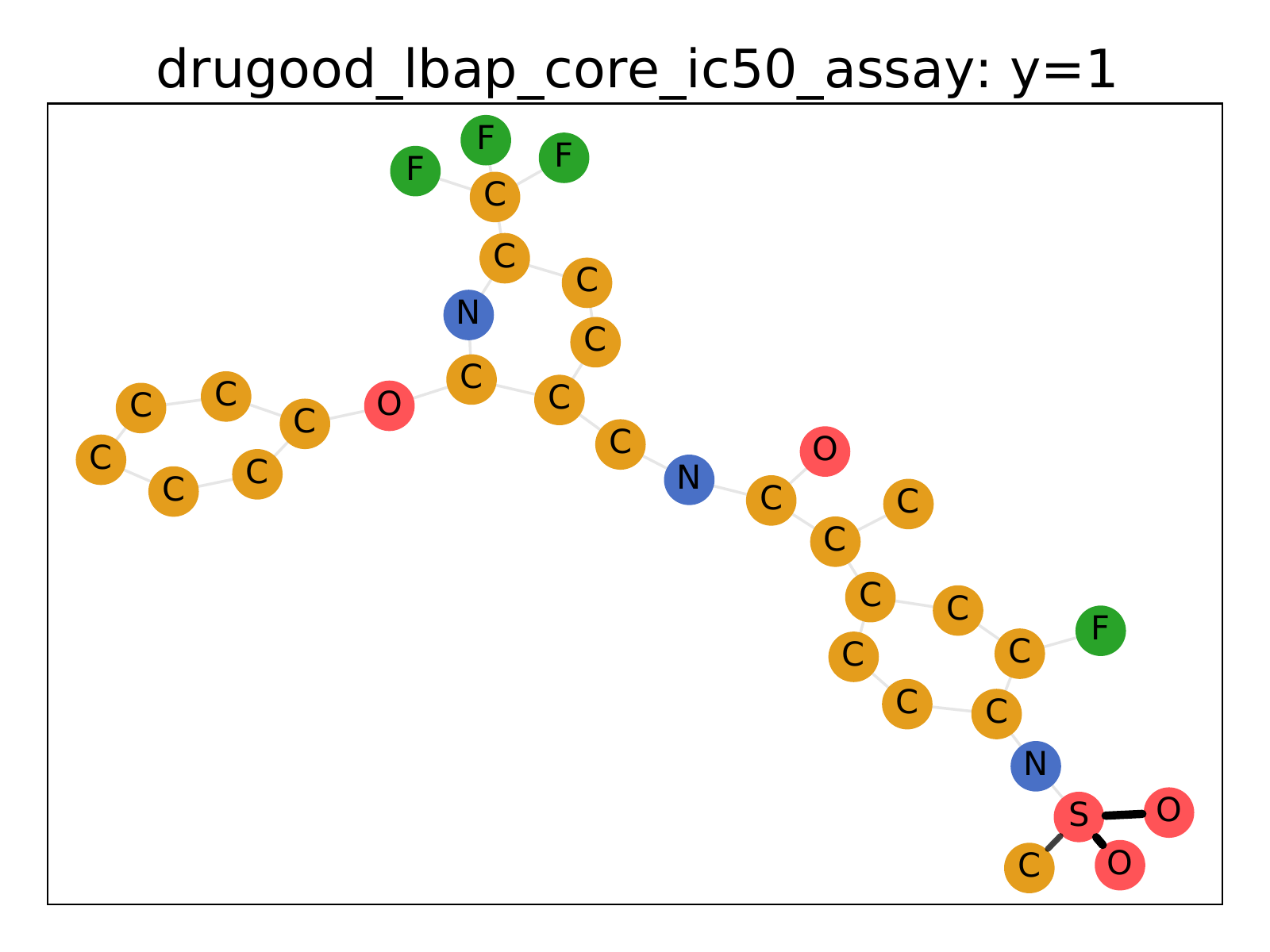}
	}
	\subfigure[]{
		\includegraphics[width=0.31\textwidth]{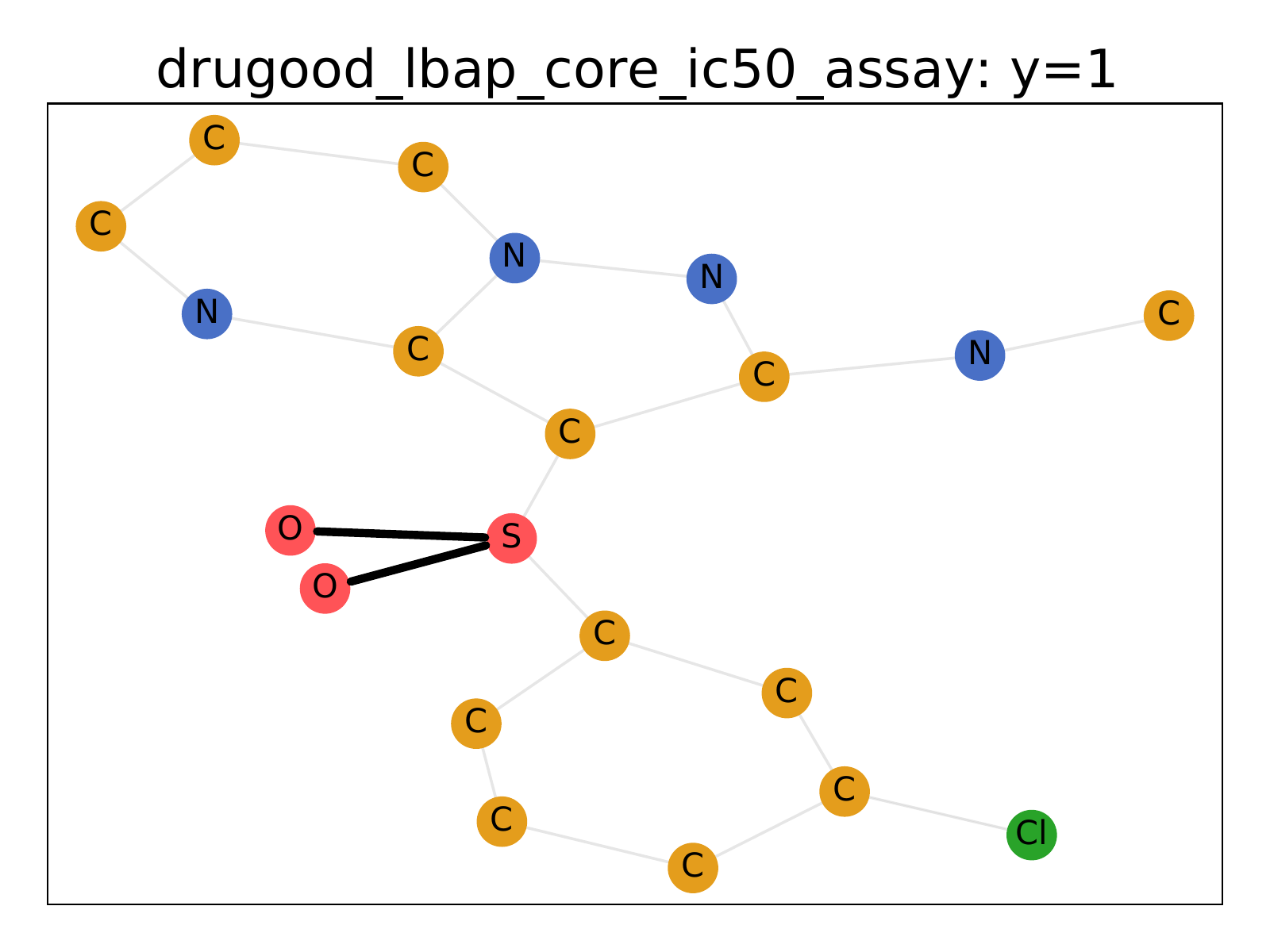}
	}
	\caption{
		Interpretation visualization of activate examples ($y=1$) from DrugOOD-Assay.}
	\label{fig:assay_viz_act_appdx}
\end{figure}

\begin{figure}[H]
	\centering
	\subfigure[]{
		\includegraphics[width=0.31\textwidth]{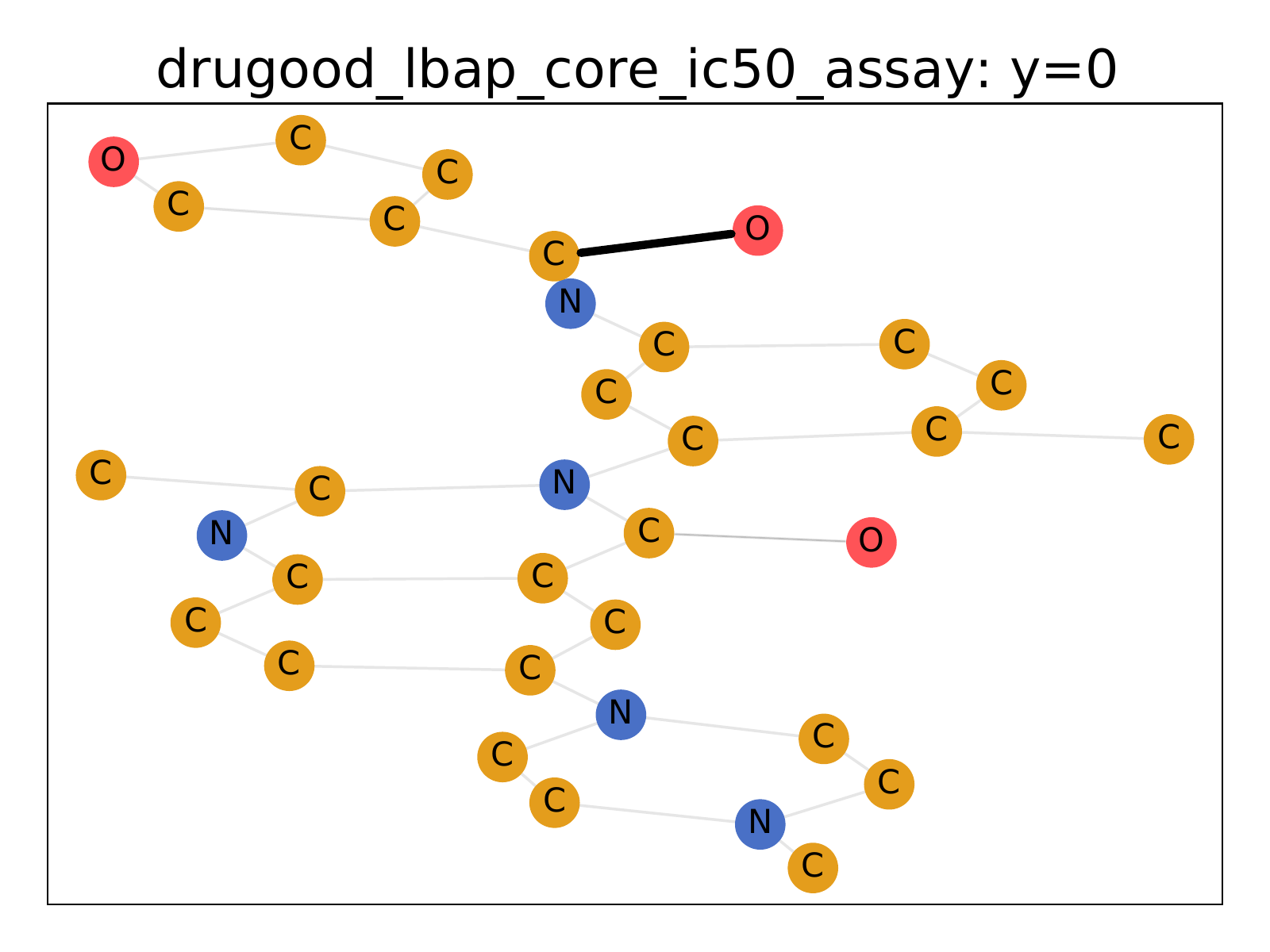}
	}
	\subfigure[]{
		\includegraphics[width=0.31\textwidth]{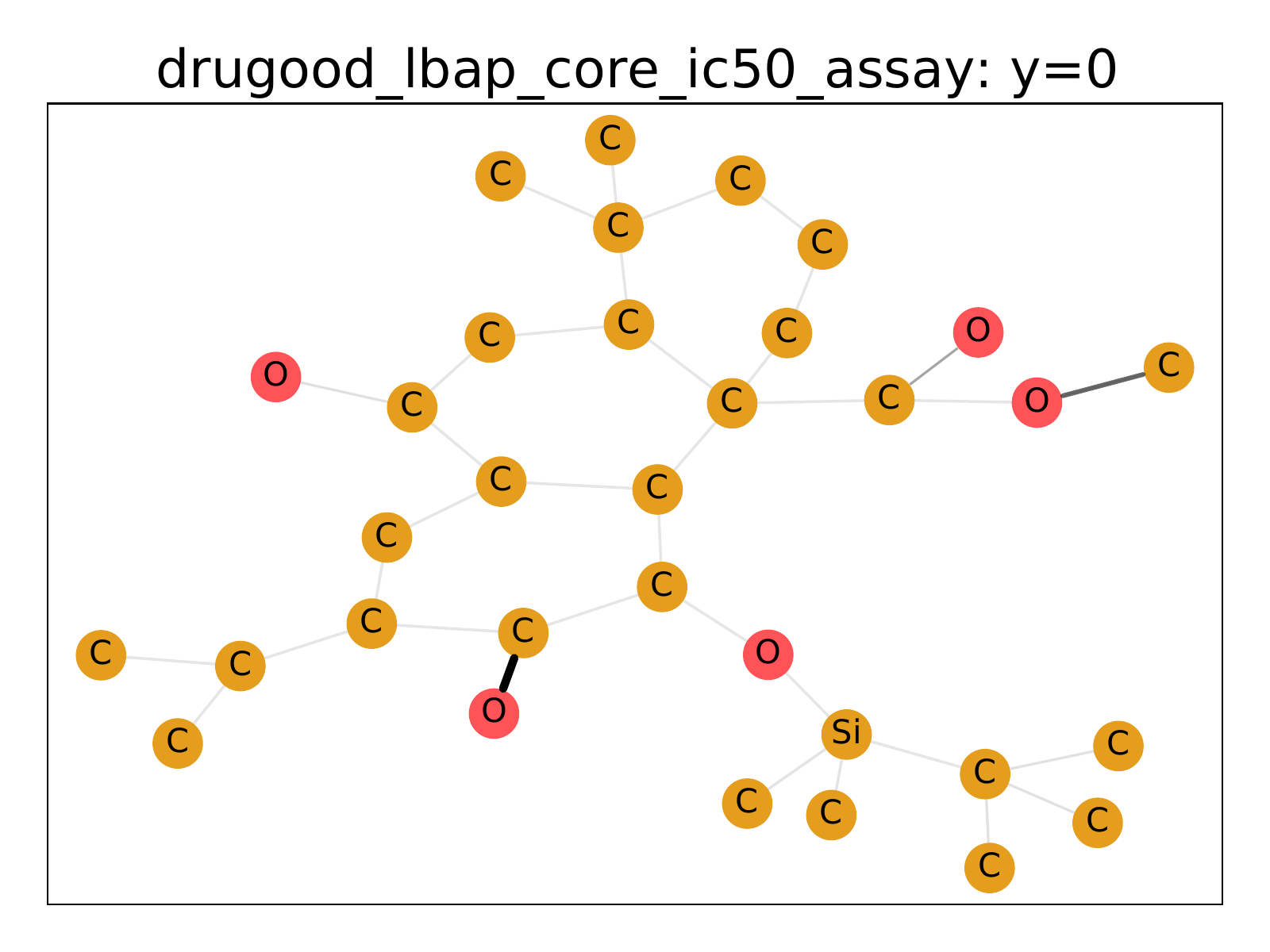}
	}
	\subfigure[]{
		\includegraphics[width=0.31\textwidth]{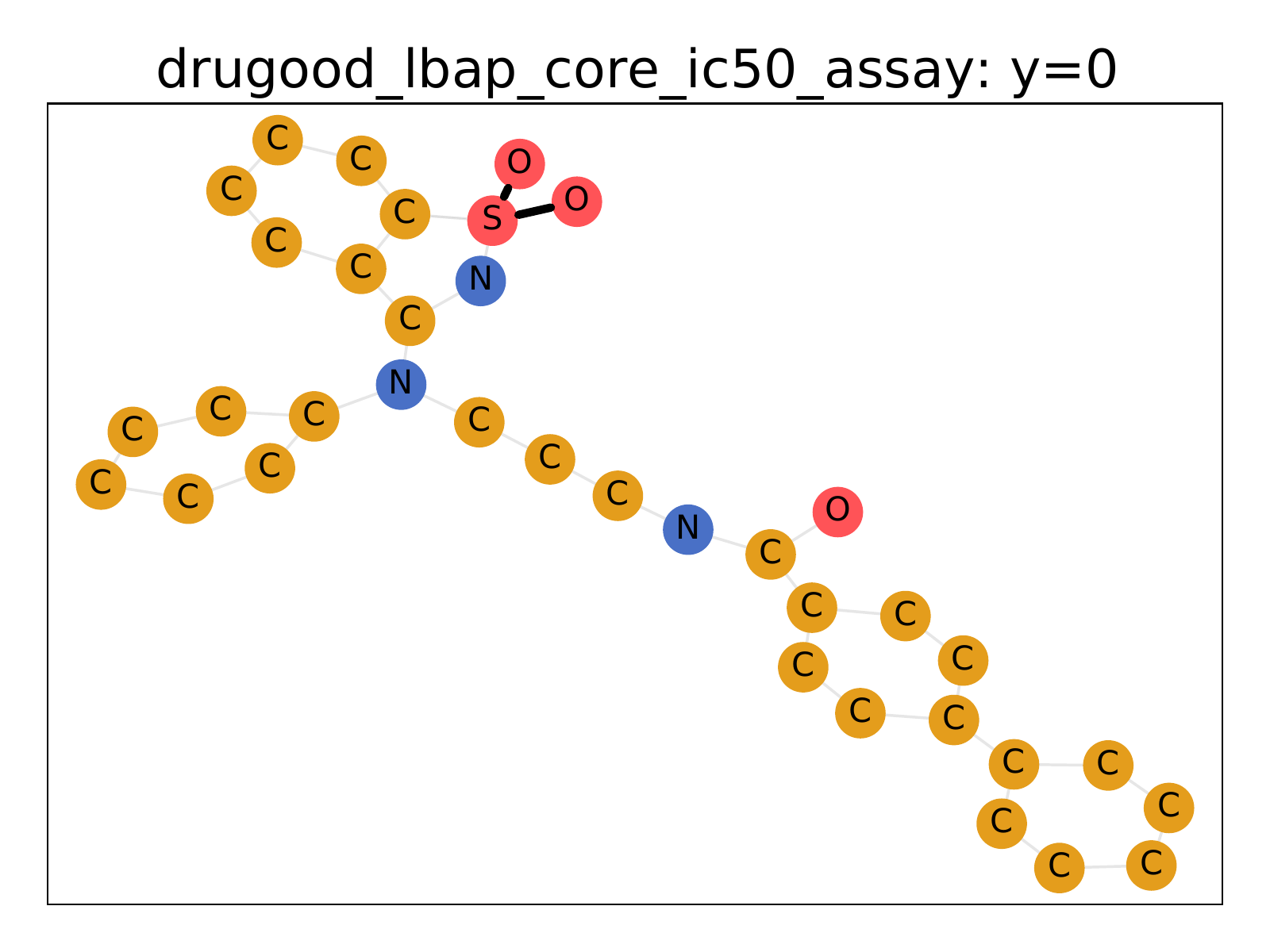}
	}
	\caption{
		Interpretation visualization of inactivate examples ($y=0$) from DrugOOD-Assay.}
	\label{fig:assay_viz_inact_appdx}
\end{figure}

\begin{figure}[H]
	\centering
	\subfigure[]{
		\includegraphics[width=0.31\textwidth]{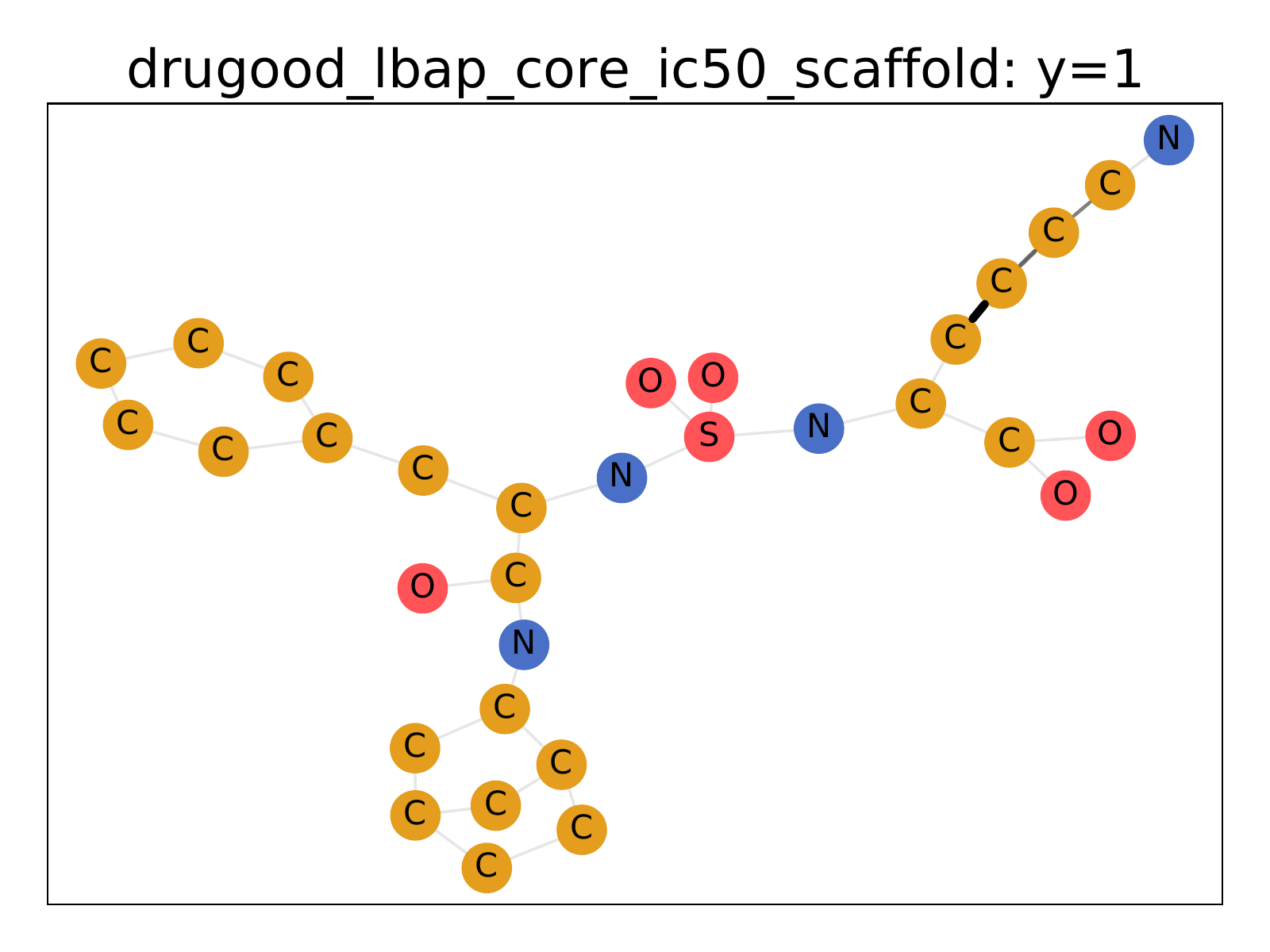}
	}
	\subfigure[]{
		\includegraphics[width=0.31\textwidth]{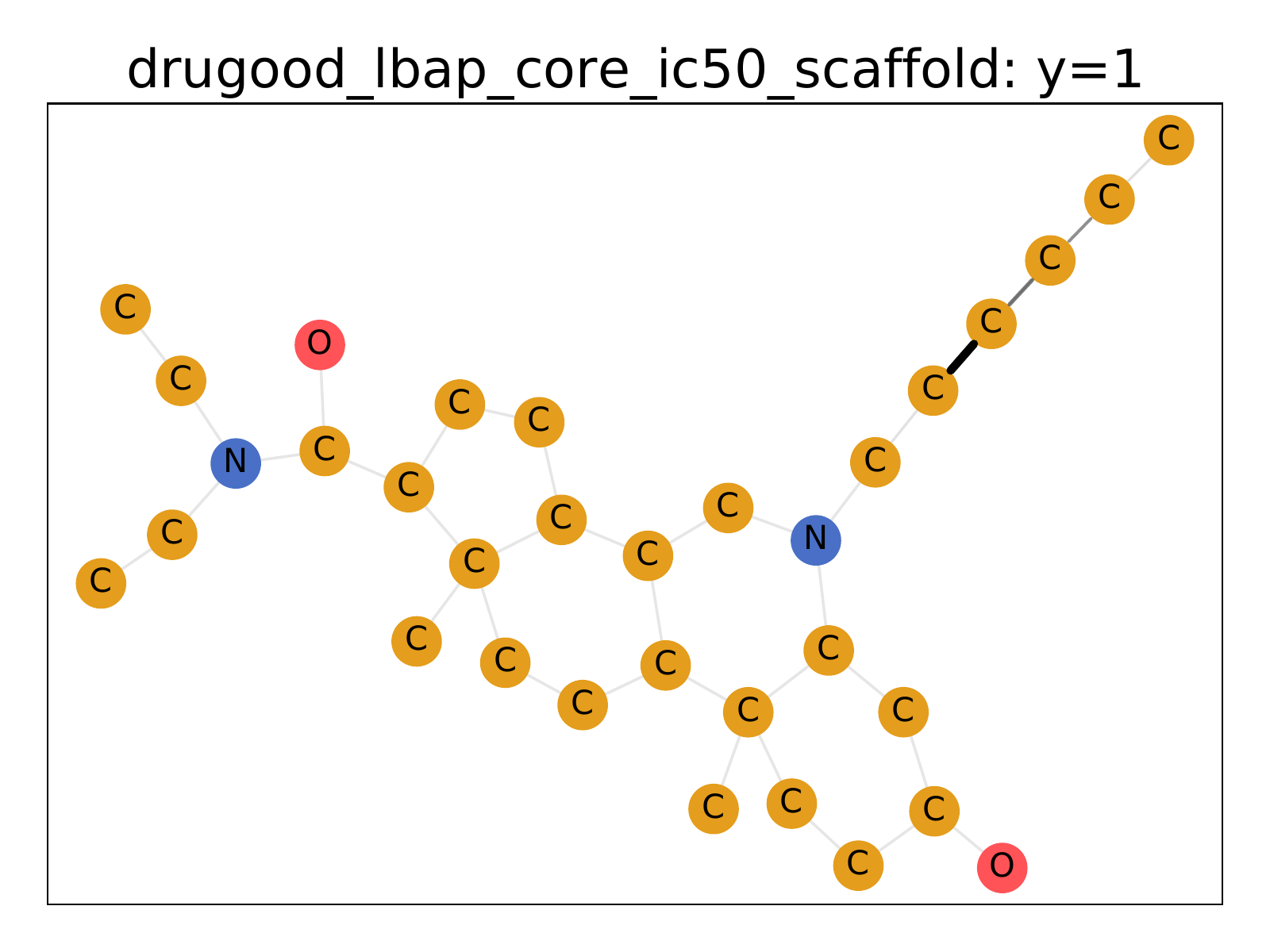}
	}
	\subfigure[]{
		\includegraphics[width=0.31\textwidth]{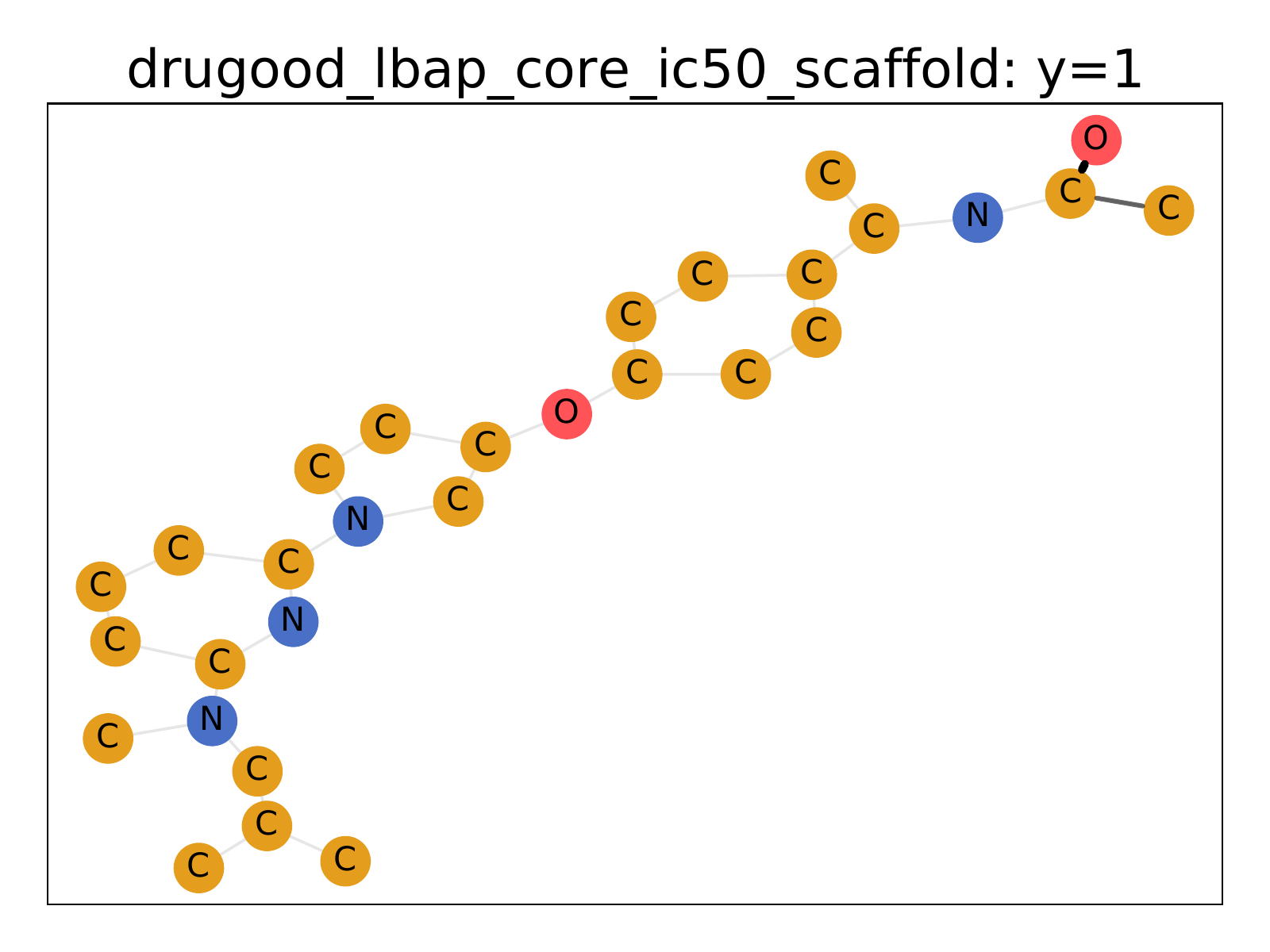}
	}
	\caption{
		Interpretation visualization of activate examples ($y=1$) from DrugOOD-Scaffold.}
	\label{fig:scaffold_viz_act_appdx}
\end{figure}

\begin{figure}[H]
	\centering
	\subfigure[]{
		\includegraphics[width=0.31\textwidth]{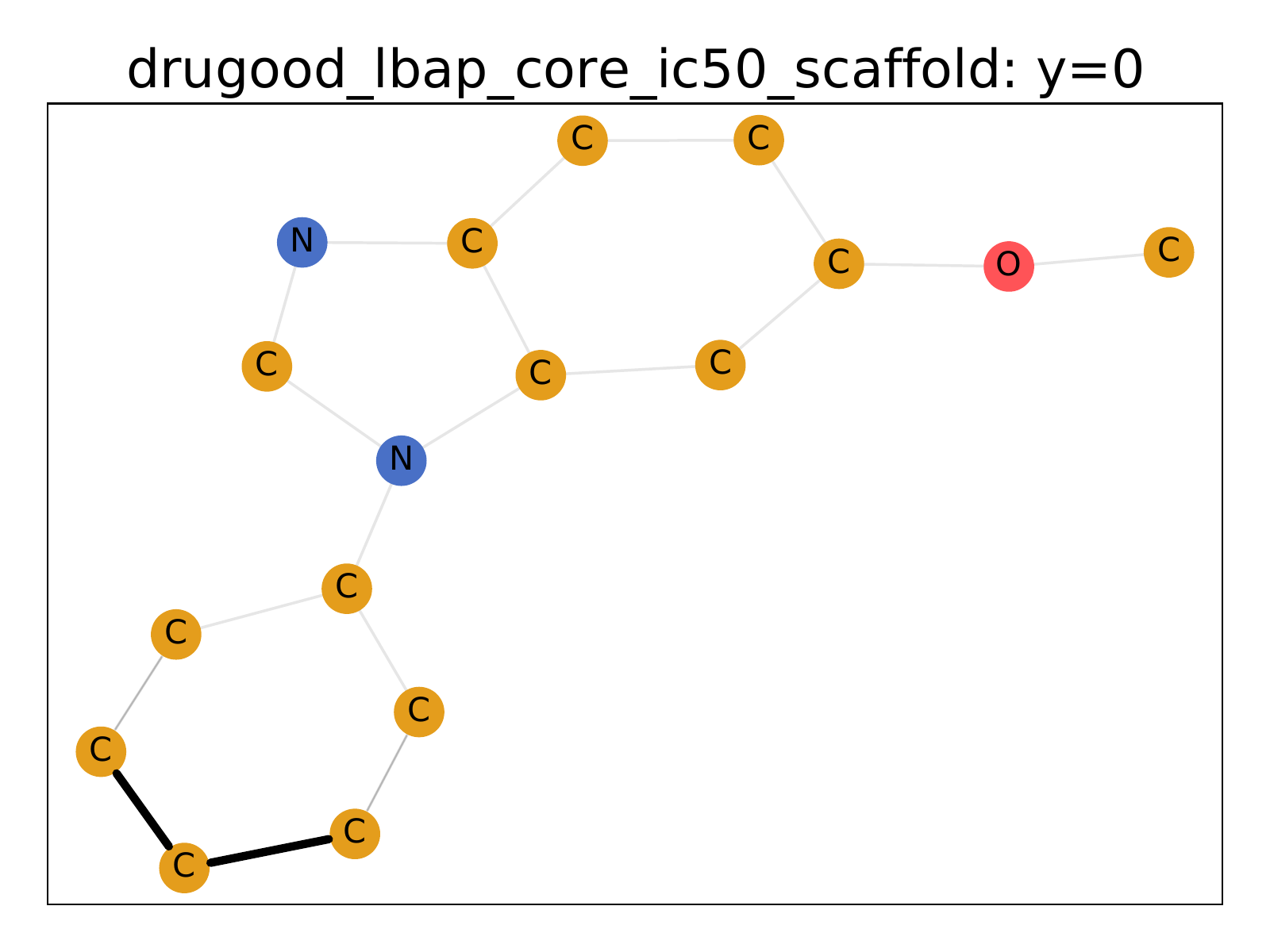}
	}
	\subfigure[]{
		\includegraphics[width=0.31\textwidth]{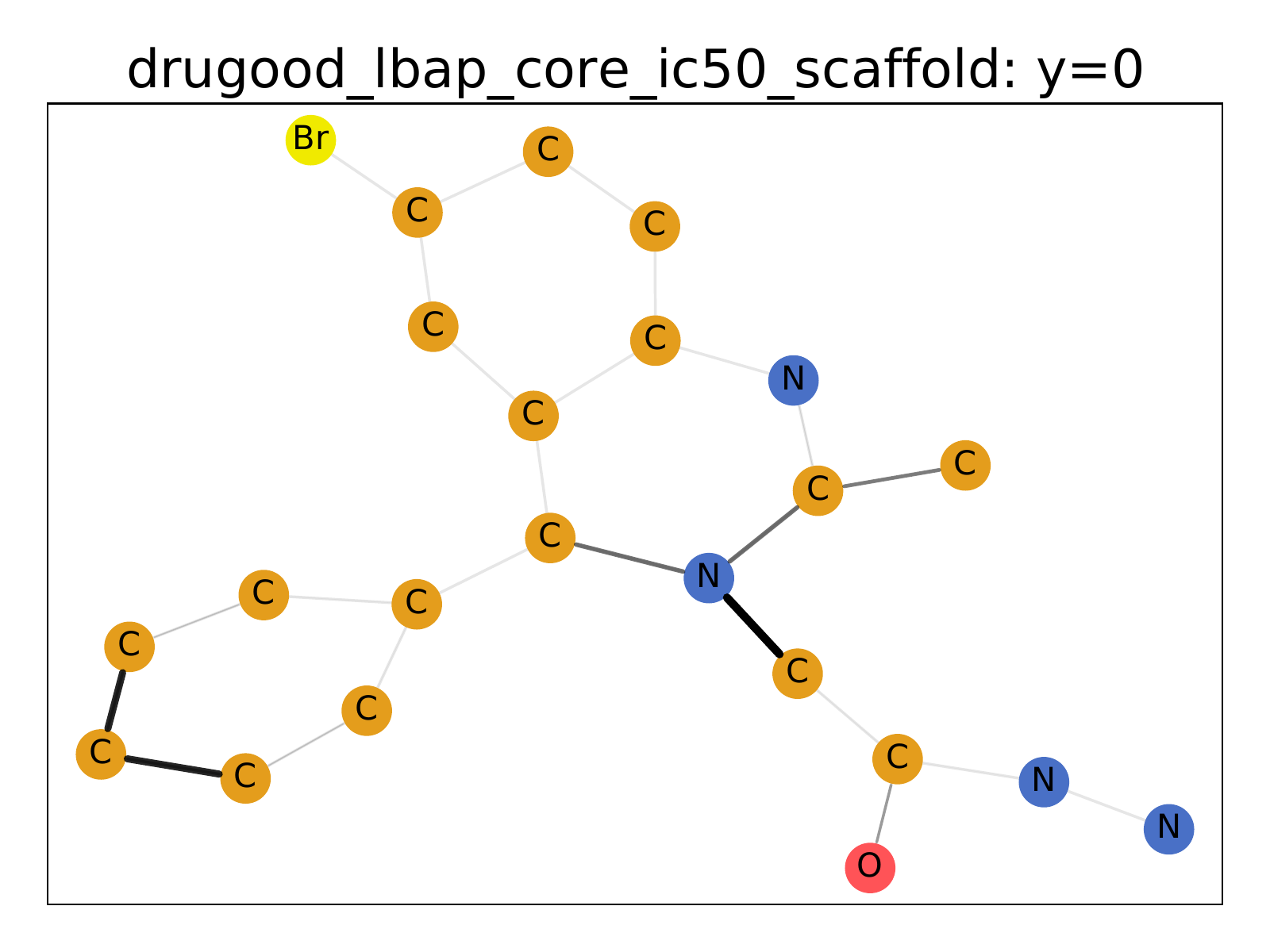}
	}
	\subfigure[]{
		\includegraphics[width=0.31\textwidth]{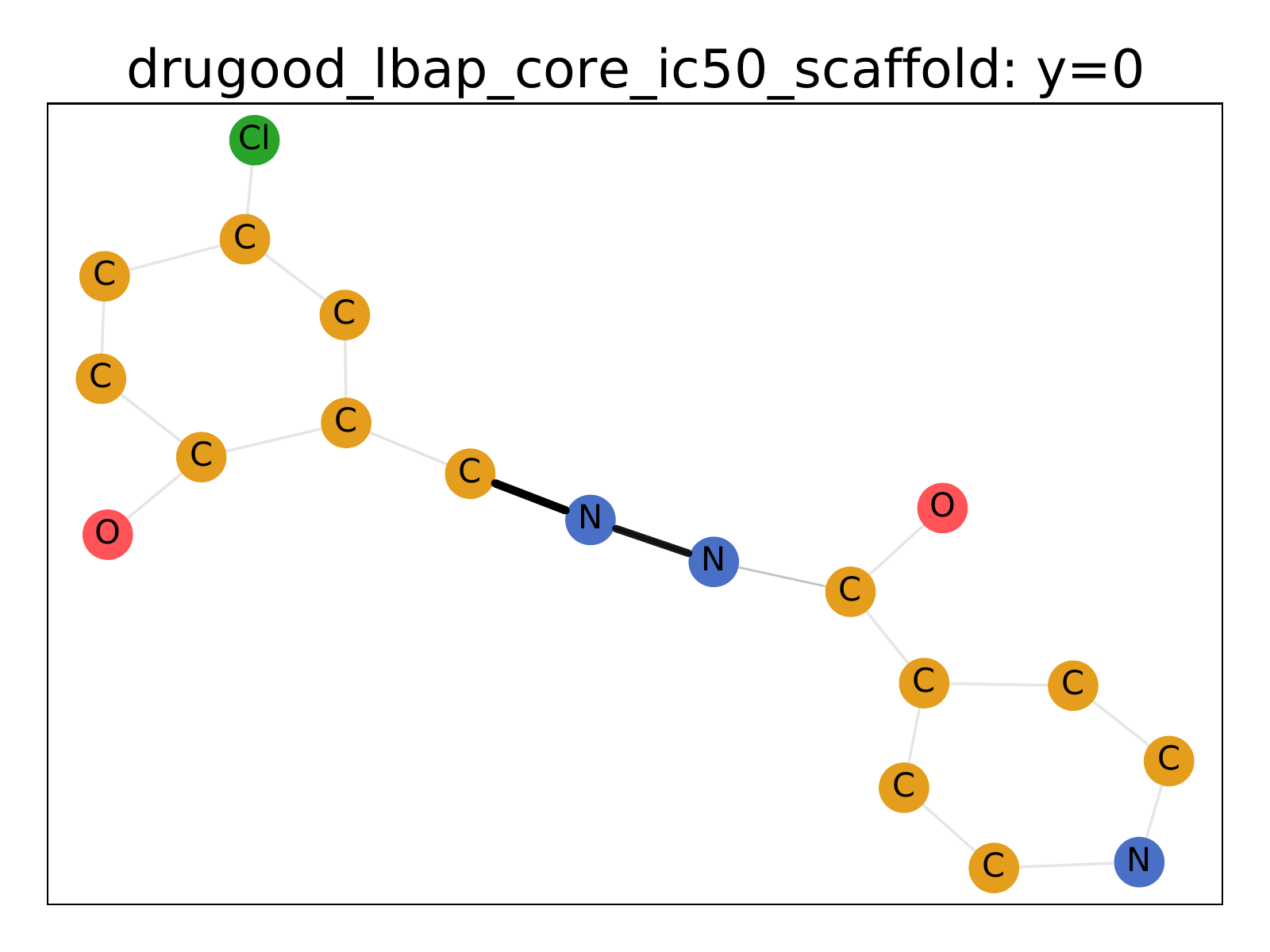}
	}
	\caption{
		Interpretation visualization of inactivate examples ($y=0$) from DrugOOD-Scaffold.}
	\label{fig:scaffold_viz_inact_appdx}
\end{figure}

\begin{figure}[H]
	\centering
	\subfigure[]{
		\includegraphics[width=0.31\textwidth]{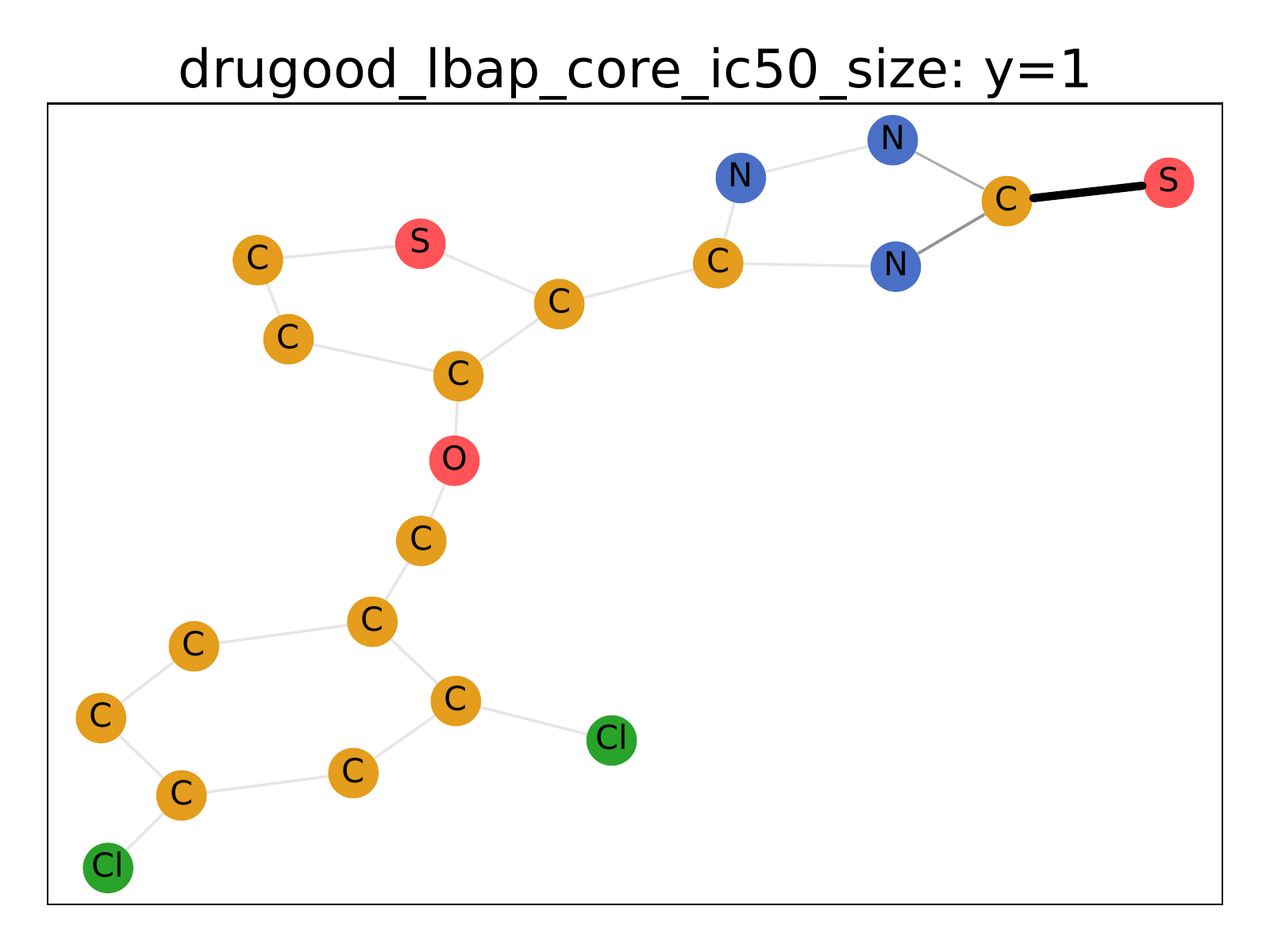}
	}
	\subfigure[]{
		\includegraphics[width=0.31\textwidth]{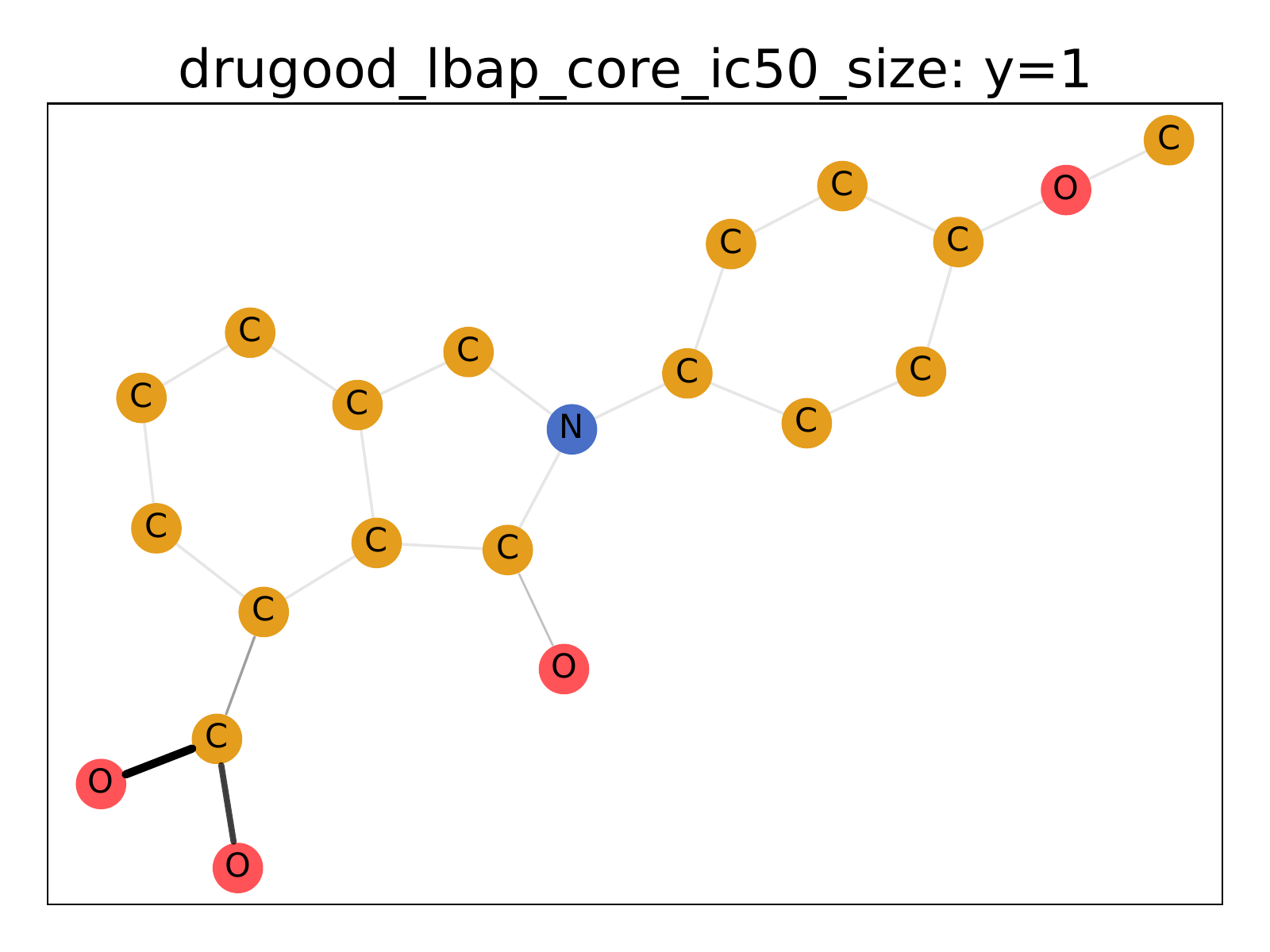}
	}
	\subfigure[]{
		\includegraphics[width=0.31\textwidth]{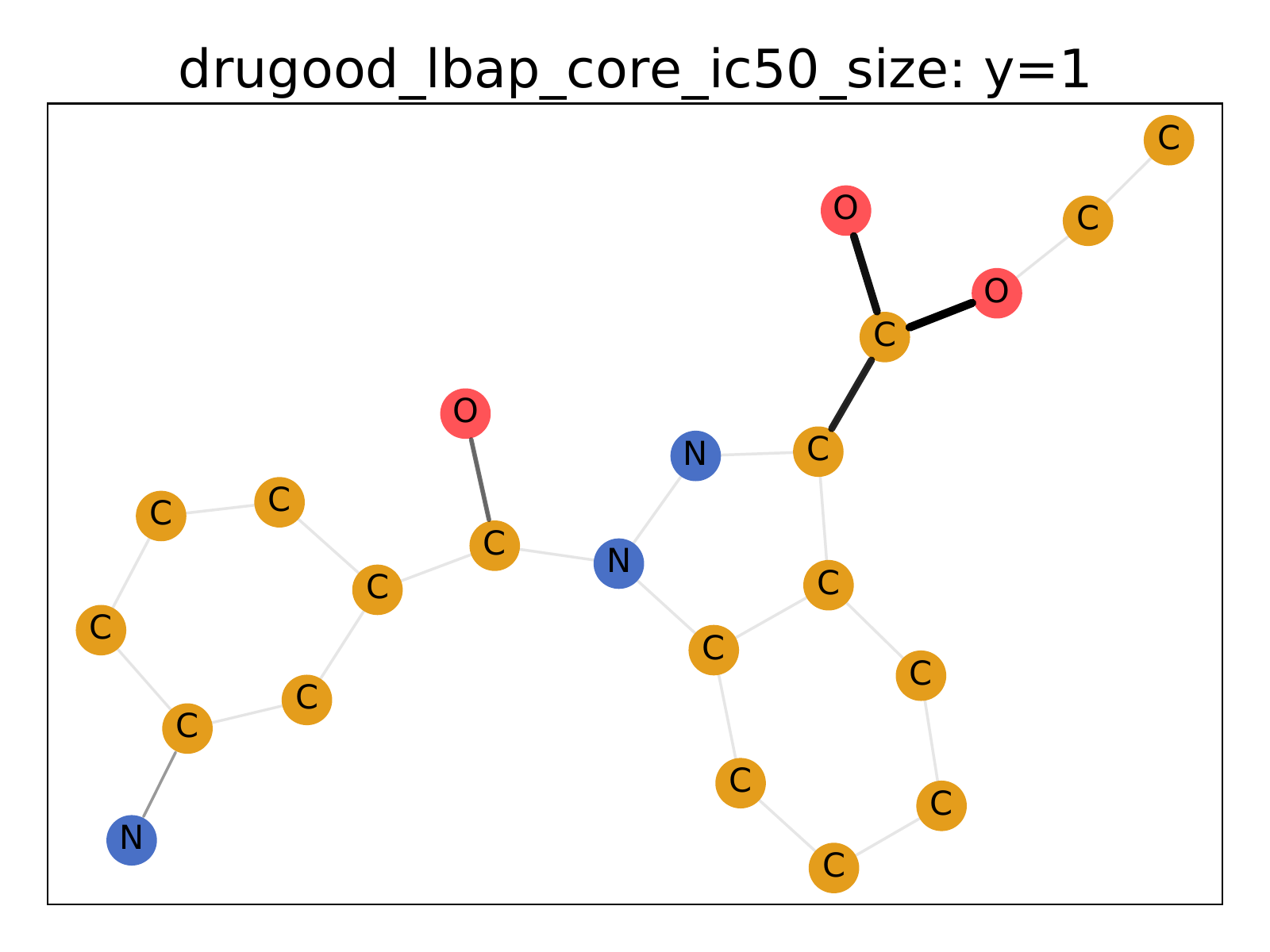}
	}
	\caption{
		Interpretation visualization of activate examples ($y=1$) from DrugOOD-Size.}
	\label{fig:size_viz_act_appdx}
\end{figure}

\begin{figure}[H]
	\centering
	\subfigure[]{
		\includegraphics[width=0.31\textwidth]{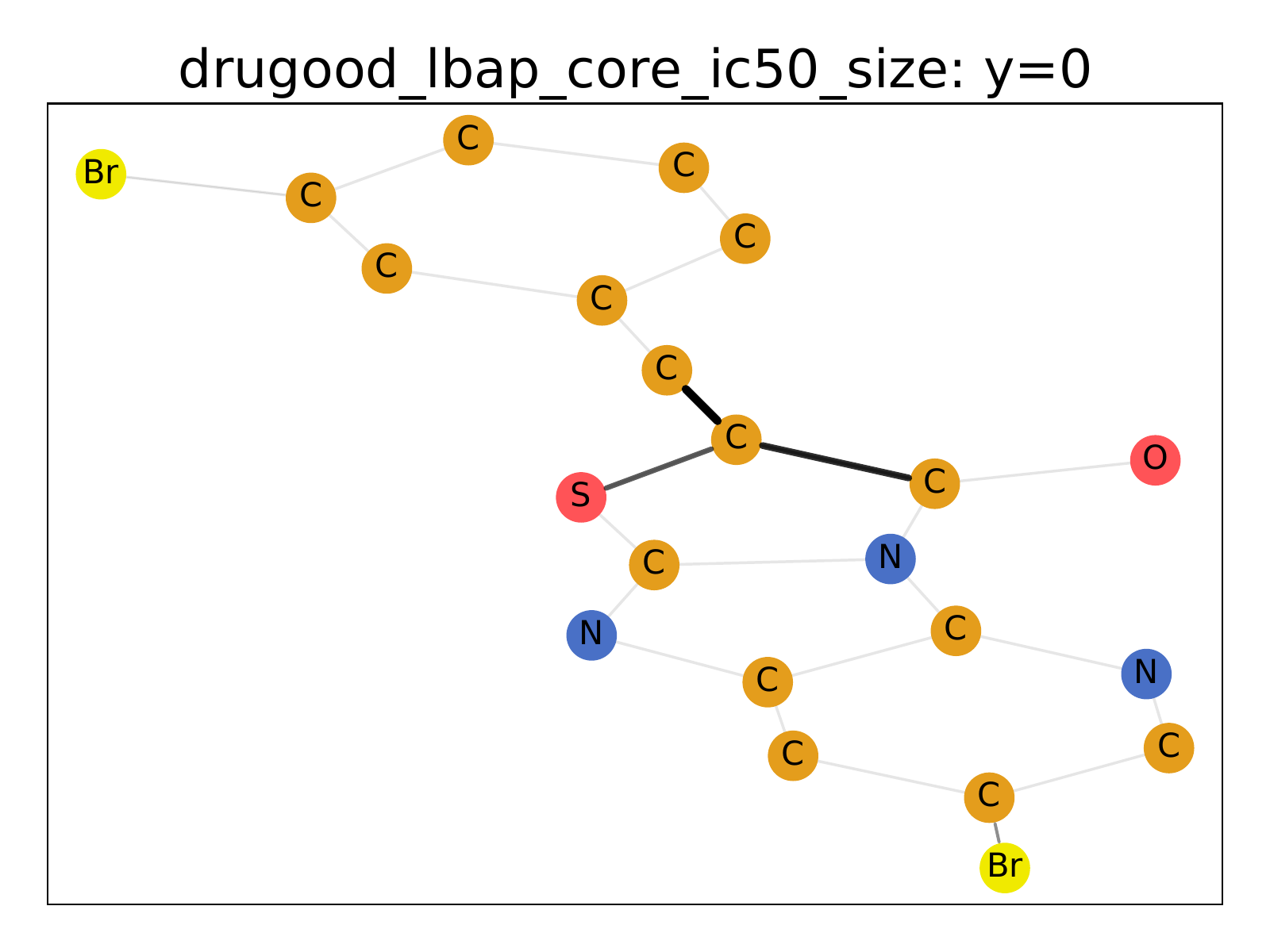}
	}
	\subfigure[]{
		\includegraphics[width=0.31\textwidth]{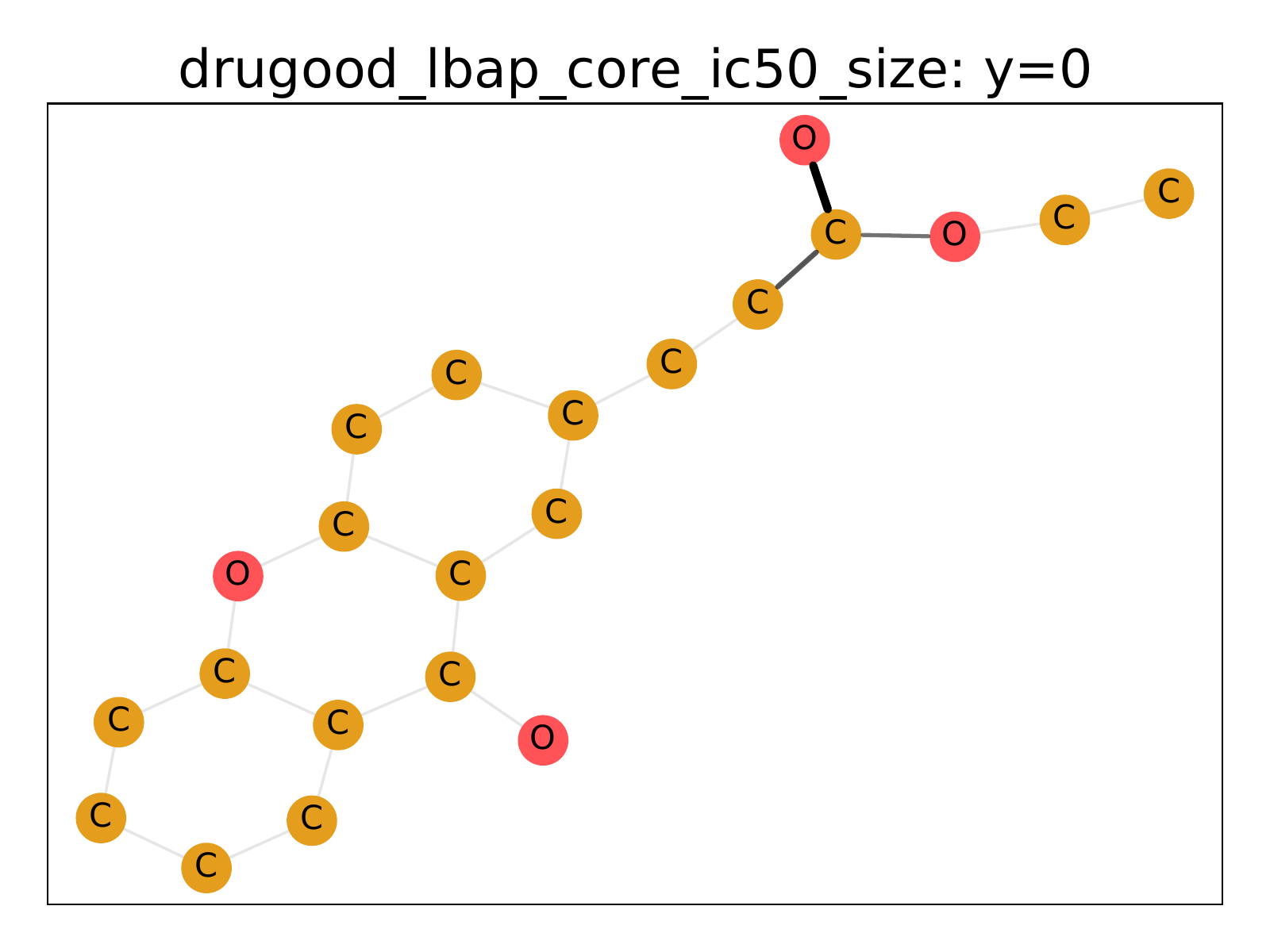}
	}
	\subfigure[]{
		\includegraphics[width=0.31\textwidth]{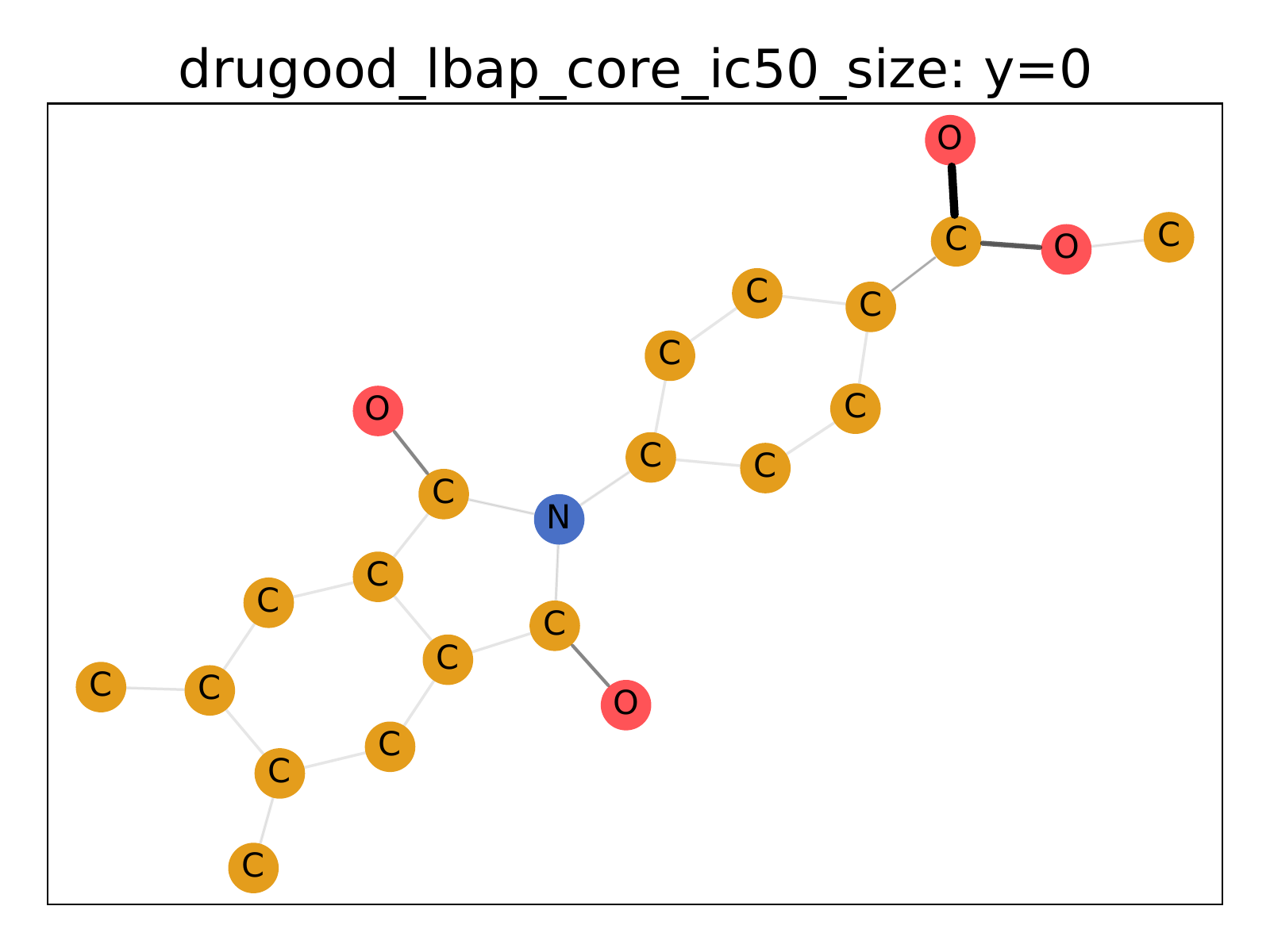}
	}
	\caption{
		Interpretation visualization of inactivate examples ($y=0$) from DrugOOD-Size.}
	\label{fig:size_viz_inact_appdx}
\end{figure}

\end{document}